\documentclass{article}

\usepackage[final]{neurips_2021}

\usepackage[utf8]{inputenc} 
\usepackage[T1]{fontenc}    
\usepackage{hyperref}       
\usepackage{url}            
\usepackage{booktabs}       
\usepackage{amsfonts}       
\usepackage{nicefrac}       
\usepackage{microtype}      
\usepackage{xcolor}         

\usepackage{booktabs}
\usepackage{algorithm} 
\usepackage{algpseudocode}
\usepackage{mathrsfs}
\usepackage{multirow}
\usepackage{color}
\usepackage{colortbl}
\usepackage{arydshln}
\usepackage{array}
\usepackage{bm}
\usepackage{rotating}
\usepackage{adjustbox}
\usepackage{color}

\usepackage{amsmath}
\usepackage{amsthm}
\newtheorem{theorem}{Theorem}
\newtheorem{lemma}{Lemma}

\newtheorem{proposition}{Proposition}
\newtheorem{definition}{Definition}
\usepackage{multirow}
\usepackage{subfigure}
\usepackage{float}
\usepackage{longtable}
\usepackage{colortbl}
\usepackage{wrapfig}
\usepackage{dsfont}
\usepackage{amssymb}
\usepackage{bbding}

\usepackage{authblk}

\title{Rethinking the Pruning Criteria for Convolutional Neural Network}

\begin{document}

\author{\textbf{Zhongzhan Huang$^1$} \quad \textbf{Wenqi Shao$^{2,3\ddag}$} \quad \textbf{Xinjiang Wang$^3$} \quad \textbf{Liang Lin$^1$} \quad \textbf{Ping Luo$^4$}\thanks{Corresponding author: pluo.lhi@gmail.com;\quad $^\ddag$ co-first author.}\\
\vspace{-0.25cm}
\textsuperscript{\rm 1}Sun Yat-Sen University,
\textsuperscript{\rm 2}The Chinese University of Hong Kong,\\
\textsuperscript{\rm 3}SenseTime Research,\textsuperscript{\rm 4}The University of Hong Kong 
}


\maketitle

\begin{abstract}
Channel pruning is a popular technique for compressing convolutional neural networks~(CNNs), where various pruning criteria have been proposed to remove the redundant filters. From our comprehensive experiments, we found two blind spots of pruning criteria: (1)~Similarity: There are some strong similarities among several primary pruning criteria that are widely cited and compared. According to these criteria, the ranks of filters’ \textit{Importance Score} are almost identical, resulting in similar pruned structures. (2)~Applicability: The filters' \textit{Importance Score} measured by some pruning criteria are too close to distinguish the network redundancy well. In this paper, we analyze the above blind spots on different types of pruning criteria with layer-wise pruning or global pruning. We also break some stereotypes, such as that the results of $\ell_1$ and $\ell_2$ pruning are not always similar. These analyses are based on the empirical experiments and our assumption~(\textit{Convolutional Weight Distribution Assumption}) that the well-trained convolutional filters in each layer approximately follow a Gaussian-alike distribution. This assumption has been verified through systematic and extensive statistical tests.

\end{abstract}

	\section{Introduction}
	\label{Introduction}
	Pruning~\cite{lecun1990optimal,hassibi1993second, han2015deep,heyang} a trained neural network is commonly seen in network compression. In particular, for CNNs, channel pruning refers to the pruning of the filters in the convolutional layers. There are several critical factors for channel pruning. 
	\textbf{Procedures}. One-shot method~\cite{li2016pruning}: Train a network from scratch; Use a certain criterion to calculate filters’ \textit{Importance Score}, and prune the filters which have small \textit{Importance Score}; After additional training, the pruned network can recover its accuracy to some extent. Iterative method~\cite{lecun1990optimal,he2018soft,frankle2018the}: Unlike One-shot methods, they prune and fine-tune a network alternately.
	\textbf{Criteria}. The filters' \textit{Importance Score} can be definded by a given criterion. From different ideas, many types of pruning criteria have been proposed, such as Norm-based~\cite{li2016pruning}, Activation-based~\cite{hu2016network,luo2017entropy}, Importance-based~\cite{molchanov2016pruning, molchanov2019importance}, BN-based~\cite{liu2017learning} and so on. \textbf{Strategy}. 
	Layer-wise pruning: In each layer, we can sort and prune the filters, which have small \textit{Importance Score} measured by a given criterion. Global pruning: Different from layer-wise pruning, global pruning~\cite{liu2017learning,hecap} sort the filters from all the layers through their \textit{Importance Score} and prune them. 
	  
	\begin{table}[h]
	\centering

	\caption{
	An example to illustrate the phenomenon that different criteria may select the similar sequence of filters for pruning. Taking VGG16~(3$^{\rm rd}$ Conv) and ResNet18~(12$^{\rm th}$ Conv) on Norm-based criteria as examples. The pruned filters' index ~(the ranks of filters’  \textit{Importance Score}) are almost the same, which lead to the similar pruned structures.}
	\resizebox{\columnwidth}{!}{%
	\begin{tabular}{lllll}
		
		\hline
		Criteria & \multicolumn{1}{l}{Model} & \multicolumn{1}{l}{Pruned Filters' Index~(Top 8)}& \multicolumn{1}{l}{Model} & \multicolumn{1}{l}{Pruned Filters' Index~(Top 8)}\\
		\hline
		$\ell_1$    &ResNet18       &[111, 212, 33, 61, 68, 152, 171, 45] &VGG16       &[102, 28, 9, 88, 66, 109, 86, 45]\\
		$\ell_2$    &ResNet18       &[111, 33, 212, 61, 171, 42, 243, 129] &VGG16       &[102, 28, 88, 9, 109, 66, 86, 45]\\
		$\mathbf{GM}$    &ResNet18       &[111, 212, 33, 61, 68, 45, 171, 42] &VGG16       &[102, 28, 9, 88, 109, 66, 45, 86]\\
		$\mathbf{Fermat}$ &ResNet18       &[111, 212, 33, 61, 45, 171, 42, 68] &VGG16 &[102, 28, 88, 9, 109, 66, 45, 86]\\
		\hline

	\end{tabular}%
}
\label{filtersort}%
	\vspace{-0.4cm}
\end{table}%

	\begin{figure*}[t]
		\centering
		\includegraphics[width=1\textwidth]{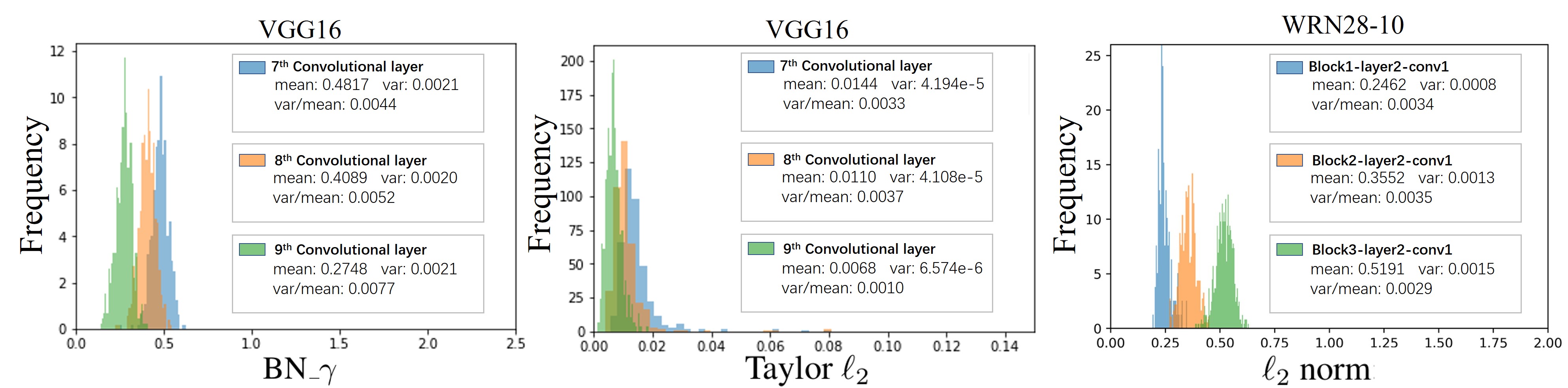}
				\vspace{-0.7cm}
		\caption{Visualization of Applicability problem, \textit{i.e.,} the histograms of the \textit{Importance Score} measured by different types of pruning criteria~(like BN\_$\gamma$, Taylor $\ell_2$ and $\ell_2$ norm). The \textit{Importance Score} in each layer are close enough, which implies that it is hard for these criteria to distinguish redundant filters well in layer-wise pruing.} 
\vspace{-0.2cm}
		\label{fig:app}
	\end{figure*}


	%
	
	In this work, we conduct our investigation on a variety of pruning criteria.
	As one of the simplest and most effective channel pruning criteria, $\ell_1$ pruning~\cite{li2016pruning} is widely used in practice. The core idea of this criterion is to sort the $\ell_1$ norm of filters in one layer and then prune the filters with a small $\ell_1$ norm. Similarly, there is $\ell_2$ pruning which instead leverages the $\ell_2$ norm~\cite{frankle2018the,he2018soft}. $\ell_1$ and $\ell_2$ can be seen as the criteria which use absolute \textit{Importance Score} of filters.
	Through the study of the distribution of norm, \cite{heyang} demonstrates that these criteria should satisfy two conditions: (1)~the variance of the norm of the filters cannot be too small;  (2)~the minimum norm of the filters should be small enough. Since these two conditions do not always hold, a new criterion considering the relative \textit{Importance Score} of the filters is proposed~\cite{heyang}. Since this criterion uses the Fermat point~(\textit{i.e.}, geometric median~\cite{cohen2016geometric}), we call this method $\mathbf{Fermat}$. Due to the high calculation cost of Fermat point, \cite{heyang} further relaxed the $\mathbf{Fermat}$ and then introduced another criterion denotes as $\mathbf{GM}$. To illustrate each of the pruning criteria, let $F_{ij}\in \mathbb{R}^{N_i\times k\times k}$ represent the $j^{\rm th}$ filter of the $i^{\rm th}$ convolutional layer, where $N_i$ is the number of input channels for $i^{\rm th}$ layer and $k$ denotes the kernel size of the convolutional filter. In $i^{\rm th}$ layer, there are $N_{i+1}$ filters. For each criteria, details are shown in Table~\ref{tab:criteria}, where $\mathbf{F}$ denotes the Fermat point of $F_{ij}$ in Euclidean space. These four pruning criteria are called Norm-based pruning in this paper as they utilize norm in their design.

	Previous works~\cite{luo2017thinet,han2015deep,ding2019global,dong2017learning,renda2020comparing}, including the criteria mentioned above, the main concerns commonly consist of (a)~How much the model was compressed; (b)~How much performance was restored; (c)~The inference efficiency of the pruned network and (d)~The cost of finding the pruned network. However, few works discussed the following two blind spots about the pruning criteria:
	
		\begin{wraptable}{r}{6cm}
	\vspace{-0.7cm}
	\caption{Norm-based pruning criteria. }
	\vspace{-0.2cm}
	\begin{center}
		\begin{small}
			
			\begin{tabular}{ll}
				\hline
				Criterion & Details of \textit{Importance Score}\\
				\hline
				$\ell_1$~\cite{li2016pruning}    & $||F_{ij}||_1$\\
				$\ell_2$~\cite{frankle2018the}    & $||F_{ij}||_2$\\
				$\mathbf{Fermat}$~\cite{heyang}    & $||\mathbf{F} - F_{ij}||_2$\\
				$\mathbf{GM}$~\cite{heyang}    & $\sum_{k=1}^{N_{i+1}}||F_{ik}-F_{ij}||_2$\\
				\hline
			\end{tabular}
			
		\end{small}
	\end{center}
	\label{tab:criteria}
	\vspace{-0.4cm}
\end{wraptable}
	
	\textbf{(1)~Similarity: What are the actual differences among these pruning criteria?} Taking the VGG16 and ResNet18 on ImageNet as an example, we show the ranks of filters’ \textit{Importance Score} under different criteria in Table~\ref{filtersort}. It is obvious that they have almost the same sequence, leading to similar pruned structures. In this situation, the criteria used absolute \textit{Importance Score} of filters~($\ell_1$,$\ell_2$) and the criteria used relative \textit{Importance Score} of filters~($\mathbf{Fermat}$, $\mathbf{GM}$) may not be significantly different. 

	\textbf{(2)~Applicability: What is the applicability of these pruning criteria to prune the CNNs?} There is a toy example w.r.t. $\ell_2$ criterion. If the $\ell_2$ norm of the filters in one layer are 0.9, 0.8, 0.4 and 0.01, according to \textit{smaller-norm-less-informative assumption}~\cite{ye2018rethinking}, it’s apparent that we should prune the last filter. However, if the norm are close, such as 0.91, 0.92, 0.93, 0.92, it is hard to determine which filter should be pruned even though the first one is the smallest. In Fig.~\ref{fig:app}, we demonstrate some real examples, \textit{i.e.,} the visualization of Applicability problem under different networks and criteria.

 In this paper, we provide comprehensive observations and in-depth analysis of these two blind spots. Before that, in Section~\ref{distribution}, we propose an assumption about the parameters distribution of CNNs, called \textit{Convolution Weight Distribution Assumption}~(CWDA), and use it as a theoretical tool to analyze the two blind spots. We explore the Similarity and Applicability problem of pruning criteria in the following order: (1)~Norm-based criteria~(layer-wise pruning) in Section~\ref{sec:norm}; (2)~Other types of criteria~(layer-wise pruning) in Section~\ref{sec:others}; (3) and different types of criteria~(global pruning) in Section~\ref{sec:global}. Last but not least, we provide further discussion on: (i) the conditions for CWDA to be satisfied, (ii) how our findings help the community in Section~\ref{Discussion}. In order to focus on the pruning criteria, all the pruning experiments are based on the relatively simple pruning procedure, \textit{i.e.,} one-shot method.

	The main \textbf{contributions} of this work are two-fold:
	
	\textbf{(1)}~We analyze the Applicability problem and the Similarity of different types of pruning criteria. These two blind spots can guide and motivate researchers to design more reasonable criteria. We also break some stereotypes, such as that the results of $\ell_1$ and $\ell_2$ pruning are not always similar.
	
	\textbf{(2)}~We propose and verify an assumption called CWDA, which reveals that the well-trained convolutional filters approximately follow a Gaussian-alike distribution. Using CWDA, we succeeded in explaining the multiple observations about these two blind spots theoretically.

		\section{Weight Distribution Assumption}
	\label{distribution}
	In this section,  we propose and verify an assumption about the parameters distribution of the convolutional filters.
	
	\textbf{(Convolution Weight Distribution Assumption)}~Let $F_{ij}\in \mathbb{R}^{N_i\times k\times k}$ be the $j^{\rm th}$ well-trained filter of the $i^{\rm th}$ convolutional layer. In general\footnote{In Section~\ref{Discussion}, we make further discussion and analysis on the conditions for CWDA to be satisfied.}, in $i^{\rm th}$ layer, $F_{ij}~( j=1,2,...,N_{i+1})$ are i.i.d and follow such a distribution:
	\begin{equation}
	F_{ij} \sim \mathbf{N}(\mathbf{0}, \mathbf{\Sigma}^i_{\text{diag}} + \epsilon\cdot\mathbf{\Sigma}^i_{\text{block}}),
	\label{cwda_org}
	\end{equation}
	where $\mathbf{\Sigma}^i_{\text{block}} = \mathrm{diag}(K_1,K_2,...,K_{N_i})$ is a block diagonal matrix and the diagonal elements of $\mathbf{\Sigma}^i_{\text{block}}$ are 0. $\epsilon$ is a small constant. The values of the off-block-diagonal elements are 0 and $K_l \in R^{k^2\times k^2}, l=1,2,...,N_i$. $\mathbf{\Sigma}^i_{\text{diag}}= \mathrm{diag}(a_1,a_2,...,a_{N_i \times k \times k})$ is a diagonal matrix and the elements of $\mathbf{\Sigma}^i_{\text{diag}}$ are close enough.

	This assumption is based on the observation shown in the Fig.~\ref{fig:cwda}. To estimate $\mathbf{\Sigma}^i_{\text{diag}} + \epsilon\cdot\mathbf{\Sigma}^i_{\text{block}}$, we use the correlation matrix  $FF^T$ where $F \in \mathbb{R}^{(N_i\times k \times k) \times N_{i+1}}$ denotes all the parameters in $i^{\rm th}$ layer. Taking a convolutional layer of ResNet18 trained on ImageNet as an example, we find that $FF^T$ is a block diagonal matrix. Specifically, each block is a $k^2 \times k^2$ matrix and the off-diagonal elements are close to 0. We visualize the $j^{\rm th}$ filter $F_{ij}\in \mathbb{R}^{N_i\times k \times k}$ in $i^{\rm th}$ layer in Fig.~\ref{fig:cwda}(c), and this phenomenon reveals that the parameters in the same channel of $F_{ij}$ tend to be linearly correlated, and the parameters of any two different channels~(yellow and green channel in Fig.~\ref{fig:cwda}(c)) in $F_{ij}$ only have a low linear correlation.
	\begin{wrapfigure}{r}{7cm}
		\includegraphics[width=0.95\linewidth]{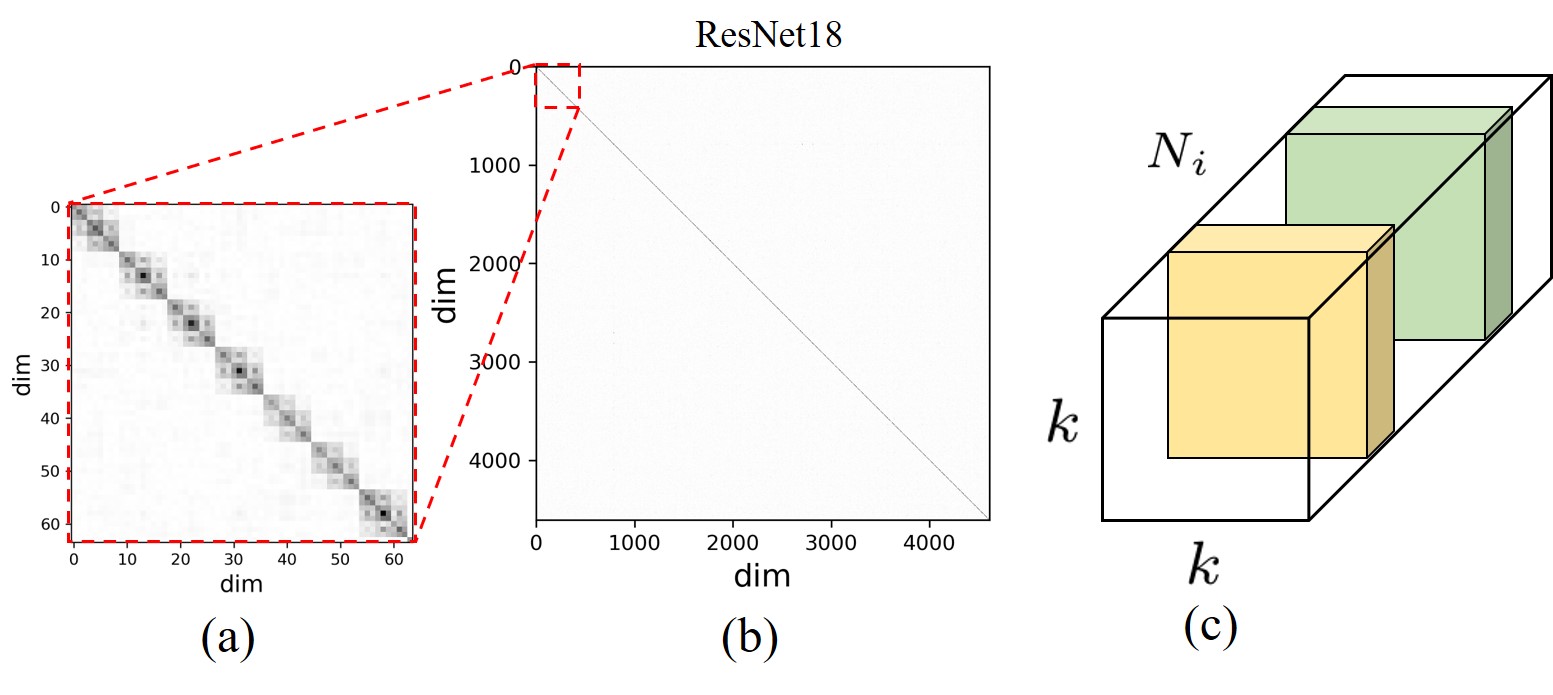}
		\vspace{-0.5cm}
		\caption{(a-b)~Visualization of $FF^T$ in ResNet-18 trained on ImageNet dataset. More experiments can be found in Appendix~\ref{app:diag_matrix}. These experiments are based on torchvison model zoo~\cite{pytorch}, which can guarantee the generality and reproducibility. (c)~A convolutional filter. $k$ is the kernel size and $N_i$ denotes the number of input channels.}
		\label{fig:cwda}
		\vspace{-0.6cm}
\end{wrapfigure}	
	


	\subsection{Statistical test for CWDA}
	\label{Statistical test}
	In fact, CWDA is not easy to be verified, \textit{e.g.}, for ResNet164 trained on Cifar100, the number of filters in the first stage is only 16, which is too small to be used to estimate the statistics in CWDA accurately. Thus, We consider verifying four \textbf{necessary conditions} of CWDA: 
	
	(1)~\textbf{Gaussian.}~Whether the weights of $F_{ij}$ approximately follows a Gaussian-alike distribution; 
	(2)~\textbf{Variance.}~Whether the variance of the diagonal elements of $\Sigma_{\text{diag}}$ are small enough; 
	(3)~\textbf{Mean.}~Whether the mean of weights of $F_{ij}$ is close to 0. 
	(4)~\textbf{The magnitude of $\epsilon$.}~Whether $\epsilon$ is small enough.
	
	The results of the tests are shown in Appendix~\ref{app:Statistical Test}, where we consider a variety of factors for the statistical tests, including different network structure, optimizer, regularization, initialization, dataset, training strategy, and other tasks in computer vision~(\textit{e.g}., semantic segmentation, detection and so on). The test results show that CWDA has a great generality for CNNs.

	\section{About the Norm-based criteria}
	\label{sec:norm}
	We start from the criteria in Table~\ref{tab:criteria}, which are widely cited and compared~\cite{liu2020joint,li2020group,he2020learning,liu2020rethinking,li2020eagleeye}. 
	\subsection{Similarity}
	\label{Experiment and theory}
	In this section, we further verify the observation that the Norm-based pruning criteria in Table~\ref{tab:criteria} are highly similar from two perspectives. Empirically, we conducted large amount of experiments on image classification to investigate the similarities. Theoretically, we rigorously prove the similarities of the criteria in Table~\ref{tab:criteria} in layer-wise pruning under CWDA.

\begin{wrapfigure}{r}{7cm}
	\centering
    \includegraphics[width=1\linewidth]{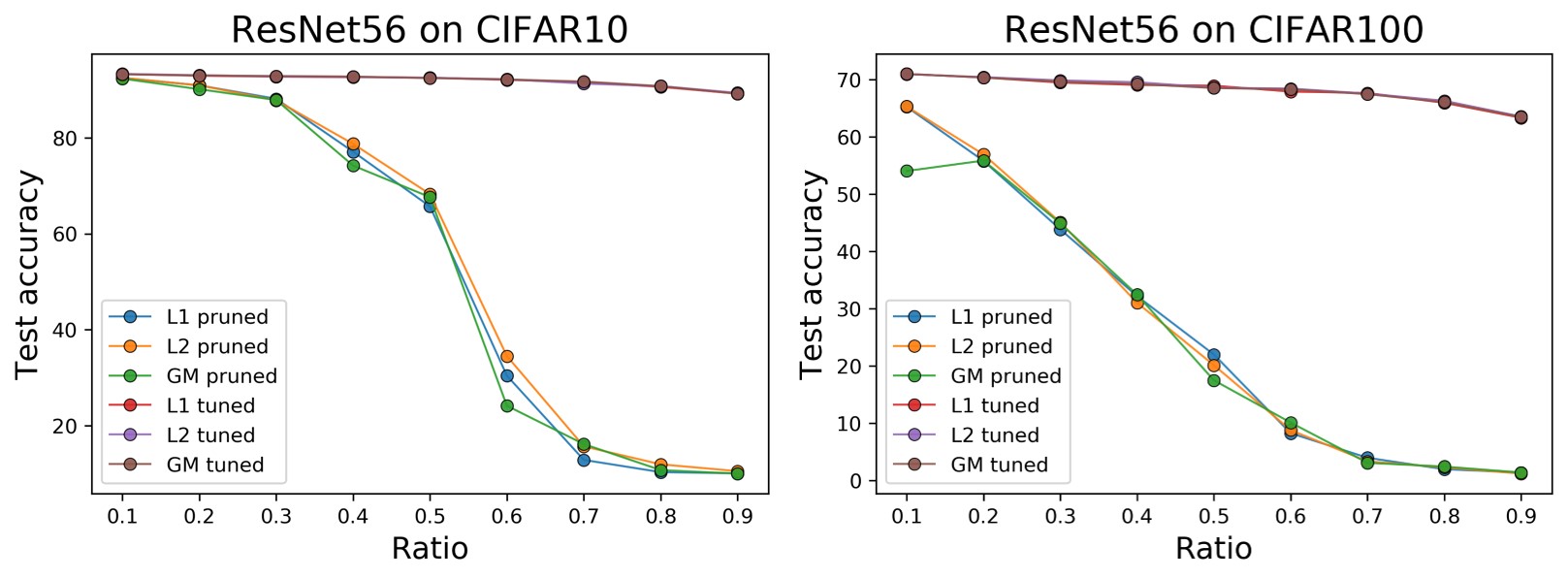}
	\vspace{-0.35cm}
	\caption{Test accuracy of the ResNet56 on CIFAR10/100 while using different pruning ratios. ``L1 pruned'' and ``L1 tuned'' denote the test accuracy of the ResNet56 after $\ell_1$ pruning and fine-tuning, respectively. If ratio is 0.5, we prune 50\% filters in all layers.}
	\label{fig:exp_pruned_tuned2}
	\vspace{-0.2cm}
\end{wrapfigure}

	\textbf{Empirical Analysis}. (1)~In Fig.~\ref{fig:exp_pruned_tuned2}, we show the test accuracy of the ResNet56 after pruning and fine-tuning under different pruning ratios and datasets. The test accuracy curves of different pruning criteria at different stages are very close under different pruning ratios. This phenomenon implies that those pruned networks using different Norm-based criteria are very similar, and there are strong similarities among these pruning criteria. The experiments about other commonly used configs of pruning ratio can be found in Appendix~\ref{app:cls}.
	(2)~In Fig.~\ref{fig:vggmore}, we show the Spearman's rank correlation coefficient\footnote{Sp is a nonparametric measurement of ranking correlation, and it assesses how well the relationship between two variables can be described using a monotonic function,  \textit{i.e.}, filters ranking sequence in the same layer under two criteria in this paper.}~(Sp) between different pruning criteria. The Sp in most convolutional layers are more than 0.9, which means the network structures are almost the same after pruning. Note that the Sp in transition layer are relatively small, and the transition layer refers to the layer where the dimensions of the filter change, like the layer between stage 1 and stage 2 of a ResNet. The reason for this phenomenon may be that the layers in these areas are sensitive. It is interesting but will not greatly impact the structural similarity of the whole pruned network. The similar observations are shown in Fig.~2 in \cite{ding2019global}, Fig.~6 and Fig.~10 in \cite{li2016pruning}. 
	\begin{figure*} [htbp]
		\vspace{-0.1cm}
	\centering 
	\includegraphics[width=0.9\linewidth]{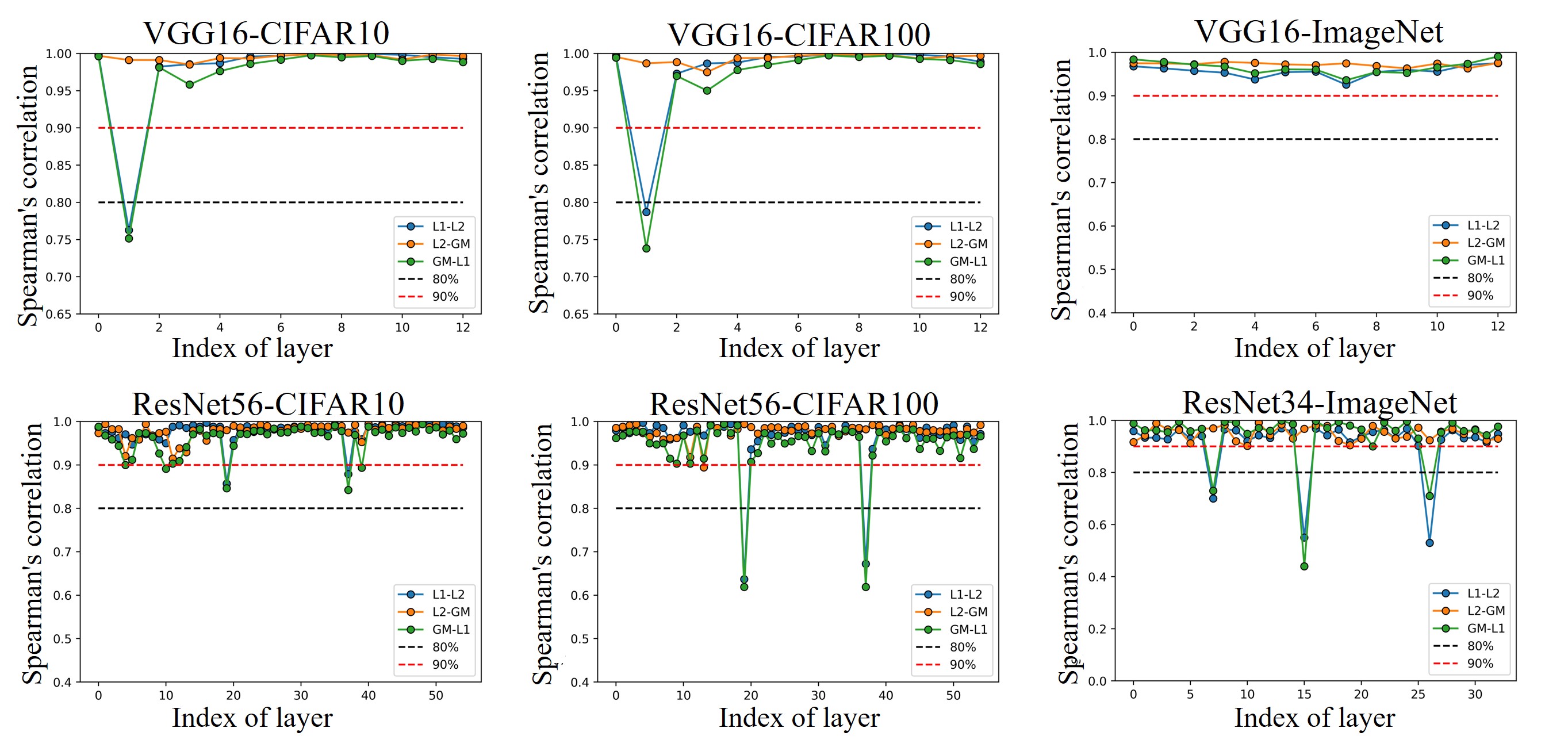}
	
	\centering 
	\caption{Spearman's rank correlation coefficient~(Sp) between different pruning criteria on several networks and datasets~(more experiments can be found in Appendix~\ref{app:sp_network}). } 
	\vspace{-0.2cm}
	\label{fig:vggmore}
\end{figure*}

	
	\textbf{Theoretical Analysis}. Besides the experimental verification, the similarities via using layer-wise pruning among the criteria in Table~\ref{tab:criteria} can also be proved theoretically in this section. Let $C_1$ and $C_2$ be two pruning criteria to calculate the \textit{Importance Score} for
	all convolutional filters in one layer. If they can produce the similar ranks of \textit{Importance Score}, we define that $C_1$ and $C_2$ are \textit{approximately monotonic} to each other and use $C_1 \cong C_2$  to represent this relationship. In Section~\ref{Experiment and theory}, we use the Sp to describe this relationship but it's hard to be analyzed theoretically. Therefore, we focus on a stronger condition. Let $\mathbf{X} = (x_1,x_2,...,x_k)$ and $\mathbf{Y} = (y_1,y_2,...,y_k)$ be two given sequences\footnote{Since $\mathbf{X}$ is not random variables here, $\mathbb{E}(\mathbf{X})$ and $\mathbf{Var}(\mathbf{X})$ denote the average value $\sum_{i=1}^kx_i/k$ and the sample variance $\sum_{i=1}^k(x_i-\mathbb{E}(\mathbf{X}))/(k-1)$, respectively.}.
	we first normalize their magnitude, \textit{i.e.}, let $\widehat{\mathbf{X}} = \mathbf{X}/\mathbb{E}(\mathbf{X})$ and $\widehat{\mathbf{Y}} = \mathbf{Y}/\mathbb{E}(\mathbf{Y})$~. This operation does not change the ranking sequence of the elements of $\mathbf{X}$ and $\mathbf{Y}$, because $\mathbb{E}(\mathbf{X})$ and $\mathbb{E}(\mathbf{Y})$ are constants, \textit{i.e.}, $\mathbf{\widehat{X}} \cong \mathbf{\widehat{Y}} \Leftrightarrow \mathbf{X} \cong \mathbf{Y}$.~After that, if both $\mathbf{Var(\widehat{\mathbf{X}}/\widehat{\mathbf{Y}})}$ and $\mathbf{Var(\widehat{\mathbf{Y}}/\widehat{\mathbf{X}})}$ are small enough, then the Sp between $\mathbf{X}$ and $\mathbf{Y}$ is close to 1, where $\widehat{\mathbf{X}}/\widehat{\mathbf{Y}} = (\widehat{x_1}/\widehat{y_1},..,\widehat{x_k}/\widehat{y_k})$. The reason is that in these situations, the ratio $\widehat{\mathbf{X}}/\widehat{\mathbf{Y}}$ and $\widehat{\mathbf{Y}}/\widehat{\mathbf{X}}$ will be close to two constants $a,b$. For any $1 \leq i \leq k$, $\widehat{x_i} \approx a\cdot \widehat{y_i}$ and $\widehat{y_i} \approx b\cdot \widehat{x_i}$. So, $ab \approx 1$ and $a,b \neq 0$. Therefore, there exists an \textit{approximately monotonic} mapping from $\widehat{y_i}$ to $\widehat{x_i}$~(linear function), which makes the Sp between $\mathbf{X}$ and $\mathbf{Y}$ close to 1. With this basic fact, we propose the Theorem \ref{theo:similarity-layer}, which implies that many Norm-based pruning criteria produces almost the same ranks of \textit{Importance Score}.
	
\begin{theorem}\label{theo:similarity-layer}
Let $n-$dimension random variable $X$ meet CWDA, and the pair of criteria $(C_1,C_2)$ is one of $(\ell_1,\ell_2)$, $(\ell_2,\mathbf{Fermat})$ or $(\mathbf{Fermat},\mathbf{GM})$, we have  
	\begin{equation}
	\mathbf{max}\left\{\mathbf{Var}_{X}\left(\frac{\widehat{C}_2(X)}{\widehat{C}_1(X)}\right),\mathbf{Var}_{X}\left(\frac{\widehat{C}_1(X)}{\widehat{C}_2(X)}\right) \right\}\lesssim B(n),
	\end{equation}
	\label{theo:bound}
	where $\widehat{C}_1(X)$ denotes $C_1(X)/\mathbb{E}(C_1(X))$ and $\widehat{C}_2(X)$ denotes $C_2(X)/\mathbb{E}(C_2(X))$. $B(n)$ denotes the upper bound of left-hand side and when $n$ is large enough, $B(n) \to 0$.
\end{theorem}
\begin{proof}
	(See Appendix \ref{proof:theorm1}).\qedhere
\end{proof}

	In specific, for $i^{\rm th}$ convolutional layer of a CNN, since $F_{ij}\in \mathbb{R}^n$, $j= 1,2,...N_{i+1}$, meet CWDA and the dimension $n$ is generally large, we can obtain $\ell_1 \cong \ell_2$, $\ell_2 \cong \mathbf{Fermat}$ and $\mathbf{Fermat} \cong \mathbf{GM}$ according to Theorem~\ref{theo:bound}. Therefore,  we have $\ell_1 \cong \ell_2 \cong \mathbf{Fermat} \cong \mathbf{GM}$, which verifies the strong similarities among the criteria shown in Table~\ref{tab:criteria}.

	\subsection{Applicability} 
\label{Applicability}	


	In this section, we analyze the Applicability problem of the Norm-based criteria. In Fig.~\ref{fig:app}~(Right), we know that there are some cases where the values of \textit{Importance Score} measured by $\ell_2$ criterion are very close~(e.g., the distribution looks sharp), which make $\ell_2$ criterion cannot distinguish the redundant filters well. It's related to the variance of \textit{Importance Score}. \cite{heyang} argue that a \textit{small norm deviation}~(the values of variance of \textit{Importance Score} are small) makes it difficult to find an appropriate threshold to select filters to prune. However, even if the values of the variance are large, it still cannot guarantee to solve this problem. Since the magnitude of these \textit{Importance Score} may be much greater than the values of the variance, we can use the mean of \textit{Importance Score} to represent their magnitude. Therefore, we consider using a relative variance $\mathbf{Var}_r[C(F_A)]$ to describe the Applicability problem. Let $\mathbb{E}[C(F_A)] > 0$ and
	\begin{equation}
	    \mathbf{Var}_r[C(F_A)] = \mathbf{Var}[C(F_A)]/\mathbb{E}[C(F_A)],
	    \label{eqn:appli}
	\end{equation}
	where $C$ is a given pruning criterion and $F_A$ denotes the filters in layer $A$. The criterion $C$ for layer $A$ has Applicability problem when $\mathbf{Var}_r[C(F_A)]$ is close to 0. Then we introduce the Proposition~\ref{prop:mean_var_cri} to provide the estimation of the mean and variance w.r.t. different criteria when the CWDA is hold:
	
	 

	
    \begin{proposition}
    If the convolutional filters $F_A$ in layer $A$ meet CWDA, then we have following estimations:
\begin{table}[H]
	\vspace{-0.4cm}
	\begin{center}
		\begin{small}
			\begin{tabular}{lll}
				\hline
				 Criterion & Mean & Variance\\
				\hline
				$\ell_1(F_A)$    & $\sqrt{2/\pi}\sigma_Ad_A$& $(1-\frac{2}{\pi})\sigma_A^2d_A$\\
				$\ell_2(F_A)$    & $\sqrt{2}\sigma_A\Gamma(\frac{d_A+1}{2})/\Gamma(\frac{d_A}{2})$& $\sigma_A^2/2$\\
				$\mathbf{Fermat}(F_A)$   & $\sqrt{2}\sigma_A\Gamma(\frac{d_A+1}{2})/\Gamma(\frac{d_A}{2})$& $\sigma_A^2/2$\\
				\hline
			\end{tabular}
		\end{small}
	\end{center}
	\label{criteria}
\end{table}
where $d_A$ and $\sigma_A^2$ denote the dimension of $F_A$ and the variance of the weights in layer $A$, respectively.
   \label{prop:mean_var_cri}
   \vspace{-0.4cm}
    \end{proposition}
\begin{proof}
	(See Appendix \ref{app:prop}).\qedhere
	\vspace{-0.3cm}
\end{proof}
Based on the Proposition~\ref{prop:mean_var_cri}, we further provide the theoretical analysis for each criteria:

(i)~For $\ell_2(F_A)$. From Proposition~\ref{prop:mean_var_cri}, we can obtain that 
\begin{align}
\vspace{-0.3cm}
\mathbf{Var}_r[\ell_2(F_A)]
&= \frac{\sigma_A^2}{2} / [\sqrt{2}\sigma_A\Gamma(\frac{d_A+1}{2})/\Gamma(\frac{d_A}{2})] = O(\sigma_A/g(d_A)), 
\label{eqn:rv_l2}
\end{align}	
where $g(d_A) = \Gamma(\frac{d_A+1}{2})/\Gamma(\frac{d_A}{2})$ is a monotonically increasing function w.r.t $d_A$. From Eq.~(\ref{eqn:rv_l2}), $\mathbf{Var}_r[\ell_2(F_A)]$ depend on $\sigma_A$ and $d_A$. When $\sigma_A$ is small or $d_A$ is large enough, $\mathbf{Var}_r[\ell_2(F_A)]$ tends to be 0.



	(ii) For $\mathbf{Fermat}(F_A)$. From the proof in Appendix~\ref{proof:l1vsl2}, we know that the Fermat point $\mathbf{F}$ of $F_A$ and the origin $\mathbf{0}$ approximately coincide. From Table~\ref{criteria}, $||\mathbf{F}-F_A||_2 \approx ||\mathbf{0} - F_A||_2 = ||F_A||_2$. Therefore, the mean and variance of $\mathbf{Fermat}(F_A)$ are the same as $\ell_2(F_A)$'s in Proposition~\ref{prop:mean_var_cri}. Hence, a similar conclusion can be obtained for $\mathbf{Fermat}$ criterion. \textit{i.e.,} the \textit{Importance Score} tends to be identical and it’s hard to distinguish the network redundancy well when $\sigma_A$ is small or $d_A$ is large enough.

	(iii)~For $\ell_1(F_A)$. Intuitively, the $\ell_1$ criterion should have the same conclusion as the $\ell_2$ criterion. However, given the Proposition~\ref{prop:mean_var_cri}, we can obtain that 
\begin{align}
\vspace{-0.5cm}
\mathbf{Var}_r[\ell_1(F_A)]
&= (1-\frac{2}{\pi})\sigma_A^2d_A / [\sqrt{2/\pi}\sigma_Ad_A] = \epsilon(\pi)\cdot\sigma_A, 
\label{eqn:rv_l1}
\vspace{-0.3cm}
\end{align}	
where $\epsilon(\pi)<1$ is a constant w.r.t $\pi$. Note that $\mathbf{Var}_r[\ell_1(F_A)]$ only depend on $\sigma_A$, but not the dimension $n$. Moreover, for the common network structures, like VGG, ResNet shown in Fig.~\ref{fig:magnitude}~(b) and (d), the dimension of the filters are usually large enough. Therefore, compared with $\ell_2$, $\ell_1$ criterion is relatively not prone to have Applicability problems, unless the $\sigma_A$ is very small.


\begin{figure*} [htbp]
	\centering 
	\includegraphics[width=0.92\linewidth]{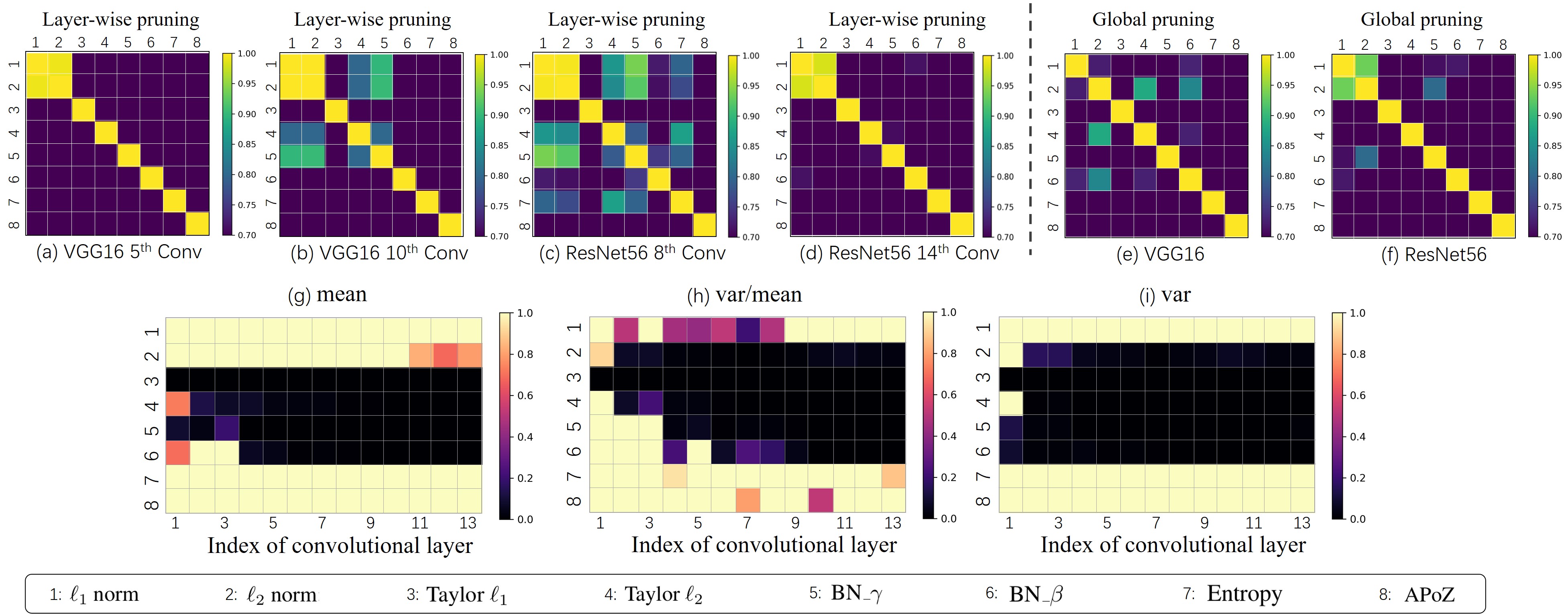}
	\centering 
	\caption{The Similarity and Applicability problem for different types of pruning criteria in layer-wise or global pruning. } 
	\label{fig:other_simi}
\end{figure*}


\section{About other types of pruning criteria}
\label{sec:others}

In this section, we study the Similarity and Applicability problem in other types of pruning criteria through numerical experiments, such as Activation-based pruning~\cite{hu2016network,luo2017entropy}, Importance-based pruning~\cite{molchanov2016pruning, molchanov2019importance} and BN-based pruning~\cite{liu2017learning}. For each type, we choose two representative criteria and we call them: (1)~Norm-based: $\ell_1$ and $\ell_2$; (2)~Importance-based: Taylor $\ell_1$ and  Taylor $\ell_2$~\cite{molchanov2016pruning, molchanov2019importance, molchanov2019taylor}; (3)~BN-based:  BN\_$\gamma$\footnote{The empirical result for slimming training~\cite{liu2017learning} is shown in Appendix~\ref{slimming}.} and BN\_$\beta$~\cite{liu2017learning}; (4) Activation-based:  Entropy~\cite{luo2017entropy} and APoZ~\cite{hu2016network}. The details of these criteria can be found in Appendix \ref{app:other_criteria}.

\textbf{The Similarity for different types of pruning criteria}. In Fig.~\ref{fig:other_simi}~(a-d), we show the Sp between different types of pruning criteria, and only the Sp greater than 0.7 are shown because if Sp $<$ 0.7, it means that there is no strong similarity between two criteria in the current layer. 

According to the Sp shown in Fig.~\ref{fig:other_simi}~(a-d), we obtain the following observations: 
(1)~As verified in Section~\ref{Experiment and theory}, $\ell_1$ and $\ell_2$ can maintain a strong similarity in each layer; 
(2)~In the layers shown in Fig.~\ref{fig:other_simi}~(a) and Fig.~\ref{fig:other_simi}~(d), the Sp between most different pruning criteria are not large in these layers, which indicates that these criteria have great differences in the redundancy measurement of convolutional filters. This may lead to a phenomenon that one criterion considers a convolutional filter to be important, while another considers it redundant. We find a specific example which is shown in Appendix~\ref{app:case};
(3)~Intuitively, the same type of criteria should be similar. However, Fig.~\ref{fig:other_simi}~(b) and Fig.~\ref{fig:other_simi}~(c) show that the Sp between Taylor $\ell_1$ and Taylor $\ell_2$ is not large, but Taylor $\ell_2$ has strong similarity with both two Norm-based criteria. Moreover, the Sp between BN\_$\gamma$ and each Norm-based criteria exceeds 0.9, but it is not large in other layers~(Fig.~\ref{fig:other_simi}~(a) and Fig.~\ref{fig:other_simi}~(d)). These phenomena are worthy of further study.

	\begin{figure*} [htbp]
		\vspace{-0.2cm}
	\centering 
	\includegraphics[width=1.0\linewidth]{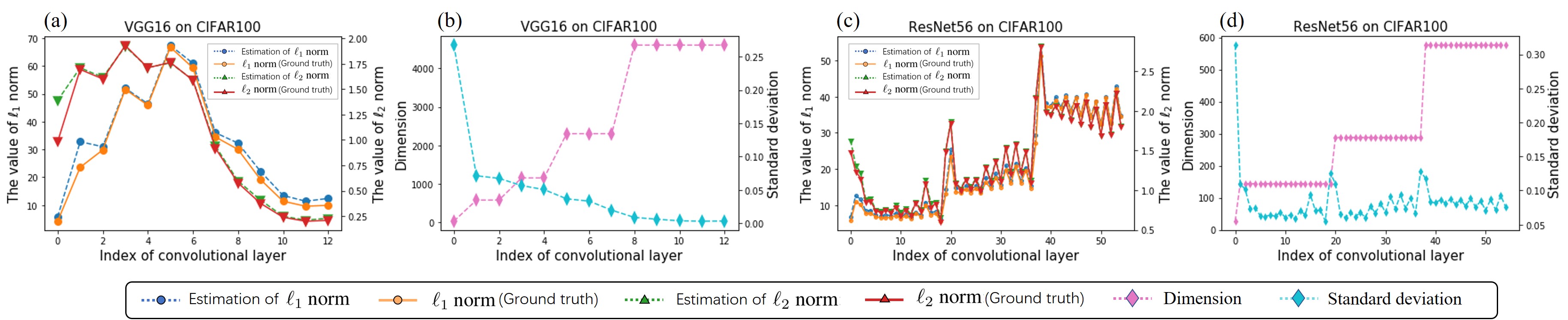}
	\centering 
	\vspace{-0.4cm}
	\caption{The magnitude of the \textit{Importance Score} measured by $\ell_1$ and $\ell_2$ criteria. } 
	\label{fig:magnitude}
\end{figure*}

\textbf{The Applicability for different types of pruning criteria}. According to the analysis in Section~\ref{Applicability}, the Applicability problem depends on the mean and variance of the \textit{Importance Score}. Fig.~\ref{fig:other_simi}~(g-i) shows the result of the \textit{Importance Score} measured by different pruning criteria on each layer of VGG16. Due to the difference in the magnitude of \textit{Importance Score} for different criteria, for the convenience of visualization, the value greater than 1 is represented by 1. 
 
  First, we analyze the Norm-based criteria. In most layers, the relative variance $\mathbf{Var}_r[\ell_2]$ is much smaller than that of $\mathbf{Var}_r[\ell_1]$, which means that the $\ell_2$ pruning has Applicability problem in VGG16, while the $\ell_1$ does not. This is consistent with our conclusion in Section~\ref{Applicability}. 
  Next, for the Activation-based criteria, the relative variance $\mathbf{Var}_r$ is large in each layer, which means that these two Activation-based criteria can distinguish the network redundancy well from their measured filters' \textit{Importance Score}. However, for the Importance-based and BN-based criteria, their relative variance $\mathbf{Var}_r$ are close to 0. According to Section~\ref{Applicability}, these criteria have Applicability problem, especially in the deeper layers (e.g., from 6$^{\rm th}$ layer to the last layer).

	\begin{wrapfigure}{r}{7cm}
		\vspace{-0.3cm}
    \includegraphics[width=1\linewidth]{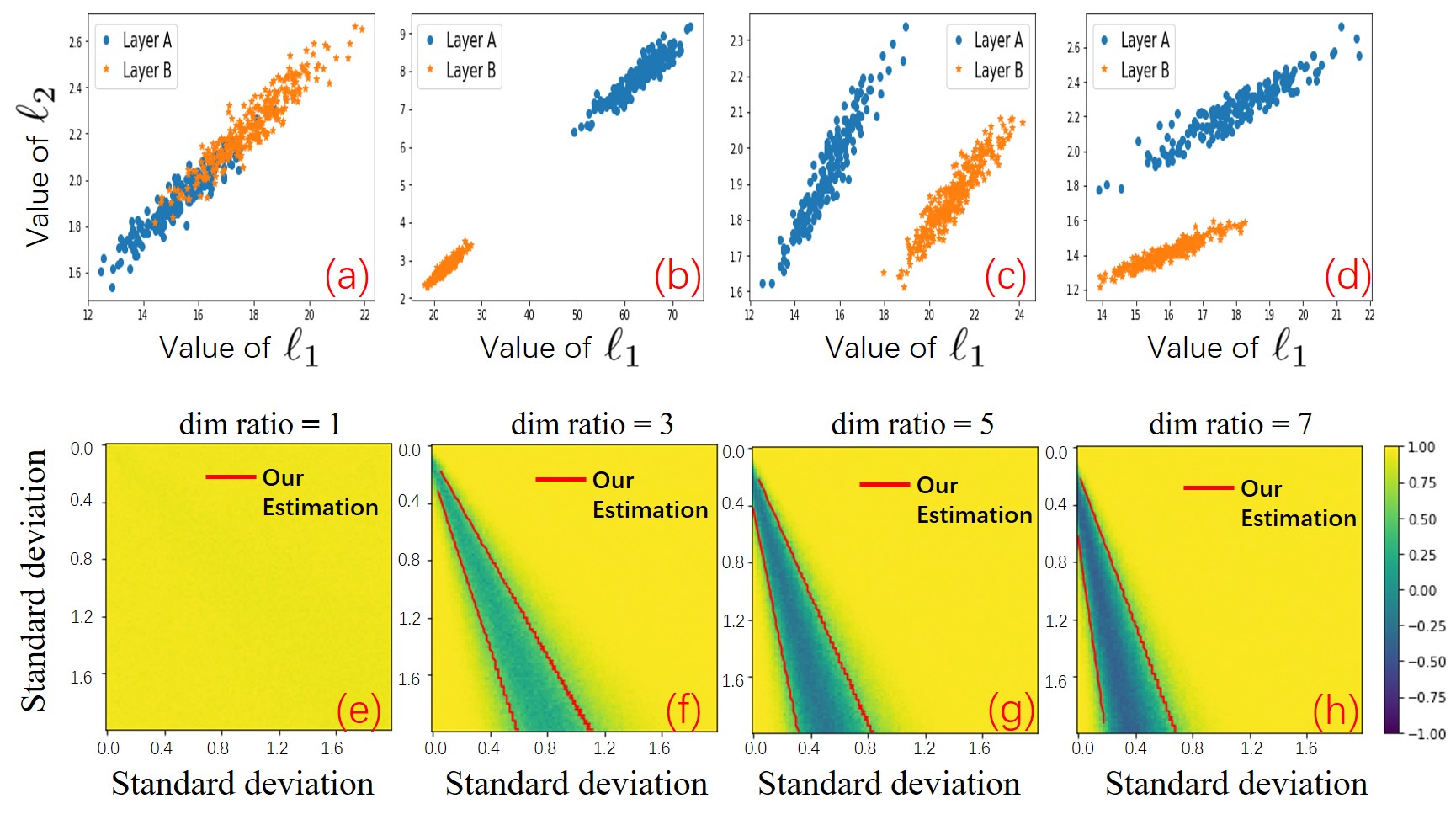}
	\vspace{-0.5cm}
	\caption{The global pruning simulation for the unpruned network with only two layers.}
	\label{fig:exp_pruned_tuned}
	\vspace{-0.2cm}
\end{wrapfigure}

\section{About global pruning}\label{sec:global}
	
	Compared with layer-wise pruning, global pruning is more widely~\cite{liu2018rethinking, molchanov2016pruning,liu2017learning} used in the current research of channel pruning. Therefore, in this section we may also analyze the Similarity and Applicability problem of global pruning.
	
	\textbf{Applicability while using global pruning}. In fact, for global pruning, both $\ell_1$ and $\ell_2$ criteria are not prone to Applicability problems. From Proposition~\ref{prop:mean_var_cri}, we show that the estimations for the mean of \textit{Importance Score} in layer $A$ for $\ell_1$ and $\ell_2$ are $\sigma_A\cdot d_A\sqrt{\frac{2}{\pi}}$ and $\sqrt{2}\sigma_A\cdot\Gamma(\frac{d_A+1}{2})/\Gamma(\frac{d_A}{2})$, respectively. Since $\sigma_A$ and $d_A$ are quite different, shown in Fig.~\ref{fig:magnitude}~(b) and (d), hence the variance of the \textit{Importance Score} may be large in this situation. 
	Fig.~\ref{fig:magnitude}~(a) and (c) show such kind of difference of the magnitude on different convolutional layers. In addition, from our estimations in Fig.~\ref{fig:magnitude}~(c), this inconsistent magnitude can be explained for another common problem in practical applications of global pruning: the ResNet is easily pruned off. As shown in Fig.~\ref{fig:magnitude}~(c), we take ResNet56 as an example. Since the \textit{Importance Score} in first stage is much smaller than the \textit{Importance Score} in the deeper layer, global pruning will give priority to prune the convolutional filters of the first stage. For problem, we suggest that some normalization tricks should be implemented or a protection mechanism should be established, \textit{e.g.}, a mechanism which can ensure that each layer has at least a certain number of convolutional filters that will not be pruned. Unlike some previous works~\cite{hecap,chin2020towards,wang2019cop}, which make suggestions from qualitative observation, we provide a quantitative view to illustrate that these tricks are necessary.

	

	\textbf{Similarity while using global pruning}. In Fig.~\ref{fig:other_simi}~(e-f), we show the similarity of different types of pruning criteria using global pruning on VGG16 and ResNet56. Comparing to the results from the layer-wise pruning shown in Fig.~\ref{fig:other_simi}~(a-d), we can find that the similarities of most pruning criteria are quite different in global pruning. In addition, the same criteria may have different results for different network structures in global pruning, \textit{e.g.,} in Fig.~\ref{fig:other_simi}~(e), we can find $\ell_2 \cong$ Taylor $\ell_2$ and BN$_\gamma \cong \ell_2$, but this observation does not hold in Fig.~\ref{fig:other_simi}~(f). In particular, different from the result about ResNet56 in Fig.~\ref{fig:other_simi}~(f), the similarity between $\ell_1$ and $\ell_2$ is not as strong as the one in the layer-wise case. This phenomenon is counter intuitive.
	
	To understand this phenomenon, we first consider about a simple case, \textit{i.e.,} the unpruned network has only two convolutional layers~(layer $A$ and layer $B$). The filters in these two layers are $F_A = (F_A^1,F_A^2,...,F_A^n)$ and $F_B = (F_B^1,F_B^2,...,F_B^m)$. According to CWDA, for $1\leq i \leq n$ and $1 \leq j \leq m$, $F_A^i$ and $F_B^j$ can follow $N(\mathbf{0},\sigma_A^2\mathbf{I}_{d_A})$ and $N(\mathbf{0},\sigma_B^2\mathbf{I}_{d_B})$, respectively. Next, we show Sp between \textit{Importance Score} measured by $\ell_1$ and $\ell_2$ pruning in different dimension ratio $d_A/d_B$, $\sigma_A$ and $\sigma_B$ in Fig.~\ref{fig:exp_pruned_tuned}~(e-h). Moreover, to analyze this phenomenon concisely, we draw some scatter plots as shown in Fig.~\ref{fig:exp_pruned_tuned}~(a-d), where the coordinates of each point are given by (value of $\ell_1$, value of $\ell_2$). The set of the points consisting of the filters in layer $A$ is called group-$A$. Then we introduce the Proposition 2.
    \begin{proposition}
    If the convolutional filters $F_A$ in layer $A$ meet CWDA, then $\mathbb{E}[\ell_1(F_A)/\ell_2(F_A)]$ and $\mathbb{E}[\ell_2(F_A)/\ell_1(F_A)]$ only depend on their dimension $d_A$.
    \label{prop:slope}
    \vspace{-0.3cm}
    \end{proposition}
\begin{proof}
	(See Appendix \ref{app:prop}).\qedhere
	\vspace{-0.4cm}
\end{proof}

    Now we analyze the simple case under different situations:

	(1)~For $d_A/d_B = 1$. If $\sigma_A^2 = \sigma_B^2$, in fact, it's the same situation as layer-wise pruning. From Theorem~\ref{theo:similarity-layer}, we know that group-$A$ and group-$B$ coincide and approximately lie on the same line, resulting $\ell_1 \cong \ell_2$ . If $\sigma_A^2 \not= \sigma_B^2$, group-$A$ and group-$B$ lie on two lines, respectively. However, these two lines have the same slope based on Proposition~\ref{prop:slope}, as shown in Fig.~\ref{fig:exp_pruned_tuned}~(a). For these reasons, we have $\ell_1 \cong \ell_2$ when $d_A/d_B = 1$.
	
	(2)~For $d_A/d_B \not= 1$. In Fig.~\ref{fig:exp_pruned_tuned}~(b-d), there are three main situations about the position relationship between group-$A$ and group-$B$. In Fig.~\ref{fig:exp_pruned_tuned}~(b), according to Theorem~\ref{theo:similarity-layer}, the points in group-$A$ and group-$B$ are monotonic respectively. Moreover, their \textit{Importance Score} measured by $\ell_1$ and $\ell_2$ do not overlap, which make $\ell_1$ and $\ell_2$ are \textit{approximately monotonic} overall. Thus, $\ell_1 \cong \ell_2$. However, for Fig.~\ref{fig:exp_pruned_tuned}~(c-d), the Sp is small since the points in these two group are not monotonic (the \textit{Importance Score} measured by $\ell_1$ or $\ell_2$ has a large overlap). From Proposition~\ref{prop:mean_var_cri} and the approximation $\Gamma(\frac{d_A+1}{2})/\Gamma(\frac{d_A}{2}) \approx \sqrt{d_A/2}$~(Appendix~\ref{proof:l1vsl2}), these two situations can be described as: 
\begin{equation}
\sigma_{A} d_{A} \approx \sigma_{B} d_{B}\quad or\quad \sigma_{A} \sqrt{d_{A}} \approx \sigma_{B} \sqrt{d_{B}},
\label{eqn:overlap}
\end{equation}
where $d_A \not= d_B$. Through Eq.~(\ref{eqn:overlap}) we can obtain the two red lines shown in Fig.~\ref{fig:exp_pruned_tuned}~(f-h). It can be seen that the area surrounded by these two red lines is consistent with the area where the Sp is relatively small, which means our analysis is reasonable. Based on the above analysis, we can summarize the conditions about $\ell_1 \cong \ell_2$ in global pruning for two convolutional layers as shown in Table~\ref{tab:twolayer_global}.

\begin{wraptable}{r}{7cm}
\vspace{-0.6cm}
  \centering
  \small
  \caption{The conditions about $\ell_1 \cong \ell_2$ in global pruning for two layers~(layer $A$ and layer $B$)}
  	\resizebox{0.5\columnwidth}{!}{
    \begin{tabular}{rccc|c}
    \hline
          & \multicolumn{1}{l}{$d_A = d_B$?} & \multicolumn{1}{l}{$\frac{\sigma_{A}}{\sigma_{B}}  \approx \frac{d_{B}}{d_{A}} $?} & \multicolumn{1}{l|}{$ \frac{\sigma_{A}}{\sigma_{B}}   \approx \frac{\sqrt{d_{B}}}{\sqrt{d_{A}}} $?} & \multicolumn{1}{l}{$\ell_1 \cong \ell_2$?} \\
    \hline
    (1)     &{\color{black}{\Checkmark}}       & \textbf{--}      & \textbf{--}      &{\color{green}{\Checkmark}}  \\
    (2)     &{\color{black}{\XSolidBrush}}       &{\color{black}{\XSolidBrush}}       &{\color{black}{\XSolidBrush}}       &{\color{green}{\Checkmark}}   \\
    (3)     &{\color{black}{\XSolidBrush}}     &{\color{black}{\Checkmark}}       & \textbf{--}      &{\color{red}{\XSolidBrush}}  \\
    (4)     &{\color{black}{\XSolidBrush}}      & \textbf{--}      &{\color{black}{\Checkmark}}       &{\color{red}{\XSolidBrush}}  \\

    \hline
    \end{tabular}%
    }
  \label{tab:twolayer_global}%
  \vspace{-0.2cm}
\end{wraptable}

Next, we go back to the the situation about real neural networks in Fig.~\ref{fig:other_simi}~(e-f).
 (1)~For ResNet56. As shown in Fig.\ref{fig:magnitude}~(d), the dimensions of the filters in each stage are almost the same. From Table~\ref{tab:twolayer_global}~(1), the pruning results after $\ell_1$ and $\ell_2$ pruning in each stage are similar. And, the magnitudes of the \textit{Importance Score} in each stage are very different, since Table~\ref{tab:twolayer_global}~(2), we can obtain that $\ell_1 \cong \ell_2$ for ResNet56.

(2) For VGG16. As shown in Fig.\ref{fig:magnitude}~(a-b), compared with ResNet56, VGG16 has some layers with different dimensions but similar \textit{Importance Score} measured by $\ell_1$ or $\ell_2$, such as ``layer 2'' and ``layer 8'' for $\ell_2$ criterion in Fig.\ref{fig:magnitude}~(a). From Table~\ref{tab:twolayer_global}~(3-4), these pairs of layers make the Sp small, which explain why the result of $\ell_1$ and $\ell_2$ pruning is not similar in Fig.~\ref{fig:other_simi}~(e) for VGG16. In Appendix~\ref{app:support}, more experiments show that we can increase the Sp in global pruning by ignoring part of these pairs of layers, which support our analysis.


	\section{Discussion} 
	\label{Discussion}
	
	\subsection{Why CWDA sometimes does not hold?} 
	\label{why}
	CWDA may not always hold. As shown in Appendix~\ref{app:Statistical Test}, a small number of convolutional filters may not pass all statistical tests. In this section, we try to analyze this phenomenon.

	(1)~\textbf{The network is not trained well enough}. The distribution of parameters should be discussed \textbf{only when} the network is trained well. If the network does not converge, it is easy to construct a scenario which does not satisfy CWDA, \textit{e.g.}, for a network with uniform initialization, when it is only be trained for a few epochs, the distribution of parameters may be still close to a uniform distribution. At this time, the distribution obviously does not satisfy CWDA. A specific example is in Appendix~\ref{app:other_result}.

	(2)~\textbf{The number of filters is insufficient.} In Appendix~\ref{app:Statistical Test}, the layers that can not pass the statistical tests are almost those whose position is in the front of the network. A common characteristic of these layers is that they have a few filters, which may not estimate statistics well. Taking the second convolutional layer~(64 filters) in VGG16 on CIFAR10 as an example, first, the filters in this layer can not pass all the statistical tests. And then the Sp in this transition layer is relatively small, as shown in Fig.~\ref{fig:vggmore}. However, in Fig.~\ref{fig:vgg_change}, we change the number of filters in this layer from 64 to 128 or 256. After that, the Sp increases significantly, and the filters can pass all the statistical tests when the number of filters is 256. These observations suggest that the number of filters is a major factor for CWDA to be hold.

 		\begin{figure*} [h]
	\centering 
	\vspace{-0.2cm}
	\includegraphics[width=0.92\linewidth]{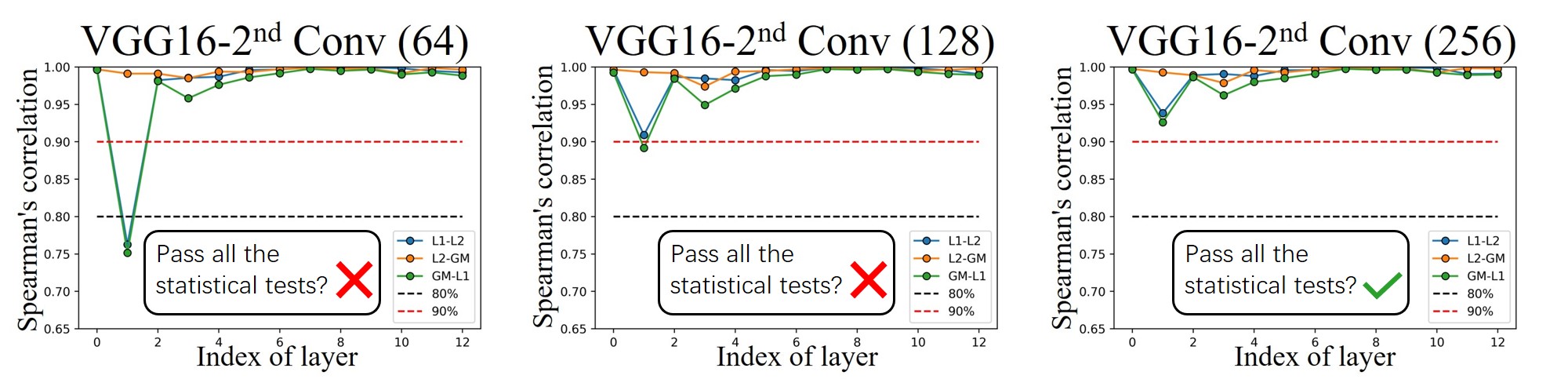}
	
	\centering 
	\vspace{-0.2cm}
	\caption{The Sp between different pruning criteria on VGG16~(CIFAR10). The number of filters in the second convolutional layers is changed from 64 to 256. The filters in this layer can pass all the statistical tests when the number of filters is 256.} 
	\vspace{-0.2cm}
	\label{fig:vgg_change}
\end{figure*}


%

	\subsection{How our findings help the community?} 
	\label{how}	
	
	(1)~We propose an assumption about the parameters distribution of the CNNs called CWDA, which is an effective theoretical tool for analyzing convolutional filter. In this paper, CWDA is successfully used to explain many phenomena in the Similarity and Applicability of pruning criteria. In addition, it also explains why the ResNet is easily pruned off in global pruning. In Section~\ref{Statistical test}, since CWDA can pass statistical tests in various situations, it can be expected that it can also be used as an effective and concise analysis tool for other CNNs-related areas, \textbf{not just} pruning area.
	
	(2)~In this paper, we study the Similarity and Applicability problem about pruning criteria, which can guide and motivate the researchers to design more reasonable criteria. For Applicability problem, we suggest that, intuitively, it is reasonable that the \textit{Importance Score} should be distinguishable for the proposed novel criteria. For Similarity, as more and more criteria are proposed, these criteria can be used for ensemble
learning to enhance their pruning performance~\cite{he2020learning}. In this case,
the similarity analysis between criteria in this paper is important, because highly similar criteria cannot bring gains to ensemble learning.
	
	(3)~In pruning area, $\ell_1$ and $\ell_2$ are usually regarded as the same pruning criteria, which is intuitive. In layer-wise pruning, we do prove that the $\ell_1$ and $\ell_2$ pruning are almost the same. However, in global pruning, the pruning results by these two criteria are sometimes very different. In addition, compared with $\ell_1$ criterion, $\ell_2$ criterion is prone to Applicability problems. These counter-intuitive phenomena enlighten us that we can't just rely on intuition when analyzing problems.


	

\begin{table}[htbp]
  \centering
  \caption{The random pruning results of VGGNet with different criteria which have the Applicability problem. The VGG16 and VGG19 are trained on CIFAR100. The unpruned baseline accuracy of VGG16 and VGG19 are 72.99 and 73.42, respectively.}
  \resizebox{\columnwidth}{!}{%
    \begin{tabular}{cl|cccc|cccc}
    \hline
    \textbf{Model} & \textbf{criterion} & \textbf{min (r=10\%)} & \textbf{max (r=10\%)} & \textbf{mean (r=10\%)} & \textbf{$\Delta$} & \textbf{min (r=20\%)} & \textbf{max (r=20\%)}& \textbf{mean  (r=20\%)} & \textbf{$\Delta$} \\
    \hline
    \multirow{5}[1]{*}{\begin{sideways}VGG16\end{sideways}} & $\ell_2$    & 71.41 & 72.65 & 71.75 & 1.24  & 71.01 & 72.47 & 71.32 & 1.46 \\
          & Taylor $\ell_1$  & 71.67 & 72.34 & 71.89 & 0.67  & 71.32 & 72.32 & 71.45 & 1.01 \\
          & Taylor $\ell_2$  & 71.87 & 72.37 & 71.91 & 0.5   & 71.66 & 72.27 & 71.65 & 0.61 \\
          & BN$_\gamma$ & 71.09 & 71.66 & 71.36 & 0.57  & 71.02 & 71.57 & 71.12 & 0.55 \\
          & BN$_\beta$ & 71.15 & 72.58 & 71.43 & 1.43  & 71.06 & 72.11 & 71.87 & 1.05 \\
    \hline
    \multirow{5}[2]{*}{\begin{sideways}VGG19\end{sideways}} & $\ell_2$     & 71.99 & 73.15 & 72.26 & 1.16  & 71.11 & 73.02 & 72.15 & 1.91 \\
          & Taylor $\ell_1$  & 71.67 & 73.04 & 72.23 & 1.37  & 71.6  & 72.98 & 72.24 & 1.38 \\
          & Taylor $\ell_2$  & 72.12 & 72.99 & 72.28 & 0.87  & 72.04 & 72.83 & 72.54 & 0.79 \\
          & BN$_\gamma$ & 72.01 & 73.23 & 72.25 & 1.22  & 71.98 & 72.32 & 72.12 & 0.34 \\
          & BN$_\beta$ & 72.25 & 73.23 & 72.41 & 0.98  & 72.04 & 72.65 & 72.33 & 0.61 \\
    \hline
    \end{tabular}%
    }
  \label{tab:addlabel}%
\end{table}%

(4)~Similar to the setting in Fig.~\ref{fig:other_simi}, we can explore the effect of pruning filters with similar \textit{Importance Score} on the performance. First, we find that the criteria ($\ell_2$,Taylor $\ell_1$, Taylor $\ell_2$, BN$_{\gamma}$  and BN$_{\beta}$) for VGGNet can cause the Applicability problem in most layers (Fig.~\ref{fig:other_simi}). As such, we randomly select 10\% or 20\% filters to be pruned by the uniform distribution $U[0,1]$ in each layer, and the selective filters will be in similar \textit{Importance Score}. Finally, we finetune the pruned model (there are 20 random repeated experiments). $\Delta$ denotes the difference between max acc. and min acc. (\textit{i.e.} max acc. - min acc.) . Since their \textit{Importance Score} are very similar, when the network is pruned and finetuned, it can be expected that the performance should be similar in these repeated experiments. However, from the results in the above table, although the \textit{Importance Score} of the pruned filters is very close, we can still get pruning results with very different results (\textit{e.g.} the $\Delta$ of VGG16 on $\ell_2$ are more than 1). It means that these criteria may not really represent the importance of convolutional filters. Therefore, it is necessary to re-evaluate the correctness of the existing pruning criteria. 


\textbf{Acknowledgments}. Z. Huang gratefully acknowledges the technical and writing support from Mingfu Liang~(Northwestern University), Senwei Liang~(Purdue University) and Wei He~(Nanyang Technological University). Moreover, he sincerely thanks Mingfu Liang for offering his self-purchasing GPUs and Qinyi Cai (NetEase, Inc.) for checking part of the proof in this paper. This work was supported in part by the General Research Fund of Hong Kong No.27208720, the National Key R\&D Program of China under Grant No. 2020AAA0109700, the National Science Foundation of China under Grant  No.61836012 and 61876224, the National High Level Talents Special
Support Plan (Ten Thousand Talents Program), and GD-NSF (no.2017A030312006).

\clearpage
\onecolumn
\appendix

\clearpage
\section{Related Proposition}
\label{app:prop}
\begin{proposition}[Amoroso distribution]~The Amoroso distribution is a four parameter, continuous, univariate, unimodal probability density, with semi-infinite range~\cite{crooks2012survey}. And its probability density function is
	\begin{equation}
	\mathbf{Amoroso}(X|a,\theta,\alpha,\beta) = \frac{1}{\Gamma(\alpha)}|\frac{\beta}{\theta}|(\frac{X-a}{\theta})^{\alpha\beta-1}\exp \left\{ -(\frac{X-a}{\theta})^\beta\right\},
	\end{equation}
	for $x,a,\theta, \alpha, \beta \in \mathbb{R}, \alpha >0$ and range $x \geq a$ if $\theta > 0$, $x \leq a$ if $\theta < 0$. The mean and variance of Amoroso distribution are
	\begin{equation}
	\mathbb{E}_{X \sim \mathbf{Amoroso}(X|a,\theta,\alpha,\beta)}X = a + \theta\cdot \frac{\Gamma(\alpha + \frac{1}{\beta})}{\Gamma(\alpha)},
	\label{equ:mean_amo}
	\end{equation}
	and
	\begin{equation}
	\mathbf{Var}_{X \sim \mathbf{Amoroso}(X|a,\theta,\alpha,\beta)}X = \theta^2 \left[ \frac{\Gamma(\alpha + \frac{2}{\beta})}{\Gamma(\alpha)} - \frac{\Gamma(\alpha + \frac{1}{\beta})^2}{\Gamma(\alpha)^2} \right].
	\label{equ:var_amo}
	\end{equation}
	\vspace{0.3cm}
	\label{pro:Amoroso}
\end{proposition}

\begin{proposition}[Half-normal distribution]~Let random variable $X$ follow a normal distribution $N(0,\sigma ^{2})$, then $Y = |X|$ follows a half-normal distribution~\cite{pescim2010beta}. Moreover, $Y$ also follows $\mathbf{Amoroso}(x|0,\sqrt{2}\sigma,\frac{1}{2},2)$. By Eq.~(\ref{equ:mean_amo}) and Eq.~(\ref{equ:var_amo}), the mean and variance of half-normal distribution are
	\begin{equation}
	\mathbb{E}_{X \sim N(0,\sigma ^{2})} |X| = \sigma\sqrt{2/\pi},
	\label{equ:mean_half}
	\end{equation}
	and
	\begin{equation}
	\mathbf{Var}_{X \sim N(0,\sigma ^{2})} |X| = \sigma^2\left(1 - \frac{2}{\pi}\right).
	\label{equ:var_half}
	\end{equation}
	
	\vspace{0.3cm}
	\label{pro:half normal}
\end{proposition}

\begin{proposition}[Scaled Chi distribution]~Let $X = (x_1,x_2,...x_k)$ and $x_i,i=1,...,k$ are $k$ independent, normally distributed random variables  with mean 0 and standard deviation $\sigma$. The statistic $\ell_2(X) = \sqrt{\sum_{i=1}^kx_i^2}$ follows Scaled Chi distribution~\cite{crooks2012survey}. Moreover, $\ell_2(X)$ also follows $\mathbf{Amoroso}(x|0,\sqrt{2}\sigma,\frac{k}{2},2)$. By Eq.~(\ref{equ:mean_amo}) and Eq.~(\ref{equ:var_amo}), the mean and variance of Scaled Chi distribution are
	\begin{equation}
	\mathbb{E}_{X \sim N(\mathbf{0},\sigma ^{2}\cdot \mathbf{I_k})} [\ell_2(X)]^j = 2^{j/2}\sigma^j\cdot \frac{\Gamma(\frac{k+j}{2})}{\Gamma(\frac{k}{2})},
	\label{equ:mean_chi}
	\end{equation}
	and
	\begin{equation}
	\mathbf{Var}_{X \sim N(\mathbf{0},\sigma ^{2}\cdot \mathbf{I_k})} \ell_2(X) = 2\sigma^2 \left[ \frac{\Gamma(\frac{k}{2}+1)}{\Gamma(\frac{k}{2})} - \frac{\Gamma(\frac{k+1}{2})^2}{\Gamma(\frac{k}{2})^2} \right].
	\label{equ:var_chi}
	\end{equation}
	\vspace{0.3cm}
	\label{pro:chi}
\end{proposition}

\begin{proposition}[Stirling's formula]\footnote{\url{en.wikipedia.org/wiki/Stirling's approximation}} For big enough $x$ and $x \in \mathbb{R}^+$, we have an approximation of Gamma function:
	\begin{equation}
	\Gamma(x+1) \approx \sqrt{2\pi x}\left(\frac{x}{e}\right)^x.
	\end{equation}
	\label{pro:stir}
	\vspace{0.3cm}
\end{proposition}

\begin{proposition}[FKG inequality] If $f$ and $g$ are increasing functions on $\mathbb{R}^n$ ~\cite{graham1983applications}, we have 
	\begin{equation}
	\mathbb{E}(f)\mathbb{E}(g) \leq \mathbb{E}(fg).
	\end{equation}
	Say that a function on $\mathbb{R}^n$ is increasing if it is an increasing function in each of its arguments.(\textit{i.e.}, for fixed values of the other arguments).
	\label{pro:fkg}
	\vspace{0.3cm}
\end{proposition}

\begin{proposition} Let $f(X,Y)$ is a two dimensional differentiable function. According to Taylor theorem~\cite{hormander1983analysis}, we have
	\begin{equation}
	f(X,Y) = f(\mathbb{E}(X),\mathbb{E}(Y)) + \sum_{cyc}(X - \mathbb{E}(X))\frac{\partial}{\partial X}f(\mathbb{E}(X),\mathbb{E}(Y)) + Remainder1,
	\label{pro:taylor1}
	\end{equation}
	
	\begin{equation}        
	\begin{aligned}
	f(X,Y) &= f(\mathbb{E}(X),\mathbb{E}(Y)) + \sum_{cyc}(X - \mathbb{E}(X))\frac{\partial}{\partial X}f(\mathbb{E}(X),\mathbb{E}(Y)) + \\
	& \frac{1}{2}\sum_{cyc}(X - \mathbb{E}(X))^T\frac{\partial^2}{\partial X^2}f(\mathbb{E}(X),\mathbb{E}(Y))(X - \mathbb{E}(X)) + Remainder2
	\end{aligned}   
	\label{pro:taylor2}
	\end{equation}
	\vspace{0.3cm}
\end{proposition}

\begin{lemma} Let $X$ and $Y$ are random variables. Then we have such an estimation
	\begin{equation}
	\mathbf{Var}\left(\frac{X}{Y}\right) \approx \left(\frac{\mathbb{E}(X)}{\mathbb{E}(Y)}\right)^2 \left( \frac{\mathbf{Var}X}{\mathbb{E}(X)^2}  + \frac{\mathbf{Var}Y}{\mathbb{E}(Y)^2} -2\frac{\mathbf{Cov}(X,Y)}{\mathbb{E}(X)\mathbb{E}(Y)} \right).
	\end{equation}
	\label{lemma:varx_y}
\end{lemma}{}
\begin{proof}
	Let $f(X,Y) = X/Y$, according to the definition of variance, we have
	\begin{align*}
	\mathbf{Var} f(X,Y)&= \mathbb{E}[f(X,Y) - \mathbb{E}(f(X,Y))]^2\\ 
	&\approx \mathbb{E}[f(X,Y) - \mathbb{E}\left\{f(\mathbb{E}(X),\mathbb{E}(Y)) + \sum_{cyc}(X - \mathbb{E}(X))\frac{\partial}{\partial X}f(\mathbb{E}(X),\mathbb{E}(Y))\right\}]^2 \tag*{from Eq. (\ref{pro:taylor1})}\\
	& = \mathbb{E}[f(X,Y) - f(\mathbb{E}(X),\mathbb{E}(Y)) - \sum_{cyc}\mathbb{E}(X - \mathbb{E}(X))\frac{\partial}{\partial X}f(\mathbb{E}(X),\mathbb{E}(Y))]^2\\
	&=  \mathbb{E}[f(X,Y) - f(\mathbb{E}(X),\mathbb{E}(Y))]^2\\  
	&\approx \mathbb{E}[\sum_{cyc}(X - \mathbb{E}(X))\frac{\partial}{\partial X}f(\mathbb{E}(X),\mathbb{E}(Y))]^2\tag*{from Eq. (\ref{pro:taylor1})}\\
	& = 2\mathbf{Cov}(X,Y) \frac{\partial}{\partial X}f(\mathbb{E}(X),\mathbb{E}(Y)) \frac{\partial}{\partial Y}f(\mathbb{E}(X),\mathbb{E}(Y)) +  \sum_{cyc} [\frac{\partial}{\partial X}f(\mathbb{E}(X),\mathbb{E}(Y))]^2 \cdot \mathbf{Var}X\\
	& = 2\mathbf{Cov}(X,Y) \cdot \frac{1}{\mathbb{E}(Y)} \cdot \left(-\frac{\mathbb{E}(X)}{(\mathbb{E}(Y))^2}\right) + \frac{1}{(\mathbb{E}(Y))^2}\cdot \mathbf{Var}X + \frac{(\mathbb{E}X)^2}{(\mathbb{E}Y)^4}\cdot \mathbf{Var}Y\\
	& = \left(\frac{\mathbb{E}(X)}{\mathbb{E}(Y)}\right)^2 \left( \frac{\mathbf{Var}X}{\mathbb{E}(X)^2}  + \frac{\mathbf{Var}Y}{\mathbb{E}(Y)^2} -2\frac{\mathbf{Cov}(X,Y)}{\mathbb{E}(X)\mathbb{E}(Y)} \right).
	\end{align*}
	\qedhere
\end{proof}

	From Eq.(\ref{pro:taylor2}) and \textbf{Lemma 1}, we also can obtain an estimation of $\mathbb{E}(\mathbf{A}/\mathbf{B})$, where $\mathbf{A}$ and $\mathbf{B}$ are two random variables. \textit{i.e.},
	\begin{equation}
	\mathbb{E}\left(\frac{\mathbf{A}}{\mathbf{B}}\right) \approx \frac{\mathbb{E}\mathbf{A}}{\mathbb{E}\mathbf{B}} + \mathbf{Var(B)}\cdot \frac{\mathbb{E}\mathbf{A}}{(\mathbb{E}\mathbf{B})^3}.
	\label{estimation:E}
	\end{equation}

\begin{lemma} For big enough $x$ and $x \in \mathbb{R}^+$, we have
	\begin{equation}
	\lim_{x \to +\infty} \left[\frac{\Gamma(\frac{x+1}{2})}{\Gamma(\frac{x}{2})}\right]^2\cdot \frac{1}{x} = \frac{1}{2}.
		\label{eq:1_2}
	\end{equation}

	And 
	\begin{equation}
	\lim_{x \to +\infty} \frac{\Gamma(\frac{x}{2}+1)}{\Gamma(\frac{x}{2})} - \left[\frac{\Gamma(\frac{x+1}{2})}{\Gamma(\frac{x}{2})}\right]^2 = \frac{1}{4}.
	\label{eq:1_4}
	\end{equation}
	\label{lemma:gamma1}
\end{lemma}{}

\begin{proof}
	\begin{align*}
	\lim_{x \to +\infty}\left[\frac{\Gamma(\frac{x+1}{2})}{\Gamma(\frac{x}{2})}\right]^2\cdot \frac{1}{x} & \approx \lim_{x \to +\infty}\left( \frac{\sqrt{2\pi(\frac{x-1}{2})}\cdot(\frac{x-1}{2e})^{\frac{x-1}{2}}}{\sqrt{2\pi(\frac{x-2}{2})}\cdot(\frac{x-2}{2e})^{\frac{x-2}{2}}}\right)^2\cdot \frac{1}{x}\tag*{from Proposition. \ref{pro:stir}} \\
	& = \lim_{x \to +\infty}\left(\frac{x-1}{x-2}\right)\cdot \frac{(\frac{x-1}{2e})^{x-2}}{(\frac{x-2}{2e})^{x-2}} \cdot \left(\frac{x-1}{2e}\right)\cdot \frac{1}{x} \\
	&= \lim_{x \to +\infty}\left(1+\frac{1}{x-2}\right)^{x-2}\cdot \frac{x-1}{x-2}\cdot\frac{x-1}{2e}\cdot\frac{1}{x}\\
	& = \frac{1}{2}
	\end{align*}
	on the other hand, we have
	\begin{align*}
	\lim_{x \to +\infty}\frac{\Gamma(\frac{x}{2}+1)}{\Gamma(\frac{x}{2})} - \left[\frac{\Gamma(\frac{x+1}{2})}{\Gamma(\frac{x}{2})}\right]^2 &= \lim_{x \to +\infty} \frac{x}{2} - \left(1+\frac{1}{x-2}\right)^{x-2}\cdot \frac{x-1}{x-2}\cdot\frac{x-1}{2e}\\
	& = \lim_{x \to +\infty} \frac{x}{2e}\left(e - (1+\frac{1}{x})^x\right)\\
	&= \frac{1}{2}\left(-\frac{\frac{1}{e}(-e)}{2}\right)\\
	& = \frac{1}{4}
	\end{align*}
	
	\qedhere
\end{proof}

\begin{proposition}
KL divergence between two distributions $P$ and $Q$ of a continuous random variable is given by $D_{K L}(p \| q)=\int_{x} p(x) \log \frac{p(x)}{q(x)}$. And probabilty density function of multivariate Normal distribution is given by $p(\mathbf{x})=\frac{1}{(2 \pi)^{k / 2}|\Sigma|^{1 / 2}} \exp \left(-\frac{1}{2}(\mathbf{x}-\boldsymbol{\mu})^{T} \Sigma^{-1}(\mathbf{x}-\boldsymbol{\mu})\right)$. Let our two Normal distributions be $\mathcal{N}\left(\boldsymbol{\mu}_{\boldsymbol{p}}, \Sigma_{p}\right)$ and $\mathcal{N}\left(\boldsymbol{\mu}_{q}, \Sigma_{q}\right)$, both $k$ dimensional. we have
\begin{equation}
    D_{K L}(p \| q)=\frac{1}{2}\left[\log \frac{\left|\Sigma_{q}\right|}{\left|\Sigma_{p}\right|}-k+\left(\boldsymbol{\mu}_{\boldsymbol{p}}-\boldsymbol{\mu}_{\boldsymbol{q}}\right)^{T} \Sigma_{q}^{-1}\left(\boldsymbol{\mu}_{\boldsymbol{p}}-\boldsymbol{\mu}_{\boldsymbol{q}}\right)+\operatorname{tr}\left\{\Sigma_{q}^{-1} \Sigma_{p}\right\}\right].
\end{equation}
\label{prop:two_gaussian}
\end{proposition}

\begin{proposition}[Jacobi's formula] If $A$ is a differentiable map from the real numbers to $n \times n$ matrices, 
	\begin{equation}
\frac{d}{d t} \operatorname{det} A(t)=\operatorname{tr}\left(\operatorname{adj}(A(t)) \frac{d A(t)}{d t}\right).
	\end{equation}
	\label{prop:jacobi}
\end{proposition}

\begin{proposition} For random variable $X$ with $\mu$ and $\sigma^2$ as mean and variance, then we can use Taylor expansion to obtain:
	\begin{equation}
\left\{\begin{array}{l}
\mathbb{E}(\log X) \approx \log \mu-\frac{\sigma^{2}}{2 \mu^{2}} \\
\mathbf{Var}(\log X) \approx \frac{\sigma^{2}}{\mu^{2}}
\end{array}\right..
	\end{equation}
	\label{prop:log}
\end{proposition}

\begin{proposition} Given $n$ normal distributions $N(0,\sigma_i^2), 1\leq i \leq n$ and $(X_{i1},X_{i2},...,X_{im})$ are sample from $N(0,\sigma_i^2), 1\leq j \leq m$. then
	\begin{equation}
\mathbf{Var}_{1\leq i \leq n,1\leq j \leq m}(X_{ij}) = \frac{1}{n}\sum_{i=1}^n \sigma_i^2.
	\end{equation}
	\label{prop:all_var}
\end{proposition}
\begin{proof}

	\begin{align}
	\mathbf{Var}_{1\leq i \leq n,1\leq j \leq m}(X_{ij}) 
	&= \frac{1}{mn}\sum_{i=1}^n\sum_{j=1}^m[X_{ij}-\mathbb{E}(X_{ij})]^2\\
	&=\frac{1}{n}\{\frac{1}{m}\sum_{j=1}^m[X_{ij}-\mathbb{E}(X_{1j})]^2 + ... + \frac{1}{m}\sum_{j=1}^m[X_{nj}-\mathbb{E}(X_{nj})]^2\} \tag*{Since $\mathbb{E}(X_{ij})= 0,1\leq i \leq n,1\leq j \leq m$}\\
	& =\frac{1}{n}\{\sigma_1^2 + ...+\sigma_n^2\}
	\end{align}

\end{proof}

\begin{lemma} For a matrix $\mathbf{B}\in R^{n\times n}$ and a small constant $\epsilon$, we have:
\begin{equation}
det(\mathbf{I}_n + \epsilon \mathbf{B}) = 1+\epsilon \operatorname{tr}(\mathbf{B}) + O(\epsilon^2).
\end{equation}
\label{lemma:det} 
\end{lemma}

\begin{proof}
First, we regard $det(\mathbf{I}_n + \epsilon \mathbf{B})$ as a function w.r.t $\epsilon$. Since Proposition~\ref{prop:jacobi}, we have:

	\begin{align}
	\frac{d}{d\epsilon}det(\mathbf{I}_n + \epsilon \mathbf{B})|_{\epsilon = 0} &= \operatorname{tr}\{
	\operatorname{adj}(\mathbf{I}_n + \epsilon \mathbf{B})\mathbf{B}\}|_{\epsilon = 0}\\
	&=\operatorname{tr}\{det(\mathbf{I}_n + \epsilon \mathbf{B})\cdot(\mathbf{I}_n + \epsilon \mathbf{B})^{-1}\mathbf{B}\}|_{\epsilon = 0}\\
	& =det(\mathbf{I}_n + \epsilon \mathbf{B})\cdot \operatorname{tr}\{(\mathbf{I}_n + \epsilon \mathbf{B})^{-1}\mathbf{B}\}|_{\epsilon = 0}\\
	& = \operatorname{tr}(\mathbf{B})
	\end{align}
Using Taylor expansion for $det(\mathbf{I}_n + \epsilon \mathbf{B})$, we have $\frac{d}{d\epsilon}det(\mathbf{I}_n + \epsilon \mathbf{B}) = det(\mathbf{I}_n) +  \frac{d}{d\epsilon}det(\mathbf{I}_n + \epsilon \mathbf{B})|_{\epsilon = 0}\cdot \epsilon + O(\epsilon^2)$. In other words, $det(\mathbf{I}_n + \epsilon \mathbf{B}) = 1+\epsilon \operatorname{tr}(\mathbf{B}) + O(\epsilon^2)$.

\end{proof}

\subsection{The proof of Proposition~\ref{prop:mean_var_cri}}

(\textbf{Proposition}~\ref{prop:mean_var_cri})~If the convolutional filters $F_A$ in layer $A$ meet CWDA, then we have following estimations:
\begin{table}[H]
	\begin{center}
		\begin{small}
			\begin{tabular}{lll}
				\hline
				 Criterion & Mean & Variance\\
				\hline
				$\ell_1(F_A)$    & $\sqrt{2/\pi}\sigma_Ad_A$& $(1-\frac{2}{\pi})\sigma_A^2d_A$\\
				$\ell_2(F_A)$    & $\sqrt{2}\sigma_A\Gamma(\frac{d_A+1}{2})/\Gamma(\frac{d_A}{2})$& $\sigma_A^2/2$\\
				$\mathbf{Fermat}(F_A)$   & $\sqrt{2}\sigma_A\Gamma(\frac{d_A+1}{2})/\Gamma(\frac{d_A}{2})$& $\sigma_A^2/2$\\
				\hline
			\end{tabular}
		\end{small}
	\end{center}
\end{table}
where $d_A$ and $\sigma_A^2$ denote the dimension of $F_A$ and the variance of the weights in layer $A$, respectively.
\begin{proof}
    According to Appendix~\ref{app:relax}, Eq.~(\ref{eq:1_4}), Proposition~\ref{pro:half normal} and Proposition~\ref{pro:chi}, we can obtain the mean and variance of $\ell_1(F_A)$ and $\ell_2(F_A)$. Moreover, From the Theorem~\ref{theo:fermat}, we know that the Fermat point $\mathbf{F}$ of $F_A$ and the origin $\mathbf{0}$ approximately coincide. According to Table~\ref{criteria}, $||\mathbf{F}-F_A||_2 \approx ||\mathbf{0} - F_A||_2 = ||F_A||_2$. Therefore, the mean and variance of $\mathbf{Fermat}(F_A)$ are the same as $\ell_2(F_A)$'s in Proposition~\ref{prop:mean_var_cri}.

\end{proof}

\subsection{The proof of Proposition~\ref{prop:slope}}

(\textbf{Proposition}~\ref{prop:slope})~If the convolutional filters $F_A$ in layer $A$ meet CWDA, then $\mathbb{E}[\ell_1(F_A)/\ell_2(F_A)]$ and $\mathbb{E}[\ell_2(F_A)/\ell_1(F_A)]$ only depend on their dimension $d_A$.
\begin{proof}
From Eq.~(\ref{estimation:E}), we have:

  	\begin{align*}
  	\mathbb{E}[\frac{\ell_1(F_A)}{\ell_2(F_A)}]& \approx \frac{\mathbb{E}[\ell_1(F_A)]}{\mathbb{E}[\ell_2(F_A)]} + \mathbf{Var}[\ell_2(F_A)] \cdot \frac{\mathbb{E}[\ell_1(F_A)]}{\mathbb{E}[\ell_2(F_A)]^3} \\
	&= \frac{\sqrt{2/\pi}\sigma_Ad_A}{\sqrt{2}\sigma_A\Gamma(\frac{d_A+1}{2})/\Gamma(\frac{d_A}{2})} + \sigma_A^2/2 \cdot \frac{\sqrt{2/\pi}\sigma_Ad_A}{[\sqrt{2}\sigma_A\Gamma(\frac{d_A+1}{2})/\Gamma(\frac{d_A}{2})]^3} \tag*{from Proposition. \ref{prop:mean_var_cri}}\\
	&\approx O(\sqrt{d_A}) + O(\frac{1}{\sqrt{d_A}}) \tag*{from Eq.~ (\ref{eq:1_2})}\\
	\end{align*}  
Similarly, we can prove that $\mathbb{E}[\ell_2(F_A)/\ell_1(F_A)]$ also only depend on their dimension $d_A$.
    
  	\begin{align*}
  	\mathbb{E}[\frac{\ell_2(F_A)}{\ell_1(F_A)}]& \approx \frac{\mathbb{E}[\ell_2(F_A)]}{\mathbb{E}[\ell_1(F_A)]} + \mathbf{Var}[\ell_1(F_A)] \cdot \frac{\mathbb{E}[\ell_2(F_A)]}{\mathbb{E}[\ell_1(F_A)]^3} \\
	&= \frac{\sqrt{2}\sigma_A\Gamma(\frac{d_A+1}{2})/\Gamma(\frac{d_A}{2})}{\sqrt{2/\pi}\sigma_Ad_A} + (1-\frac{2}{\pi})\sigma_A^2d_A \cdot \frac{\sqrt{2}\sigma_A\Gamma(\frac{d_A+1}{2})/\Gamma(\frac{d_A}{2})}{[\sqrt{2/\pi}\sigma_Ad_A]^3} \tag*{from Proposition. \ref{prop:mean_var_cri}}\\
	&\approx  O(\frac{1}{\sqrt{d_A}}) +O(\frac{1}{d_A^{1.5}}) \tag*{from Eq.~ (\ref{eq:1_2})}\\
	\end{align*}

\end{proof}

\section{The relaxation for CWDA}
\label{app:relax}


	\textbf{(Convolution Weight Distribution Assumption)}~Let $F_{ij}\in \mathbb{R}^{N_i\times k\times k}$ be the $j^{\rm th}$ well-trained filter of the $i^{\rm th}$ convolutional layer. In general\footnote{In Section~\ref{Discussion}, we make further discussion and analysis on the conditions for CWDA to be satisfied.}, in $i^{\rm th}$ layer, $F_{ij}~( j=1,2,...,N_{i+1})$ are i.i.d and follow such a distribution:
	\begin{equation}
	F_{ij} \sim \mathbf{N}(\mathbf{0}, \mathbf{\Sigma}^i_{\text{diag}} + \epsilon\cdot\mathbf{\Sigma}^i_{\text{block}}),
	\end{equation}
	where $\mathbf{\Sigma}^i_{\text{block}} = \mathrm{diag}(K_1,K_2,...,K_{N_i})$ is a block diagonal matrix and the diagonal elements of $\mathbf{\Sigma}^i_{\text{block}}$ are 0. $\epsilon$ is a small constant. The values of the off-block-diagonal elements are 0 and $K_l \in R^{k^2\times k^2}, l=1,2,...,N_i$. $\mathbf{\Sigma}^i_{\text{diag}}= \mathrm{diag}(a_1,a_2,...,a_{N_i \times k \times k})$ is a diagonal matrix and the elements of $\mathbf{\Sigma}^i_{\text{diag}}$ are close enough.

In Section~\ref{distribution}, we propose CWDA. In order to use this assumption conveniently, we give the following relaxation of CWDA:

	(\textbf{Convolution Weight Distribution Assumption-Relaxation})~Let $F_{ij}\in \mathbb{R}^{N_i\times k\times k}$ be the $j^{\rm th}$ well-trained filter of the $i^{\rm th}$ convolutional layer. In general, in $i^{\rm th}$ layer, $F_{ij}~( j=1,2,...,N_{i+1})$ are i.i.d and follow such a distribution:
	\begin{equation}
	F_{ij} \sim \mathbf{N}(\mathbf{0}, \sigma_{\text{layer}}^2\cdot \mathbf{I}_{N_i\times k\times k}),
	\end{equation}
	where $\sigma_{\text{layer}}^2$ is the variance of the weights in $i^{\rm th}$ convolutional layer.

Next, we analyze the gap between CWDA and CWDA-Relaxation, \textit{i.e.,} the difference between $\mathbf{N}(\mathbf{0}, \mathbf{\Sigma}^i_{\text{diag}} + \epsilon\cdot\mathbf{\Sigma}^i_{\text{block}})$ and $\mathbf{N}(\mathbf{0}, \sigma_{\text{layer}}^2\cdot \mathbf{I}_{N_i\times k\times k})$.

\vspace{0.5cm}

\begin{lemma} 
Given two $n$-dimension Gaussian distributions $\mathbf{N}(\mathbf{0}, \mathbf{\Sigma}_{\text{diag}} + \epsilon\cdot\mathbf{\Sigma}_{\text{block}})$ and $\mathbf{N}(\mathbf{0}, \mathbf{\Sigma}_{\text{diag}})$, we can estimate the KL divergence of them:

\begin{equation}
    \operatorname{KL}[\mathbf{N}(\mathbf{0}, \mathbf{\Sigma}_{\text{diag}} + \epsilon\cdot\mathbf{\Sigma}_{\text{block}}) || \mathbf{N}(\mathbf{0}, \mathbf{\Sigma}_{\text{diag}})] \approx \frac{1}{2}\log [\frac{1}{1+O(\epsilon^2)}]
\end{equation}
	where $\mathbf{\Sigma}_{\text{block}} = \mathrm{diag}(K_1,K_2,...,K_{N_i})$ is a block diagonal matrix and the diagonal elements of $\mathbf{\Sigma}_{\text{block}}$ are 0. $\epsilon$ is a small constant. The values of the off-block-diagonal elements are 0 and $K_l \in R^{k^2\times k^2}, l=1,2,...,N_i$. $\mathbf{\Sigma}_{\text{diag}}= \mathrm{diag}(a_1,a_2,...,a_{N_i \times k \times k})$ is a diagonal matrix and the elements of $\mathbf{\Sigma}_{\text{diag}}$ are close enough. $n = N_i \times k \times k$.

\label{lemma:relax1}
\end{lemma}

\begin{proof}
Since Proposition~\ref{prop:two_gaussian}, we have:

	\begin{align}
	2\operatorname{KL} &= \log\frac{det[ \mathbf{\Sigma}_{\text{diag}}]}{det[ \mathbf{\Sigma}_{\text{diag}} + \epsilon\cdot\mathbf{\Sigma}_{\text{block}}]} -n +0 + \operatorname{tr}\{\mathbf{\Sigma}_{\text{diag}}^{-1}(\mathbf{\Sigma}_{\text{diag}} + \epsilon\cdot\mathbf{\Sigma}_{\text{block}})\}\\
	&=\log\frac{det[ \mathbf{\Sigma}_{\text{diag}}]}{det[ \mathbf{\Sigma}_{\text{diag}} + \epsilon\cdot\mathbf{\Sigma}_{\text{block}}]} -n +\operatorname{tr}\{\mathbf{I}_k + \epsilon \mathbf{\Sigma}_{\text{diag}}^{-1}\mathbf{\Sigma}_{\text{block}}\}\\
	&= \log\frac{det[ \mathbf{\Sigma}_{\text{diag}}]}{det[ \mathbf{\Sigma}_{\text{diag}} + \epsilon\cdot\mathbf{\Sigma}_{\text{block}}]} \tag*{Since  the diagonal elements of $\mathbf{\Sigma}_{\text{block}}$ are 0}\\
	\end{align}
Let $\mathbf{\Sigma}_{\text{diag}} = \text{diag}(S_1,S_2,...,S_{N_i})$, where $S_j = \text{diag}(a_{(j-1)k^2+1},a_{(j-1)k^2+2},...,a_{(j-1)k^2+k^2}),j=1,2,...,N_i.$

	\begin{align}
	2\operatorname{KL} &= \log\frac{det[ \mathbf{\Sigma}_{\text{diag}}]}{det[ \mathbf{\Sigma}_{\text{diag}} + \epsilon\cdot\mathbf{\Sigma}_{\text{block}}]} \\
	&= \log \prod_{j=1}^na_k   -\log \{\prod_{h=1}^{N_i}det[S_h + \epsilon K_h]\}\\
	&=\log \prod_{j=1}^na_k - \log \{\prod_{h=1}^{N_i} det[S_h]det[\mathbf{I}_{k^2} + \epsilon S_h^{-1}K_h]\}  \tag*{Since $S_h \succeq 0$}\\
	\end{align}
Note that $S_h$ is a diagonal matrix and the diagonal elements of $K_h$ are all zero. Therefore 
\begin{equation}
    \operatorname{tr}(S_h^{-1}K_h) = 0.
    \label{eq:tr_temp}
\end{equation}
Next,

	\begin{align}
	2\operatorname{KL}
	&=\log \prod_{j=1}^na_k - \log \{\prod_{h=1}^{N_i} det[S_h]det[\mathbf{I}_{k^2} + \epsilon S_h^{-1}K_h]\}  \\
	&= \log \prod_{j=1}^na_k - \log \{\prod_{h=1}^{N_i} det[S_h]\cdot (1+\epsilon \operatorname{tr}(S_h^{-1}K_h)+O(\epsilon^2))\} \tag*{Since Lemma~\ref{lemma:det}}\\
	&= \log \prod_{j=1}^na_k - \log \{\prod_{h=1}^{N_i} det[S_h]\cdot (1+O(\epsilon^2))\} \tag*{Since Eq.~(\ref{eq:tr_temp})}\\
	&= \log \prod_{j=1}^na_k - \log \prod_{j=1}^na_k (1+O(\epsilon^2))\\
	& = \log [\frac{1}{1+O(\epsilon^2)}]
	\end{align}

\end{proof}

According to Statistical test~(2) in Section~\ref{Statistical test}, $\mathbf{N}(\mathbf{0}, \mathbf{\Sigma}_{\text{diag}})$ can be approximate to $\mathbf{N}(\mathbf{0},\frac{1}{n}\operatorname{tr}(\mathbf{\Sigma}_{\text{diag}})\mathbf{I}_n)$. In addition, from Propsition~\ref{prop:all_var} and Lemma~\ref{lemma:relax1}, while $\epsilon$ is small enough, the distribution $\mathbf{N}(\mathbf{0}, \mathbf{\Sigma}_{\text{diag}} + \epsilon\cdot\mathbf{\Sigma}_{\text{block}})$ can be approximate to $\mathbf{N}(\mathbf{0}, \sigma_{\text{layer}}^2\cdot \mathbf{I}_{N_i\times k\times k})$. The analysis in this paper are based on \textit{Convolution Weight Distribution Assumption-Relaxation} and we use it to explain successfully many phenomena in the Similarity and Applicability problem of pruning criteria.

\clearpage
\section{Proof of Theorem~\ref{theo:bound}}
\label{proof:theorm1}
	\textbf{Theorem 1.} Let $n-$dimension random variable $X$ meet CWDA, and the pair of criteria $(C_1,C_2)$ is one of $(\ell_1,\ell_2)$, $(\ell_2,\mathbf{Fermat})$ or $(\mathbf{Fermat},\mathbf{GM})$, we have  
	\begin{equation}
	\mathbf{max}\left\{\mathbf{Var}_{X}\left(\frac{\widehat{C}_2(X)}{\widehat{C}_1(X)}\right),\mathbf{Var}_{X}\left(\frac{\widehat{C}_1(X)}{\widehat{C}_2(X)}\right) \right\}\lesssim B(n).
	\label{bound}
	\end{equation}
	where $\widehat{C}_1(X)$ denotes $C_1(X)/\mathbb{E}(C_1(X))$ and $\widehat{C}_2(X)$ denotes $C_2(X)/\mathbb{E}(C_2(X))$. $B(n)$ denotes the upper bound of left-hand side and when $n$ is large enough, $B(n) \to 0$.

For $i^{\rm th}$ layer, we use $v_j$ to represent $F_{ij}$, $j= 1,2,...N$. And $v_j$ meets CWDA. Since Appendix~\ref{app:relax}, we use the following three points to prove Theorem~\ref{theo:bound}.

\textbf{(1)~For} $(\ell_2 ,\ell_1)$. 
In fact, $\ell_2 \cong \ell_1$ (their importance rankings are similar) is not trivial. Generally speaking, for convolutional filters, $\mathbf{dim}(v_j)$ is large enough. Since $v_i$ satisfies CWDA, from Theorem \ref{theo:l1vsl2}, we know that the variance of ratio between $\widehat{\ell}_1$ and $\widehat{\ell}_2$ have a bound $O(\mathbf{dim}(v_j)^{-1})$, which means $\ell_2$ and $\ell_1$ are \textit{appropriate monotonic}. Specific numerical validation is shown in Fig.~\ref{fig:bound} of Appendix \ref{proof:l1vsl2}).

\begin{theorem} Let $X\sim N(\mathbf{0},c^2\cdot\mathbf{I}_n)$, we have  
	\begin{equation}
	\mathbf{max}\left\{\mathbf{Var}_{X}\left(\frac{\widehat{\ell}_2(X)}{\widehat{\ell}_1(X)}\right),\mathbf{Var}_{X}\left(\frac{\widehat{\ell}_1(X)}{\widehat{\ell}_2(X)}\right) \right\}\lesssim \frac{1}{n}.
	\end{equation}
	\label{theo:l1vsl2}
	where $\widehat{\ell}_1(X)$ denotes $\ell_1(X)/\mathbb{E}(\ell_1(X))$ and $\widehat{\ell}_2(X)$ denotes $\ell_2(X)/\mathbb{E}(\ell_2(X))$. $c$ is a constant.
\end{theorem}
\begin{proof}
	(See Appendix \ref{proof:l1vsl2}).\qedhere
\end{proof}

\textbf{(2)~For} $(\ell_2, \mathbf{Fermat})$. Since $v_i$ satisfies CWDA, from Theorem \ref{theo:fermat}, we know that the Fermat point of $v_i$ and the origin $\mathbf{0}$ approximately coincide. According to Table~\ref{tab:criteria}, $||\mathbf{Fermat}-v_i||_2 \approx ||\mathbf{0} - v_i||_2 = ||v_i||_2$. Therefore, from Theorem~\ref{theo:l1vsl2}, the bound $B(n)$ for the ($\ell_1$, $\mathbf{Fermat}$) and ($\ell_2$, $\mathbf{Fermat}$) are $\frac{1}{n}$ and 0, respectively.
Moreover, since CWDA, the centroid of $v_i$ is $\mathbf{G} = \frac{1}{n}\sum_{i=1}^N v_i = \mathbf{0}$. Hence,
\begin{equation}
\mathbf{G} = \mathbf{0} \approx \mathbf{Fermat}.
\label{chonghe}
\end{equation}
\begin{theorem} Let random variable $v_i \in \mathbb{R}^k$ and they are i.i.d and follow normal distribution $N(\mathbf{0},\sigma^2\mathbf{I}_k)$. For $F \in \mathbb{R}^k$, we have $\mathbf{argmin}_F\left\{\mathbb{E}_{v_i \sim N(\mathbf{0},\sigma^2\mathbf{I}_k)} \sum_{i=1}^n ||F-v_i||_2\right\} = \mathbf{0}.$
	\label{theo:fermat}
\end{theorem}
\begin{proof}
	(See Appendix \ref{proof:fermat}).\qedhere
\end{proof}

\textbf{(3)~For} $(\mathbf{GM}, \mathbf{Fermat})$. First, we show the following two theorems:
\begin{theorem} For $n$ random variables $a_i \in \mathbb{R}^k$ follow $N(\mathbf{0},c^2\cdot \mathbf{I}_k)$.When $k$ is large enough, we have such an estimation: 
	\begin{equation}
	\mathbf{Var}_{a_i}\frac{F_1(a_i)}{F_2(a_i)}\approx \frac{1}{2nk}, \quad \mathbf{Var}_{a_i}\frac{F_2(a_i)}{F_1(a_i)}\approx \frac{1}{2nk},
	\end{equation}
	
	where $F_1(a_i) = \sum_{i=1}^n||a_i||_2/ \mathbb{E}(\sum_{i=1}^n||a_i||_2)$ and $F_2(a_i) = \sum_{i=1}^n||a_i||_2^2/ \mathbb{E}(\sum_{i=1}^n||a_i||_2^2)$.
	\label{theo:var}
\end{theorem}
\begin{proof}
	(See Appendix \ref{proof:var}).\qedhere
\end{proof}

\begin{theorem} Let $v_0,v_1,...,v_k$ be the $k+1$ vectors in $n$ dimensional Euclidean space $\mathbb{E}^n$. For all $P$ in $\mathbb{E}^n$,
	\begin{equation}
	\sum_{i=0}^k||P-v_i||_2^2 = \sum_{i=0}^k||G-v_i||_2^2 + (k+1)||P-G||_2^2,
	\label{ggdl}
	\end{equation}
	\label{theo:ggdl}
	where $G$ is the centroid of $v_i$, will hold if it satisfies one of the following conditions:

	(1)if $k\geq n$ and $\mathbf{rank}(v_1-v_0,v_2-v_0,...,v_k-v_0)=n$.

	(2)if $k < n$ and $(v_1-v_0,v_2-v_0,...,v_k-v_0)$ are linearly independent.

	(3)if $v_i \sim N(\mathbf{0},c^2\cdot \mathbf{I}_n)$, Eq.(\ref{ggdl}) holds with probability 1.
	
\end{theorem}
\begin{proof}
	(See Appendix \ref{proof:ggdl}).\qedhere
\end{proof}
Let $P \in \{v_1,v_2,...,v_N\}$. Since $v_i \sim N(\mathbf{0},c^2\cdot \mathbf{I})$, we can obtain that $a_i = P - v_i \sim N(\mathbf{0},2c^2\cdot \mathbf{I})$ if $P \neq v_i$. According to the analysis in Section \ref{Experiment and theory} and Theorem \ref{theo:var}, we have 
\begin{equation}
\sum_{i=1}^n||a_i||_2 \cong \sum_{i=1}^n||a_i||_2^2,
\label{temp12345}
\end{equation}
Next, we can prove $(k+1)\color{red}{||P-F||_2^2}$~($\mathbf{Fermat}$) and $\color{blue}{\sum_{i=1}^N||P-v_i||_2}$~($\mathbf{GM}$)  are \textit{approximately monotonic}, where $P \in \{v_1,v_2,...,v_N\}$.
\begin{align}
(k+1)\color{red}{||P-F||_2^2}
&\cong (k+1)||P-G||_2^2 \tag*{Since  Eq. (\ref{chonghe})}\\
&= \sum_{i=1}^N||P-v_i||_2^2 - \sum_{i=1}^N||G-v_i||_2^2\tag*{Since Theorem \ref{theo:ggdl}}\\
&\cong\sum_{i=1}^N||P-v_i||_2 - \sum_{i=1}^N||G-v_i||_2^2 \tag*{Since Eq. (\ref{temp12345})}\\
&\cong\color{blue}{\sum_{i=1}^N||P-v_i||_2}
\label{temp5555}
\end{align}

The reason for the last equation is that $\sum_{i=1}^N||G-v_i||_2^2$ is a constant for given $v_i$.

\section{Proof of Theorem \ref{theo:l1vsl2}}
\label{proof:l1vsl2}

\textbf{Theorem \ref{theo:l1vsl2}} Let $X\sim N(\mathbf{0},c^2\cdot\mathbf{I}_n)$, we have 
$$
\mathbf{max}\left\{\mathbf{Var}_{X}\left(\frac{\widehat{\ell}_2(X)}{\widehat{\ell}_1(X)}\right),\mathbf{Var}_{X}\left(\frac{\widehat{\ell}_1(X)}{\widehat{\ell}_2(X)}\right) \right\}\lesssim \frac{1}{n}.$$
where $\widehat{\ell}_1(X)$ denotes $\ell_1(X)/\mathbb{E}(\ell_1(X))$ and $\widehat{\ell}_2(X)$ denotes $\ell_2(X)/\mathbb{E}(\ell_2(X))$.

\begin{proof}
	For the ratio $\widehat{\ell}_2(X)/\widehat{\ell}_1(X)$, we have
	\begin{align*}
	\mathbf{Var}\left(\frac{\widehat{\ell}_2(X)}{\widehat{\ell}_1(X)}\right) &= \left(\frac{\mathbb{E}(\ell_1(X))}{\mathbb{E}(\ell_2(X))}\right)^2  \mathbf{Var}\left(\frac{\ell_2(X)}{\ell_1(X)}\right)\\
	& \approx \left(\frac{\mathbb{E}(\ell_1(X))}{\mathbb{E}(\ell_2(X))}\right)^2 \left(\frac{\mathbb{E}(\ell_2(X))}{\mathbb{E}(\ell_1(X))}\right)^2 \left( \frac{\mathbf{Var}\ell_2(X)}{\mathbb{E}(\ell_2(X))^2}  + \frac{\mathbf{Var}\ell_1(X)}{\mathbb{E}(\ell_1(X))^2} -2\frac{\mathbf{Cov}(\ell_2(X),\ell_1(X))}{\mathbb{E}(\ell_2(X))\mathbb{E}(\ell_1(X))} \right)\tag*{from Lemma. \ref{lemma:varx_y}} \\
	& \leq\left( \frac{\mathbf{Var}\ell_2(X)}{\mathbb{E}(\ell_2(X))^2}  + \frac{\mathbf{Var}\ell_1(X)}{\mathbb{E}(\ell_1(X))^2} \right).\tag*{from Proposition. \ref{pro:fkg}} \\
	\end{align*}
	
	similarly, we also have
	\begin{equation}
	\mathbf{Var}\left(\frac{\widehat{\ell}_1(X)}{\widehat{\ell}_2(X)}\right) \leq\left( \frac{\mathbf{Var}\ell_2(X)}{\mathbb{E}(\ell_2(X))^2}  + \frac{\mathbf{Var}\ell_1(X)}{\mathbb{E}(\ell_1(X))^2} \right).             
	\end{equation}  
	Therefore,      
	
	\begin{align*}
	\mathbf{max}\left\{\mathbf{Var}_{X}\left(\frac{\widehat{\ell}_2(X)}{\widehat{\ell}_1(X)}\right),\mathbf{Var}_{X}\left(\frac{\widehat{\ell}_1(X)}{\widehat{\ell}_2(X)}\right) \right\} &\leq\left( \frac{\mathbf{Var}\ell_2(X)}{\mathbb{E}(\ell_2(X))^2}  + \frac{\mathbf{Var}\ell_1(X)}{\mathbb{E}(\ell_1(X))^2} \right)\\
	&=  \frac{2\sigma^2 \left[ \frac{\Gamma(\frac{n}{2}+1)}{\Gamma(\frac{n}{2})} - \frac{\Gamma(\frac{n+1}{2})^2}{\Gamma(\frac{n}{2})^2} \right]}{(\sqrt{2}\sigma\cdot\frac{\Gamma(\frac{n+1}{2})}{\Gamma(\frac{n}{2})})^2}  + \frac{\sigma^2\left(1 - \frac{2}{\pi}\right)n}{(n\cdot\sigma\sqrt{2/\pi})^2}  \tag*{from Proposition. \ref{pro:chi} and \ref{pro:half normal}}\\
	&\approx \left(\frac{1}{2n} + (\frac{\pi}{2}-1)\frac{1}{n}\right)  \tag*{from Lemma \ref{lemma:gamma1}}\\
	&= \frac{\pi-1}{2n}
	\end{align*}

	\qedhere
\end{proof}

Because the approximation is widely used in the proof of Theorem \ref{lemma:varx_y}, it is necessary to verify it numerically. As shown in Fig.~\ref{fig:bound}, we use ResNet56 on Cifar100 and ResNet110 on Cifar10 respectively to verify Theorem \ref{lemma:varx_y}. From Fig.~\ref{fig:bound}, we find that the estimationn of Theorem \ref{lemma:varx_y} is reliable, \textit{i.e.}, the estimation $O(\frac{1}{n})$ for $\mathbf{max}\left\{\mathbf{Var}_{X}\left(\frac{\widehat{\ell}_2(X)}{\widehat{\ell}_1(X)}\right),\mathbf{Var}_{X}\left(\frac{\widehat{\ell}_1(X)}{\widehat{\ell}_2(X)}\right) \right\}$ is appropriate.


\begin{figure} [t]
	\centering 
	\includegraphics[height=2in, width=2.5in]{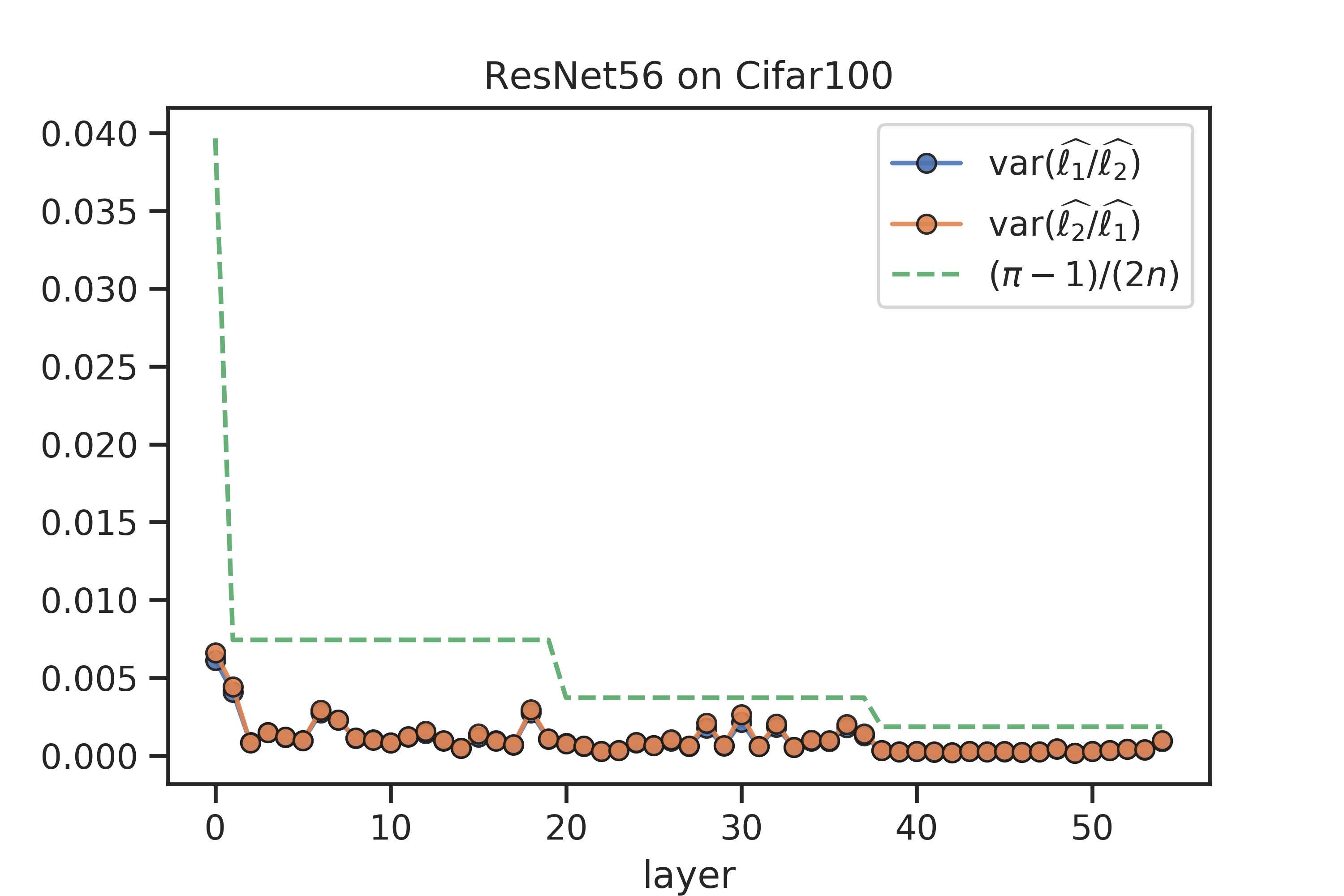} 
	\includegraphics[height=2in, width=2.5in]{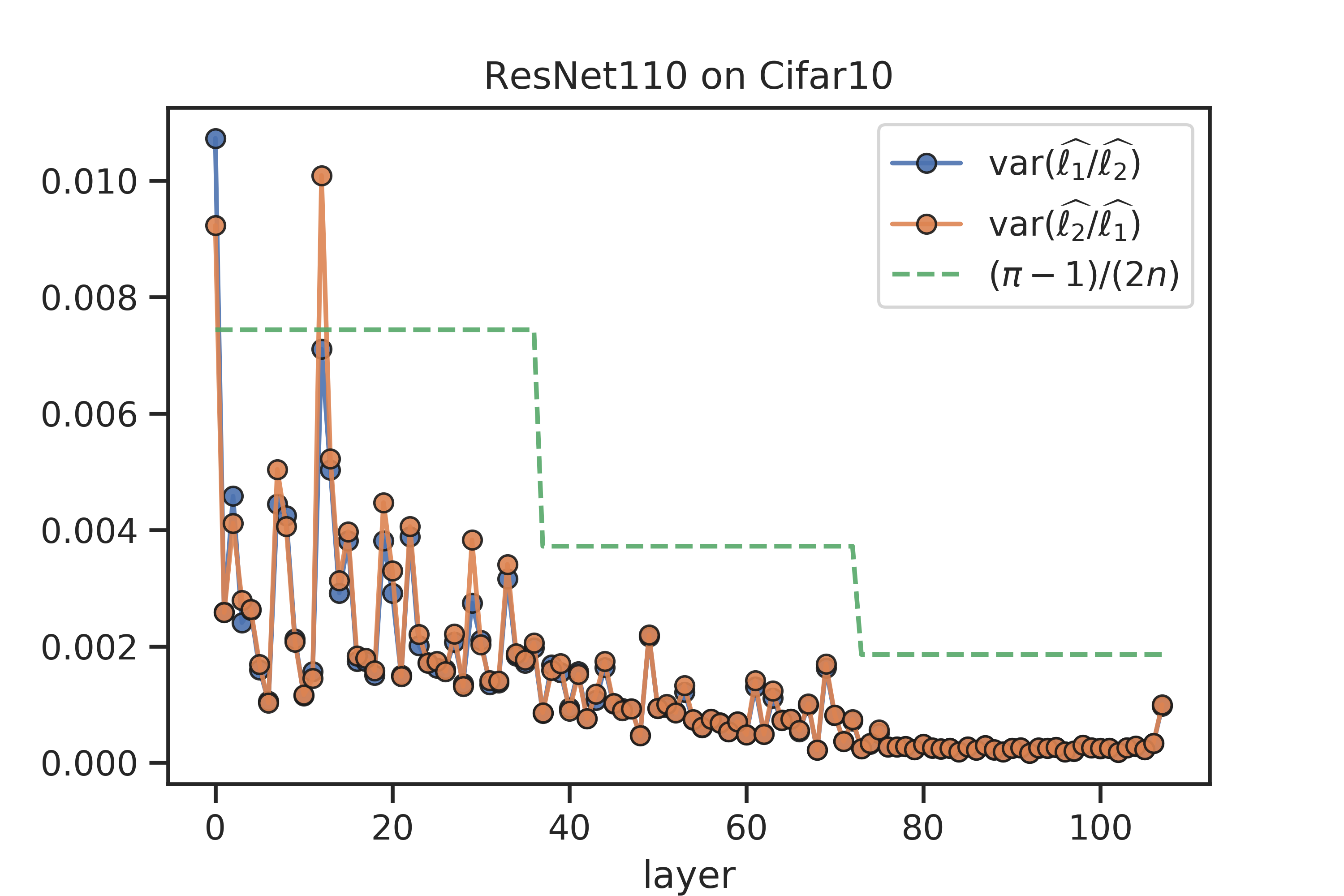} 
	\caption{The approximation of \textbf{Theorem \ref{theo:l1vsl2}}: (Left)~the example about ResNet56; (Right)~the example about ResNet110. } 
	\label{fig:bound}
\end{figure}

\section{Proof of Theorem \ref{theo:fermat}}
\label{proof:fermat}

\begin{proposition} Let $L_{p}^{(\alpha)}(x)$ denotes generalized Laguerre function, and it have following properties:
	\begin{equation}
	\frac{\partial^n}{\partial x^n}L_{p}^{(\alpha)} = (-1)^n L_{p-n}^{(\alpha+n)}(x),
	\label{pro:diff_L}
	\end{equation}
	and for $\alpha > 0$,
	\begin{equation}
	L_{-\frac{1}{2}}^{(\alpha)}(x) >0.
	\label{pro:positive}
	\end{equation}
	
\end{proposition}

\textbf{Theorem \ref{theo:fermat}.} Let random variable $v_i \in \mathbb{R}^k$. They are i.i.d and follow normal distribution $N(\mathbf{0},\sigma^2\mathbf{I}_k)$. For $F$ in $\mathbb{R}^k$, we have
$$
\mathbf{argmin}_F\left\{\mathbb{E}_{v_i \sim N(\mathbf{0},\sigma^2\mathbf{I}_k)} \sum_{i=1}^n ||F-v_i||_2\right\} = \mathbf{0}.
$$
\begin{proof}
	Let $w_i = F - v_i$ and we have $w_i \sim N(F,\sigma^2 \mathbf{I}_k)$, then
	\begin{align*}
	\mathbb{E}_{v_i \sim N(\mathbf{0},\sigma^2\mathbf{I}_k)} \sum_{i=1}^n ||F-v_i||_2  & = \sum_{i=1}^n \mathbb{E}_{v_i \sim N(\mathbf{0},\sigma^2\mathbf{I}_k)}||F-v_i||_2 \\
	& = \sum_{i=1}^n \mathbb{E}_{w_i \sim N(F,\sigma^2\mathbf{I}_k)}||w_i||_2\\
	& = n\cdot \sigma^2\sqrt{\frac{\pi}{2}}\cdot L_{\frac{1}{2}}^{(\frac{k}{2}-1)} \left(-\frac{||F||_2^2}{2\sigma^2}\right)
	\end{align*}
	The reason for the last equation is that $||w_i||_2$ follows scaled noncentral chi distribution\footnote{\href{https://pdfs.semanticscholar.org/cf53/f8c9dfa71bf17649feb86af5d7d8d294b06a.pdf}{Survey of simple,continuous,uniariate probability distribution} and \href{https://en.wikipedia.org/wiki/Noncentral_chi_distribution}{Wikipredia}.} when $w_i \sim N(F,\sigma^2\mathbf{I}_k)$. Let $T(x) = L_{\frac{1}{2}}^{(\frac{k}{2}-1)} \left(-\frac{x^2}{2\sigma^2}\right)$, we calculate the minimum of $T(x)$. From Eq.~(\ref{pro:diff_L}),
	\begin{equation}
	\frac{d}{dx}T(x) = \frac{x}{\sigma^2} \cdot L_{-\frac{1}{2}}^{(\frac{k}{2})}\left(-\frac{x^2}{2\sigma^2}\right).
	\label{first_diff}
	\end{equation}
	Since Eq.~(\ref{pro:positive}), we find that $\frac{d}{dx}T(x) > 0$ when $x > 0 $ and if $x \leq 0$, then $\frac{d}{dx}T(x) \leq 0 $. It means that $T(x)$ gets the minimizer at $||F||_2 = 0$, \textit{i.e.}, $F = \mathbf{0}$.

	\qedhere
\end{proof}

\section{Proof of Theorem \ref{theo:var}}
\label{proof:var}

\begin{lemma} For two random variables $X,Y \in \mathbb{R}^k$ follow $N(\mathbf{0},c^2\cdot \mathbf{I}_k)$ and they are i.i.d. When $k$ is large enough, we have:
	\begin{equation}
	\mathbb{E}\left(\frac{(||X||_2^2 - ||Y||_2^2)^2}{2||X||_2\cdot||Y||_2}\right)\approx 2c^2+\frac{4c^2k+1}{2k^2},
	\label{eratio}
	\end{equation}
	and
	\begin{equation}
	\mathbf{Var}\left(\frac{(||X||_2^2 - ||Y||_2^2)^2}{2||X||_2\cdot||Y||_2}\right)\lesssim 8c^4 + \frac{16c^4k+c^2}{k^2},
	\label{varratio}
	\end{equation}
	\label{lemma3}
\end{lemma}
\begin{proof}
	According to \textbf{Proposition 3} and \textbf{Lemma 2}, it is easy to know, when $k$ is large enough, that
	\begin{equation}
	\mathbb{E}\left(2||X||_2\cdot||Y||_2\right)=2c^2k,\quad  \mathbf{Var}\left(2||X||_2\cdot||Y||_2\right)=c^2+4c^4k,
	\label{evar1}
	\end{equation}  
	and
	\begin{equation}
	\mathbb{E}\left((||X||_2^2 - ||Y||_2^2)^2\right)=4c^4k,\quad  \mathbf{Var}\left((||X||_2^2 - ||Y||_2^2)^2\right)=16c^8(2k^2+3k).
	\label{evar2}
	\end{equation}      
	Since Lemma \ref{lemma:varx_y}, we have an estimation 
	\begin{align*}
	\mathbf{Var}\left(\frac{(||X||_2^2 - ||Y||_2^2)^2}{2||X||_2\cdot||Y||_2}\right) & \leq \left(\frac{\mathbb{E}(||X||_2^2 - ||Y||_2^2)^2}{\mathbb{E}2||X||_2\cdot||Y||_2}\right)^2 \left( \frac{\mathbf{Var}(||X||_2^2 - ||Y||_2^2)^2}{\mathbb{E}(||X||_2^2 - ||Y||_2^2)^2}  + \frac{\mathbf{Var}(2||X||_2\cdot||Y||_2)^2)}{\mathbb{E}(2||X||_2\cdot||Y||_2)^2} \right)    \\
	& \approx  \left(\frac{4c^4k}{2c^2k}\right)^2\cdot \left(\frac{c^2+4c^4k}{4c^4k}+\frac{16c^8(2k^2+3k)}{16c^8k^2}\right)    \tag*{Since  Eq.(\ref{evar1}) and Eq.(\ref{evar2})}\\
	& = 8c^4 + \frac{16c^4k+c^2}{k^2}.
	\end{align*}

	Therefore,
	\begin{align*}
	\mathbb{E}\left(\frac{(||X||_2^2 - ||Y||_2^2)^2}{2||X||_2\cdot||Y||_2}\right) & \approx \frac{\mathbb{E}(||X||_2^2 - ||Y||_2^2)^2}{\mathbb{E}2||X||_2\cdot||Y||_2} + \mathbf{Var}(2||X||_2\cdot||Y||_2)\cdot \frac{\mathbb{E}(||X||_2^2 - ||Y||_2^2)^2}{(\mathbb{E}2||X||_2\cdot||Y||_2)^3}  \tag*{Since  Eq.(\ref{estimation:E})}\\
	& \approx  \frac{4c^4k}{2c^2k}+\frac{4c^4k}{8c^6k^3}\cdot(c^2+4c^4k)    \tag*{Since  Eq.(\ref{evar1}) and Eq.(\ref{evar2})}\\
	& = 2c^2 + \frac{4c^2k+1}{2k^2}.
	\end{align*}    
\end{proof}
Note that, the approximation is widely used in the proof of Eq.(\ref{eratio}) and Eq.(\ref{varratio}). Hence, it is also necessary to verify it numerically. As shown in Fig.~\ref{fig:esti}, the estimation is appropriate. According to \textbf{Lemma}~\ref{lemma3}, the mathematical expectation and variance of the ratio of $(||X||_2^2 - ||Y||_2^2)^2$ and $2||X||_2\cdot||Y||_2$ are both close to 0 when $k$ is large enough and $c$ is small enough. that is,
\begin{equation}
2(||X||_2\cdot||Y||_2) \gg (||X||_2^2 - ||Y||_2^2)^2.
\label{>>}
\end{equation}
By the way, the convolutional filters easily meet the condition that $k$ is large enough.

%

\begin{figure} [t]
	\centering 
	\includegraphics[height=2in, width=2.5in]{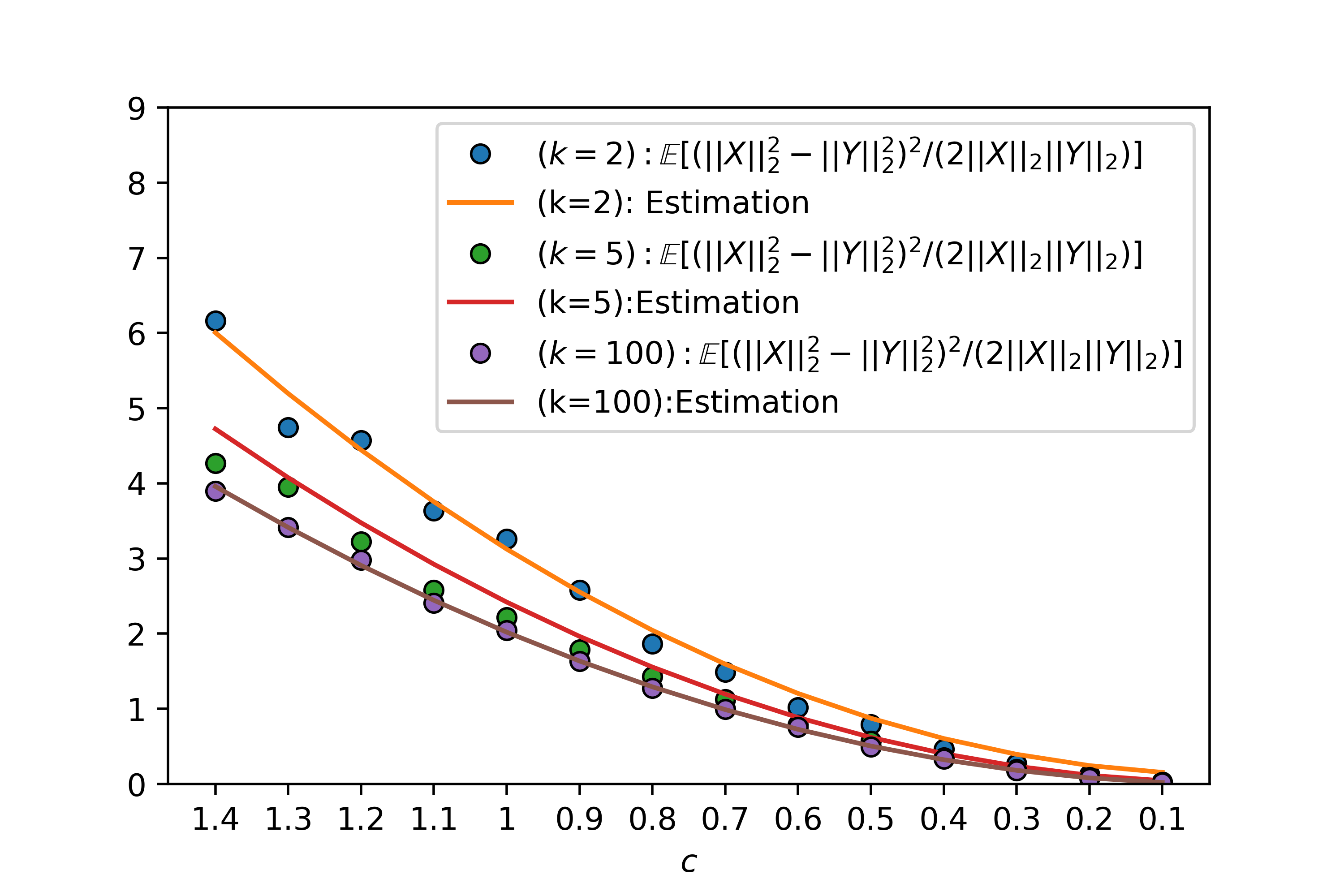} 
	\includegraphics[height=2in, width=2.5in]{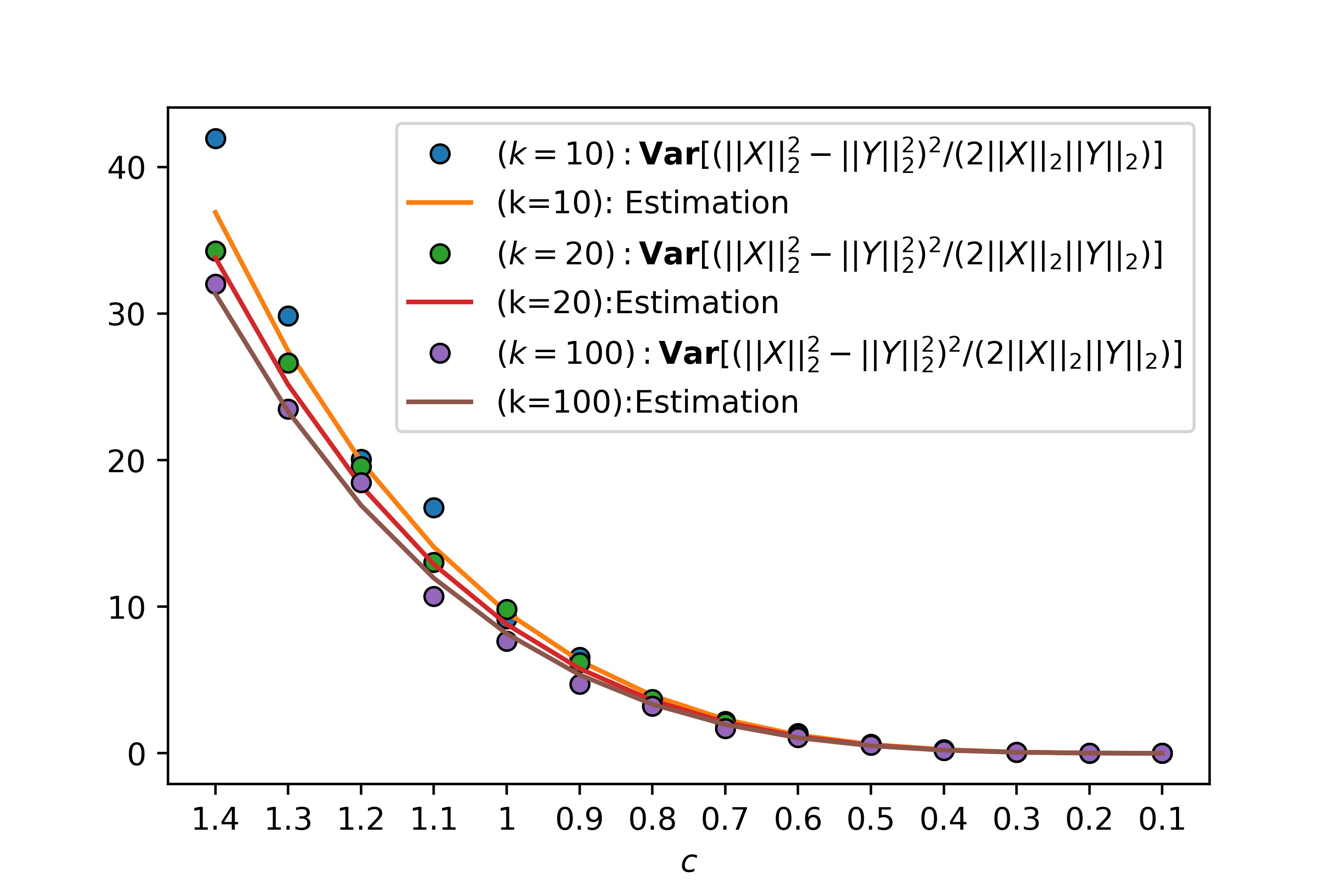} 
	\caption{(Left)~The numerical verification of Eq.(\ref{eratio}) and (Right)~The numerical verification of Eq.(\ref{varratio}). $X$ and $Y$ follow $N(\mathbf{0},c^2\cdot I_k)$.} 
	\label{fig:esti}
\end{figure}

\textbf{Theorem \ref{theo:var}.} For $n$ random variables $a_i \in \mathbb{R}^k$ follow $N(\mathbf{0},c^2\cdot \mathbf{I}_k)$.When $k$ is large enough, we have such an estimation:
$$
\mathbf{Var}_{a_i}\frac{F_1(a_i)}{F_2(a_i)}\approx \frac{1}{2nk},\quad \mathbf{Var}_{a_i}\frac{F_2(a_i)}{F_1(a_i)}\approx \frac{1}{2nk}.
$$  
where $F_1(a_i) = \sum_{i=1}^n||a_i||_2/ \mathbb{E}(\sum_{i=1}^n||a_i||_2)$ and $F_2(a_i) = \sum_{i=1}^n||a_i||_2^2/ \mathbb{E}(\sum_{i=1}^n||a_i||_2^2)$.
\begin{proof}
	Since Eq.~(\ref{equ:mean_chi}) and Eq.~(\ref{equ:var_chi}), we have
	\begin{equation}
	\mathbf{Var}_{a_i}\frac{F_1(a_i)}{F_2(a_i)} = \left(\frac{nc^2k}{nc\sqrt{k}}\right)^2\cdot\mathbf{Var}_{a_i}\left(\frac{\sum_{i=1}^n||a_i||_2}{\sum_{i=1}^n||a_i||_2^2}\right).
	\label{temptemptemp1}
	\end{equation}
	and 
	\begin{equation}
	\mathbf{Var}_{a_i}\frac{F_2(a_i)}{F_1(a_i)} = \left(\frac{nc\sqrt{k}}{nc^2k}\right)^2\cdot\mathbf{Var}_{a_i}\left(\frac{\sum_{i=1}^n||a_i||_2^2}{\sum_{i=1}^n||a_i||_2}\right).
	\label{temptemptemp2}
	\end{equation}

	According to Lagrange's identity, we have
	\begin{align*}
	\left(\sum_{i=1}^n||a_i||_2^2\right)\left(\sum_{i=1}^n1\right) &= \left(\sum_{i=1}^n||a_i||_2\right)^2 + \sum_{1\leq i<j\leq n}(||a_i||_2^2 - ||a_j||_2^2)^2 \\
	& = \sum_{i=1}^n||a_i||_2^2 + \sum_{1\leq i<j\leq n}(||a_i||_2\cdot||a_j||_2) + 2\sum_{1\leq i<j\leq n}(||a_i||_2^2 - ||a_j||_2^2)^2\\
	& \approx  \sum_{i=1}^n||a_i||_2^2 + 2\sum_{1\leq i<j\leq n}(||a_i||_2\cdot||a_j||_2)\tag*{Since  Eq.~(\ref{>>})}\\
	& = \left(\sum_{i=1}^n||a_i||_2\right)^2
	\end{align*} 
	so we have
	\begin{equation}
	\mathbf{Var}_{a_i \sim N(\mathbf{0},c^2\cdot\mathbf{I}_k)}\frac{\sum_{i=1}^n||a_i||_2}{\sum_{i=1}^n||a_i||_2^2} \approx \mathbf{Var}_{a_i \sim N(\mathbf{0},c^2\cdot\mathbf{I}_k)}\frac{n}{\sum_{i=1}^n||a_i||_2}
	\label{temp2}
	\end{equation}
	By central limit theorem, we have $\sqrt{n}(\frac{1}{n}\sum_{i=1}^n||a_i||_2-\mu)\sim N(\mathbf{0},\sigma^2)$. And let $g(x) = \frac{1}{x}$, we can use Delta method\footnote{\url{https://en.wikipedia.org/wiki/Delta_method}} to find the distribution of $g(\frac{1}{n}\sum_{i=1}^n||a_i||_2)$:
	\begin{equation}
	\sqrt{n}\left(g(\frac{\sum_{i=1}^n||a_i||_2}{n})-g(\mu)) \right) \sim N(0,\sigma^2\cdot [g\prime(\mu)]^2) = N(0,\sigma^2\cdot \frac{1}{\mu^4}).
	\label{temp3}
	\end{equation} 
	where $\mu$ and $\sigma^2$ denote the mean and variance of $||a_i||_2$ respectively. From Eq.~(\ref{temp2}), we have
	\begin{align*}
	\mathbf{Var}_{a_i \sim N(\mathbf{0},c^2\cdot\mathbf{I}_k)}\frac{\sum_{i=1}^n||a_i||_2}{\sum_{i=1}^n||a_i||_2^2} &\approx \mathbf{Var}_{a_i \sim N(\mathbf{0},c^2\cdot\mathbf{I}_k)}\frac{n}{\sum_{i=1}^n||a_i||_2}\\
	&= \sigma^2\cdot \frac{1}{\mu^4\cdot n}\tag*{Since  Eq.~(\ref{temp3})}\\
	&= 2c^2 \left[ \frac{\Gamma(\frac{k}{2}+1)}{\Gamma(\frac{k}{2})} - \frac{\Gamma(\frac{k+1}{2})^2}{\Gamma(\frac{k}{2})^2} \right]\cdot \frac{1}{(\sqrt{2}c\cdot \frac{\Gamma(\frac{k+1}{2})}{\Gamma(\frac{k}{2})})^4\cdot n}\tag*{Since  Eq.~(\ref{equ:mean_chi}) and Eq.~(\ref{equ:var_chi})}\\
	&= \frac{1}{2c^2\cdot nk^2}\tag*{Since  Lemma. \ref{lemma:gamma1}}\\
	\end{align*} 
	
	Since Eq.~(\ref{temptemptemp1}), we have
	\begin{equation}
	\mathbf{Var}_{a_i}\frac{F_1(a_i)}{F_2(a_i)} = \left(\frac{nc^2k}{nc\sqrt{k}}\right)^2\cdot\mathbf{Var}_{a_i}\left(\frac{\sum_{i=1}^n||a_i||_2}{\sum_{i=1}^n||a_i||_2^2}\right) \approx \frac{1}{2nk}.
	\label{temptemptemp3}
	\end{equation}

	Similar to Eq.~(\ref{temp2}), 
	\begin{equation}
	\mathbf{Var}_{a_i \sim N(\mathbf{0},c^2\cdot\mathbf{I}_k)}\frac{\sum_{i=1}^n||a_i||_2^2}{\sum_{i=1}^n||a_i||_2} \approx \mathbf{Var}_{a_i \sim N(\mathbf{0},c^2\cdot\mathbf{I}_k)}\frac{\sum_{i=1}^n||a_i||_2}{n}
	\label{temp4}
	\end{equation}
	
	\begin{align*}
	\mathbf{Var}_{a_i \sim N(\mathbf{0},c^2\cdot\mathbf{I}_k)}\frac{\sum_{i=1}^n||a_i||_2^2}{\sum_{i=1}^n||a_i||_2} &\approx \mathbf{Var}_{a_i \sim N(\mathbf{0},c^2\cdot\mathbf{I}_k)}\frac{\sum_{i=1}^n||a_i||_2}{n}\tag*{Similar to Eq.~(\ref{temp2})}\\
	&= \sigma^2\cdot \frac{1}{n}\tag*{Since central limit theorem}\\
	&= 2c^2 \left[ \frac{\Gamma(\frac{k}{2}+1)}{\Gamma(\frac{k}{2})} - \frac{\Gamma(\frac{k+1}{2})^2}{\Gamma(\frac{k}{2})^2} \right]\cdot \frac{1}{n}\tag*{Since Eq.~(\ref{equ:var_chi})}\\
	&= \frac{c^2}{2n}\tag*{Since  Lemma. \ref{lemma:gamma1}}\\
	\end{align*} 
	Since Eq.~(\ref{temptemptemp2}), we have
	\begin{equation}
	\mathbf{Var}_{a_i}\frac{F_2(a_i)}{F_1(a_i)} = \left(\frac{nc\sqrt{k}}{nc^2k}\right)^2\cdot\mathbf{Var}_{a_i}\left(\frac{\sum_{i=1}^n||a_i||_2^2}{\sum_{i=1}^n||a_i||_2}\right) \approx \frac{1}{2nk}.
	\label{temptemptemp4}
	\end{equation}
	From Eq.(\ref{temptemptemp3}) and Eq.(\ref{temptemptemp4}), \textbf{Theorem \ref{theo:var}} holds.
	
	\qedhere
\end{proof}
\begin{figure} [t]
	\centering 
	\includegraphics[height=2in, width=2.5in]{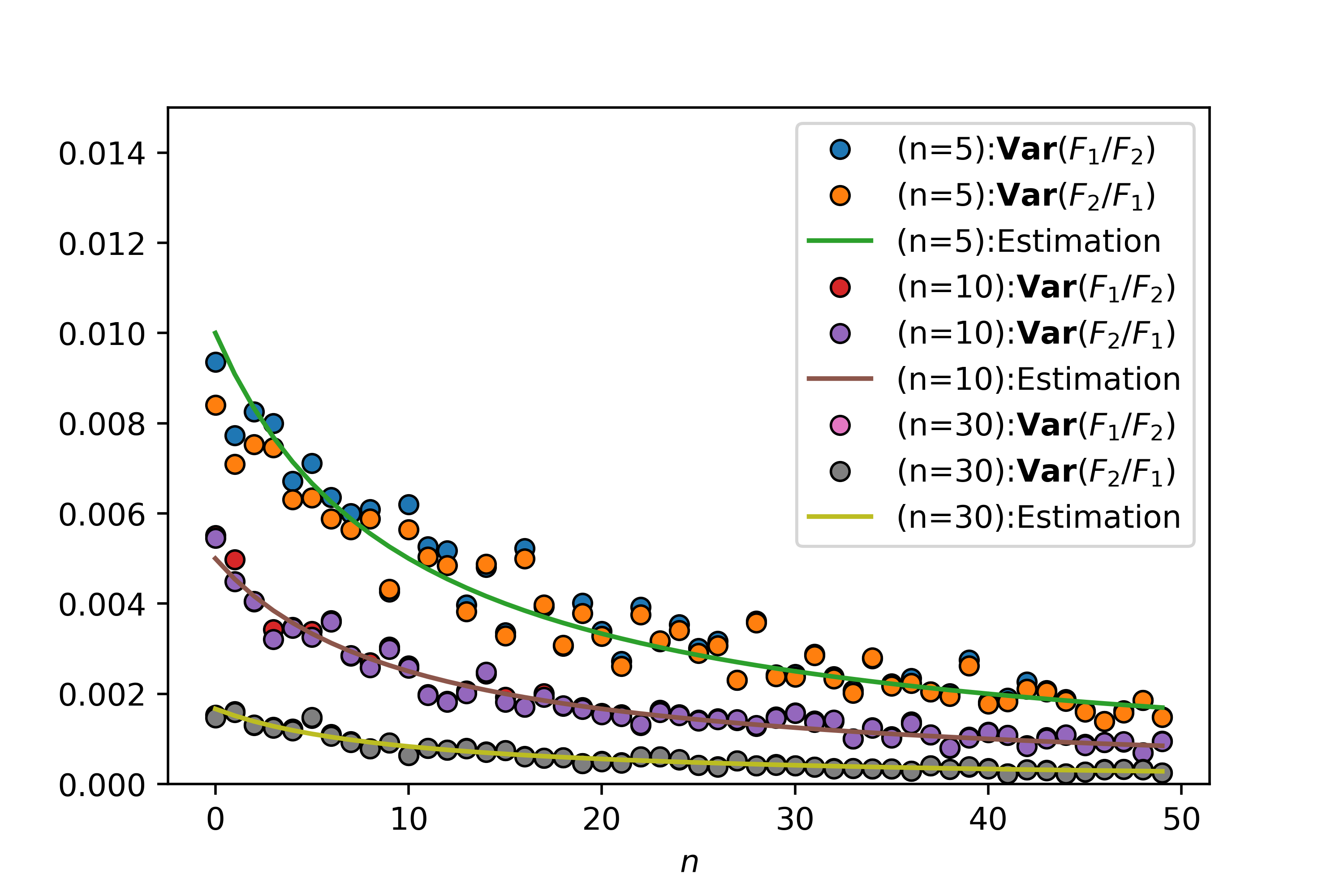} 
	\includegraphics[height=2in, width=2.5in]{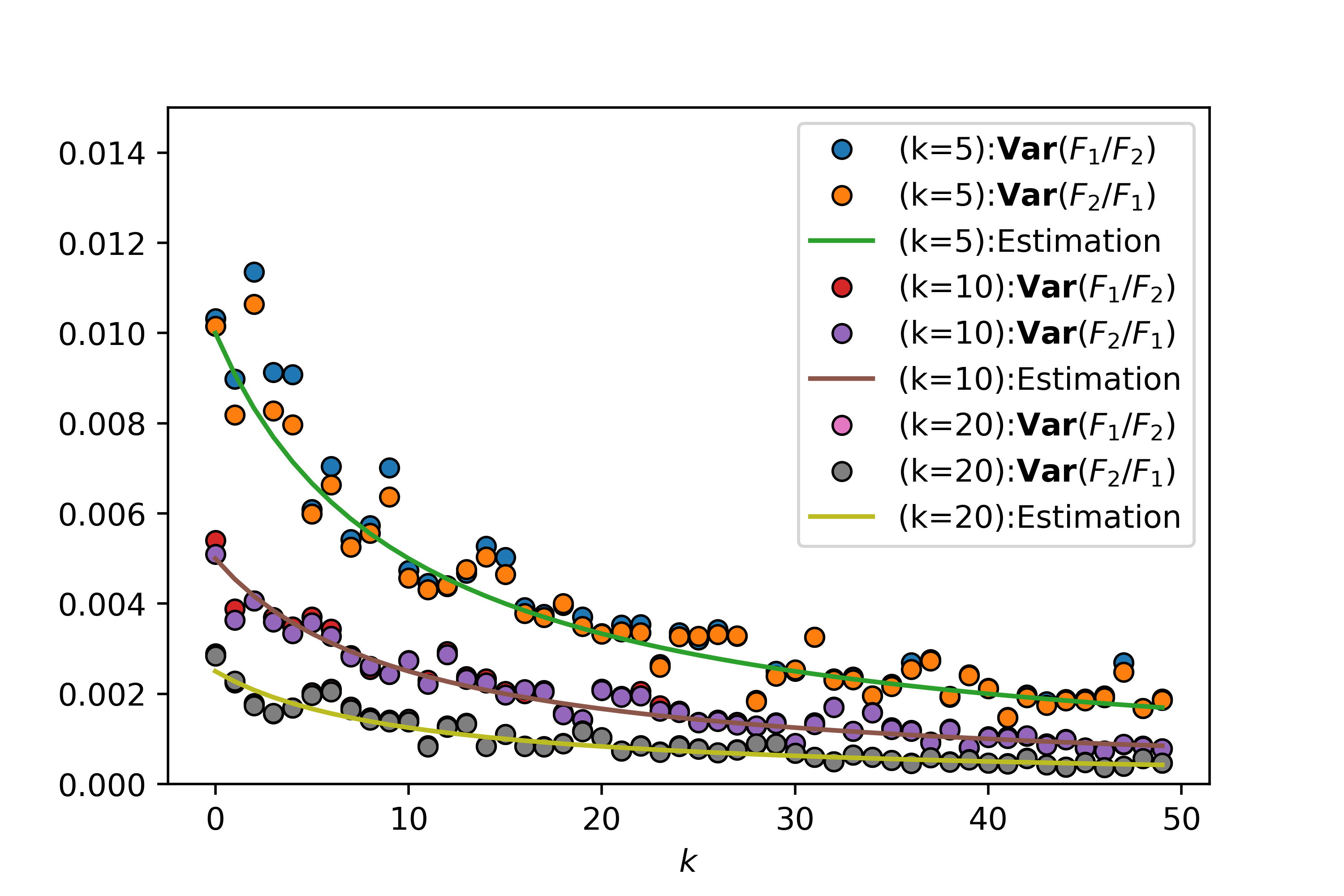} 
	\caption{A numerical verification of \textbf{Theorem \ref{theo:var}}, where $F_1 = \sum_{i=1}^n||a_i||_2/ \mathbb{E}(\sum_{i=1}^n||a_i||_2)$ and $F_2 = \sum_{i=1}^n||a_i||_2^2/ \mathbb{E}(\sum_{i=1}^n||a_i||_2^2)$.  $a_i$ follow $N(\mathbf{0},0.01^2\cdot I_k)$.} 
	\label{fig:theo3ver}
\end{figure}

In Fig.~\ref{fig:theo3ver}, we also show a numerical verification of \textbf{Theorem \ref{theo:var}}.

\section{Proof of Theorem \ref{theo:ggdl}}
\label{proof:ggdl}

\begin{proposition} For a $n \times m$ random matrix $(a_{ij})_{n \times m}$, where $a_{ij} \sim N(0,\sigma^2)$. And Eq.~(\ref{pro:matrix}) holds with probability 1.
	\begin{equation}
	\mathbf{rank}((a_{ij})_{n \times m}) = \mathbf{min}(m,n).
	\end{equation}
	\label{pro:matrix}
\end{proposition}

\begin{lemma} Let $v_0,v_1,...,v_k$ be the $k+1$ vectors in $n$ dimensional Euclidean space $V$ and $k \leq n$.
	If $\mathbf{rank}(v_1-v_0,v_2-v_0,...,v_k-v_0)$ = $n$, then $\forall x \in V$, $\exists \lambda_i (0\leq i \leq k)$, s.t. 
	\begin{equation}
	x = \sum_{i=0}^k \lambda_i \cdot v_i,
	\end{equation}
	and $\sum_{i=0}^k\lambda_i = 1$. We call $\mathbf{\lambda} = (\lambda_0,\lambda_1,...,\lambda_k)$ the generalized barycentric coordinate with respect to $(v_0,v_1,...,v_k)$. (In general, barycentric coordinate is a concept in Polytope)
	\label{lemma:linear_exp}
\end{lemma}{}
\begin{proof}
	Note that $v_i$ is the element of $n$ dimensional linear space $V$ and $\mathbf{rank}(v_1-v_0,v_2-v_0,...,v_k-v_0) = n$. It means $(v_1-v_0,v_2-v_0,...,v_k-v_0)$ form a set of basis in the linear space $V$. $\forall x \in V$,  $x - v_0$ can be expressed linearly by them, \textit{i.e.},$\exists t_i (1 \leq i \leq k)$ s.t.
	\begin{align*}
	x & =v_0 + \sum_{i=1}^k t_i(v_i - v_0)\\
	& = (1 - \sum_{i=1}^k t_i)v_0 + \sum_{i=1}^k t_iv_i.
	\end{align*}{}
	Let $\lambda_0 = (1 - \sum_{i=1}^k t_i)$ and $\lambda_i = t_i (1 \leq i \leq k)$,  Lemma ~\ref{lemma:linear_exp}  holds.
	\vspace{0.5cm}
\end{proof}

\begin{lemma}  Let $v_0,v_1,...,v_k$ be the $k+1$ vectors in $n$ dimensional Euclidean space $V$. $\forall a,b \in V$, and the generalized barycentric coordinate of $a,b$ with respect to $(v_0,v_1,...,v_k)$ are $\mathbf{\lambda} = (\lambda_0,\lambda_1,...,\lambda_k)^T$ and $\mathbf{\mu} = (\mu_0,\mu_1,...,\mu_k)^T$,respectively. Then
	\begin{equation}
	||a-b||_2^2 = (\mathbf{\lambda} - \mathbf{\mu})^TD(\mathbf{\lambda} - \mathbf{\mu}),
	\end{equation}
	where $D = (-\frac{1}{2}d_{ij})_{(k+1)\times(k+1)}$, and $d_{ij} = ||v_i-v_j||_2^2$.
	\label{lemma:jiong}
\end{lemma}{}

\begin{proof}
	Since Lemma \ref{lemma:linear_exp}, let $R = [v_0,v_1,...,v_k]_{n \times (k+1)}$, and we have $a = R\lambda$ and $b = R\mu$. Moreover,
	\begin{align}
	||a-b||_2^2 &= (a-b)^T(a-b)\\
	&= [R(\lambda - \mu)]^T[R(\lambda - \mu)]\\
	&=(\lambda - \mu)^TR^TR(\lambda - \mu).
	\end{align}
	Note that, for $D = (-\frac{1}{2}d_{ij})_{(k+1)\times(k+1)}$,
	\begin{align}
	-\frac{1}{2}d_{ij} &= -\frac{1}{2}(v_i-v_j)^T(v_i-v_j)\\
	&=v_i^Tv_j-\frac{1}{2}(v_i^Tv_i + v_j^Tv_j).
	\end{align}
	So we have $D = R^TR -  \frac{1}{2}\left((v_i^Tv_i + v_j^Tv_j)_{(k+1)\times(k+1)}\right)$. It can be further simplified to $D = R^TR -  \frac{1}{2}(V\alpha^T + \alpha V^T)$, where $V = (v_0^Tv_0,...,v_k^Tv_k)^T$ and $\alpha = (1,...,1)^T$. So
	
	\begin{align}
	||a-b||_2^2 &= (\lambda - \mu)^TR^TR(\lambda - \mu)\\
	&= (\lambda - \mu)^T(D + \frac{1}{2}(V\alpha^T + \alpha V^T))(\lambda - \mu)\\
	&=(\lambda - \mu)^TD(\lambda - \mu) + \frac{1}{2}(\lambda - \mu)^T(V\alpha^T + \alpha V^T)(\lambda - \mu),
	\end{align}
	therefore, we only need to prove $(\lambda - \mu)^T(V\alpha^T + \alpha V^T)(\lambda - \mu) = 0$. From Lemma \ref{lemma:linear_exp}, we have $\alpha^T(\lambda - \mu) = (\lambda - \mu)^T\alpha = 0$ and the Lemma \ref{lemma:jiong} holds.
	
	\qedhere
\end{proof}

\begin{definition}[Ultra dimension] For a set $U$ composed of vectors in a $n$ dimensional linear space $V$, we define $\widehat{\mathbf{dim}}(U)$ as the Ultra dimension of $U$. The definition is that if $U$ has $k$ linearly independent vectors and there are no more, then $\widehat{\mathbf{dim}}(U) = k$.
	\label{definition}
\end{definition}{}
In fact, if $ U $ is a linear subspace in $V$, then the Ultra dimension and the dimensions of the linear subspace are equivalent. If $U$ is a linear manifold, $U = \{x + v_0|x \in W \}$, where $v_0$ and $W$ are non-zero vectors and linear subspaces in $V$, respectively. And $\mathbf{dim} (W) = r$. Then
\begin{equation}
\widehat{\mathbf{dim}}(U) = \left\{ {\begin{array}{*{20}{l}}
	r, \quad v_0 \in W\\
	r +1,  v_0 \notin W
	\end{array}} \right.
\label{ultra}
\end{equation}
In other words, $\widehat{\mathbf{dim}}(U) \geq \widehat{\mathbf{dim}}(W)$ always holds.

\begin{lemma} For arbitrary $ k$~$(1 \leq k \leq n-1)$, let $ a_1, a_2, ..., a_k $ be $k$ linearly independent vectors in $n$ dimensional linear space $V$. Consider one $n-1$ dimensional linear subspace $W$ in $V$ and a non-zero vector $v_0$ in $V$. They form a linear manifold $P = \{v_0 + \alpha|\alpha \in W \}$. If $a_1, a_2, ..., a_k$ do not all belong to $P$, then there must exist $n-k$ vectors $p_1,p_2,...,p_{n-k}$ from $P$, s.t $(a_1,a_2,...,a_k,p_1,p_2,...,p_{n-k})$ are a set of basis for the linear space $V$.
	
	\label{lemma:basis}
\end{lemma}{}

\begin{proof}
	we use mathematical induction. First, show that the Lemma \ref{lemma:basis} holds for $n-k = 1$. it means we need to find a vector $p_1 \in P$ s.t. $a_1,a_2,...,a_k,p_1$ linearly independent. If $p_1$ does not exist, then $\forall p \in P$ would be linearly represented by $a_1,a_2,...,a_k$. In other word,
	\begin{equation}
	P \subset L = \mathbf{span}(a_1,a_2,...,a_k),
	\label{psubsetl}
	\end{equation}
	{\small{\textcircled{\tiny{1}}}}~For the linear manifold $P$, if $v_0 \in W$. This means that $P$ is equal to the linear subspace $W$. Since Eq.~(\ref{psubsetl}), we have $W \subset L$ and $\widehat{\mathbf{dim}}(W) = \widehat{\mathbf{dim}}(L)$. Hence, $P = W = L$. However, $a_1, a_2, ..., a_k$ do not all belong to $P$, a contradiction. 
	
	{\small{\textcircled{\tiny{2}}}}~For the linear manifold $P$, if $v_0 \notin W$, then $ \widehat {\mathbf {dim}} (P) = n$. Because $v_0 \notin W$, that is, $v_0$ cannot be represented by a set of basis of $W$. In other words, $v_0$ and a set of basis of $W$ are linearly independent. However, the dimension of $W$ is $n-1$, hence $\widehat{\mathbf{dim}}(P) = n$. From Eq.~(\ref{psubsetl}), we have $P \subset L$, so
	\begin{equation}
	n = \widehat{\mathbf{dim}}(P) \leq \widehat{\mathbf{dim}}(L) = k = n-1,
	\end{equation}
	a contradiction. Therefore, Lemma \ref{lemma:basis} holds for $n-k = 1$. Assume the induction hypothesis that Lemma \ref{lemma:basis} is true when $n - k = l$, where $1 \leq l $. when $n - k = l + 1$, \textit{i.e.}, $k = n - (l+1)$, we also can find a vector $p_1 \in P$ s.t. $a_1,a_2,...,a_k,p_1$ linearly independent. Otherwise, $\forall p \in P$ would be linearly represented by $a_1,a_2,...,a_k$. Similarly, we have Eq.~(\ref{psubsetl}). Note that, from Definition \ref{definition}, $\widehat{\mathbf{dim}}(P) \geq n-1$, hence
	\begin{equation}
	n-1 \leq \widehat{\mathbf{dim}}(P) \leq \widehat{\mathbf{dim}}(L) = k = n - (l+1).
	\end{equation}
	a contradiction. At this time, we have $k+1 = n - (l+1) +1 = n - l$ vectors $a_1,a_2,...,a_k,p_1$ which are not all on $P$. Note that $n - (n - l) = l$, using the induction hypothesis, the Lemma \ref{lemma:basis} also holds for $n-k = l$. In summary, Lemma \ref{lemma:basis} holds.

	\qedhere
\end{proof}

\textbf{Theorem \ref{theo:ggdl}.} Let $v_0,v_1,...,v_k$ be the $k+1$ vectors in $n$ dimensional Euclidean space $\mathbb{E}^n$. For all $P$ in $\mathbb{E}^n$,
$$
\sum_{i=0}^k||P-v_i||_2^2 = \sum_{i=0}^k||G-v_i||_2^2 + (k+1)||P-G||_2^2.
$$
where $G$ is the centroid of $v_i$, will hold if it satisfies one of the following conditions:

(1)if $k\geq n$ and $\mathbf{rank}(v_1-v_0,v_2-v_0,...,v_k-v_0)=n$.

(2)if $k < n$ and $(v_1-v_0,v_2-v_0,...,v_k-v_0)$ are linearly independent.

(3)if $v_i \sim N(\mathbf{0},c\cdot \mathbf{I}_n)$, Eq.(\ref{ggdl}) holds with probability 1 where $c$ is a constant.

\begin{proof}
	
	\textbf{For Theorem \ref{theo:ggdl}~(1)}. From Lemma \ref{lemma:linear_exp}, $\forall P \in E^n$ ,$\exists \mathbf{\gamma} = (\gamma_0,...,\gamma_k)$, s.t. $P$ can be represented by $\sum_{i=0}^k \gamma_iv_i$, where $\sum_{i=0}^k \gamma_i = 1$. In fact, for each $v_i$, it also can be respresented by $\sum_{j=0}^k \beta_{ij}v_i$, where $\sum_{i=0}^k \beta_{ij} = 1$. We just take 
	$(\beta_{i0},\beta_{i1},...,\beta_{ik})$ as one of the standard orthogonal basis $\epsilon_i = (0,0,...,1_{i},...0)$. According to lemma ~\ref{lemma:jiong},
	\begin{align}
	||P-v_i||_2^2 &= (\gamma - \epsilon_i)^TD(\gamma - \epsilon_i)\\
	&= \gamma^TD\gamma - 2\gamma^TD\epsilon_i + \epsilon_i^TD\epsilon_i\\
	& = \gamma^TD\gamma - 2\gamma^TD\epsilon_i .
	\end{align}
	The final equation is because the diagonal elements of the matrix are all 0. On the other hand, we have
	\begin{align}
	||G-v_i||_2^2 &= (\frac{1}{k+1}\sum_{i=0}^k\epsilon_i - \epsilon_i)^TD(\frac{1}{k+1}\sum_{i=0}^k\epsilon_i - \epsilon_i)\\
	&= \frac{1}{(k+1)^2}\alpha^TD\alpha - \frac{2}{k+1}\alpha^TD\epsilon_i +  \epsilon_i^TD\epsilon_i\\
	&= \frac{1}{(k+1)^2}\alpha^TD\alpha - \frac{2}{k+1}\alpha^TD\epsilon_i,
	\end{align}
	where $\alpha=\sum_{i=0}^k\epsilon_i$, \textit{i.e.},$\alpha = (1,1,...,1)$. Next, we consider $||P-G||_2^2$.
	\begin{align}
	||P-G||_2^2 &= (\gamma - \frac{1}{k+1}\alpha)^TD(\gamma - \frac{1}{k+1}\alpha)\\
	&=\gamma^TD\gamma + \frac{1}{(k+1)^2}\alpha^TD\alpha - \frac{2}{k+1}\gamma^TD\alpha.
	\end{align}
	In summary, we have
	\begin{align}
	\sum_{i=0}^k ||P-v_i||_2^2 - ||G-v_i||_2^2 &= (k+1)\gamma^TD\gamma - 2\gamma^TD\alpha + \frac{1}{k+1}\alpha^TD\alpha\\
	&= (k+1)||P-G||_2^2
	\end{align}
	Therefore, Theorem \ref{theo:ggdl}~(1) holds.

	\textbf{For Theorem \ref{theo:ggdl}~(2)}. Next, we prove the case of $k <n$. Obviously, Lemma ~\ref{lemma:linear_exp} does not hold. We consider about such a linear space $W_1 = \mathbf{span}(P-G)$, \textit{i.e.}, a linear space expanded by $P-G$, and its orthogonal complement $W_1^\perp$ (in $E^n$). Since dimension formula from linear space, it is easy to konw that $\mathbf{dim}(W_1^\perp) = n-1$.
	
	Two linear manifolds $T_1$ and $T_2$ are constructed as follows,
	\begin{align}
	T_1 &= \{x + G | x \in W_1^\perp\}\\
	T_2 &= \{x + G - v_0 | x \in W_1^\perp\}
	\end{align}
	$\forall v_i \in T_1$, we have $(v_i - G)^T(P-G) = 0$, Furthermore, 
	\begin{equation}
	||P - v_i||_2^2 = ||v_i - G||_2^2 + ||P - G||_2^2.
	\label{gg}
	\end{equation}
	It  is easy to know that $G - v_0$ is not 0. If $v_1-v_0,...,v_k-v_0$ are all belong to $T_2$, it means $v_1,..,v_k$ are all in $T_1$. Hence, we have Eq.~(\ref{gg}). By summing both sides of Eq.~(\ref{gg}) for $i$, it is obvious find that Theorem \ref{theo:ggdl}~(2) holds. If $v_1-v_0,...,v_k-v_0$ are not all belong to $T_2$, since Lemma \ref{lemma:basis}, there are $n-k$ vectors $p_1 - v_0,p_2 - v_0,..,p_{n-k} - v_0$ from $T_2$ s.t. they and $v_1-v_0,...,v_k-v_0$ are linearly independent, where $p_i$ obviously belongs to manifold $T_1$.
	
	At the same time, we have $ 2G-p_i \in T_1 $, we can also construct $ n-k $ new vectors $ 2G-p_i -v_0 \in T_2 $ and calculate the rank that 
	
	$$\mathbf{rank}(v_1-v_0,...,v_k-v_0, p_1-v_0,...,p_{n-k}-v_0, 2G-p_1-v_0,...,2G-p_{n-k}-v_0)$$
	\begin{align}
	&=\mathbf{rank}(v_1-v_0,...,v_k-v_0, p_1-v_0,...,p_{n-k}-v_0, 2(G-v_0),...,2(G-v_0))\\
	&=\mathbf{rank}(v_1-v_0,...,v_k-v_0, p_1-v_0,...,p_{n-k}-v_0, 0,...,0)\\
	&= n
	\end{align}
	The reason of the final equation is that $\sum_{i=1}^k (v_i - v_0)=(k+1)(G-v_0)$. Note that there are a total of $ k + (n-k) + (n-k) = n + (n-k) \geq n$ vectors, meets the lemma \ref{lemma:linear_exp} condition. For the convenience of description, we define
	\begin{align}
	L^{(1)}_i &= v_i, (0 \leq i \leq k),\\
	L^{(2)}_i &= p_i, (1 \leq i \leq n-k),\\
	L^{(3)}_i &= 2G - p_i, (1 \leq i \leq n-k).
	\end{align}
	And their centroid is
	\begin{align}
	G' &= \frac{1}{2n-k+1}\left(\sum_{i=0}^k v_i + \sum_{i=1}^{n-k}(L^{(2)}_i+L^{(3)}_i)\right)\\
	&= \frac{1}{2n-k+1}((k+1)G + 2(n-k)G)\\
	& = G
	\end{align}
	That is, the newly added vector does not change the centroid of $v_i$. On the other hand, since both $L^{(2)}_i$ and $L^{(3)}_i$ are in the linear manifold $T_1$, and it meets the conditions of the Eq.(\ref {gg}). Similar to the derivation in the Theorem \ref{theo:ggdl}~(1), we have
	\begin{align}
	(2n-k+1)||P-G||_2^2 &=  \sum_{t = L^{(1)}_i,L^{(2)}_i,L^{(3)}_i} \left(||P-t||_2^2 - ||G-t||_2^2\right)\\
	&= \sum_{i=0}^k \left(||P-v_i||_2^2 - ||G-v_i||_2^2 \right)+\sum_{t = L^{(2)}_i,L^{(3)}_i} \left(||P-t||_2^2 - ||G-t||_2^2\right)\\
	&=\sum_{i=0}^k \left(||P-v_i||_2^2 - ||G-v_i||_2^2  \right)+ 2(n-k)||P-G||_2^2
	\label{temp5}
	\end{align}
	The final equation is because both $L^{(2)}_i$ and $L^{(3)}_i$ are in the linear manifold $T_1$ and satisfy Eq.~(\ref{gg}). To simplify Eq.~(\ref{temp5}), we obtain $\sum_{i=0}^k \left(||P-v_i||_2^2 - ||G-v_i||_2^2\right) = (k+1)||P-G||_2^2$. Therefore, Theorem \ref{theo:ggdl}~(2) holds.

	\textbf{For Theorem \ref{theo:ggdl}~(3)}. When $k \geq n$, from Proposition \ref{pro:matrix}, we know that  $\mathbf{rank}(v_1-v_0,v_2-v_0,...,v_k-v_0)$ = $n$ holds with probability 1. Hence, if we use the similar deduction from Theorem \ref{theo:ggdl}~(1), we can find that Theorem \ref{theo:ggdl}~(3) holds when $k \geq n$. On the other hand, when $k < n$, we can get the same result also according to Proposition \ref{pro:matrix}. The reason is that $(v_1-v_0,v_2-v_0,...,v_k-v_0)$ are linearly independent with probability 1.

\end{proof}

\clearpage
\section{The result of Sp}
\label{app:spearman}

\begin{figure} [H]
	\centering 
	\subfigure[Sp = 0.99]{ 
		\begin{minipage}[t]{0.5\linewidth} 
			\centering 
			\includegraphics[height=2.3in, width=2.3in]{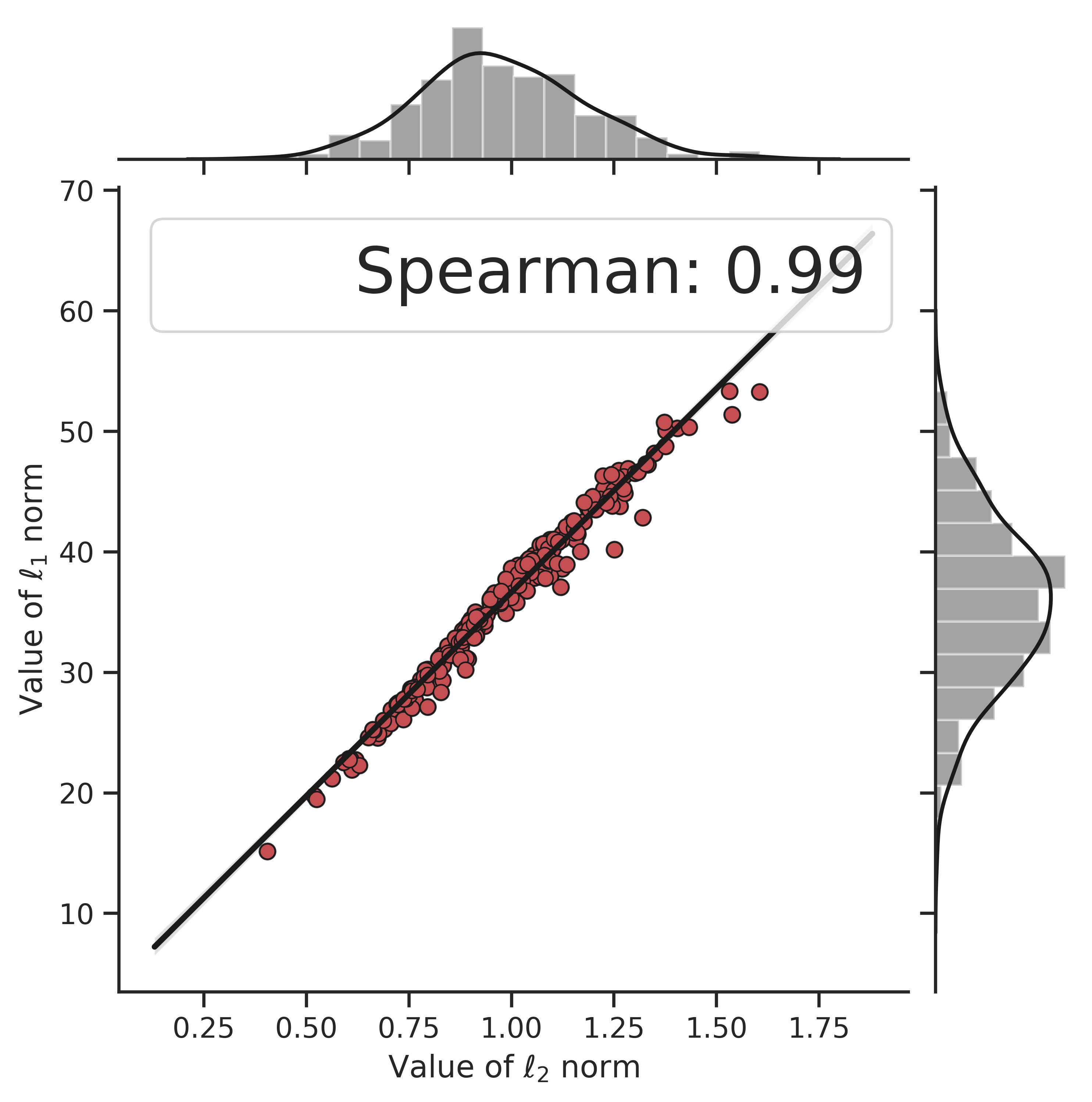} 
	\end{minipage} }%
	\subfigure[Sp = 0.99]{ 
		\begin{minipage}[t]{0.5\linewidth} 
			\centering 
			\includegraphics[height=2.3in, width=2.3in]{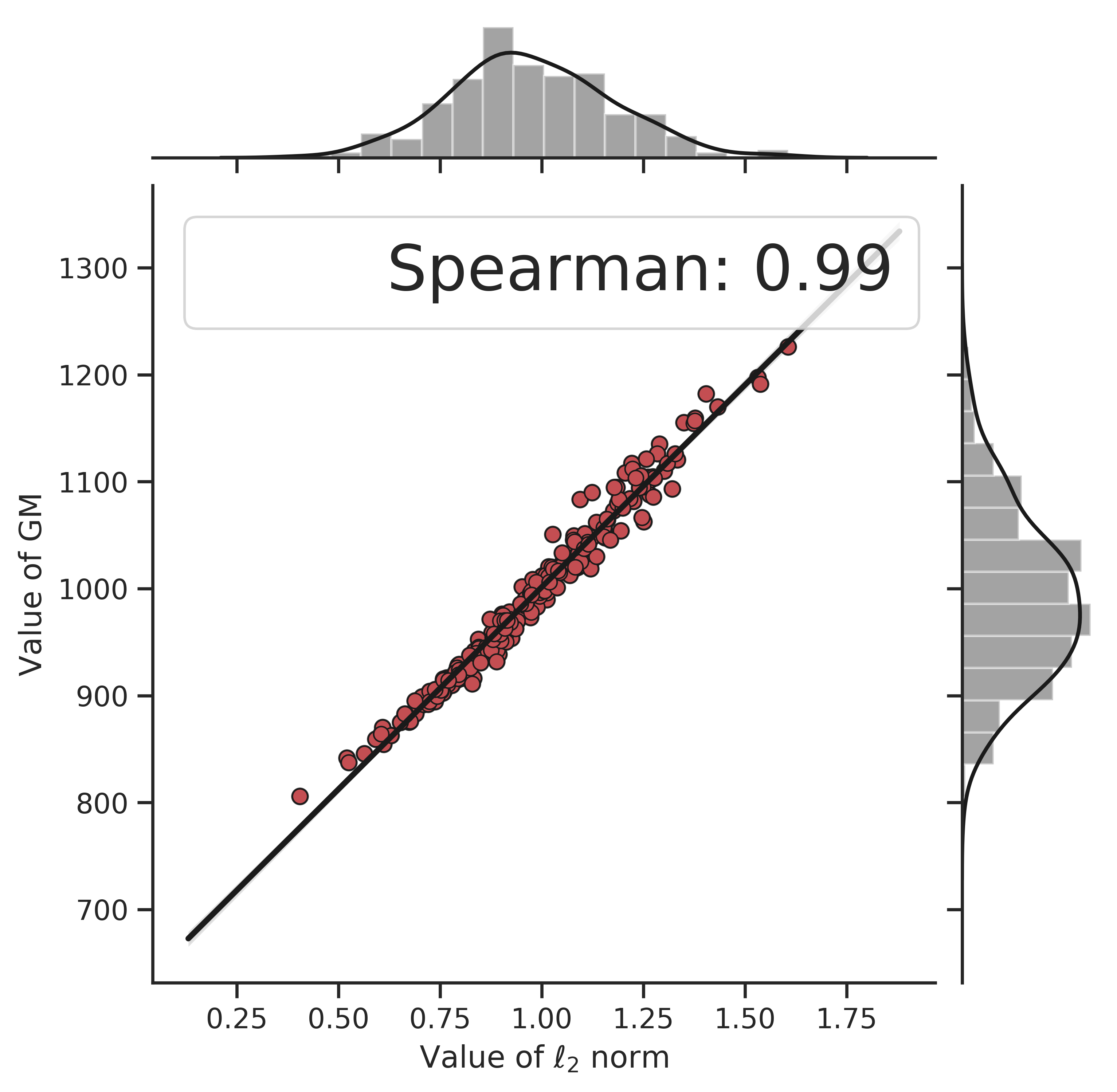} 
	\end{minipage} }%

	\subfigure[Sp = 0.99]{ 
		\begin{minipage}[t]{0.5\linewidth} 
			\centering 
			\includegraphics[height=2.3in, width=2.3in]{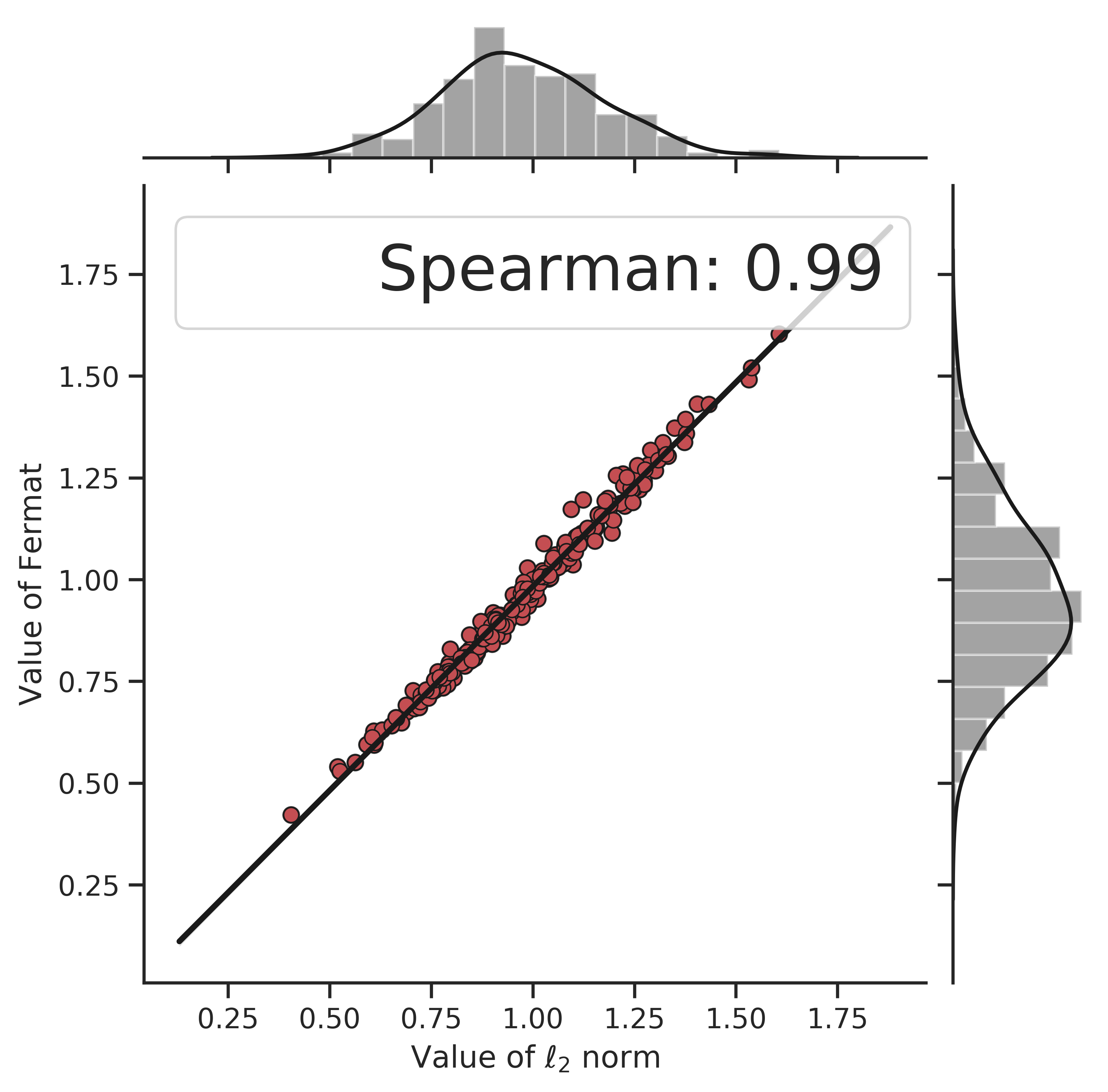} 
	\end{minipage} }%
	\subfigure[Sp = 0.98]{ 
		\begin{minipage}[t]{0.5\linewidth} 
			\centering 
			\includegraphics[height=2.3in, width=2.3in]{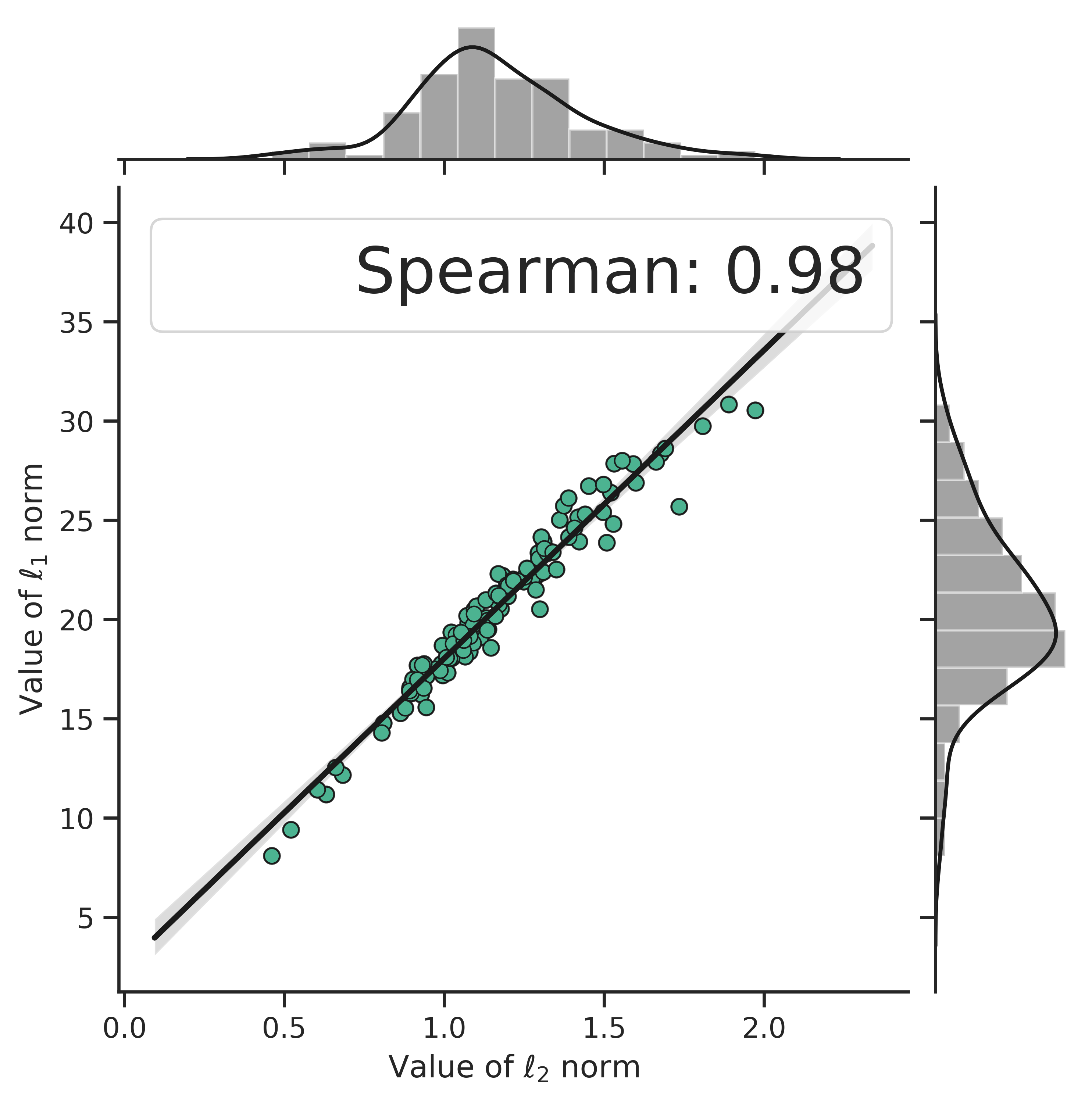} 
		\end{minipage}%
	}%

	\subfigure[Sp = 0.98]{ 
		\begin{minipage}[t]{0.5\linewidth} 
			\centering 
			\includegraphics[height=2.3in, width=2.3in]{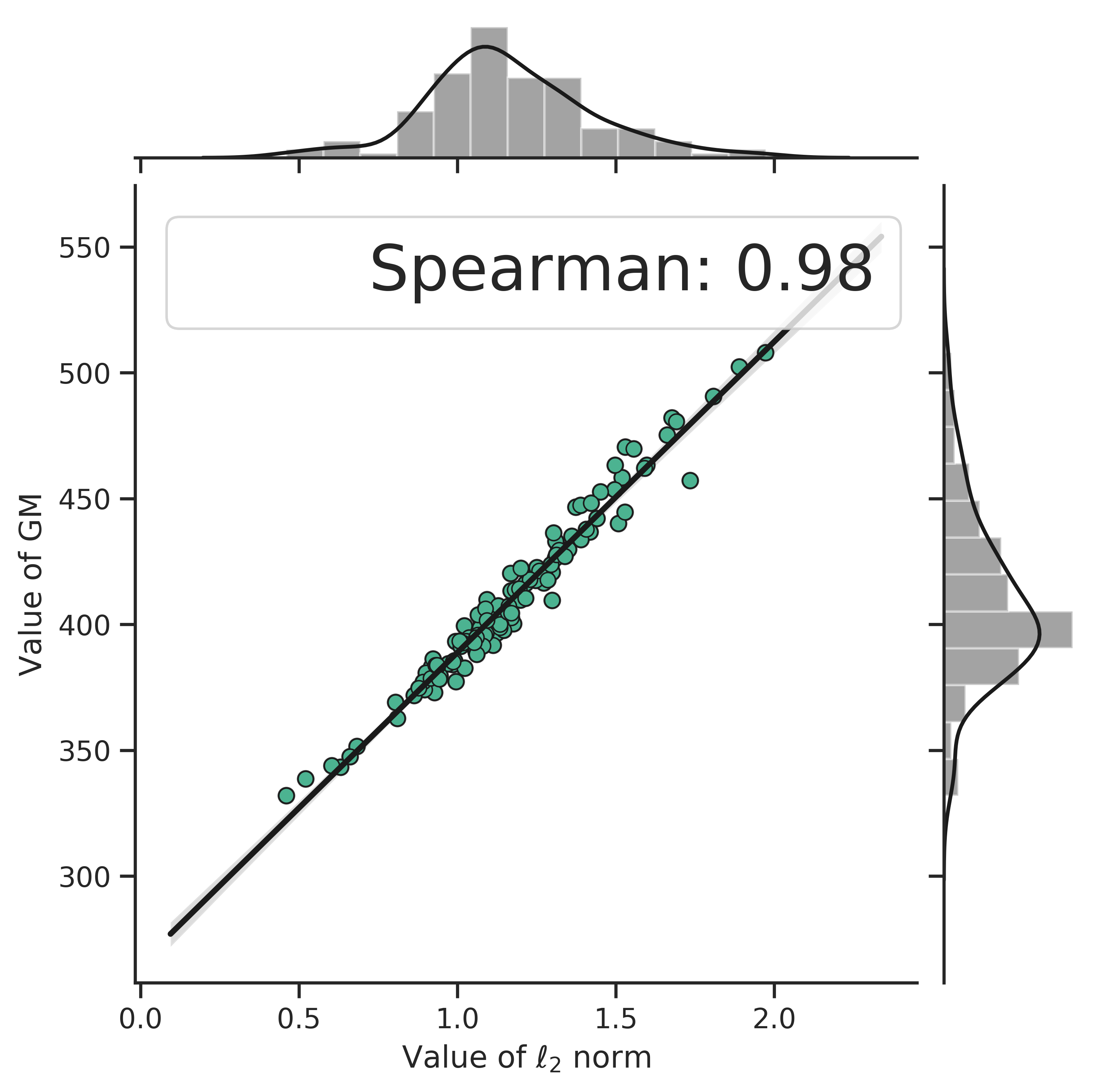} 
		\end{minipage}%
	}
	\subfigure[Sp = 1.00]{ 
		\begin{minipage}[t]{0.5\linewidth} 
			\centering 
			\includegraphics[height=2.3in, width=2.3in]{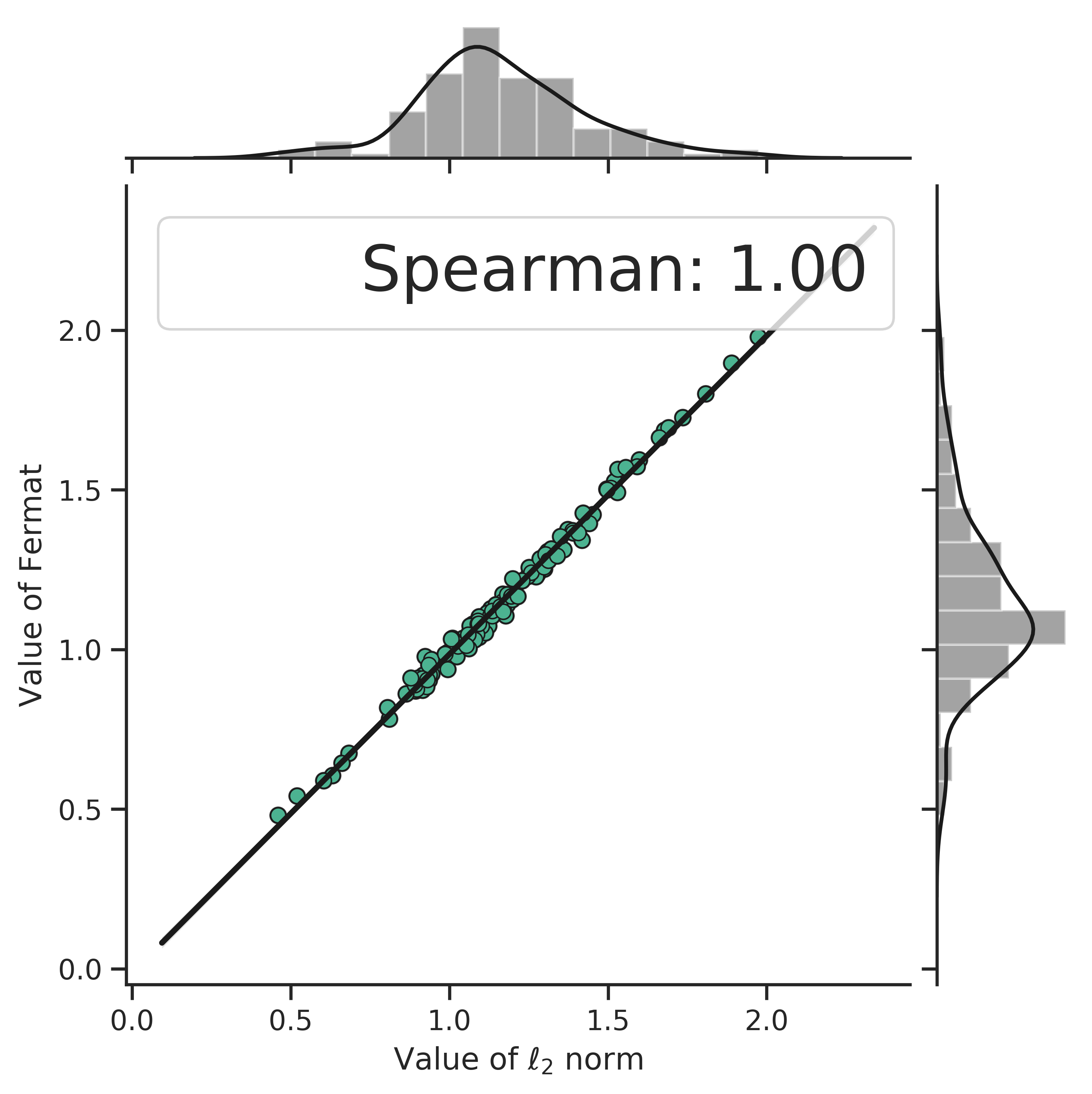} 
	\end{minipage} }%
	\centering 
	\caption{The Spearman's rank correlation coefficient~(Sp) for different criteria.~(a-c) are Sp between $\ell_1$ and $\ell_2$, $\mathbf{GM}$ and $\ell_2$, $\mathbf{Fermat}$ and $\ell_2$ from ResNet18~(12$^{\rm th}$ Conv), respectively.~The results of VGG16~(3$^{\rm rd}$ Conv) are shown in (d-f). If the Sp of two pruning criteria is close to 1, then the sequence of their pruned filters may have strong similarity.} 
	\vspace{-0.4cm}
	\label{spearman}
\end{figure}

\clearpage
\section{Other result}
\label{app:other_result}

\begin{figure} [htbp]
	\centering 
	\includegraphics[width=0.4\linewidth]{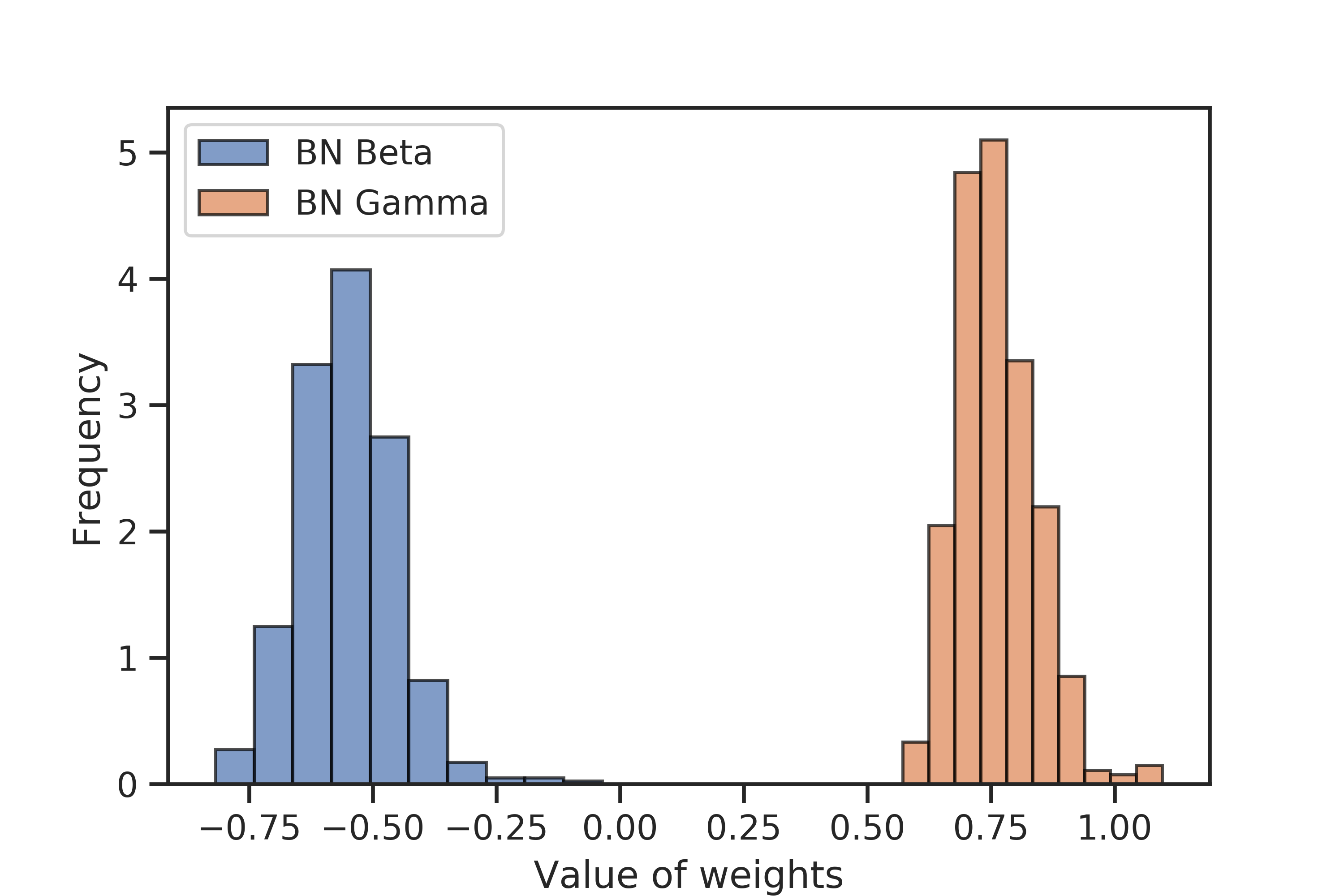}
	\includegraphics[width=0.4\linewidth]{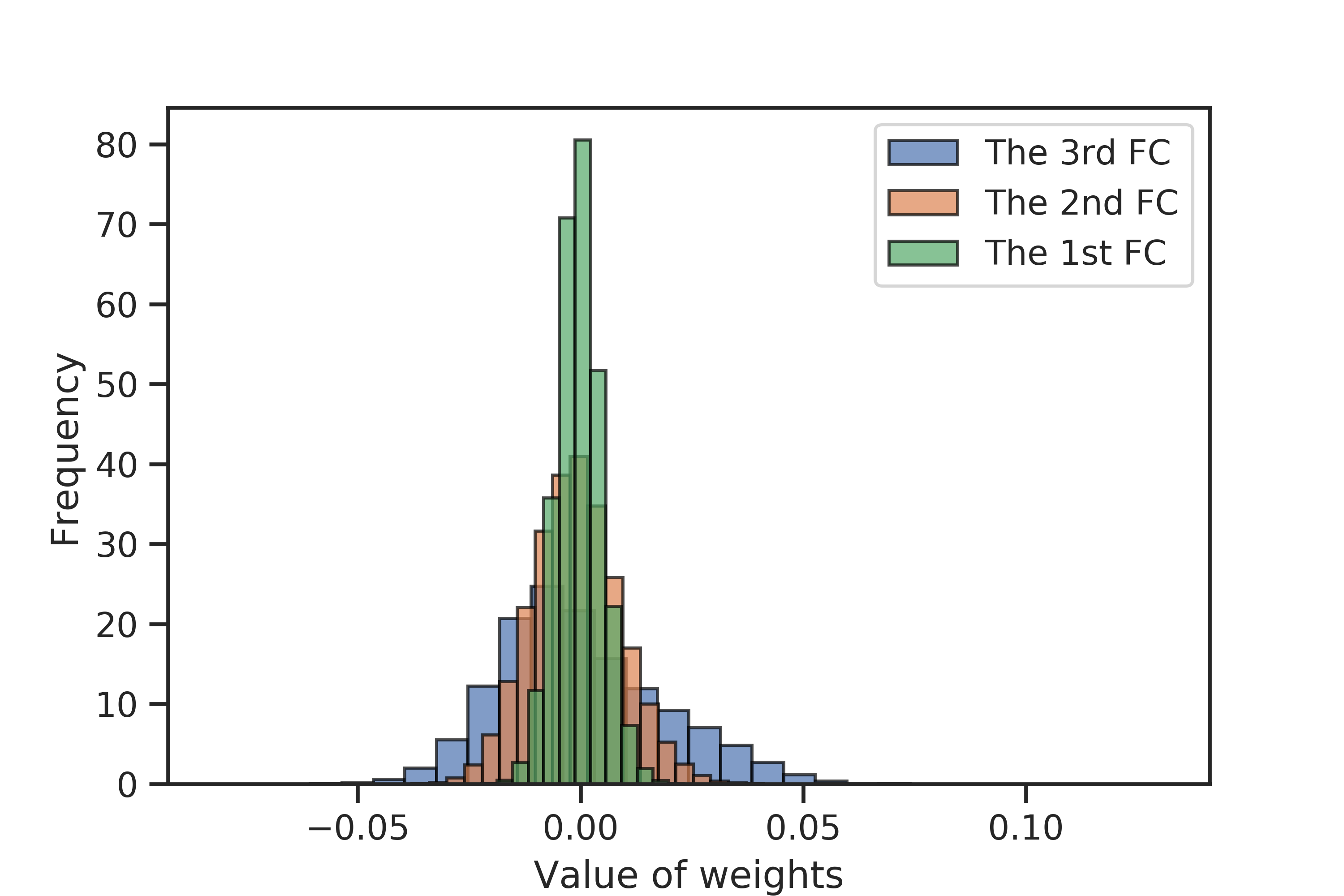} 
	\caption{The distribution about other learnable parameters. (Left): The disrtibution about the learnable parameters of batch normalization. (Rihgt): The parameters distribution of the fully-connected layers~(FC). For FC, the Sp between the criteria in Table\ref{tab:criteria} are greater than 0.9.} 
	\label{fig:fig:bnfc}
\end{figure}

In Fig~\ref{fig:fig:bnfc}, we show the other learnable parameters~(\textit{i.e.} Batch normalization~(BN) and fully connected neural network~(FC)) in VGG16-BN. For BN, the distribution of its parameters does not satisfy CWDA, and similar results are shown in \cite{Liu_2017_ICCV,tian2019luck}. Moreover, the learnable parameters of fully-connected layers also do not follow a Gaussian-alike distribution, which is consistent with the conclusion in previous work~\cite{bellido1993backpropagation,neal1995bayesian,go2004analyzing}.

\begin{figure} [htbp]
	\centering 
	\includegraphics[width=1.0\linewidth]{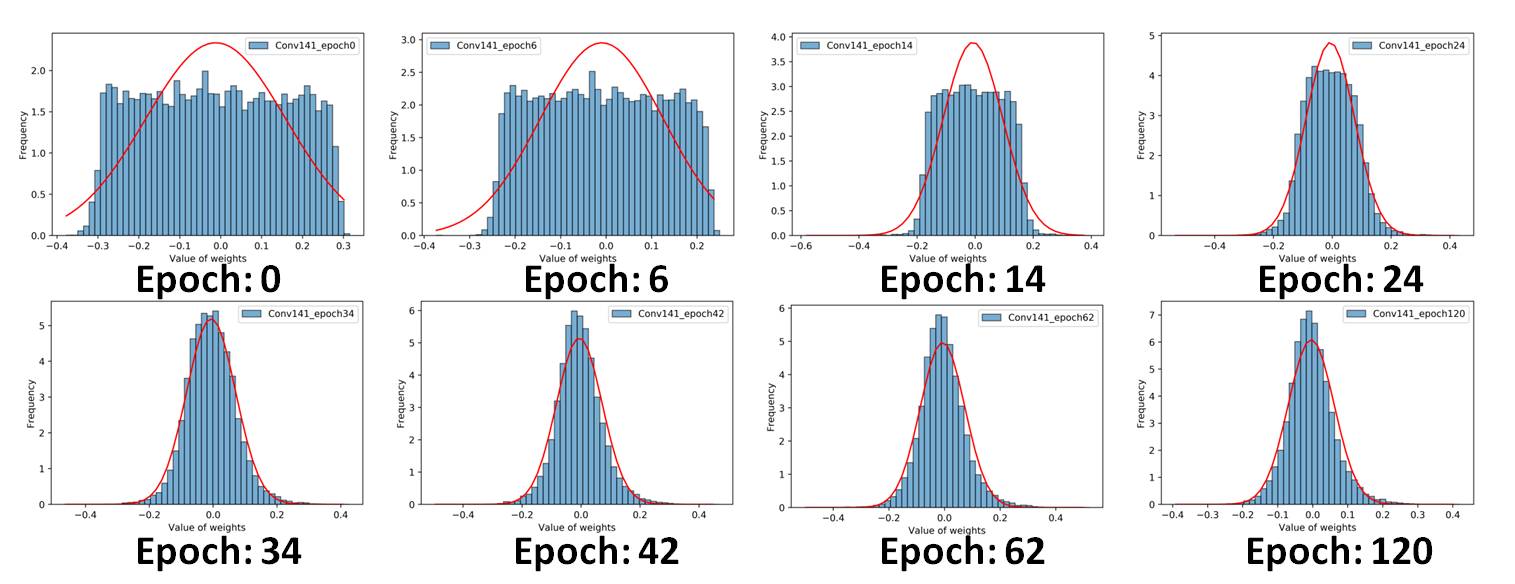} 
	\caption{The distribution of the convolutional filter~(141$^{\rm th}$ Conv) with kaiming-uniform initialization for each epoch.} 
	\label{fig:uniform}
\end{figure}    

\clearpage
\section{An interesting case for \textit{Importance Score} measured by different criteria}
\label{app:case}

The following results are the index of pruned filters obtained by the filters' \textit{Importance Score} from different types of pruning criteria. We take VGG16~(2$^{\rm nd}$) as an example. The 5$^{\rm th}$ filter in this layer is regarded as a redundant convolutional filter for APoZ criterion, but other criteria consider it to be almost the most important.

Taylor $\ell_1$:
[27, 36, 25, 11, 6, 23, 24, 16, 0, 57, 48, 53, 1, 61, 18, 55, 34, 15, 51, 58, 31, 3, 12, 21, 59, 30, 7, 38, 41, 50, 10, 33, 17, 46, 62, 13, 49, 43, 42, 47, 2, 32, 44, 20, 39, 52, 56, 40, 9, 26, 37, 22, 29, 54, 60, 8, 14, 45, 4, 63, 19, 35, 28, \textbf{\textcolor{red}{5}}]

Taylor $\ell_2$:
[23, 32, 36, 11, 62, 16, 30, 59, 10, 13, 2, 50, 38, 0, 46, 43, 21, 26, 15, 22, 7, 51, 39, 33, 14, 58, 9, 40, 57, 6, 61, 44, 20, 48, 3, 53, 41, 56, 17, 12, 18, 31, 4, 1, 25, 19, 63, 24, 54, 45, 52, 37, 55, 47, 34, 35, 8, 29, 42, 27, 49, 28, 60, \textbf{\textcolor{red}{5}}]

BN$\_\beta$:
[52, 46, 32, 21, 14, 29, 17, 0, 19, 36, 1, 51, 44, 40, 41, 60, 57, 27, 22, 53, 63, 8, 30, 26, 23, 58, 39, 18, 9, 47, 31, 35, 11, 37, 55, 45, 3, 61, 6, 4, 33, 25, 15, 48, 43, 28, 56, 2, 13, 16, 34, 20, 59, 10, 7, 24, 50, 62, 12, 49, 38, 42, \textbf{\textcolor{red}{5}}, 54]

APoZ:
[\textbf{\textcolor{red}{5}}, 10, 38, 42, 62, 24, 13, 12, 7, 28, 59, 15, 23, 11, 16, 56, 34, 35, 57, 19, 2, 49, 43, 25, 6, 63, 61, 36, 9, 27, 33, 20, 48, 58, 55, 18, 51, 31, 1, 0, 53, 37, 26, 29, 47, 60, 8, 44, 41, 46, 21, 17, 14, 32, 52, 22, 39, 3, 40, 30, 4, 45, 50, 54]

\clearpage    
\section{The details of other pruning criteria}
\label{app:other_criteria}

For notation, we denote $i^{\rm th}$ convolutional filter in layer $l$ as $F^l_i$ and the input feature maps in layer $l$ as $\mathbf{I}^l\in\mathbb{R}^{{N}\times{I^l}\times{H^l}\times{W^l}}$, where $N, I^l, H^l, W_l$ mean the train set size, number of channels, height and width respectively, $i=1,2,\cdots,\lambda_l,$ and $l=1,2,\cdots,L$. The formulation of the filters' \textit{Importance Score} under each pruning criteria are illustrated as follows: 

\textbf{Norm-based criteria:}
\begin{itemize}
	\item $\ell_1$-Norm~\cite{li2016pruning}: $||F_i^l||_1$;
	\item $\ell_2$-Norm~\cite{frankle2018the}: $||F_i^l||_2$;
\end{itemize}

\textbf{BN-based criteria~\cite{liu2017learning}:}
\begin{itemize}
	\item BN\_$\gamma$: $|\gamma_i^l|$, where $\gamma_i^l$ is the scaling factor in the Batch Normalization layer $l$;
	\item BN\_$\beta$: $|\beta_i^l|$, where $\beta_i^l$ is the shifting factor in the Batch Normalization layer $l$.
\end{itemize}

\textbf{Activation-based criteria:}
\begin{itemize}
	\item APoZ~\cite{hu2016network}: $\frac{\sum^{}_{p,q}\mathds{1}\left((|\mathbf{I}^l*F_i^l|)_{p,q}>\sigma \right)}{{N}\times{I^l}\times{H^l}\times{W^l}}$, where we set $\sigma=0.0001$ same as~\cite{luo2017entropy}, and $\mathds{1}(\cdot)$ is the indicator function, $*$ is convolution operator and $\mathbf{I}^l*F_i^l$ is the $i$-th output feature map;
	\item Entropy~\cite{luo2017entropy}: we first prepare $\mathbf{G}_i^l=GAP(\mathbf{I}^l*F_i^l)$, where $\mathbf{G}_i^l\in\mathbb{R}^{{N}\times{1}}$ and $GAP(\cdot)$ is the Global Average Pooling. Then, we estimate statistical distribution for $\mathbf{G}_i^l$ by dividing all elements in $\mathbf{G}_i^l$ into $m$ bins. Let $p_j$ is the probability of bin $j$, and the the \textit{Importance Score} score is $-\sum^m_{j=1}{p_j}\log{p_j}$.
\end{itemize}
\textbf{First order Taylor based criteria~\cite{molchanov2016pruning, molchanov2019importance, molchanov2019taylor}:}
\begin{itemize}
	\item Taylor $\ell_1$-Norm: $||\frac{\partial{loss}}{\partial{F^l_i}}\cdot{F^l_i}||_1$;
	\item Taylor $\ell_2$-Norm: $||\frac{\partial{loss}}{\partial{F^l_i}}\cdot{F^l_i}||_2$;
\end{itemize}
The $loss$ is the Cross Entropy Loss on the split training set from the original training set.

\clearpage    
\section{Additional experiments about image clasification}
\label{app:cls}

	\begin{table*}[htbp]
		\centering
		\small
		\caption{The accuracy(\%) of several networks and datasets using different pruning criteria. }
		
		\resizebox{\columnwidth}{!}{%
			\begin{tabular}{|c|c|ccc|ccc|ccc|}
				\hline
				\multicolumn{1}{|r}{} &       & \multicolumn{3}{c|}{Experiment~(1)} & \multicolumn{3}{c|}{Experiment~(2)} & \multicolumn{3}{c|}{Experiment~(3)} \\
				\cline{3-11}    \multicolumn{1}{|r}{} &       & Trained & Pruned & Fine-tuned & Trained & Pruned & Fine-tuned & Trained & Pruned & Fine-tuned \\
				\hline
				\multicolumn{1}{|p{4.25em}|}{CIFAR10} & $\ell_1$    & 93.61 & 61.21 & 93.51 & 93.21 & 54.31 & 93.22 & 93.26 & 57.74 & 93.32 \\
				\multicolumn{1}{|p{4.25em}|}{VGG16} & $\ell_2$    & 93.61 & 63.41 & 93.32 & 93.21 & 54.61 & 93.42 & 93.26 & 57.42 & 93.29 \\
				& $\mathbf{GM}$    & 93.61 & 61.22 & 93.41 & 93.21 & 53.71 & 93.25 & 93.26 & 57.46 & 93.36 \\
				\hline
				\multicolumn{1}{|p{4.25em}|}{CIFAR100} & $\ell_1$    & 72.67 & 25.91 & 71.50  & 72.99 & 20.43 & 71.36 & 72.56 & 24.01 & 71.07 \\
				\multicolumn{1}{|p{4.25em}|}{VGG16} & $\ell_2$    & 72.67 & 27.07 & 71.28 & 72.99 & 22.31 & 71.12 & 72.56 & 24.45 & 70.92 \\
				& $\mathbf{GM}$    & 72.67 & 26.37 & 71.27 & 72.99 & 21.67 & 71.26 & 72.56 & 24.26 & 70.78 \\
				\hline
				\multicolumn{1}{|p{4.25em}|}{ImageNet} & $\ell_1$    & 71.58 & 30.33 & 71.02 & 71.33 & 40.33 & 70.12 & 72.01 & 28.07 & 70.93 \\
				\multicolumn{1}{|p{4.25em}|}{VGG16} & $\ell_2$    & 71.58 & 29.47 & 70.83 & 71.33 & 40.45 & 70.13 & 72.01 & 27.89 & 71.02 \\
				& $\mathbf{GM}$    & 71.58 & 30.76 & 70.95 & 71.33 & 39.86 & 70.33 & 72.01 & 28.01 & 70.74 \\
				\hline
				\multicolumn{1}{|p{4.25em}|}{CIFAR10} & $\ell_1$    & 92.98 & 77.73 & 93.08 & 92.97 & 76.02 & 92.82 & 93.01 & 79.93 & 92.81 \\
				\multicolumn{1}{|p{4.25em}|}{ResNet56} & $\ell_2$    & 92.98 & 79.02 & 92.83 & 92.97 & 77.91 & 92.72 & 93.01 & 82.43 & 92.81 \\
				& $\mathbf{GM}$    & 92.98 & 74.26 & 92.77 & 93.2  & 73.93 & 92.61 & 93.01 & 80.48 & 92.84 \\
				\hline
				\multicolumn{1}{|p{4.25em}|}{CIFAR100} & $\ell_1$    & 71.36 & 50.64 & 70.15 & 70.02 & 52.41 & 69.19 & 70.48 & 52.19 & 69.77 \\
				\multicolumn{1}{|p{4.25em}|}{ResNet56} & $\ell_2$    & 71.36 & 53.44 & 70.16 & 70.02 & 52.73 & 69.31 & 70.48 & 52.16 & 69.62 \\
				& $\mathbf{GM}$    & 71.36 & 45.12 & 70.22 & 70.02 & 52.62 & 69.54 & 70.48 & 50.74 & 69.69 \\
				\hline
				\multicolumn{1}{|p{4.25em}|}{ImageNet} & $\ell_1$    & 73.31 & 62.22 & 73.06 & 73.16 & 54.24 & 72.99 & 73.21 & 63.12 & 73.02 \\
				\multicolumn{1}{|p{4.25em}|}{ResNet34} & $\ell_2$    & 73.31 & 62.02 & 72.91 & 73.16 & 53.64 & 72.78 & 73.21 & 62.98 & 72.86 \\
				& $\mathbf{GM}$    & 73.31 & 61.88 & 72.96 & 73.16 & 53.48 & 72.94 & 73.21 & 62.36 & 73.04 \\
				\hline
			\end{tabular}%
		}
		\label{tab:vggmore}%
	\end{table*}%
All the setting of these experiments are under can be found in \url{https://github.com/bearpaw/pytorch-classification}. Specifically, for pruning ratio:

VGG16 on CIFAR10, CIFAR100 and ImageNet:

\url{https://github.com/Eric-mingjie/rethinking-network-pruning/blob/master/cifar/l1-norm-pruning/vggprune.py#L84}

ResNet56 on CIFAR10 and CIFAR100:

\url{https://github.com/Eric-mingjie/rethinking-network-pruning/blob/master/cifar/l1-norm-pruning/res56prune.py#L94}

ResNet34 on ImageNet:

\url{https://github.com/Eric-mingjie/rethinking-network-pruning/blob/master/imagenet/l1-norm-pruning/prune.py#L138}

\clearpage    
\section{About weight decay}
\label{app:weight_decay}

\begin{figure} [htbp]
	\centering 
	\includegraphics[width=0.4\linewidth]{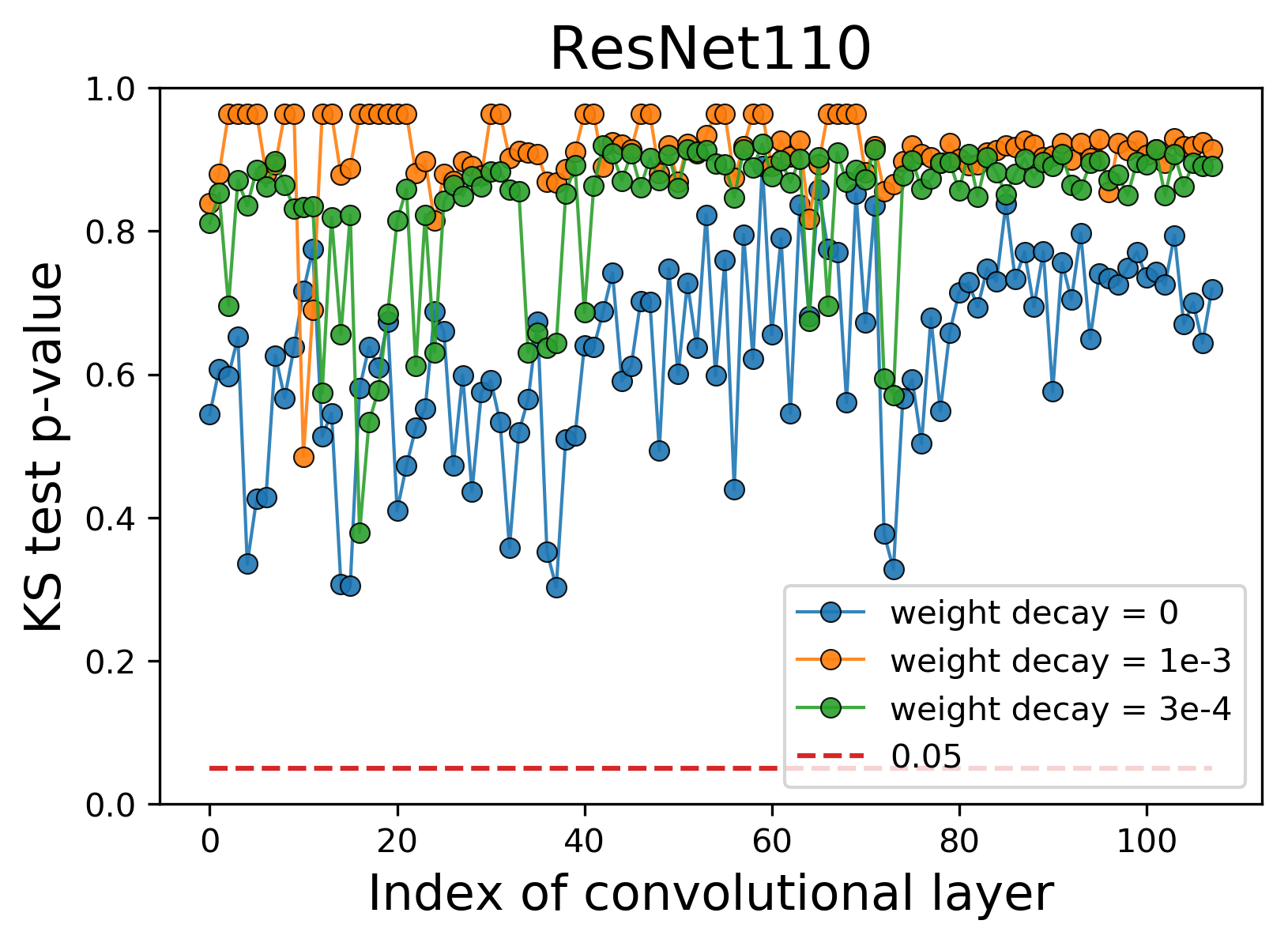}
	\includegraphics[width=0.4\linewidth]{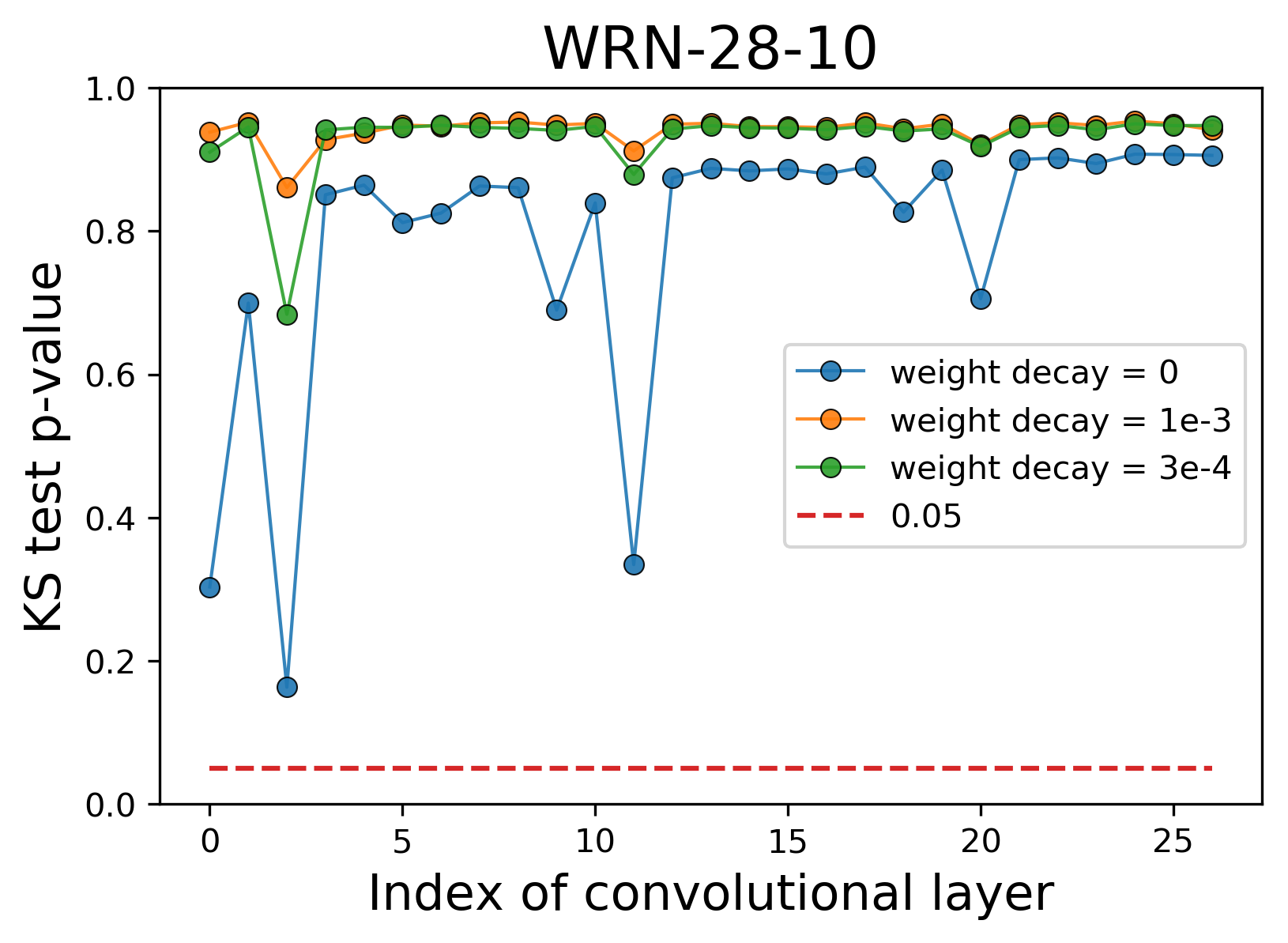} 
	\caption{KS test~\cite{lilliefors1967kolmogorov} while using different settings of weight decay.} 
	\label{fig:wd}
\end{figure}

We train the ResNet110 and WRN-28-10 on CIFAR100 with different weight decay~(1e-3, 3e-4 and 0) and use KS test to verify whether the parameters of different layers follow a normal distribution. In Fig.~\ref{fig:wd}, we can find

(1)~When weight decay~(wd) is non-zero, the normality is higher than that when weight decay is 0. 

(2)~If weight decay is 0, the p-value can still be much greater than 0.05, which means that the regularization of weight decay may not be the key reason for CWDA. The distribution of the parameters in these two networks~(weight decay is 0) are shown in Fig.~\ref{fig:wd_resnet} and Fig.~\ref{fig:wd_wrn}.

\begin{figure} [htbp]
	\centering 
	\includegraphics[width=1\linewidth]{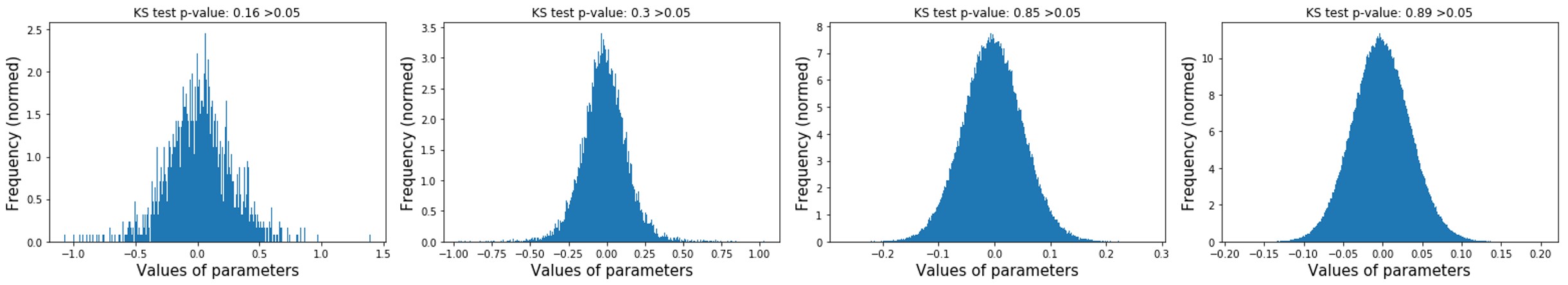}
	\caption{The distribution of parameters in different convolutional filters~(WRN-28-10, wd = 0).} 
	\label{fig:wd_wrn}
\end{figure}
\begin{figure} [htbp]
	\centering 
	\includegraphics[width=1\linewidth]{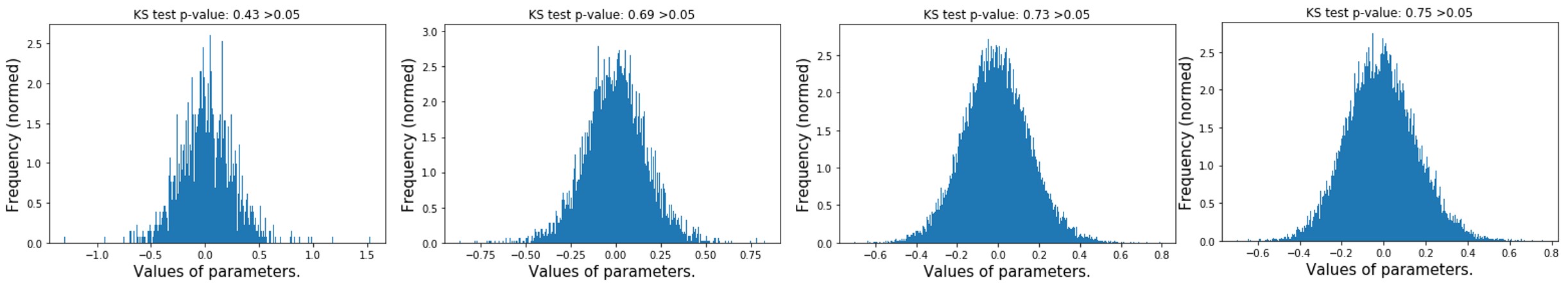}
	\caption{The distribution of parameters in different convolutional filters~(ResNet110, wd = 0).} 
	\label{fig:wd_resnet}
\end{figure}

\clearpage
\section{More visualizations of correlation matrix}
\label{app:diag_matrix}
\subsection{VGG16}

\begin{figure}[h]
	\centering 
	\includegraphics[width=0.75\linewidth]{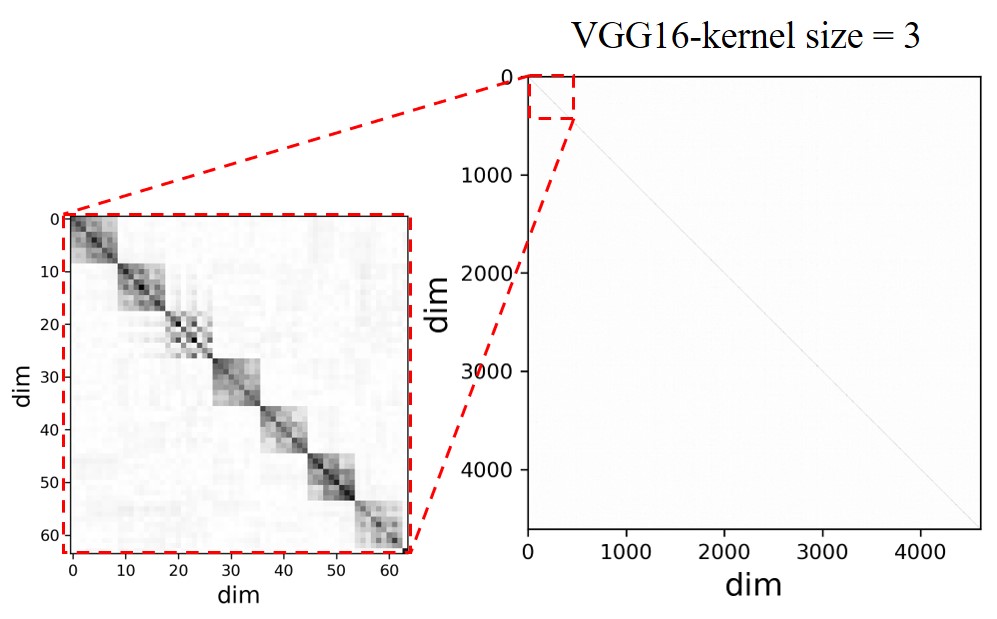}
	\label{fig:cwda_new1}
\end{figure}
\begin{figure}[h]
	\centering 
	\includegraphics[width=0.75\linewidth]{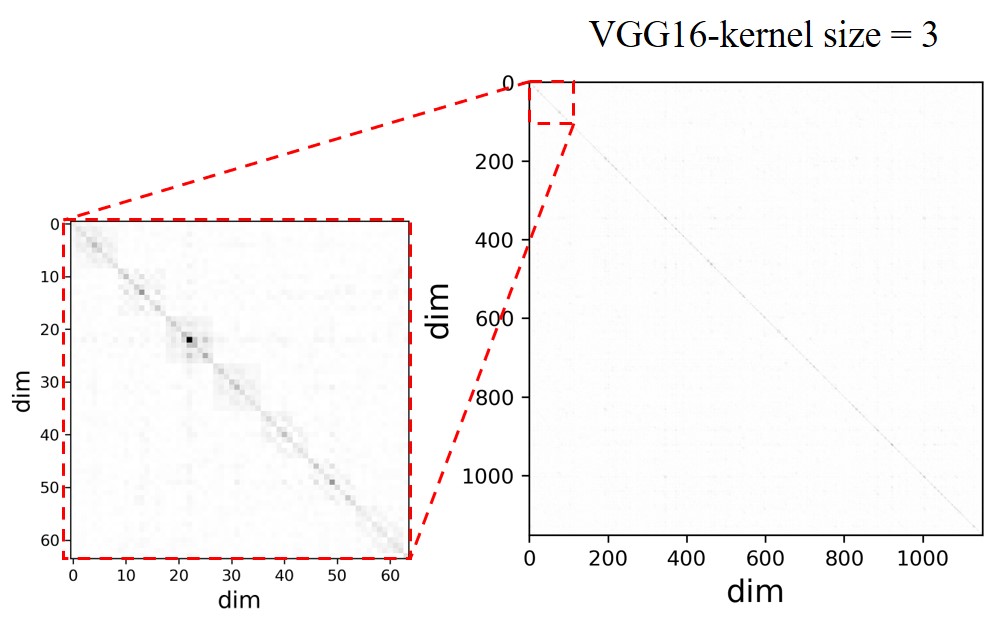}
	\label{fig:cwda_new1}
\end{figure}

\clearpage
\subsection{VGG19}

\begin{figure}[h]
	\centering 
	\includegraphics[width=0.75\linewidth]{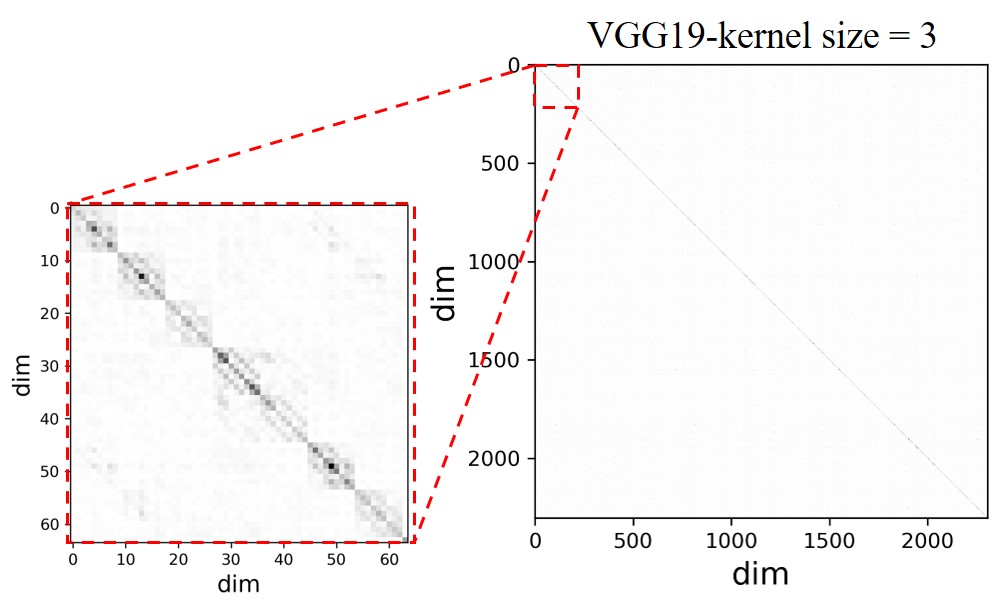}
	\label{fig:cwda_new1}
\end{figure}
\begin{figure}[h]
	\centering 
	\includegraphics[width=0.75\linewidth]{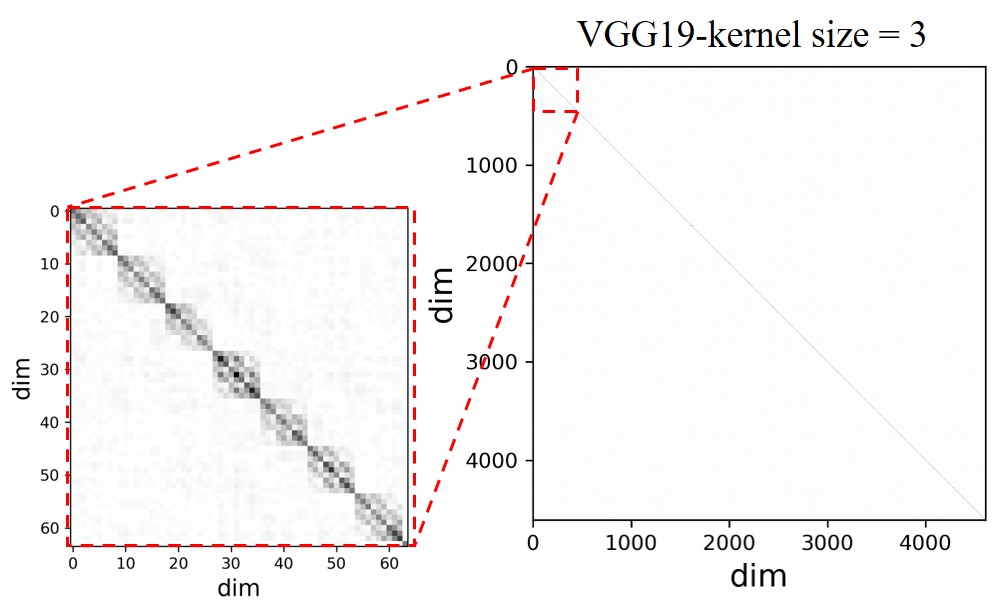}
	\label{fig:cwda_new1}
\end{figure}

\clearpage
\subsection{ResNet18}

\begin{figure}[h]
	\centering 
	\includegraphics[width=0.75\linewidth]{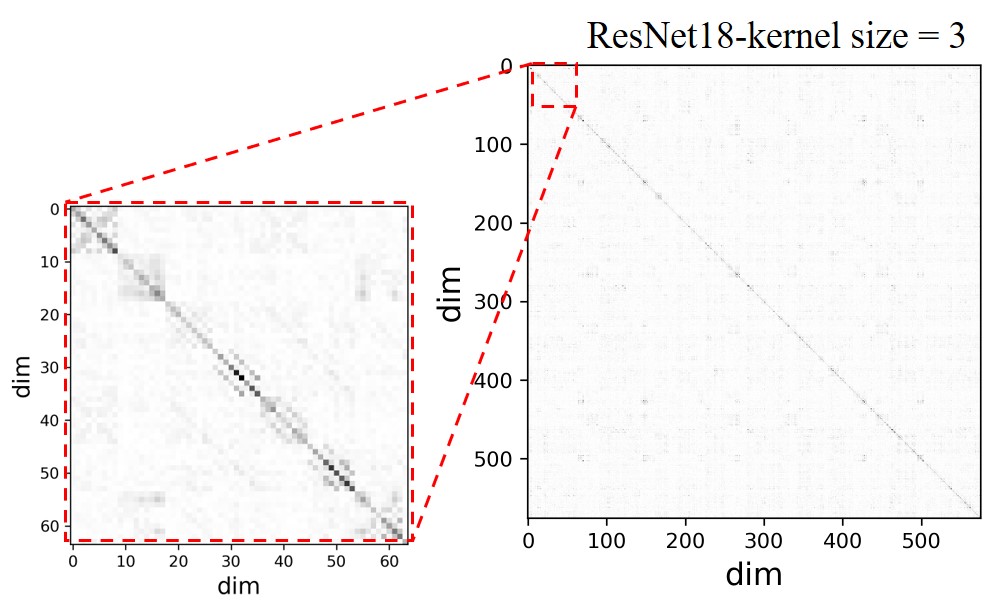}
	\label{fig:cwda_new1}
\end{figure}
\begin{figure}[h]
	\centering 
	\includegraphics[width=0.75\linewidth]{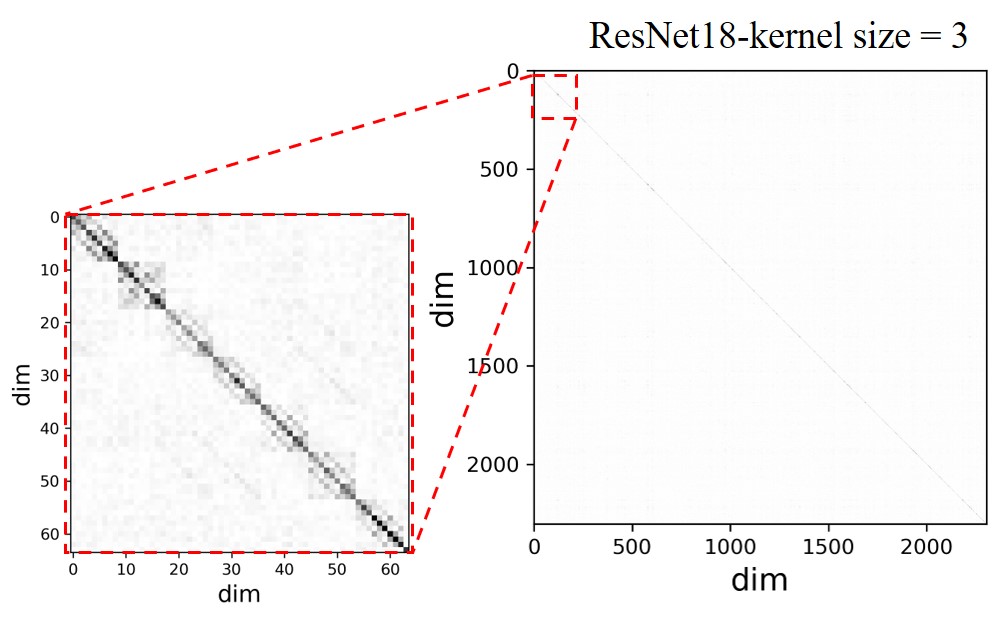}
	\label{fig:cwda_new1}
\end{figure}

\clearpage
\subsection{ResNet50}

\begin{figure}[h]
	\centering 
	\includegraphics[width=0.75\linewidth]{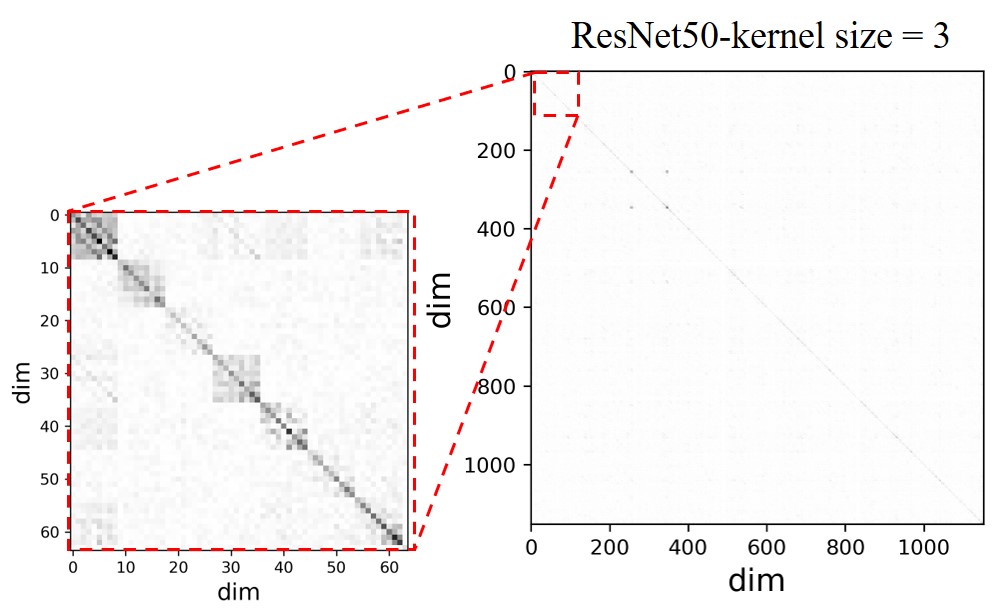}
	\label{fig:cwda_new1}
\end{figure}
\begin{figure}[h]
	\centering 
	\includegraphics[width=0.75\linewidth]{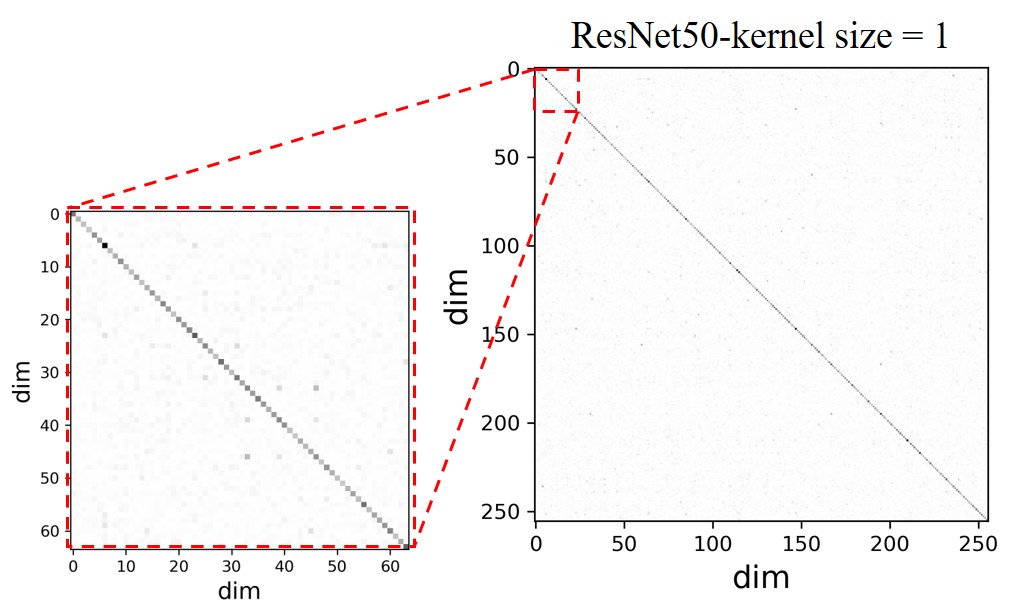}
	\label{fig:cwda_new1}
\end{figure}

\clearpage
\subsection{AlexNet}

\begin{figure}[h]
	\centering 
	\includegraphics[width=0.75\linewidth]{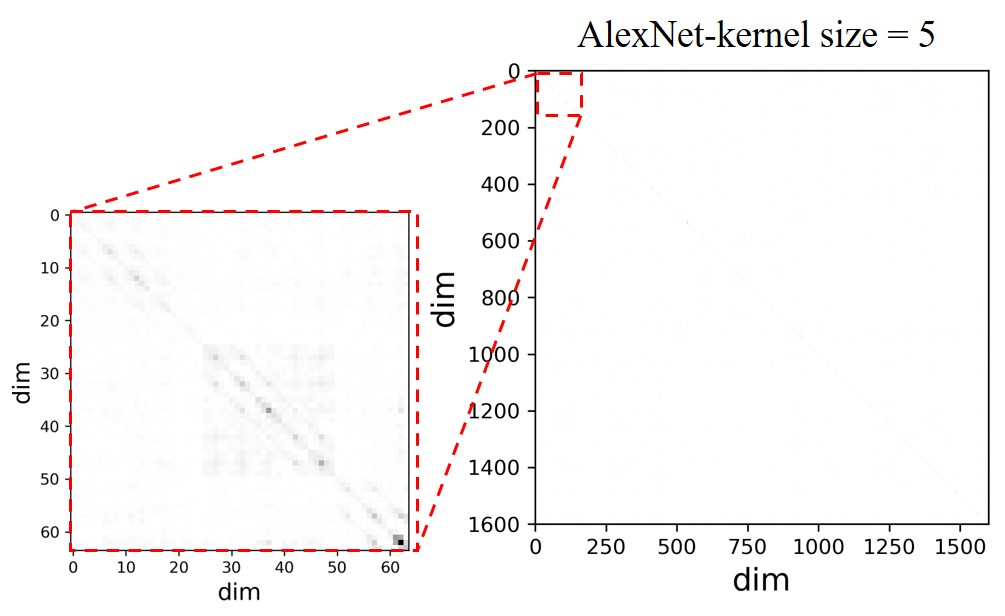}
	\label{fig:cwda_new1}
\end{figure}
\begin{figure}[h]
	\centering 
	\includegraphics[width=0.75\linewidth]{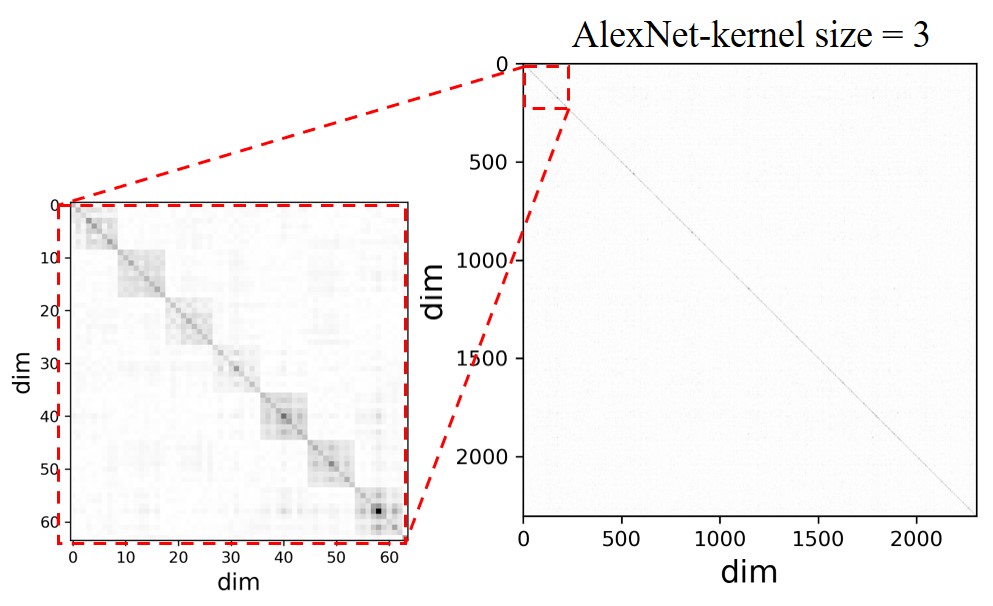}
	\label{fig:cwda_new1}
\end{figure}

\clearpage
\subsection{DenseNet}

\begin{figure}[h]
	\centering 
	\includegraphics[width=0.75\linewidth]{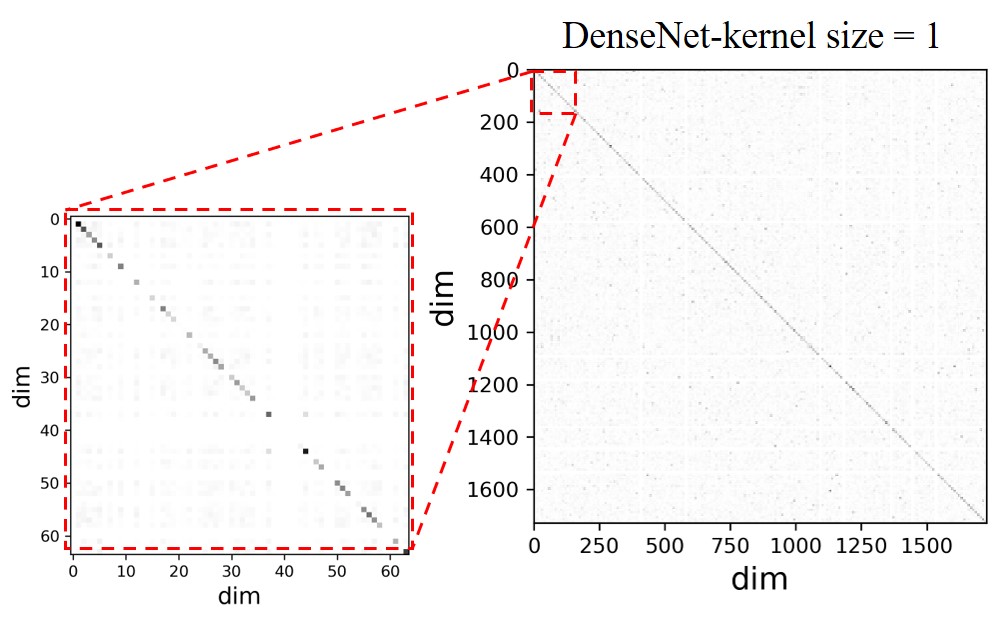}
	\label{fig:cwda_new1}
\end{figure}
\begin{figure}[h]
	\centering 
	\includegraphics[width=0.75\linewidth]{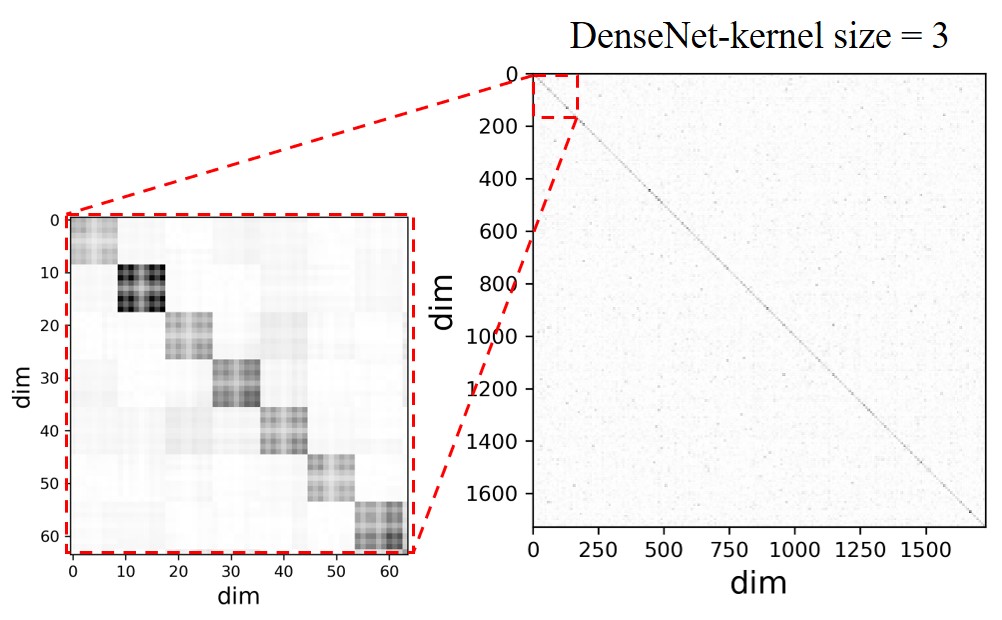}
	\label{fig:cwda_new1}
\end{figure}

\clearpage
\subsection{ResNext}

\begin{figure}[h]
	\centering 
	\includegraphics[width=0.75\linewidth]{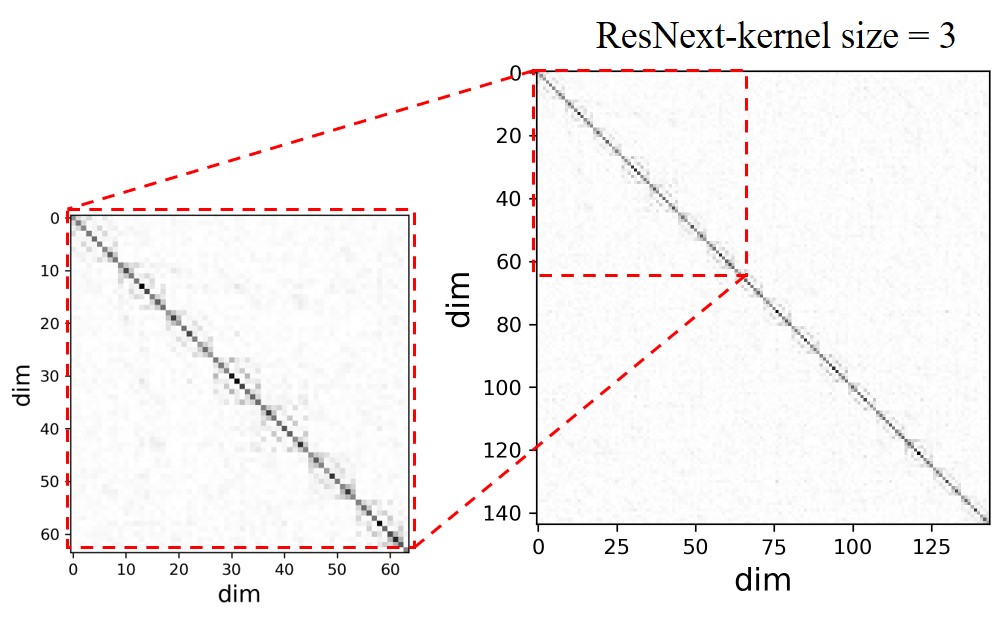}
	\label{fig:cwda_new1}
\end{figure}
\begin{figure}[h]
	\centering 
	\includegraphics[width=0.75\linewidth]{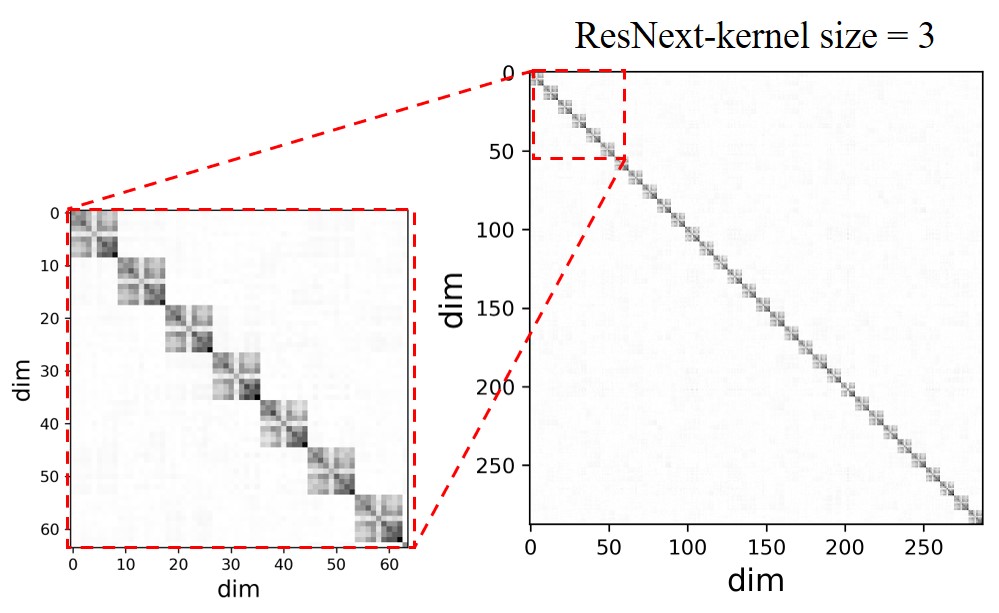}
	\label{fig:cwda_new1}
\end{figure}

\clearpage
\subsection{MobileNet}

\begin{figure}[h]
	\centering 
	\includegraphics[width=0.75\linewidth]{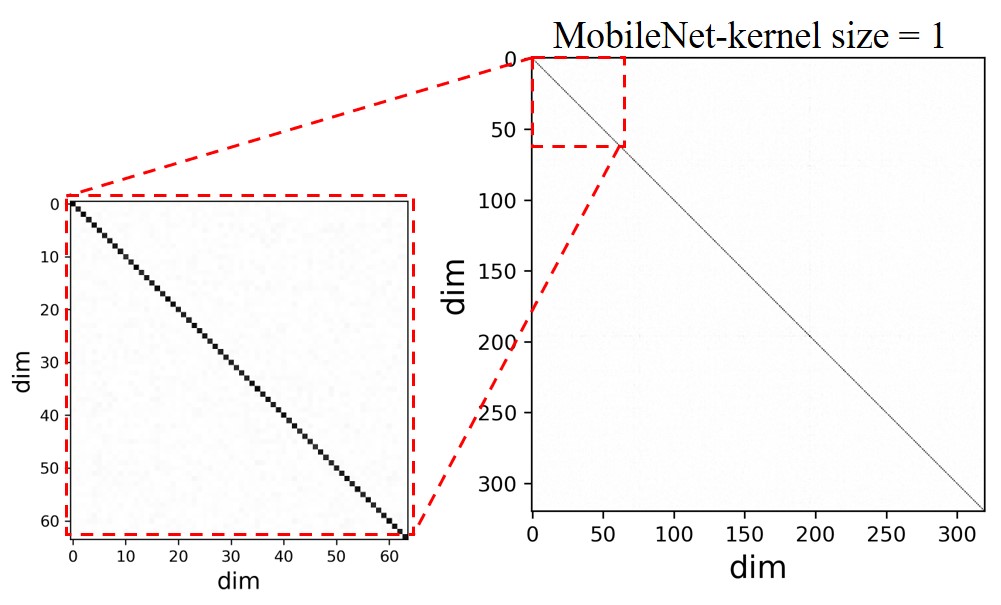}
	\label{fig:cwda_new1}
\end{figure}
\begin{figure}[h]
	\centering 
	\includegraphics[width=0.75\linewidth]{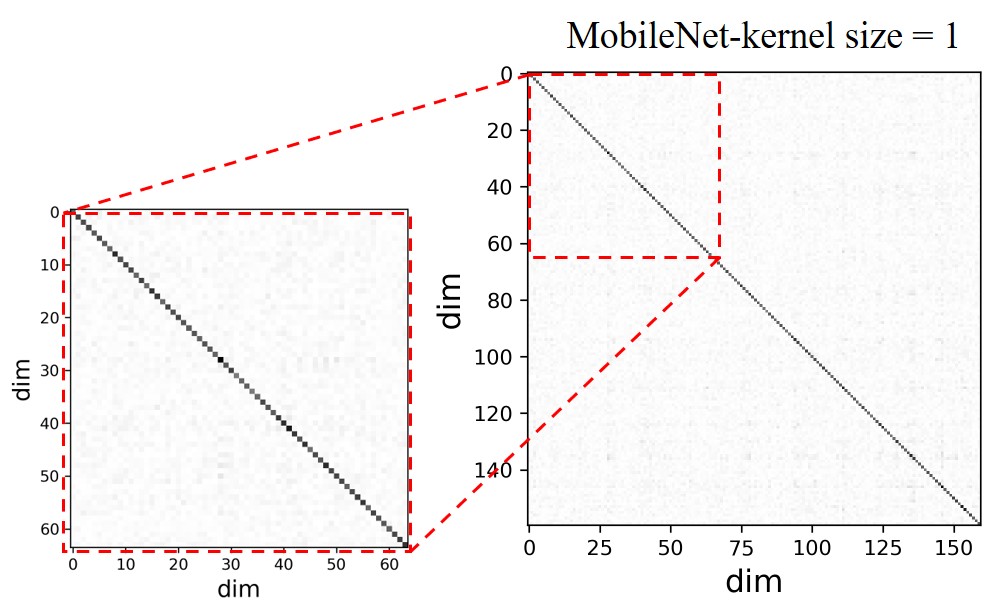}
	\label{fig:cwda_new1}
\end{figure}

\clearpage
\section{More experiments for supporting our analysis in global pruning}
\label{app:support}

\begin{figure}[h]
	\centering 
	\includegraphics[width=1\linewidth]{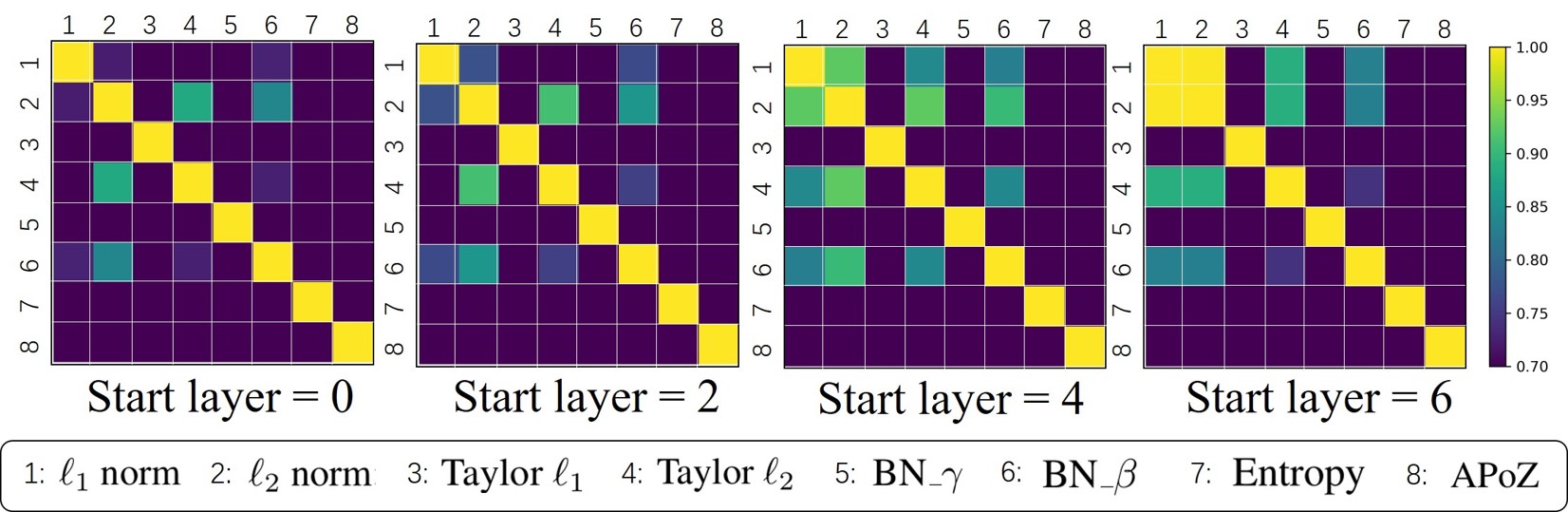} 
	\caption{Global pruning with different start layer.} 
	\label{more_exp}
\end{figure}

For VGG16. As shown in Fig.\ref{fig:magnitude}~(a-b), compared with ResNet56, VGG16 has some layers with different dimensions but similar \textit{Importance Score} measured by $\ell_1$ or $\ell_2$, such as ``layer 2'' and ``layer 8'' for $\ell_2$ criterion in Fig.\ref{fig:magnitude}~(a). From Table~\ref{tab:twolayer_global}~(3-4), these pairs of layers make the Sp small, which explain why the result of $\ell_1$ and $\ell_2$ pruning is not similar in Fig.~\ref{fig:other_simi}~(e) for VGG16. 
We consider a special class of global pruning, \textit{i.e.,} the convolutional filters from  one middle layer~(called ``Start layer'') to the last layer are pruned globally. According to our analysis and Fig.\ref{fig:magnitude}~(a-b), we can deduce that when ``Start layer'' $\geq 4$, the Sp between $\ell_1$ and $\ell_2$ is large enough. The experiments in Fig.\ref{more_exp} are consistent with our analysis, which imply our analysis is reasonable.

\clearpage    
\section{Statistical Test}
\label{app:Statistical Test}

In this section, according to Section \ref{Statistical test}, we have a series of statistical tests for the necessary conditions of CWDA. let $F_{ij}\in \mathbb{R}^{N_i\times k\times k}$ represent the $j^{\rm th}$ filter of the $i^{\rm th}$ convolutional layer.\footnote{The statistical tests about the situation with or without weight decay can be found in Appendix~\ref{app:weight_decay}.}

(1)~\textbf{Gaussian}.~We verify whether $F_{ij}$ approximatively follow a Gaussian-alike distribution. In $i^{\rm th}$ layer, we use Kolmogorov–Smirnov~(KS) test~\cite{lilliefors1967kolmogorov} to check if all the weights in the same layer follow a normal distribution.   

(2)~\textbf{Variance}.~ We verify whether the variance of the diagonal elements of $\Sigma_{\text{diag}}$ are small enough. Since Appendix~\ref{app:relax}, Let $\sigma_j$ denotes the standard deviation of all the weights of filter $F_{ij}$ in $i^{\rm th}$ layer. We use Student's t test~\cite{efron1969student} to check if the variance of these $\sigma_j$ is small enough. The null hypothesis $H_0$ and the alternative hypothesis $H_1$ are:
$$ H_0: \mathbf{Var}(\sigma_1^2,\sigma_2^2,..,\sigma_{N_i}^2)\leq \sigma_0^2,\quad \quad H_1: \mathbf{Var}(\sigma_1^2,\sigma_2^2,..,\sigma_{N_i}^2)> \sigma_0^2. $$
where $N_i$ denotes the number of the filters in $i^{\rm th}$ layer and $\sigma_0$ is a given real number which is 
small enough, like $\sigma_0^2 = 0.0001$.

(3)~\textbf{Mean}.~We verify whether the mean of $F_{ij}$ is 0. Let the mean of all the weights in the same layer is $\mu$. We use Student's t test~\cite{efron1969student} to check if $\mu$ is close to 0. First, we check the upper bound~(Mean-Left) of $\mu$, \textit{i.e.},
$$ H_0: \mu \leq \epsilon_0,\quad \quad H_1: \mu > \epsilon_0. $$
where $\epsilon_0$ is a small constant, like $\epsilon_0 = 0.01$. Next, we check the lower bound~(Mean-Right) and the null hypothesis $H_0$ and the alternative hypothesis $H_1$ are:
$$ H_0: \mu \geq -\epsilon_0,\quad \quad H_1: \mu < -\epsilon_0. $$

(4)~\textbf{Magnitude}.~We verify whether $\epsilon$ is small enough. Let $h$ denote the mean of the off-diagonal elements of $\mathbf{\Sigma}_{\text{diag}} + \epsilon\cdot\mathbf{\Sigma}_{\text{block}}$. 
$$ H_0: h \leq \epsilon_0,\quad \quad H_1: h > \epsilon_0. $$

	\begin{table*}[htbp]
		\centering
		\small
		\caption{The experiments for having the comprehensive statistical tests on CWDA.}
			\begin{tabular}{|l|l|l|}
				\hline
				\textbf{NETWORK STRUCTURE} & \textbf{OPTIMIZER} & \textbf{REGULARIZATION} \\
				\hline
				ResNet~\cite{He_2016_CVPR} & SGD~\cite{sutskever2013importance}   & L1 norm \\
				VGG~\cite{simonyan2014very}   & ASGD~\cite{polyak1992acceleration}  & L2 norm \\
				AlexNet~\cite{krizhevsky2014one} & Adam~\cite{kingma2014adam}  & RReLu~\cite{xu2015empirical} \\
				DenseNet~\cite{huang2017densely} & Adagrad~\cite{duchi2011adaptive} & Dropact~\cite{liang2018drop} \\
				PreResNet~\cite{he2016identity} & Adamax~\cite{kingma2014adam} & Autoaug~\cite{cubuk2019autoaugment} \\
				WRN~\cite{zagoruyko2016wide}   & Adadelta~\cite{zeiler2012adadelta} & Cutout~\cite{devries2017improved} \\
				ResNext~\cite{xie2017aggregated} &       & Cutmix~\cite{yun2019cutmix} \\
				\hline
				\textbf{ATTENTION MECHANISM} & \textbf{INITIALIZATION} & \textbf{DATASET} \\
				\hline
				SENet~\cite{hu2018squeeze} & Kaiming-normal~\cite{he2015delving} & CIFAR10~\cite{krizhevsky2009learning} \\
				DIANet~\cite{huang2019dianet} & Kaiming-uniform~\cite{he2015delving} & CIFAR100~\cite{krizhevsky2009learning} \\
				SRMNet~\cite{lee2019srm} & Xavier-normal~\cite{glorot2010understanding} & ImageNet~\cite{russakovsky2015imagenet} \\
				CBAM~\cite{woo2018cbam}  & Xavier-uniform~\cite{glorot2010understanding} & MNIST~\cite{lecun1998gradient} \\
				IEBN~\cite{liang2019instance}  & Orthogonal~\cite{saxe2013exact} &  \\
				SGENet~\cite{li2019spatial} &       &  \\
				\hline
				\textbf{SEGMENTATION} & \textbf{DETECTION} &  \textbf{BATCH NORMALIZATION}\\
				\hline
				SegNet~\cite{badrinarayanan2017segnet}& Faster RCNN~\cite{ren2015faster} & VGG \\
				PSPNet~\cite{zhao2017pyramid} &       & VGG-bn \\
				\hline
				\textbf{PYTORCH PRETRAIN} & \textbf{MATTING} & \textbf{LEARNING RATE} \\
				\hline
				ResNet18/34/50 & Deep image matting~\cite{xu2017deep} & Schedule150-225 \\
				VGG11/16/19 & AlphaGAN matting~\cite{lutz2018alphagan} & Schedule82-164 \\
				\cline{1-2}    \textbf{STYLE TRANSFER} & \textbf{GAN} & Schedule60-120 \\
				\cline{1-2}    Fast neural style~\cite{johnson2016perceptual} & DCGAN~\cite{radford2015unsupervised} & Cos-lr~\cite{loshchilov2016sgdr} \\
				\hline
			\end{tabular}%
		\label{tab:expassumption}
		\vspace{-0.4cm}
	\end{table*}%



\clearpage
Next, we show the passing rate about the statistical tests for different situations. ``in the front of network'' denotes whether all the failed cases are the layers whose position is in the front of the network.

For Network structure: \url{https://github.com/bearpaw/pytorch-classification}.

\begin{table}[htbp]
\small
  \centering
  \caption{\textbf{Network structure}.}
    \begin{tabular}{llccccc}
    \hline
    Experiments & Remark & \multicolumn{1}{l}{Gaussian} & \multicolumn{1}{l}{Variance} & \multicolumn{1}{l}{Mean} & \multicolumn{1}{l}{Magnitude} & in the front of network? \\
    \hline
    ResNet164 & CIFAR100 & 98.77\% & 97.55\% & 100\% & 97.55\% & {\color{green}{\Checkmark}}   \\
    VGG16 & CIFAR100 & 100\% & 93.75\% & 100\% & 100\% & {\color{green}{\Checkmark}}   \\
    AlexNet & CIFAR100 & 100\% & 100\% & 100\% & 100\% & {\color{green}{\Checkmark}}   \\
    DenseNet-BC-100-12 & CIFAR100 & 100\% & 98.99\% & 100\% & 98.99\% & {\color{green}{\Checkmark}}   \\
    PreResNet110 & CIFAR100 & 100\% & 99.08\% & 100\% & 100\% & {\color{green}{\Checkmark}}   \\
    WRN28-10 & CIFAR100 & 100\% & 100\% & 100\% & 100\% & {\color{green}{\Checkmark}}   \\
    ResNext-16x64d & CIFAR100 & 100\% & 100\% & 100\% & 100\% & {\color{green}{\Checkmark}}   \\
    \hline
    ResNet164 & CIFAR10 & 100.00\% & 97.55\% & 100\% & 97.55\% & {\color{green}{\Checkmark}}   \\
    VGG16 & CIFAR10 & 100\% & 93.75\% & 100\% & 93.75\% & {\color{green}{\Checkmark}}   \\
    AlexNet & CIFAR10 & 100\% & 100\% & 100\% & 100\% & {\color{green}{\Checkmark}}   \\
    DenseNet-BC-100-12 & CIFAR10 & 100\% & 100\% & 100\% & 98.99\% & {\color{green}{\Checkmark}}   \\
    PreResNet110 & CIFAR10 & 100\% & 99.08\% & 100\% & 100\% & {\color{green}{\Checkmark}}   \\
    WRN28-10 & CIFAR10 & 100\% & 100\% & 100\% & 100\% & {\color{green}{\Checkmark}}   \\
    ResNext-16x64d & CIFAR10 & 100\% & 100\% & 100\% & 100\% & {\color{green}{\Checkmark}}   \\
    \hline
    \end{tabular}%
  \label{tab:addlabel}%
\end{table}%

For Optimizer: \url{https://pytorch.org/docs/master/optim.html#torch-optim.}

\begin{table}[htbp]
  \centering
  \small
  \caption{\textbf{Optimizer}}
    \begin{tabular}{lcccccc}
    \hline
    Experiments & \multicolumn{1}{l}{Remark} & \multicolumn{1}{l}{Gaussian} & \multicolumn{1}{l}{Variance} & \multicolumn{1}{l}{Mean} & \multicolumn{1}{l}{Magnitude} & \multicolumn{1}{l}{in the front of network?} \\
    \hline
    ASGD  & ResNet164 & 100\% & 99.39\% & 99.39\% & 100\% & {\color{green}{\Checkmark}}   \\
    Adam  & ResNet164 & 99.39\% & 90.18\% & 100\% & 99.39\% & {\color{red}{\XSolidBrush}} \\
    Adagrad & ResNet164 & 100\% & 99.39\% & 100\% & 100\% & {\color{green}{\Checkmark}}   \\
    Adamax & ResNet164 & 100\% & 96.93\% & 100\% & 99.39\% & {\color{red}{\XSolidBrush}} \\
    Adadelta & ResNet164 & 100\% & 100\% & 100\% & 100\% & {\color{green}{\Checkmark}}   \\
    SGD   & ResNet164 & 98.77\% & 97.55\% & 100\% & 97.53\% & {\color{green}{\Checkmark}}   \\
    \hline
    ASGD  & VGG16 & 100\% & 100\% & 93.75\% & 100\% & {\color{green}{\Checkmark}}   \\
    Adam  & VGG16 & 93.75\% & 93.75\% & 100\% & 100.00\% & {\color{green}{\Checkmark}}   \\
    Adagrad & VGG16 & 100\% & 100\% & 100\% & 100\% & {\color{green}{\Checkmark}}   \\
    Adamax & VGG16 & 100\% & 100\% & 100\% & 93.75\% & {\color{red}{\XSolidBrush}} \\
    Adadelta & VGG16 & 100\% & 100\% & 100\% & 100\% & {\color{green}{\Checkmark}}   \\
    SGD   & VGG16 & 100\% & 93.75\% & 100\% & 100\% & {\color{green}{\Checkmark}}   \\
    \hline
    ASGD  & AlexNet & 100\% & 100\% & 100\% & 100\% & {\color{green}{\Checkmark}}   \\
    Adam  & AlexNet & 100\% & 100\% & 100\% & 100\% & {\color{green}{\Checkmark}}   \\
    Adagrad & AlexNet & 100\% & 100\% & 100\% & 100\% & {\color{green}{\Checkmark}}   \\
    Adamax & AlexNet & 100\% & 100\% & 100\% & 100\% & {\color{green}{\Checkmark}}   \\
    Adadelta & AlexNet & 100\% & 100\% & 100\% & 100\% & {\color{green}{\Checkmark}}   \\
    SGD   & AlexNet & 100\% & 100\% & 100\% & 100\% & {\color{green}{\Checkmark}}   \\
    \hline
    \end{tabular}%
  \label{tab:addlabel}%
\end{table}%

For Regularization:\url{https://github.com/LeungSamWai/Drop-Activation}

\url{https://github.com/uoguelph-mlrg/Cutout}

\url{https://github.com/clovaai/CutMix-PyTorch}

\url{https://github.com/DeepVoltaire/AutoAugment}

\begin{table}[htbp]
  \centering
  \caption{\textbf{Regularization}}
    \begin{tabular}{lcccccc}
    \hline
    Experiments & \multicolumn{1}{l}{Remark} & \multicolumn{1}{l}{Gaussian} & \multicolumn{1}{l}{Variance} & \multicolumn{1}{l}{Mean} & \multicolumn{1}{l}{Magnitude} & in the front of network? \\
    \hline
    L1 norm & ResNet164 & 100\% & 99.39\% & 99.39\% & 100\% & {\color{green}{\Checkmark}} \\
    L2 norm & ResNet164 & 98.77\% & 97.53\% & 100\% & 97.53\% & {\color{green}{\Checkmark}} \\
    RReLU & ResNet164 & 100\% & 99.39\% & 100\% & 100\% & {\color{green}{\Checkmark}} \\
    Dropact & ResNet164 & 100\% & 96.93\% & 100\% & 99.39\% & {\color{green}{\Checkmark}} \\
    Autoaugment & ResNet164 & 100\% & 96.93\% & 100\% & 99.39\% & {\color{green}{\Checkmark}} \\
    Cutout & ResNet164 & 100\% & 100\% & 100\% & 100\% & {\color{green}{\Checkmark}} \\
    Cutmix & ResNet164 & 98.77\% & 97.53\% & 100\% & 97.53\% & {\color{green}{\Checkmark}} \\
    \hline
    L1 norm & WRN28-10 & 100\% & 96.43\% & 100\% & 96.43\% & {\color{green}{\Checkmark}} \\
    L2 norm & WRN28-10 & 100\% & 100\% & 100\% & 100\% & {\color{green}{\Checkmark}} \\
    RReLU & WRN28-10 & 100\% & 96.43\% & 100\% & 100\% & {\color{green}{\Checkmark}} \\
    Dropact & WRN28-10 & 100\% & 96.43\% & 100\% & 100\% & {\color{green}{\Checkmark}} \\
    Autoaugment & WRN28-10 & 100\% & 96.43\% & 100\% & 100\% & {\color{green}{\Checkmark}} \\
    Cutout & WRN28-10 & 100\% & 96.43\% & 100\% & 100\% & {\color{green}{\Checkmark}} \\
    Cutmix & WRN28-10 & 100\% & 100\% & 100\% & 100\% & {\color{green}{\Checkmark}} \\
    \hline
    L1 norm & VGG16 & 100\% & 93.75\% & 100\% & 100\% & {\color{green}{\Checkmark}} \\
    L2 norm & VGG16 & 100\% & 93.75\% & 100\% & 100\% & {\color{green}{\Checkmark}} \\
    RReLU & VGG16 & 100\% & 93.75\% & 100\% & 93.75\% & {\color{green}{\Checkmark}} \\
    Dropact & VGG16 & 100\% & 93.75\% & 100\% & 100\% & {\color{green}{\Checkmark}} \\
    Autoaugment & VGG16 & 100\% & 93.75\% & 100\% & 100\% & {\color{green}{\Checkmark}} \\
    Cutout & VGG16 & 100\% & 93.75\% & 93.75\% & 93.75\% & {\color{green}{\Checkmark}} \\
    Cutmix & VGG16 & 100\% & 93.75\% & 100\% & 100\% & {\color{green}{\Checkmark}} \\
    \hline
    L1 norm & PreResNet110 & 100\% & 99.08\% & 100\% & 100\% & {\color{green}{\Checkmark}} \\
    L2 norm & PreResNet110 & 100\% & 99.08\% & 100\% & 100\% & {\color{green}{\Checkmark}} \\
    RReLU & PreResNet110 & 100\% & 100\% & 100\% & 100\% & {\color{green}{\Checkmark}} \\
    Dropact & PreResNet110 & 100\% & 99.08\% & 100\% & 100\% & {\color{green}{\Checkmark}} \\
    Autoaugment & PreResNet110 & 100\% & 100\% & 100\% & 100\% & {\color{green}{\Checkmark}} \\
    Cutout & PreResNet110 & 100\% & 99.08\% & 99.08\% & 99.08\% & {\color{green}{\Checkmark}} \\
    Cutmix & PreResNet110 & 100\% & 99.08\% & 100\% & 100\% & {\color{green}{\Checkmark}} \\
    \hline
    L1 norm & AlexNet & 100\% & 100\% & 100\% & 100\% & {\color{green}{\Checkmark}} \\
    L2 norm & AlexNet & 100\% & 100\% & 100\% & 100\% & {\color{green}{\Checkmark}} \\
    RReLU & AlexNet & 100\% & 100\% & 100\% & 100\% & {\color{green}{\Checkmark}} \\
    Dropact & AlexNet & 100\% & 100\% & 100\% & 100\% & {\color{green}{\Checkmark}} \\
    Autoaugment & AlexNet & 100\% & 100\% & 100\% & 100\% & {\color{green}{\Checkmark}} \\
    Cutout & AlexNet & 100\% & 100\% & 100\% & 100\% & {\color{green}{\Checkmark}} \\
    Cutmix & AlexNet & 100\% & 100\% & 100\% & 100\% & {\color{green}{\Checkmark}} \\
    \hline
    L1 norm & DenseNet-BC-100-12 & 100\% & 98.99\% & 100\% & 98.99\% & {\color{green}{\Checkmark}} \\
    L2 norm & DenseNet-BC-100-12 & 100\% & 98.99\% & 100\% & 98.99\% & {\color{green}{\Checkmark}} \\
    RReLU & DenseNet-BC-100-12 & 100\% & 98.99\% & 100\% & 98.99\% & {\color{green}{\Checkmark}} \\
    Dropact & DenseNet-BC-100-12 & 98.99\% & 98.99\% & 98.99\% & 98.99\% & {\color{green}{\Checkmark}} \\
    Autoaugment & DenseNet-BC-100-12 & 100\% & 98.99\% & 100\% & 98.99\% & {\color{green}{\Checkmark}} \\
    Cutout & DenseNet-BC-100-12 & 100\% & 98.99\% & 98.99\% & 98.99\% & {\color{green}{\Checkmark}} \\
    Cutmix & DenseNet-BC-100-12 & 100\% & 98.99\% & 100\% & 98.99\% & {\color{green}{\Checkmark}} \\
    \hline
    \end{tabular}%
  \label{tab:addlabel}%
\end{table}%

For Attention:\url{https://github.com/moskomule/senet.pytorch}

\url{https://github.com/gbup-group/DIANet}

\url{https://github.com/EvgenyKashin/SRMnet}

\url{https://github.com/luuuyi/CBAM.PyTorch}

\url{https://github.com/gbup-group/IEBN}

\url{https://github.com/implus/PytorchInsight}

\begin{table}[htbp]
  \centering
  \caption{\textbf{Attention}}
    \begin{tabular}{lcccccc}
    \hline
    Experiments & \multicolumn{1}{l}{Remark} & \multicolumn{1}{l}{Gaussian} & \multicolumn{1}{l}{Variance} & \multicolumn{1}{l}{Mean} & \multicolumn{1}{l}{Magnitude} & \multicolumn{1}{l}{in the front of network?} \\
    \hline
    SENet & ResNet164 & 99.39\% & 99.39\% & 100\% & 100\% & {\color{green}{\Checkmark}} \\
    DIANet & ResNet164 & 99.39\% & 99.39\% & 100\% & 100\% & {\color{green}{\Checkmark}} \\
    SRMNet & ResNet164 & 99.39\% & 97.55\% & 100\% & 99.39\% & {\color{green}{\Checkmark}} \\
    CBAM  & ResNet164 & 99.39\% & 99.39\% & 100\% & 100\% & {\color{green}{\Checkmark}} \\
    IEBN  & ResNet164 & 99.39\% & 99.39\% & 99.39\% & 99.39\% & {\color{green}{\Checkmark}} \\
    SGENet & ResNet164 & 99.39\% & 98.77\% & 100\% & 100\% & {\color{green}{\Checkmark}} \\
    \hline
    SENet & VGG16 & 100\% & 93.75\% & 100\% & 100\% & {\color{green}{\Checkmark}} \\
    DIANet & VGG16 & 100\% & 93.75\% & 100\% & 93.75\% & {\color{green}{\Checkmark}} \\
    SRMNet & VGG16 & 100\% & 100\% & 100\% & 100\% & {\color{green}{\Checkmark}} \\
    CBAM  & VGG16 & 100\% & 93.75\% & 100\% & 100\% & {\color{green}{\Checkmark}} \\
    IEBN  & VGG16 & 100\% & 93.75\% & 93.75\% & 93.75\% & {\color{green}{\Checkmark}} \\
    SGENet & VGG16 & 100\% & 93.75\% & 100\% & 100\% & {\color{green}{\Checkmark}} \\
    \hline
    SENet & PreResNet110 & 99.08\% & 100\% & 100\% & 100\% & {\color{green}{\Checkmark}} \\
    DIANet & PreResNet110 & 100\% & 99.08\% & 100\% & 100\% & {\color{green}{\Checkmark}} \\
    SRMNet & PreResNet110 & 100\% & 99.08\% & 99.08\% & 100\% & {\color{green}{\Checkmark}} \\
    CBAM  & PreResNet110 & 100\% & 100\% & 100\% & 100\% & - \\
    IEBN  & PreResNet110 & 100\% & 99.08\% & 100\% & 99.08\% & {\color{green}{\Checkmark}} \\
    SGENet & PreResNet110 & 100\% & 100\% & 100\% & 99.08\% & {\color{green}{\Checkmark}} \\
    \hline
    SENet & DenseNet-BC-100-12 & 100\% & 100\% & 100\% & 100\% & {\color{green}{\Checkmark}} \\
    DIANet & DenseNet-BC-100-12 & 98.99\% & 98.99\% & 100\% & 100\% & {\color{green}{\Checkmark}} \\
    SRMNet & DenseNet-BC-100-12 & 100\% & 98.99\% & 98.99\% & 98.99\% & {\color{green}{\Checkmark}} \\
    CBAM  & DenseNet-BC-100-12 & 100\% & 100\% & 100\% & 98.99\% & {\color{green}{\Checkmark}} \\
    IEBN  & DenseNet-BC-100-12 & 100\% & 98.99\% & 100\% & 100\% & {\color{green}{\Checkmark}} \\
    SGENet & DenseNet-BC-100-12 & 100\% & 100\% & 98.99\% & 100\% & {\color{green}{\Checkmark}} \\
    \hline
    SENet & WRN28-10 & 100\% & 96.43\% & 100\% & 100\% & {\color{green}{\Checkmark}} \\
    DIANet & WRN28-10 & 100\% & 96.43\% & 100\% & 100\% & {\color{green}{\Checkmark}} \\
    SRMNet & WRN28-10 & 100\% & 96.43\% & 100\% & 100\% & {\color{green}{\Checkmark}} \\
    CBAM  & WRN28-10 & 100\% & 96.43\% & 100\% & 100\% & {\color{green}{\Checkmark}} \\
    IEBN  & WRN28-10 & 100\% & 96.43\% & 100\% & 100\% & {\color{green}{\Checkmark}} \\
    SGENet & WRN28-10 & 100\% & 96.43\% & 100\% & 100\% & {\color{green}{\Checkmark}} \\
    \hline
    \end{tabular}%
  \label{tab:addlabel}%
\end{table}%

For initialization: 

\url{https://pytorch.org/docs/master/nn.init.html#nn-init-doc.}

\begin{table}[htbp]
  \centering
  \caption{\textbf{Initialization}}
    \begin{tabular}{lcccccc}
    \hline
    Experiments & \multicolumn{1}{l}{Remark} & \multicolumn{1}{l}{Gaussian} & \multicolumn{1}{l}{Variance} & \multicolumn{1}{l}{Mean} & \multicolumn{1}{l}{Magnitude} & \multicolumn{1}{l}{in the front of network?} \\
    \hline
    Kaiming-uniform & ResNet164 & 98.77\% & 97.55\% & 100\% & 100\% & {\color{green}{\Checkmark}} \\
    Kaiming-normal & ResNet164 & 98.77\% & 97.53\% & 100\% & 97.55\% & {\color{green}{\Checkmark}} \\
    Xavier-normal & ResNet164 & 98.77\% & 96.32\% & 100\% & 97.55\% & {\color{green}{\Checkmark}} \\
    Xarier-uniform & ResNet164 & 98.16\% & 96.32\% & 100\% & 99.39\% & {\color{green}{\Checkmark}} \\
    Orthogonal & ResNet164 & 97.55\% & 96.32\% & 100\% & 100\% & {\color{green}{\Checkmark}} \\
    \hline
    Kaiming-uniform & VGG16 & 100\% & 93.75\% & 100\% & 100\% & {\color{green}{\Checkmark}} \\
    Kaiming-normal & VGG16 & 100\% & 93.75\% & 100\% & 100\% & {\color{green}{\Checkmark}} \\
    Xavier-normal & VGG16 & 100\% & 93.75\% & 100\% & 93.75\% & {\color{green}{\Checkmark}} \\
    Xarier-uniform & VGG16 & 100\% & 93.75\% & 100\% & 93.75\% & {\color{green}{\Checkmark}} \\
    Orthogonal & VGG16 & 100\% & 93.75\% & 93.75\% & 93.75\% & {\color{green}{\Checkmark}} \\
    \hline
    Kaiming-uniform & WRN28-10 & 100\% & 96.43\% & 100\% & 100\% & {\color{green}{\Checkmark}} \\
    Kaiming-normal & WRN28-10 & 100\% & 100\% & 100\% & 100\% & {\color{green}{\Checkmark}} \\
    Xavier-normal & WRN28-10 & 100\% & 96.43\% & 100\% & 100\% & {\color{green}{\Checkmark}} \\
    Xarier-uniform & WRN28-10 & 100\% & 96.43\% & 100\% & 100\% & {\color{green}{\Checkmark}} \\
    Orthogonal & WRN28-10 & 100\% & 96.43\% & 100\% & 100\% & {\color{green}{\Checkmark}} \\
    \hline
    Kaiming-uniform & PreResNet110 & 100\% & 99.08\% & 100\% & 100\% & {\color{green}{\Checkmark}} \\
    Kaiming-normal & PreResNet110 & 100\% & 99.08\% & 100\% & 100\% & {\color{green}{\Checkmark}} \\
    Xavier-normal & PreResNet110 & 100\% & 100\% & 100\% & 100\% & {\color{green}{\Checkmark}} \\
    Xarier-uniform & PreResNet110 & 100\% & 99.08\% & 100\% & 100\% & {\color{green}{\Checkmark}} \\
    Orthogonal & PreResNet110 & 100\% & 100\% & 100\% & 100\% & {\color{green}{\Checkmark}} \\
    \hline
    Kaiming-uniform & AlexNet & 100\% & 100\% & 100\% & 100\% & {\color{green}{\Checkmark}} \\
    Kaiming-normal & AlexNet & 100\% & 100\% & 100\% & 100\% & {\color{green}{\Checkmark}} \\
    Xavier-normal & AlexNet & 100\% & 100\% & 100\% & 100\% & {\color{green}{\Checkmark}} \\
    Xarier-uniform & AlexNet & 100\% & 100\% & 100\% & 100\% & {\color{green}{\Checkmark}} \\
    Orthogonal & AlexNet & 100\% & 100\% & 100\% & 100\% & {\color{green}{\Checkmark}} \\
    \hline
    Kaiming-uniform & DenseNet-BC-100-12 & 100\% & 98.99\% & 100\% & 98.99\% & {\color{green}{\Checkmark}} \\
    Kaiming-normal & DenseNet-BC-100-12 & 100\% & 98.99\% & 100\% & 98.99\% & {\color{green}{\Checkmark}} \\
    Xavier-normal & DenseNet-BC-100-12 & 100\% & 98.99\% & 100\% & 98.99\% & {\color{green}{\Checkmark}} \\
    Xarier-uniform & DenseNet-BC-100-12 & 98.99\% & 98.99\% & 98.99\% & 98.99\% & {\color{green}{\Checkmark}} \\
    Orthogonal & DenseNet-BC-100-12 & 100\% & 98.99\% & 100\% & 98.99\% & {\color{green}{\Checkmark}} \\
    \hline
    \end{tabular}%
  \label{tab:addlabel}%
\end{table}%

For dataset:
\begin{table}[htbp]
  \centering
  \caption{\textbf{Dataset}}
    \begin{tabular}{lcccccc}
    \hline
    Experiments & \multicolumn{1}{l}{Remark} & \multicolumn{1}{l}{Gaussian} & \multicolumn{1}{l}{Variance} & \multicolumn{1}{l}{Mean} & \multicolumn{1}{l}{Magnitude} & in the front of network? \\
    \hline
    CIFAR10 & WRN28-10 & 100\% & 96.43\% & 100\% & 100\% & {\color{green}{\Checkmark}} \\
    CIFAR100 & WRN28-10 & 100\% & 100\% & 100\% & 100\% &  {\color{green}{\Checkmark}} \\
    ImageNet & WRN28-10 & 100\% & 96.43\% & 100\% & 100\% & {\color{green}{\Checkmark}} \\
    MINIST & WRN28-10 & 100\% & 96.43\% & 100\% & 96\%  & {\color{green}{\Checkmark}} \\
    \hline
    \end{tabular}%
  \label{tab:addlabel}%
\end{table}%

For other tasks:

\url{https://github.com/meetshah1995/pytorch-semse}

\url{https://github.com/jwyang/faster-rcnn.pytorch}

\url{https://github.com/speedinghzl/pytorch-segmentation-toolbox}

\url{https://github.com/foamliu/Deep-Image-Matting-PyTorch}

\url{https://github.com/CDOTAD/AlphaGAN-Matting}

\url{https://github.com/abhiskk/fast-neural-style}

\url{https://github.com/csinva/gan-pretrained-pytorch}

\begin{table}[htbp]
  \centering
  \caption{\textbf{Other tasks}}
    \begin{tabular}{lcccccc}
    \hline
    Experiments & \multicolumn{1}{l}{Remark} & \multicolumn{1}{l}{Gaussian} & \multicolumn{1}{l}{Variance} & \multicolumn{1}{l}{Mean} & \multicolumn{1}{l}{Magnitude} & \multicolumn{1}{l}{in the front of network?} \\
    \hline
    SgeNet(Cityscapes) & Segmentation & 100\% & 100\% & 100\% & 100\% & {\color{green}{\Checkmark}} \\
    PSPNet(Cityscapes) & Segmentation & 100\% & 99.12\% & 100\% & 99.12\% & {\color{green}{\Checkmark}} \\
    \hline
    ResNet101(COCO) & Faster RCNN & 100\% & 99.05\% & 100\% & 100\% & {\color{red}{\XSolidBrush}} \\
    ResNet101(VOC2007) & Faster RCNN & 100\% & 99.05\% & 100\% & 100\% & {\color{red}{\XSolidBrush}} \\
    VGG16(Visual Genome) & Faster RCNN & 100\% & 93.75\% & 100\% & 100\% & {\color{green}{\Checkmark}} \\
    \hline
    AlphaGAN & Image matting & 100\% & 95.00\% & 100\% & 95.00\% & {\color{green}{\Checkmark}} \\
    Deep image matting & Image matting & 100\% & 100\% & 100\% & 100\% & {\color{green}{\Checkmark}} \\
    \hline
    Fast neural style  & candy & 86.67\% & 100\% & 100\% & 100\% & {\color{red}{\XSolidBrush}} \\
    Fast neural style  & mosaic & 93.33\% & 100\% & 100\% & 100\% & {\color{green}{\Checkmark}} \\
    Fast neural style  & starry night & 86.67\% & 100\% & 100\% & 100\% & {\color{red}{\XSolidBrush}} \\
    Fast neural style  & udnie & 66.67\% & 100\% & 100\% & 100\% & {\color{red}{\XSolidBrush}} \\
    \hline
    DCGAN(MNIST) & GAN   & 100\% & 100\% & 100\% & 100\% & {\color{green}{\Checkmark}} \\
    DCGAN(CIFAR10) & GAN   & 100\% & 100\% & 100\% & 100\% & {\color{green}{\Checkmark}} \\
    DCGAN(CIFAR100) & GAN   & 100\% & 100\% & 100\% & 100\% & {\color{green}{\Checkmark}} \\
    \hline
    VGG19(CIFAR10) & without BN & 100\% & 100\% & 100\% & 100\% & {\color{green}{\Checkmark}} \\
    VGG19(CIFAR10) & with BN & 93.75\% & 100\% & 100\% & 100\% & {\color{green}{\Checkmark}} \\
    VGG19(CIFAR10-lr) & schedule(82-164) & 93.75\% & 100\% & 100\% & 100\% & {\color{green}{\Checkmark}} \\
    VGG19(CIFAR10-lr) & schedule(60-120) & 93.75\% & 100\% & 100\% & 100\% & {\color{green}{\Checkmark}} \\
    VGG19(CIFAR10-lr) & coslr & 93.75\% & 100\% & 100\% & 100\% & {\color{green}{\Checkmark}} \\
    \hline
    \end{tabular}%
  \label{tab:addlabel}%
\end{table}%

For pytorch pretrain:\url{http://pytorch.org/docs/master/torchvision/index.html.}
\begin{table}[htbp]
  \centering
  \caption{\textbf{Pytorch pretrian}}
    \begin{tabular}{lcccccc}
    \hline
    Experiments & \multicolumn{1}{l}{Remark} & \multicolumn{1}{l}{Gaussian} & \multicolumn{1}{l}{Variance} & \multicolumn{1}{l}{Mean} & \multicolumn{1}{l}{Magnitude} & \multicolumn{1}{l}{in the front of network?} \\
    \hline
    VGG11 & ImageNet & 100\% & 75.00\% & 100\% & 75.00\% & {\color{green}{\Checkmark}}  \\
    VGG16 & ImageNet & 100\% & 84.62\% & 100\% & 100\% & {\color{green}{\Checkmark}}  \\
    VGG19 & ImageNet & 100\% & 87.50\% & 100\% & 100\% & {\color{green}{\Checkmark}}  \\
    ResNet18 & ImageNet & 100\% & 88.24\% & 100\% & 100\% & {\color{green}{\Checkmark}}  \\
    ResNet34 & ImageNet & 100\% & 88.24\% & 100\% & 96.97\% & {\color{green}{\Checkmark}}  \\
    ResNet50 & ImageNet & 100\% & 83.67\% & 100\% & 100\% & {\color{red}{\XSolidBrush}} \\
    \hline
    \end{tabular}%
  \label{tab:addlabel}%
\end{table}%

\clearpage
\section{Training through slimming}
\label{slimming}

\begin{figure} [htbp]
	\centering 
	\includegraphics[width=1\linewidth]{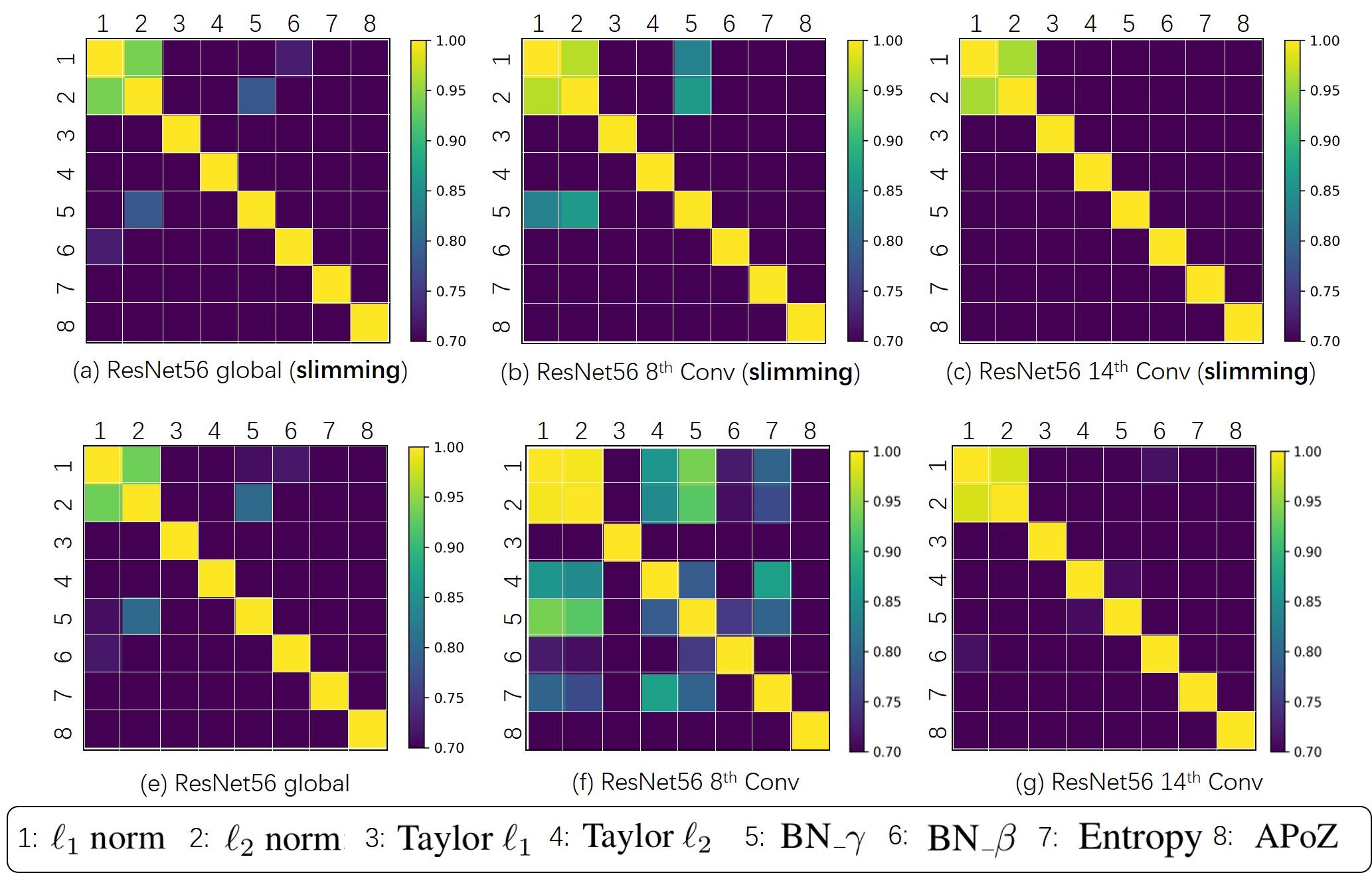}
	\caption{The Similarity for different criteria with/without slimming \cite{Liu_2017_ICCV}. } 
	\label{fig:slimming}
\end{figure}

As a representative of the BN-based pruning method, slimming pruning\cite{Liu_2017_ICCV} can not be directly compared with the criteria mentioned in the paper because it adopts a special training method. Therefore, we use the training method in \cite{Liu_2017_ICCV} to train another ResNet56 on cifar100. Then, the analysis of similarities between 8 different pruning criteria on such a model is shown in Fig.~\ref{fig:slimming}.

In this situation, the fifth criterion BN\_$\gamma$ is the method introduced in \cite{Liu_2017_ICCV}. From Fig.~\ref{fig:slimming}, there is no significant difference in the result of the similarity between ResNet56 obtained by slimming method and resnet56 trained in general.

\clearpage
\section{More experiments of Sp in Norm-based criteria}
\label{app:sp_network}

\begin{figure}[H]
	\centering 
	\includegraphics[width=0.4\linewidth]{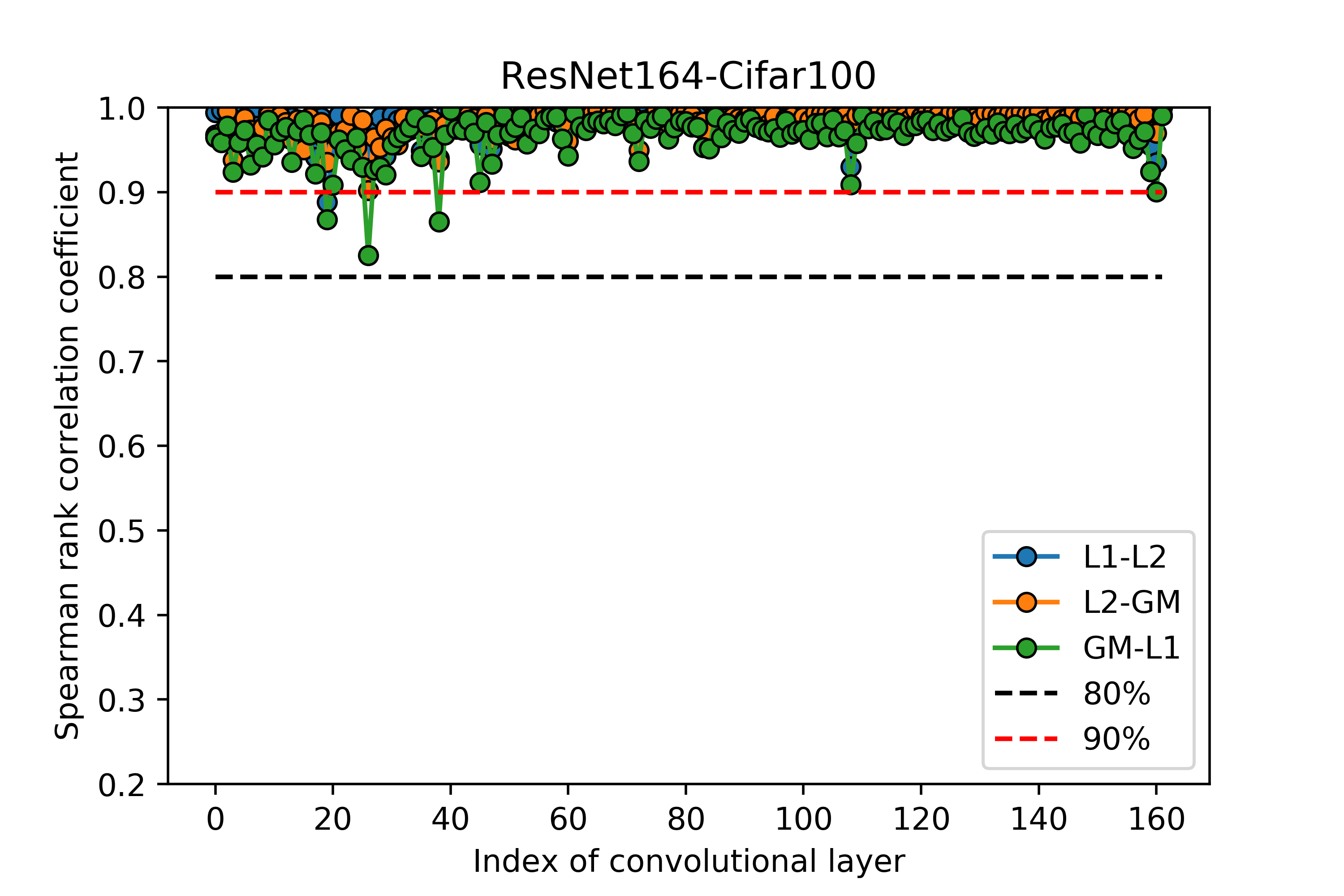} 
	\includegraphics[width=0.4\linewidth]{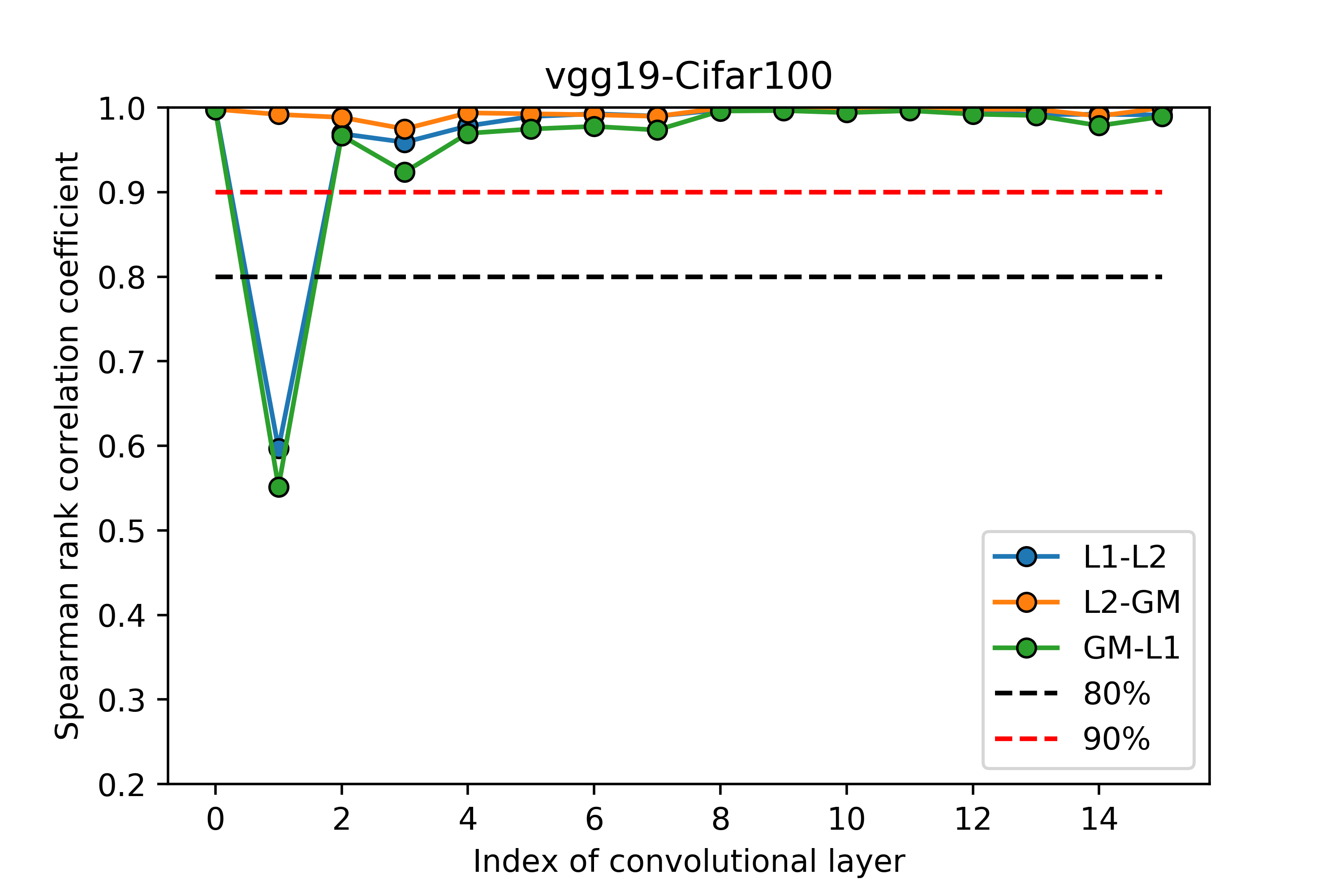} 
	\includegraphics[width=0.4\linewidth]{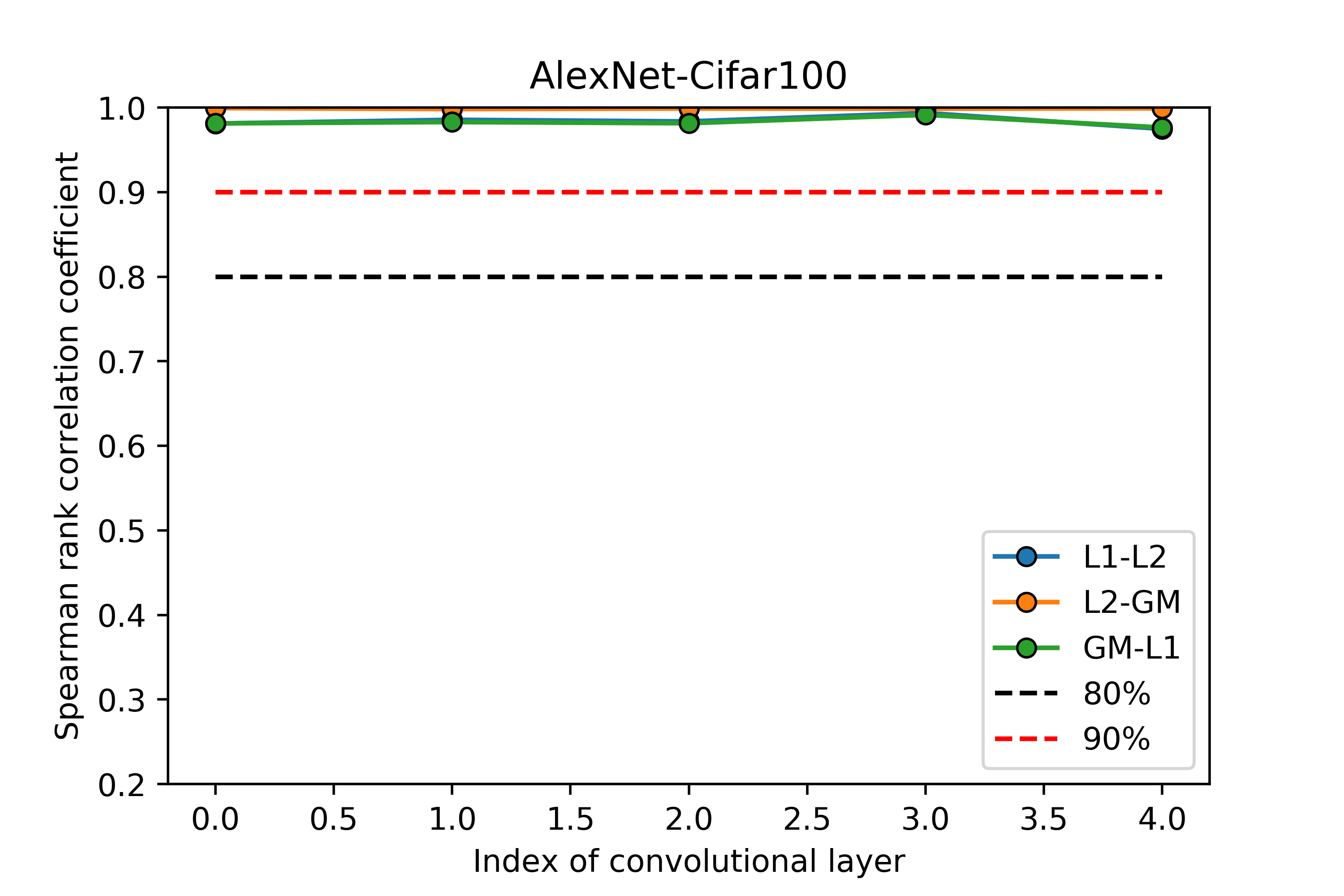} 
	\includegraphics[width=0.4\linewidth]{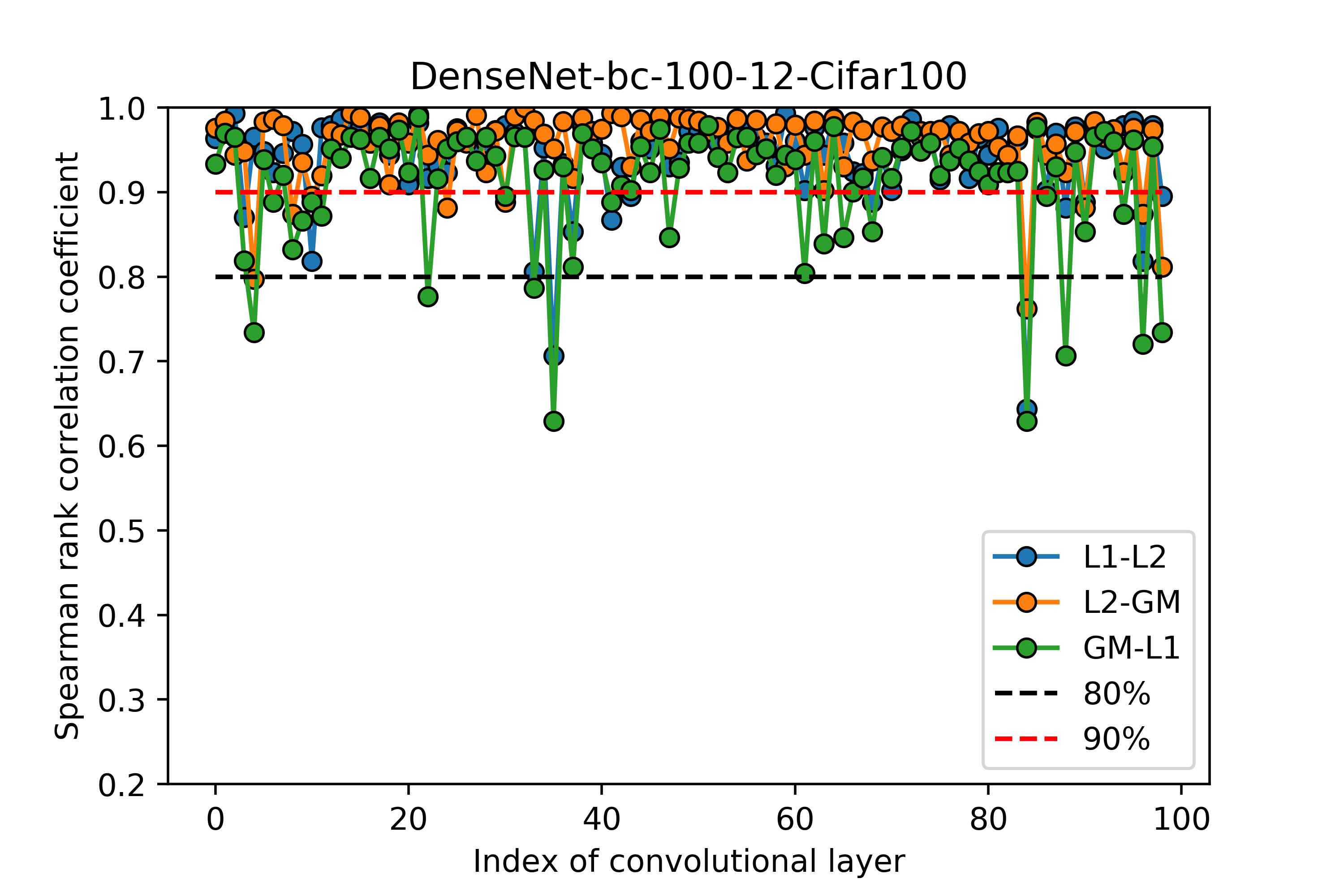} 
	\includegraphics[width=0.4\linewidth]{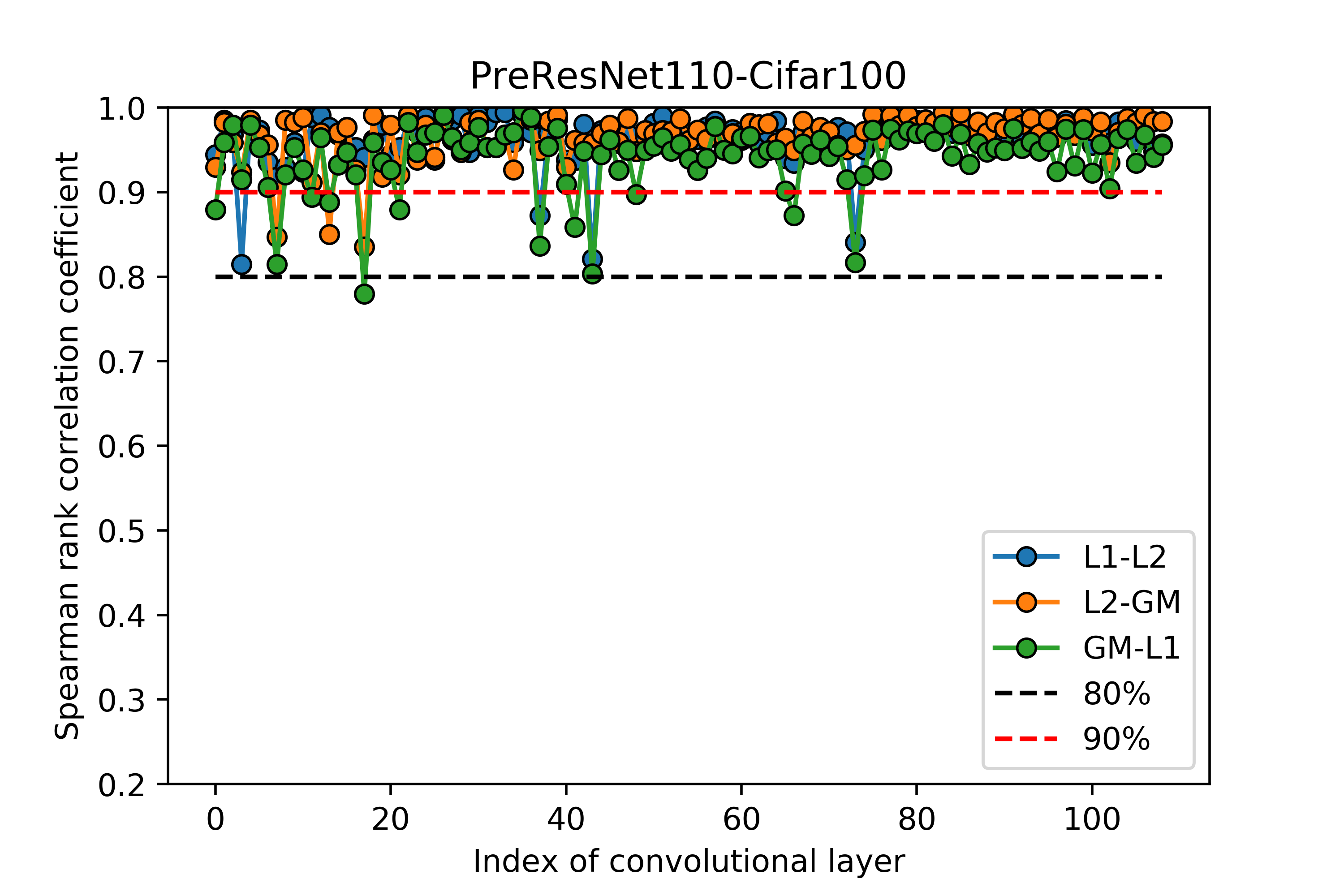} 
	\includegraphics[width=0.4\linewidth]{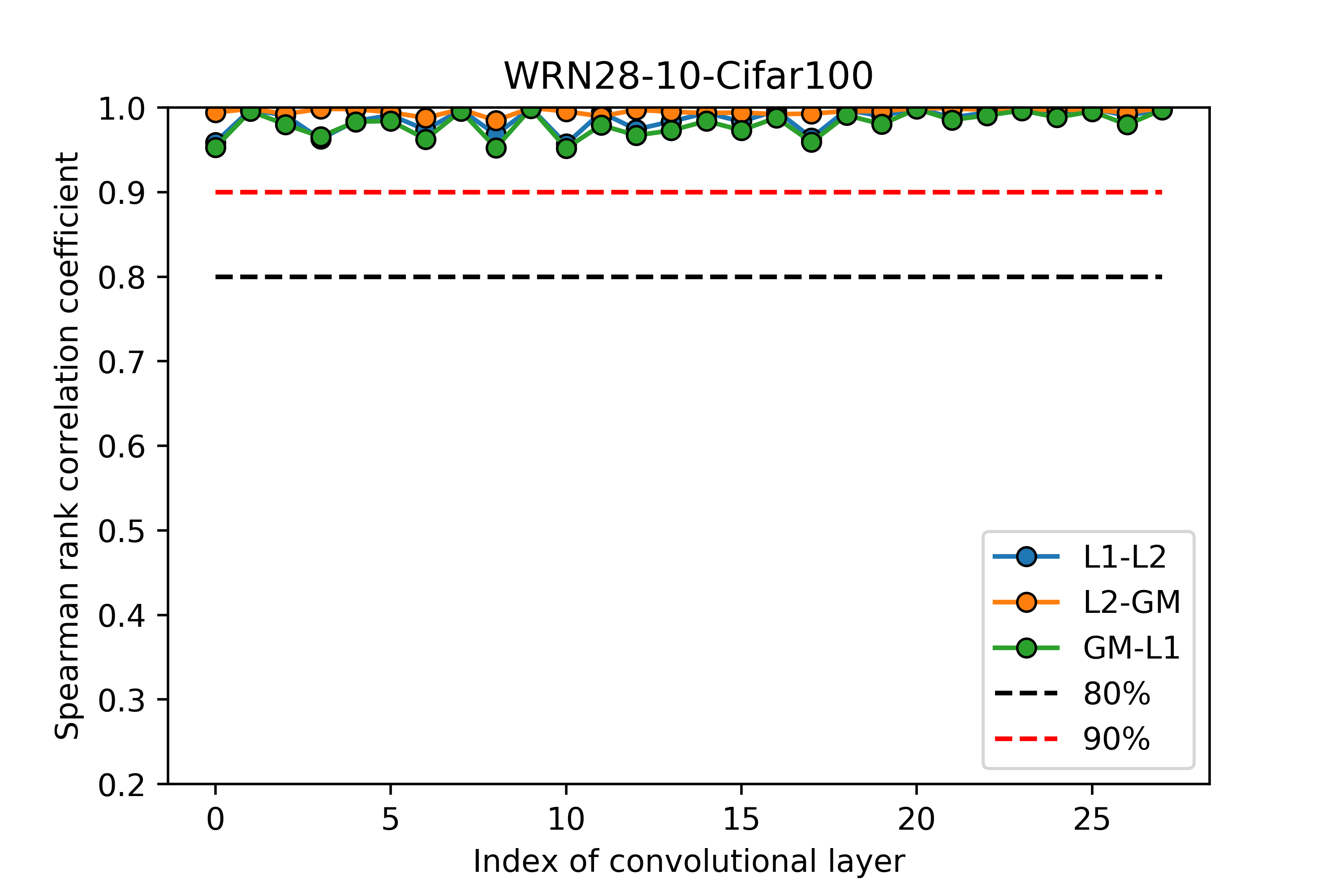} 
	\caption{Network Structure}
\end{figure}
\begin{figure}
	\centering 
	\includegraphics[width=0.3\linewidth]{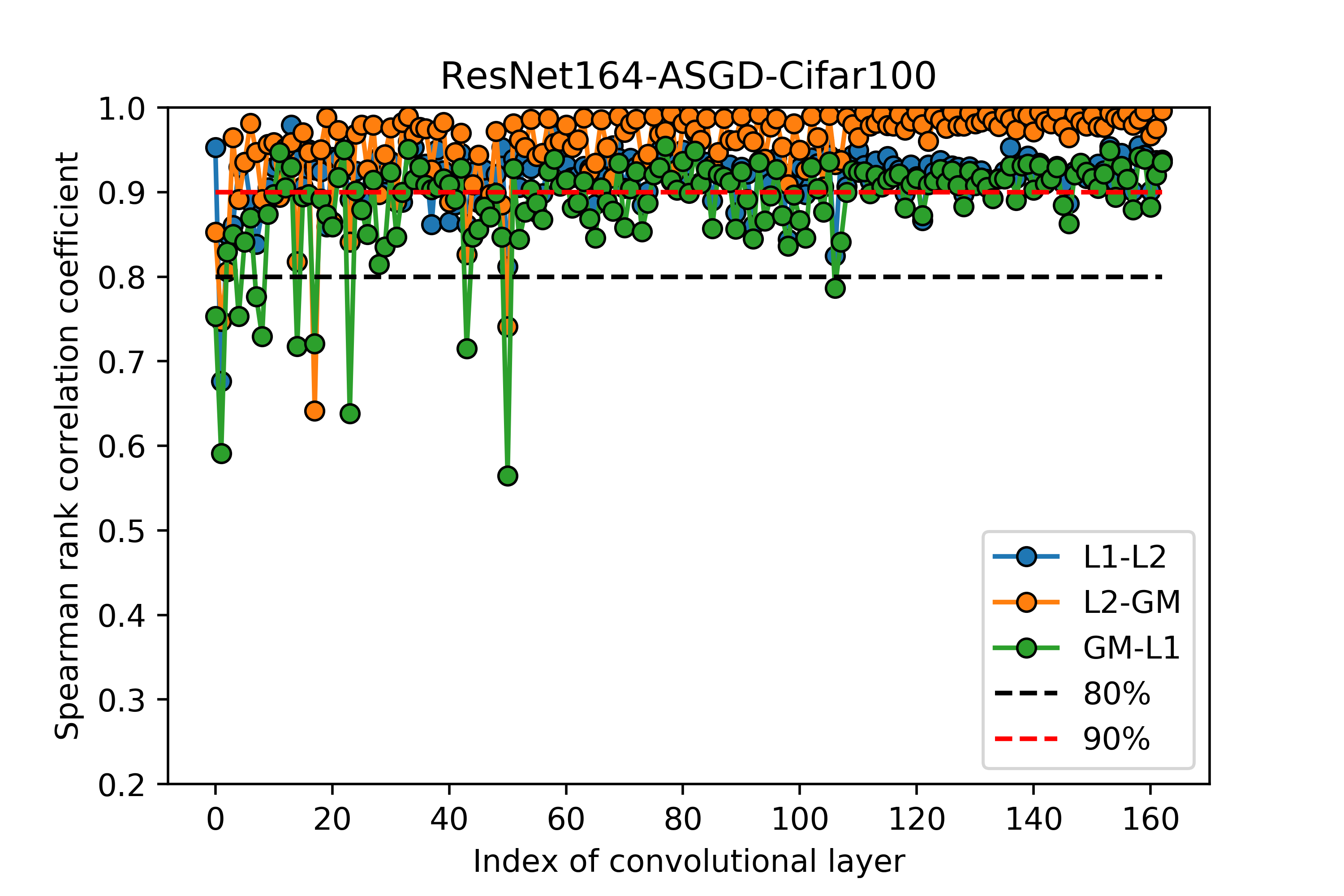} 
	\includegraphics[width=0.3\linewidth]{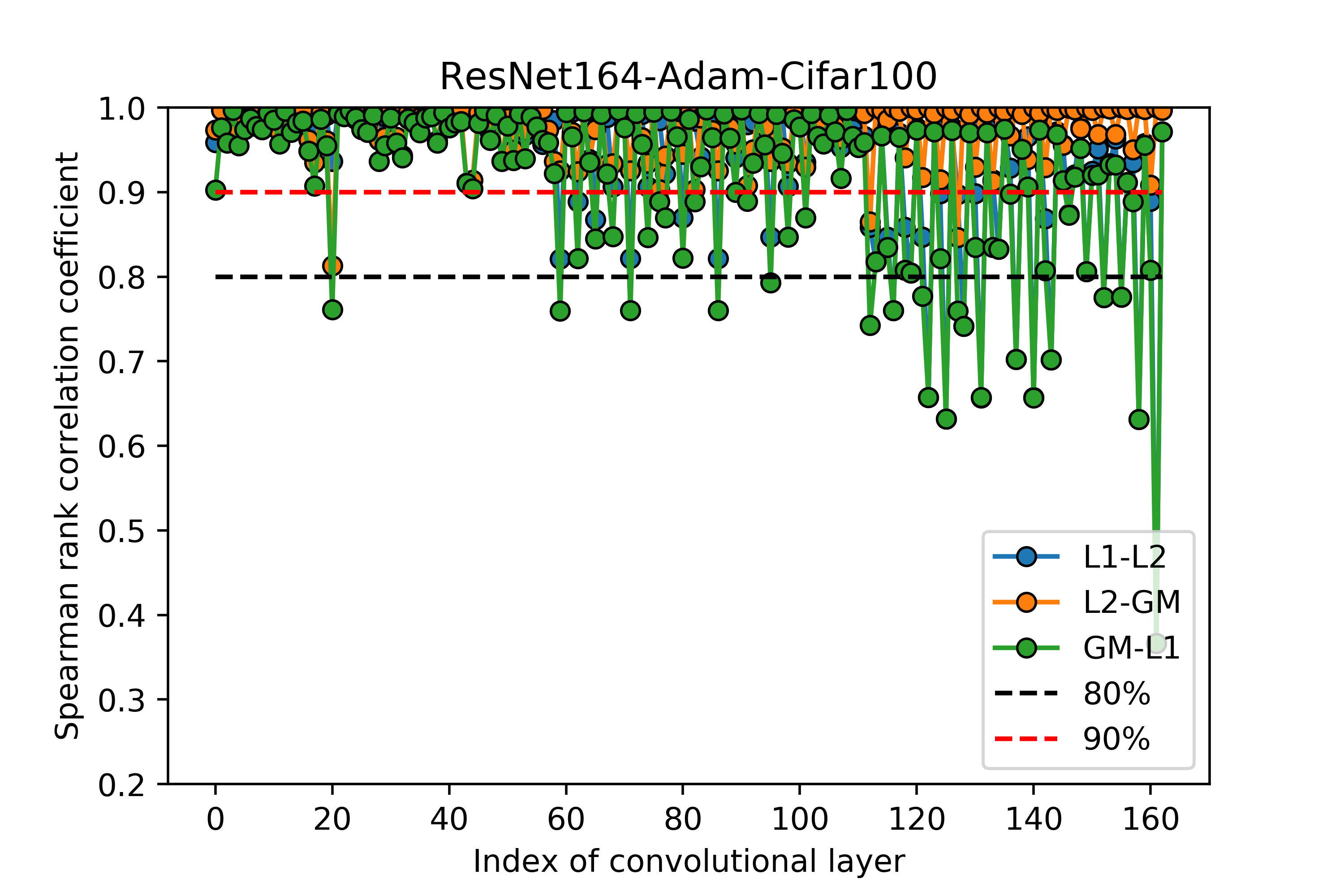} 
	\includegraphics[width=0.3\linewidth]{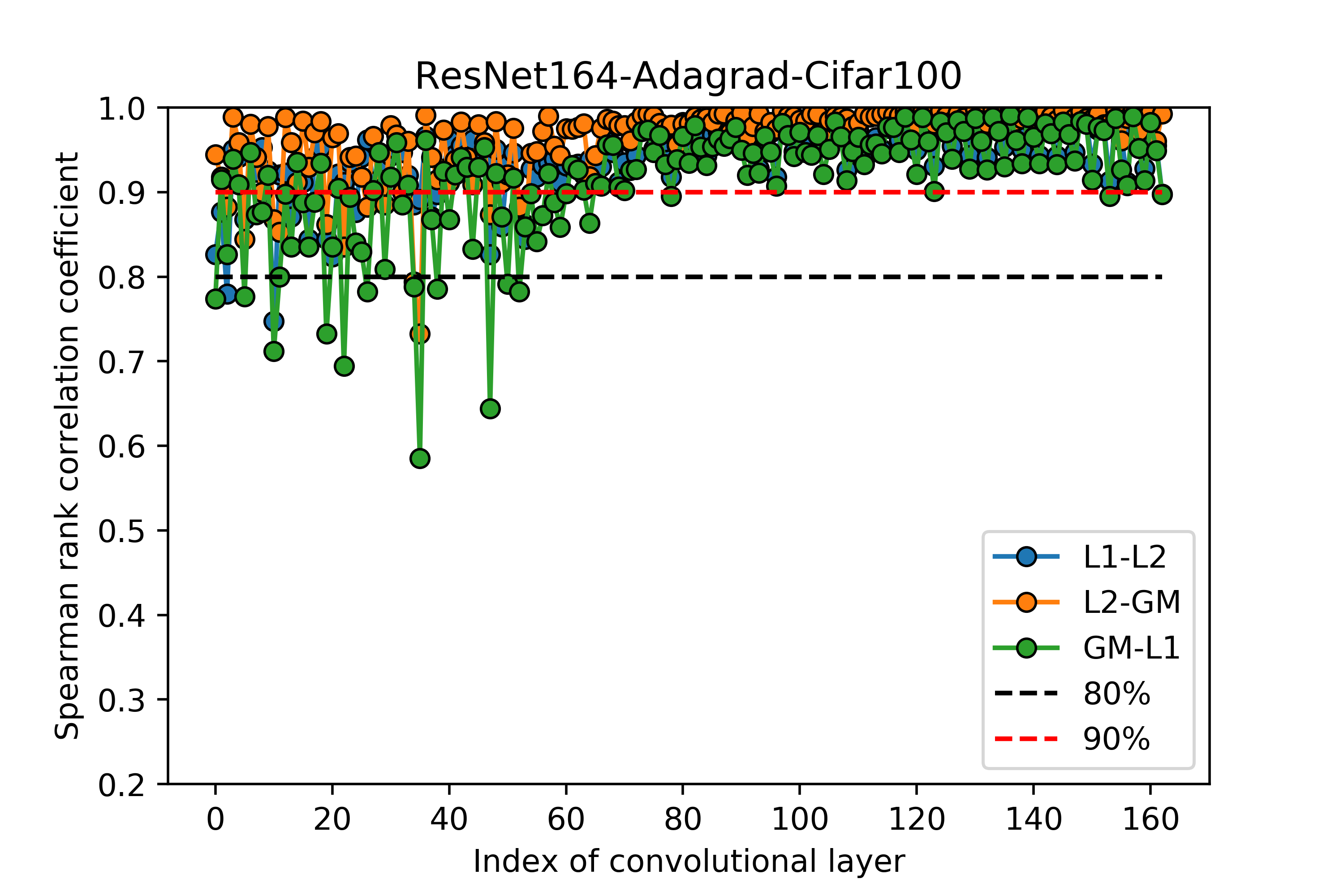} 
	\includegraphics[width=0.3\linewidth]{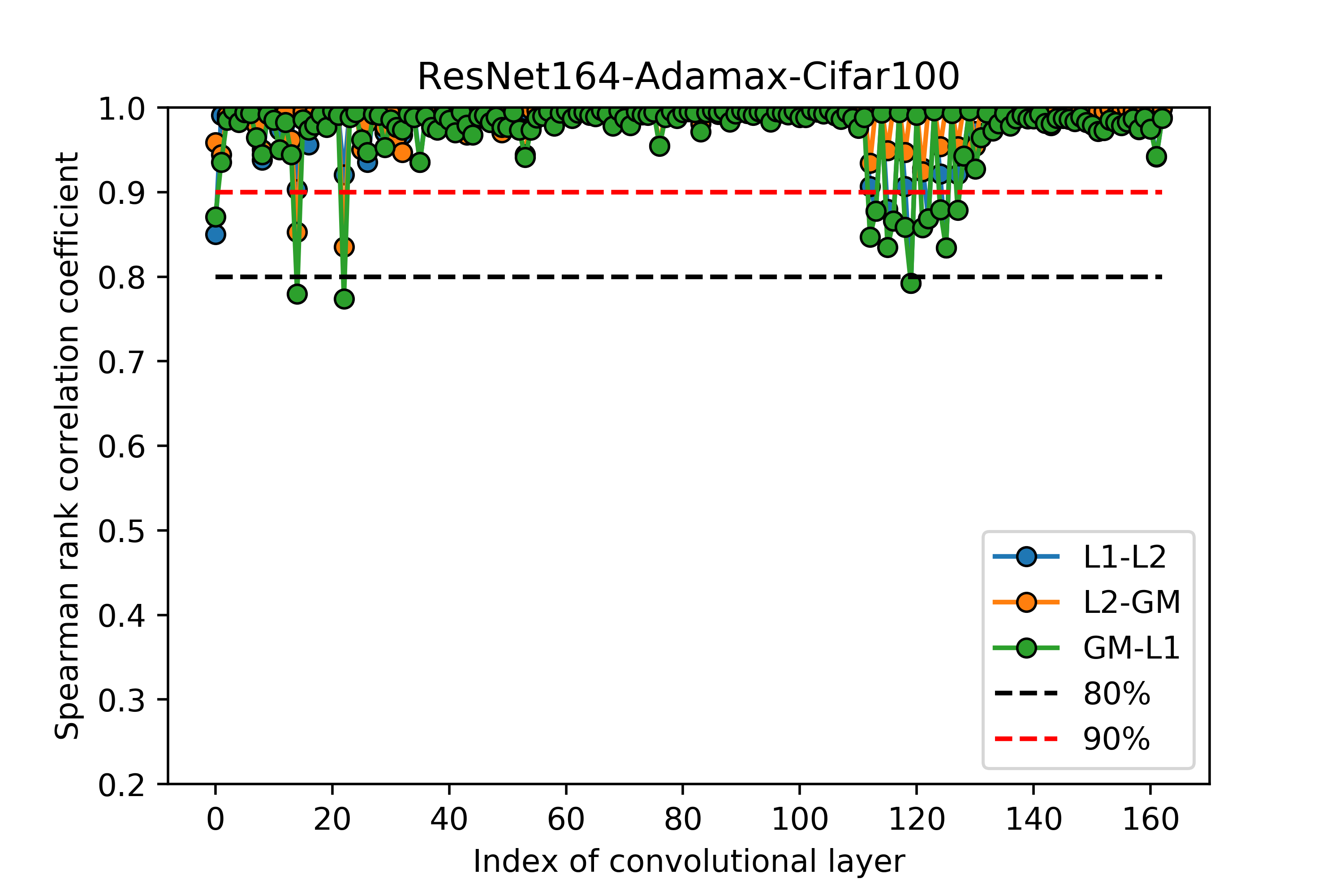} 
	\includegraphics[width=0.3\linewidth]{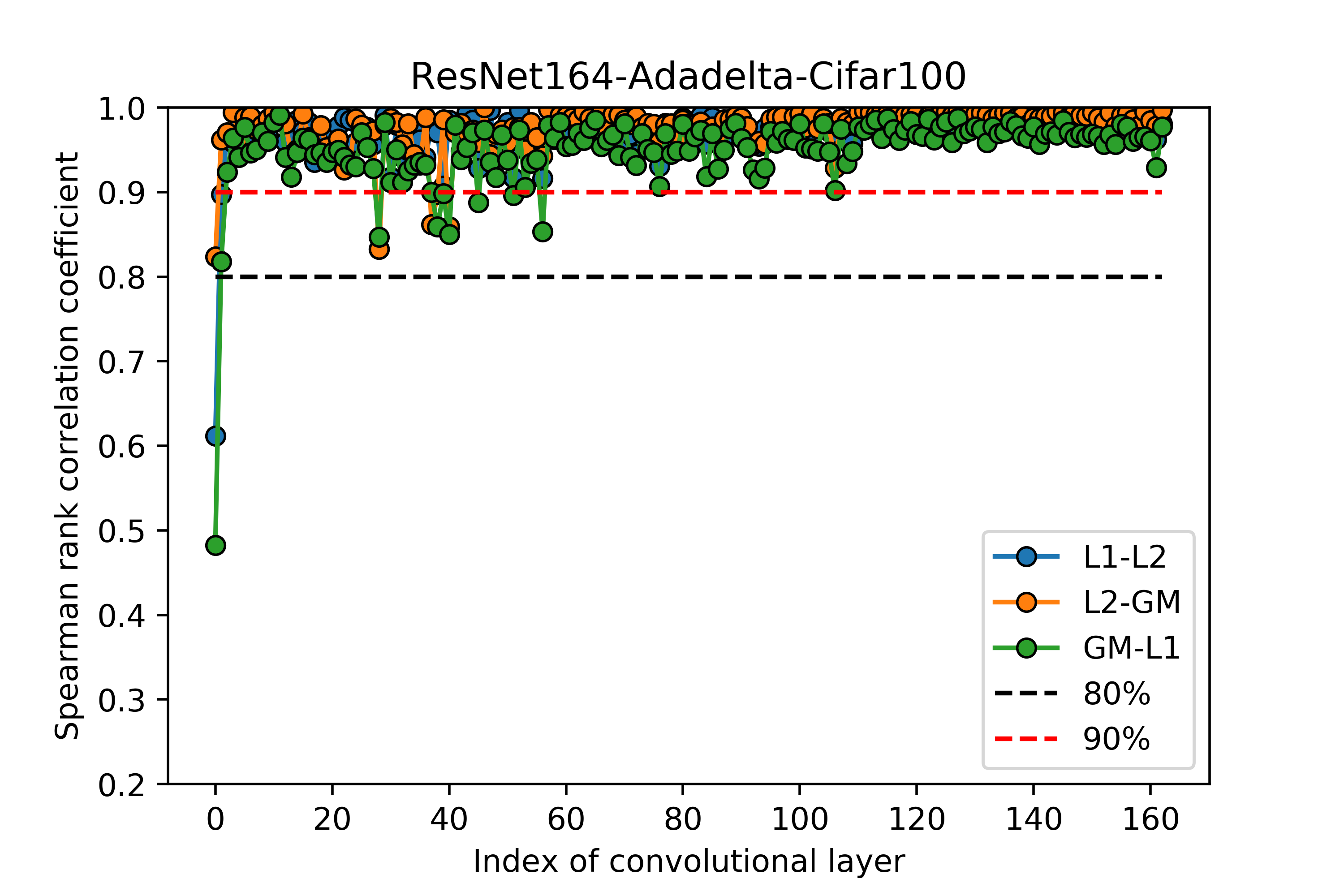} 
	\caption{Optimizer}
\end{figure}

\begin{figure}
	\centering 
	\includegraphics[width=0.4\linewidth]{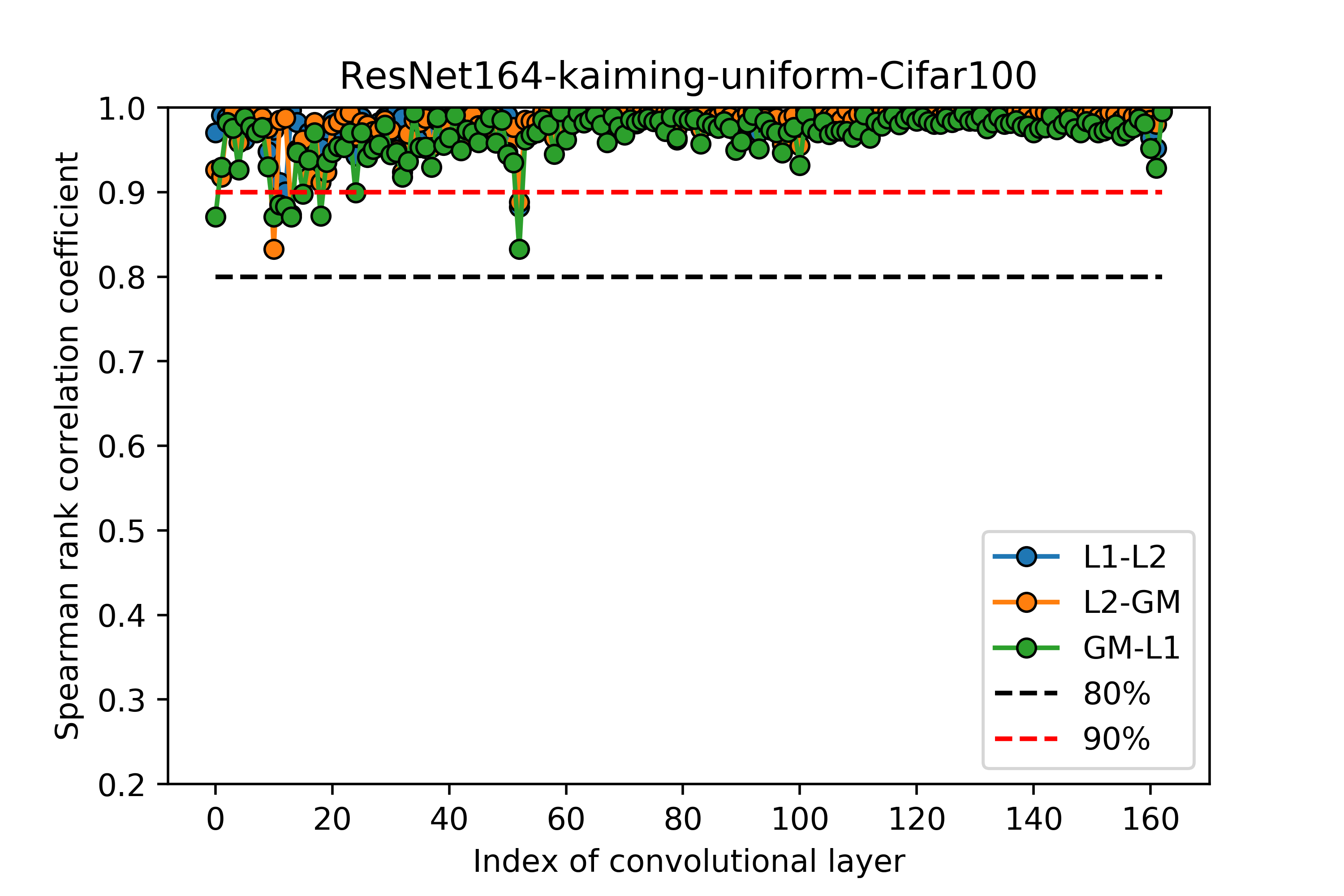} 
	\includegraphics[width=0.4\linewidth]{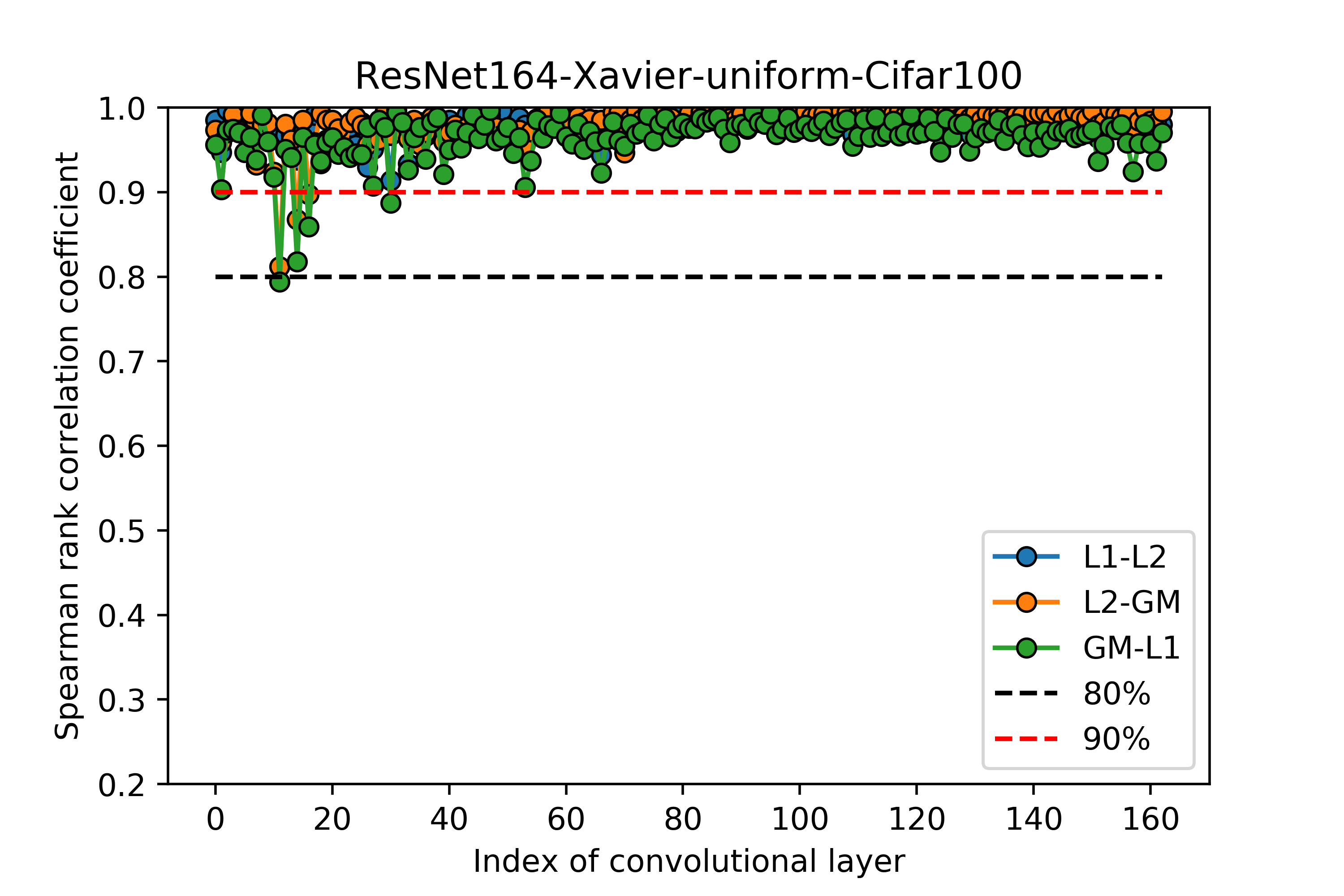} 
	\includegraphics[width=0.4\linewidth]{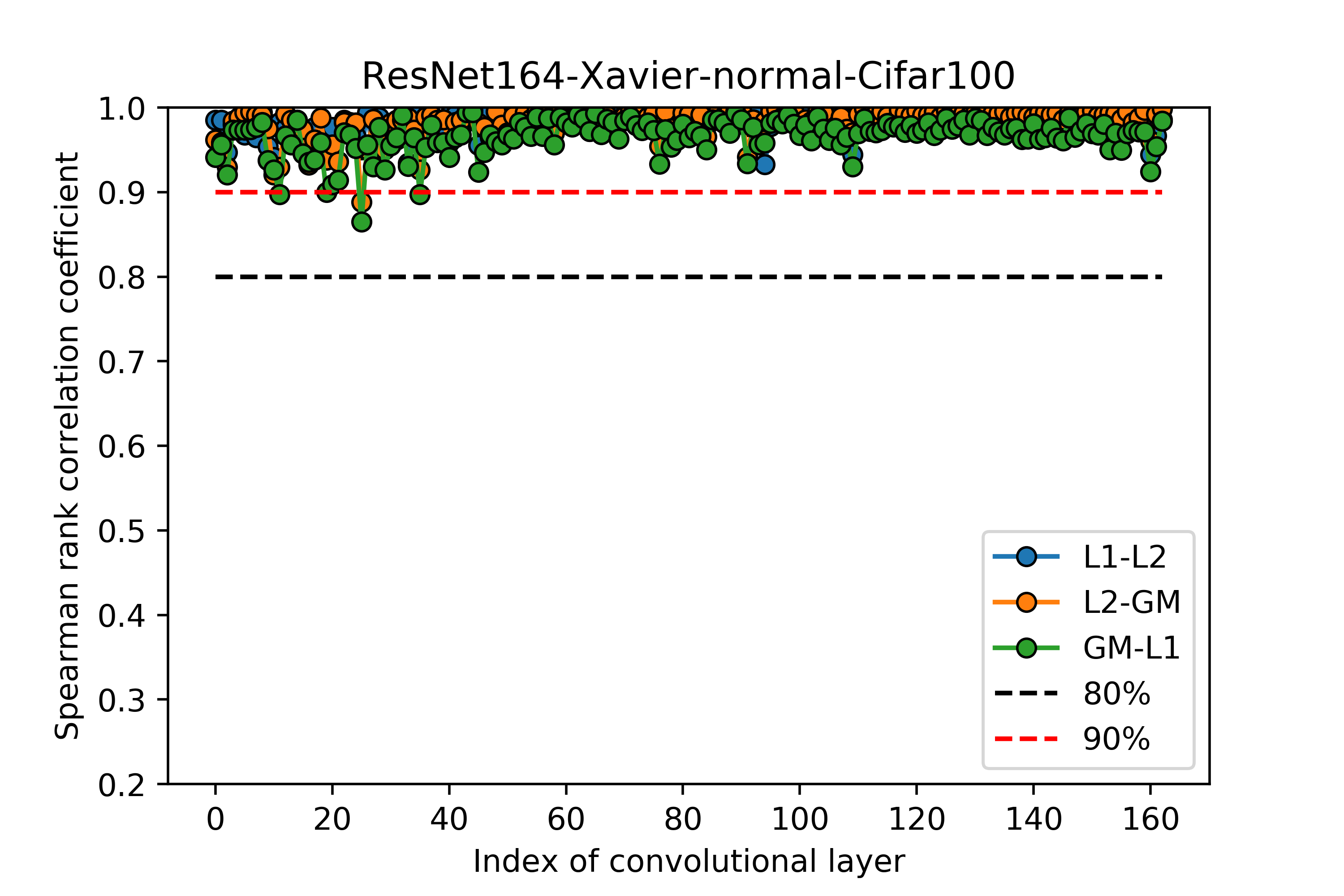} 
	\includegraphics[width=0.4\linewidth]{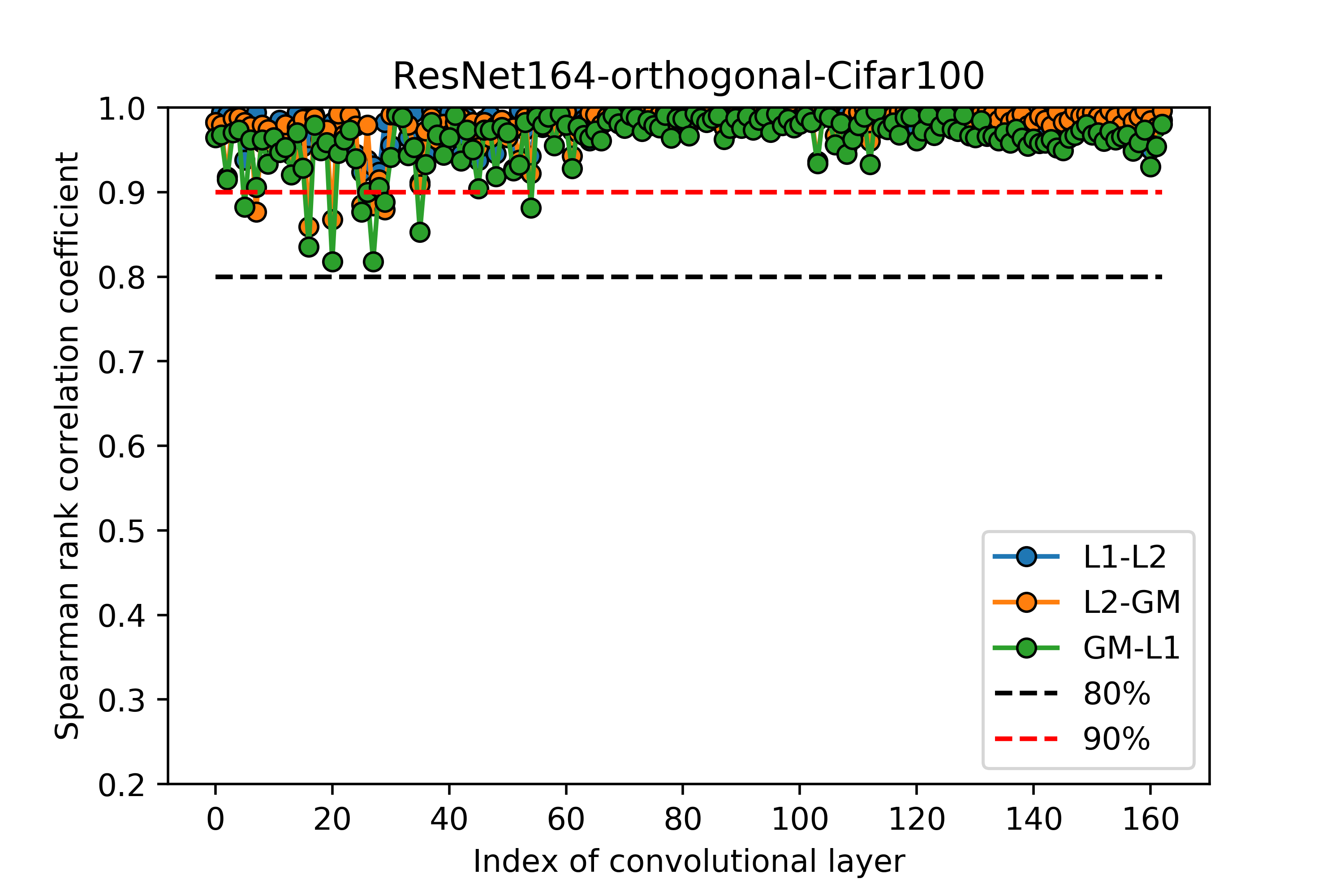} 
	\caption{Initialization}
\end{figure}    

\begin{figure}
	\centering 
	\includegraphics[height=2.2in, width=2.5in]{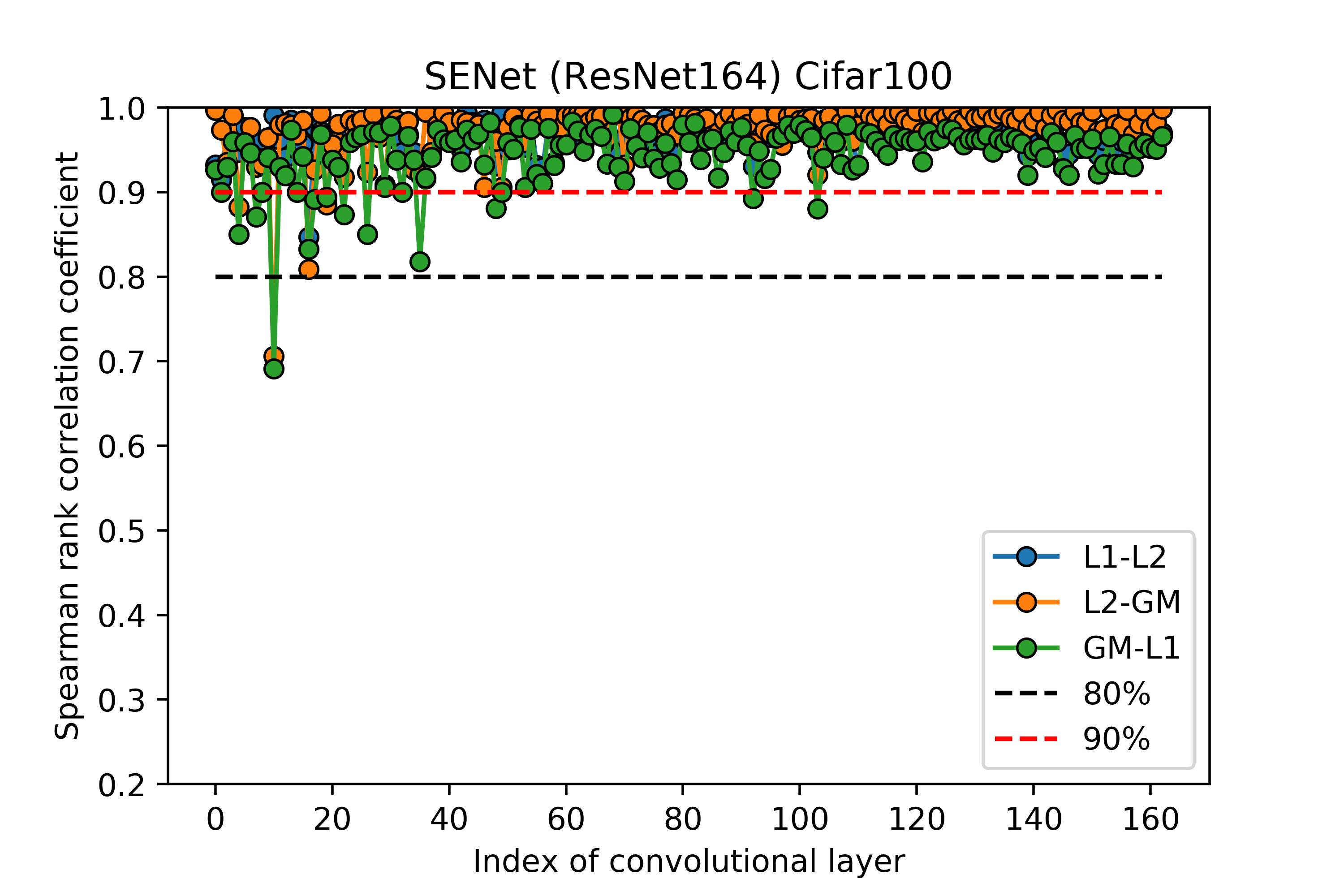} 
	\includegraphics[height=2.2in, width=2.5in]{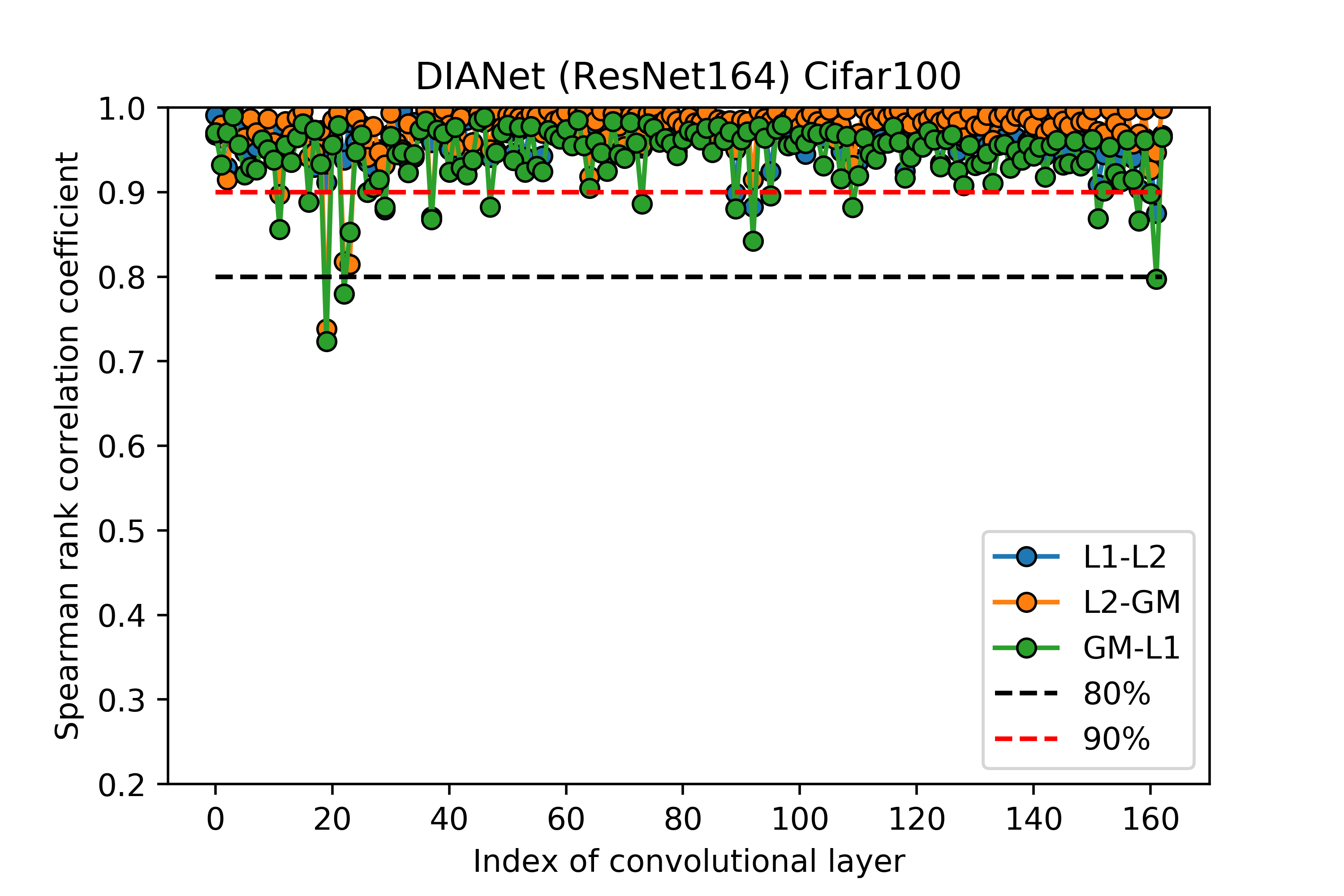} 
	\includegraphics[height=2.2in, width=2.5in]{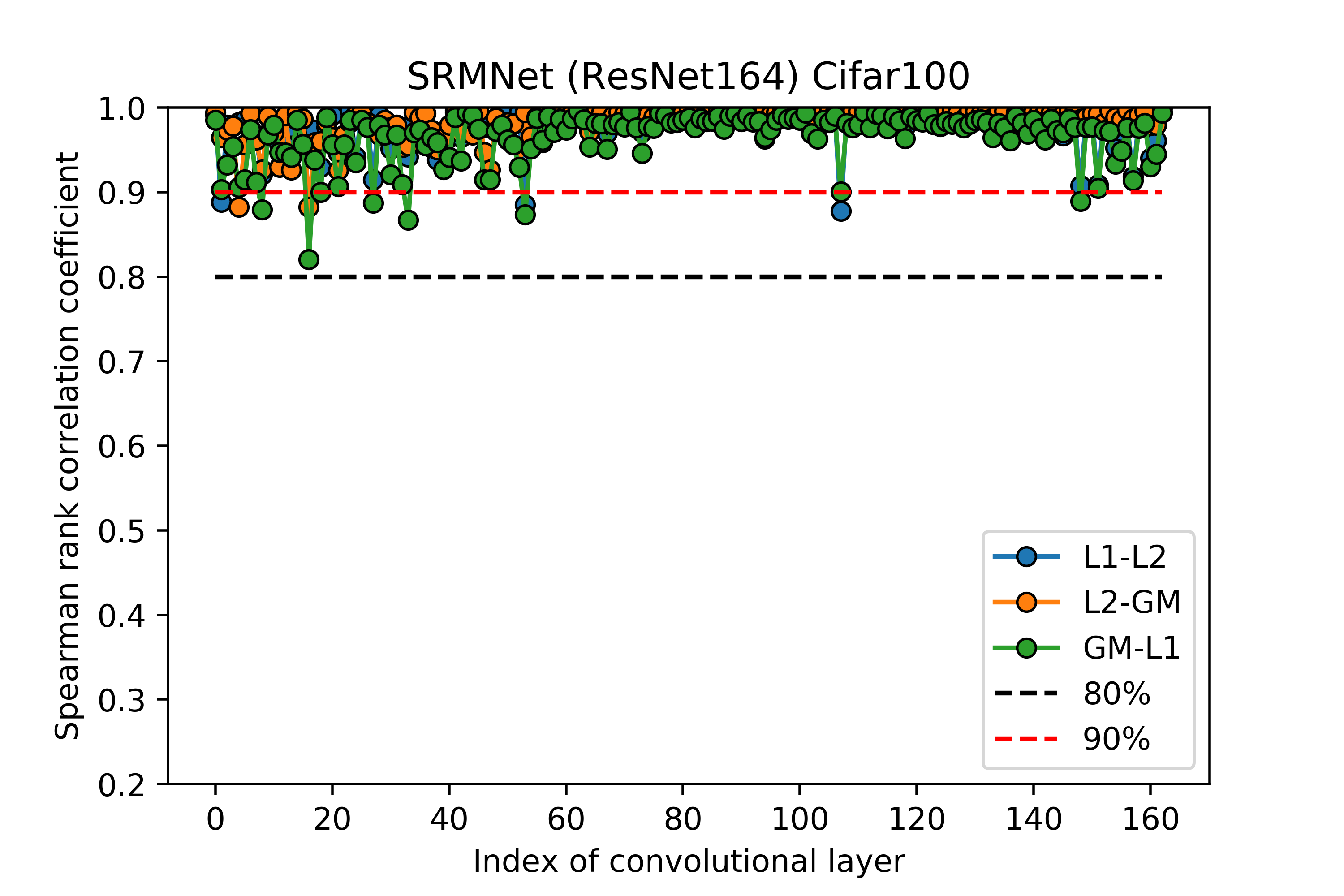} 
	\includegraphics[height=2.2in, width=2.5in]{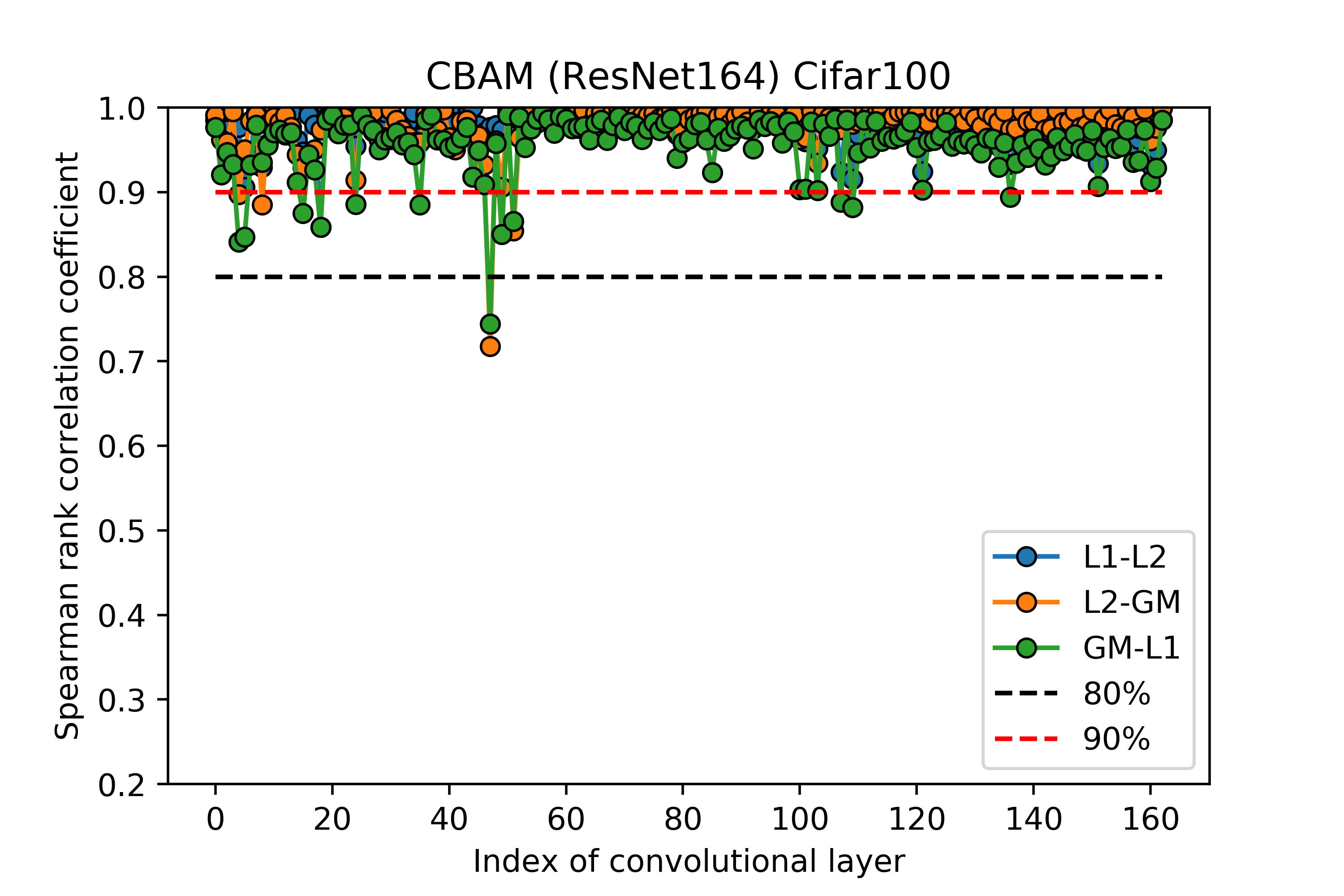} 
	\includegraphics[height=2.2in, width=2.5in]{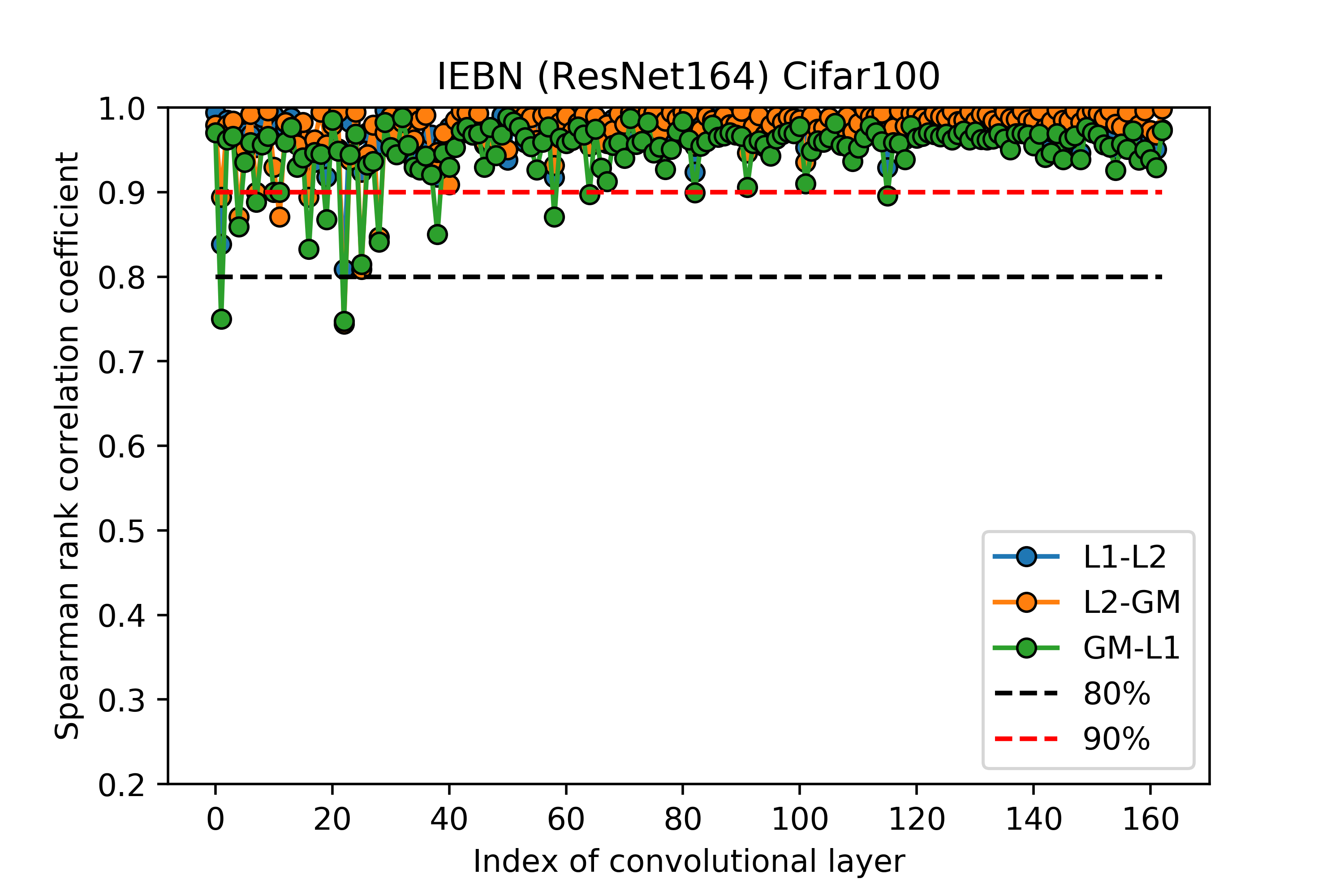} 
	\includegraphics[height=2.2in, width=2.5in]{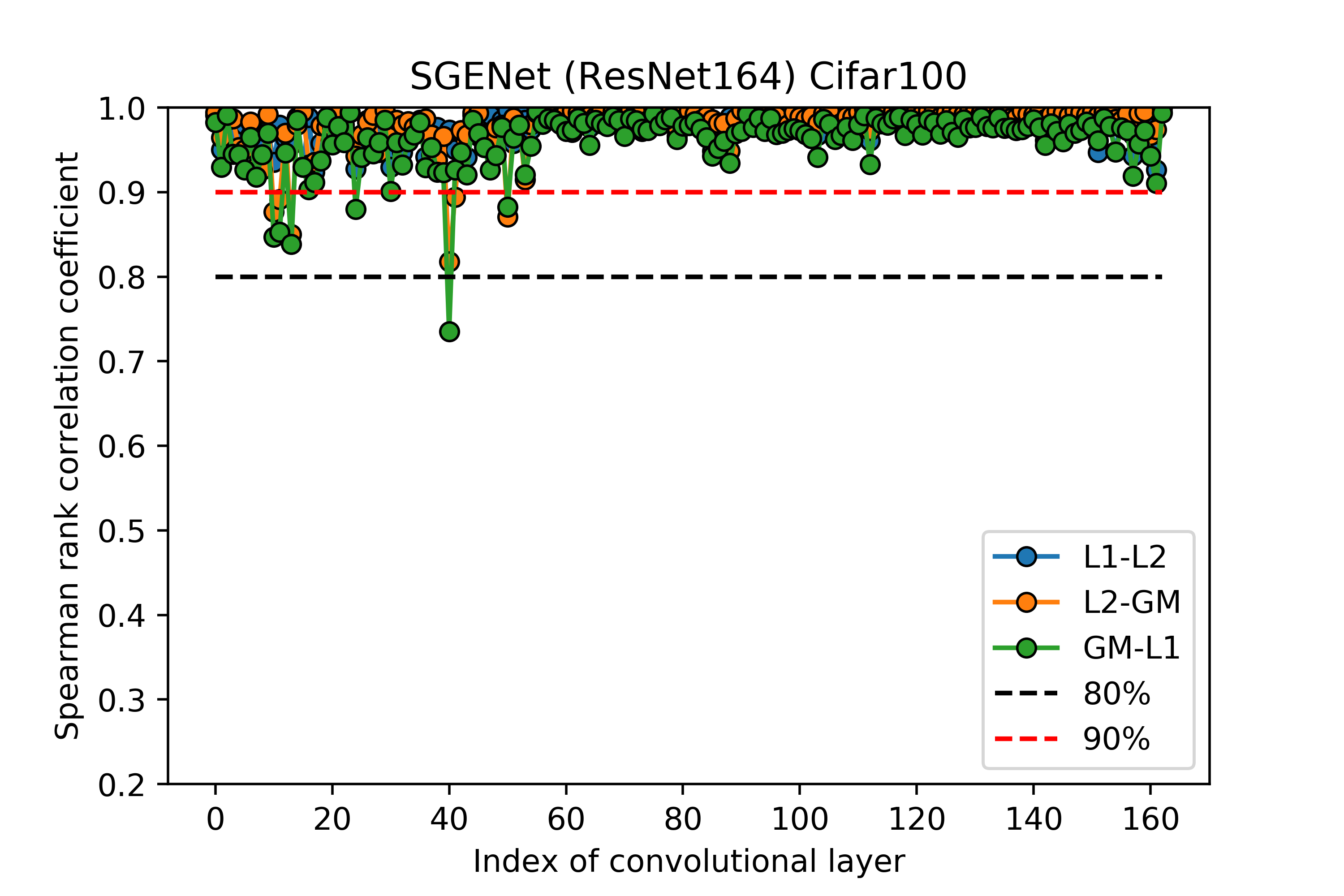} 
	
	\caption{Attention mechanism}
\end{figure}

\begin{figure}
	\centering 
	\includegraphics[height=2.2in, width=2.5in]{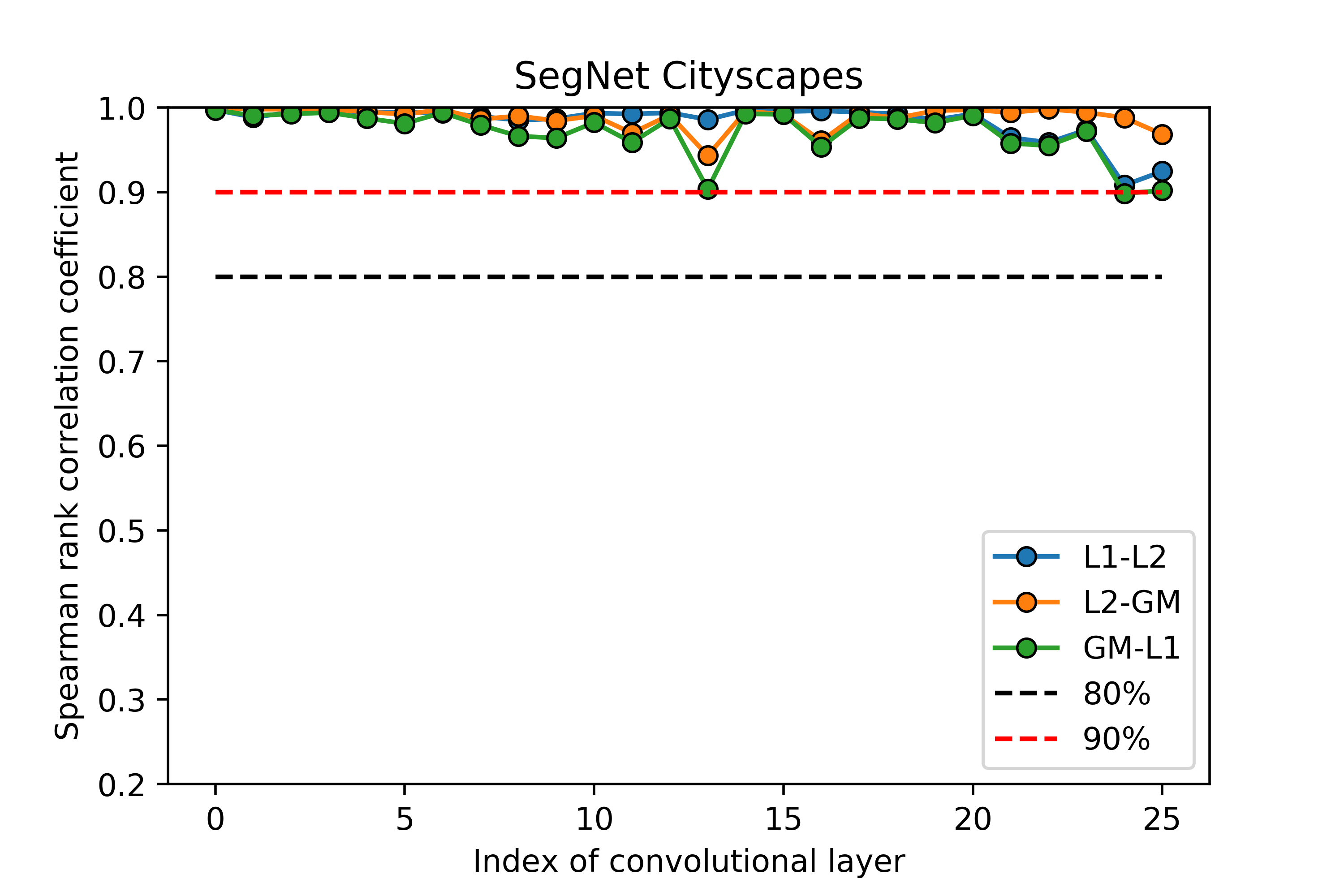} 
	\includegraphics[height=2.2in, width=2.5in]{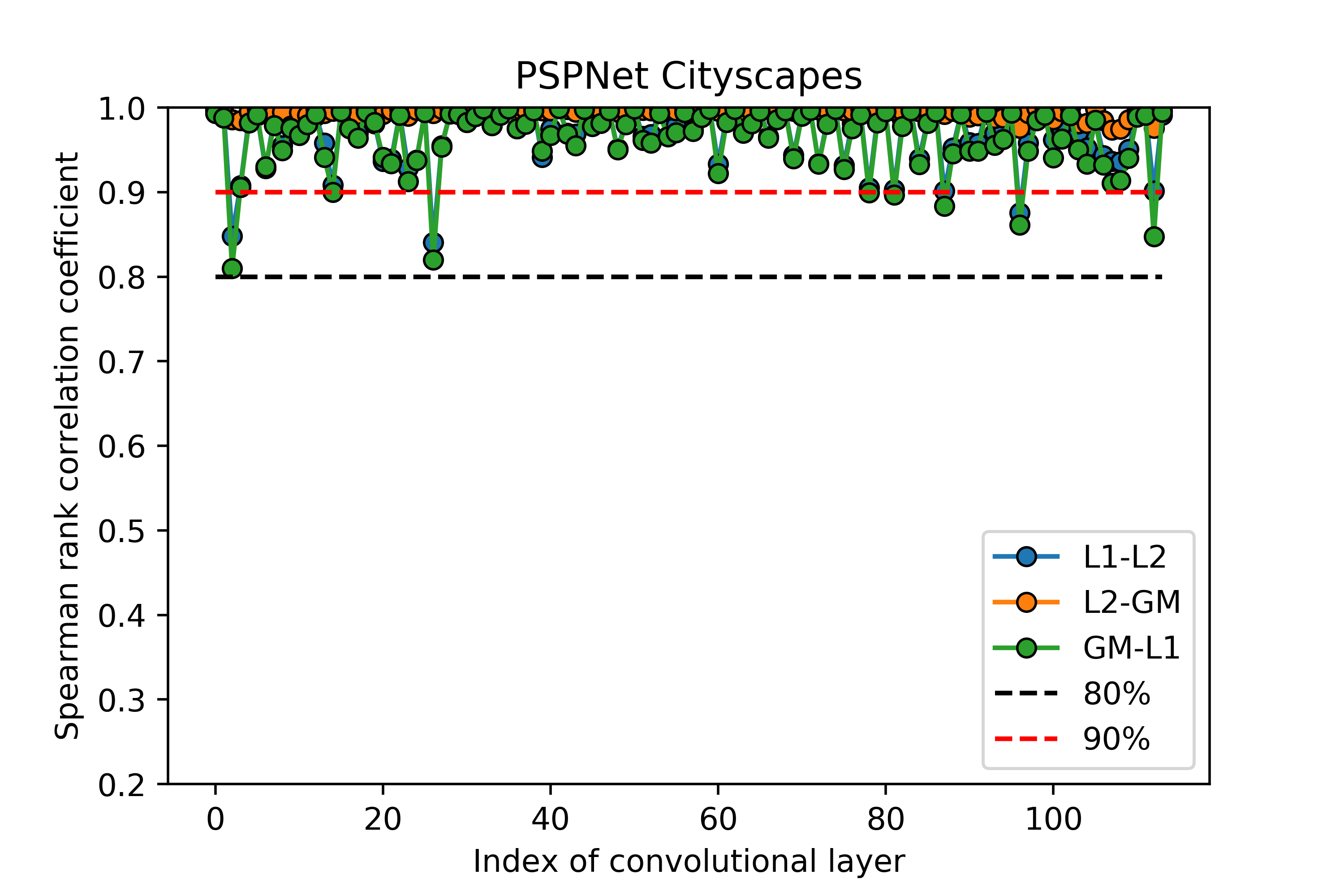} 
	\caption{Other task: segmentation}
\end{figure}    

\begin{figure}
	\centering 
	\includegraphics[height=2.2in, width=2.5in]{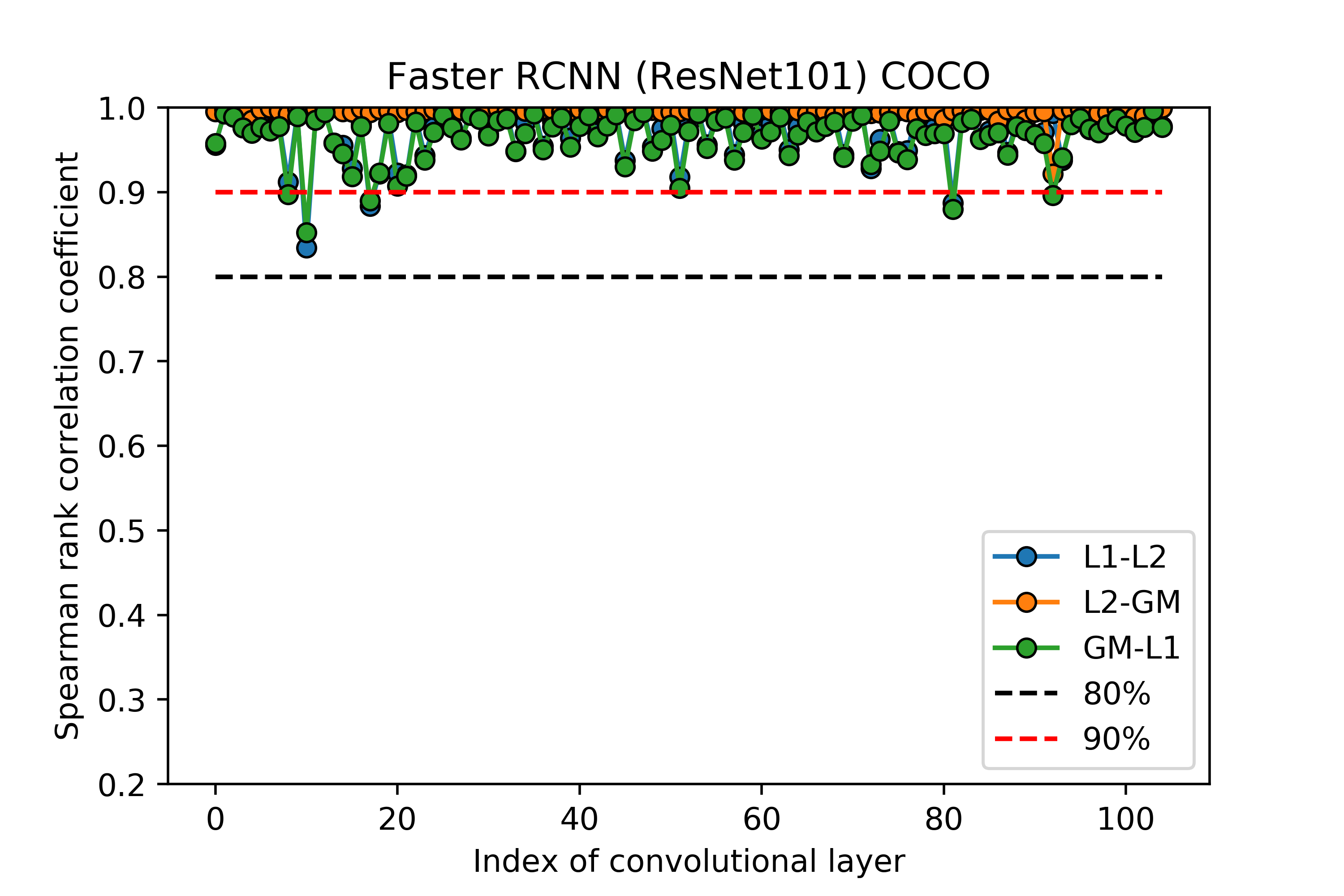} 
	\includegraphics[height=2.2in, width=2.5in]{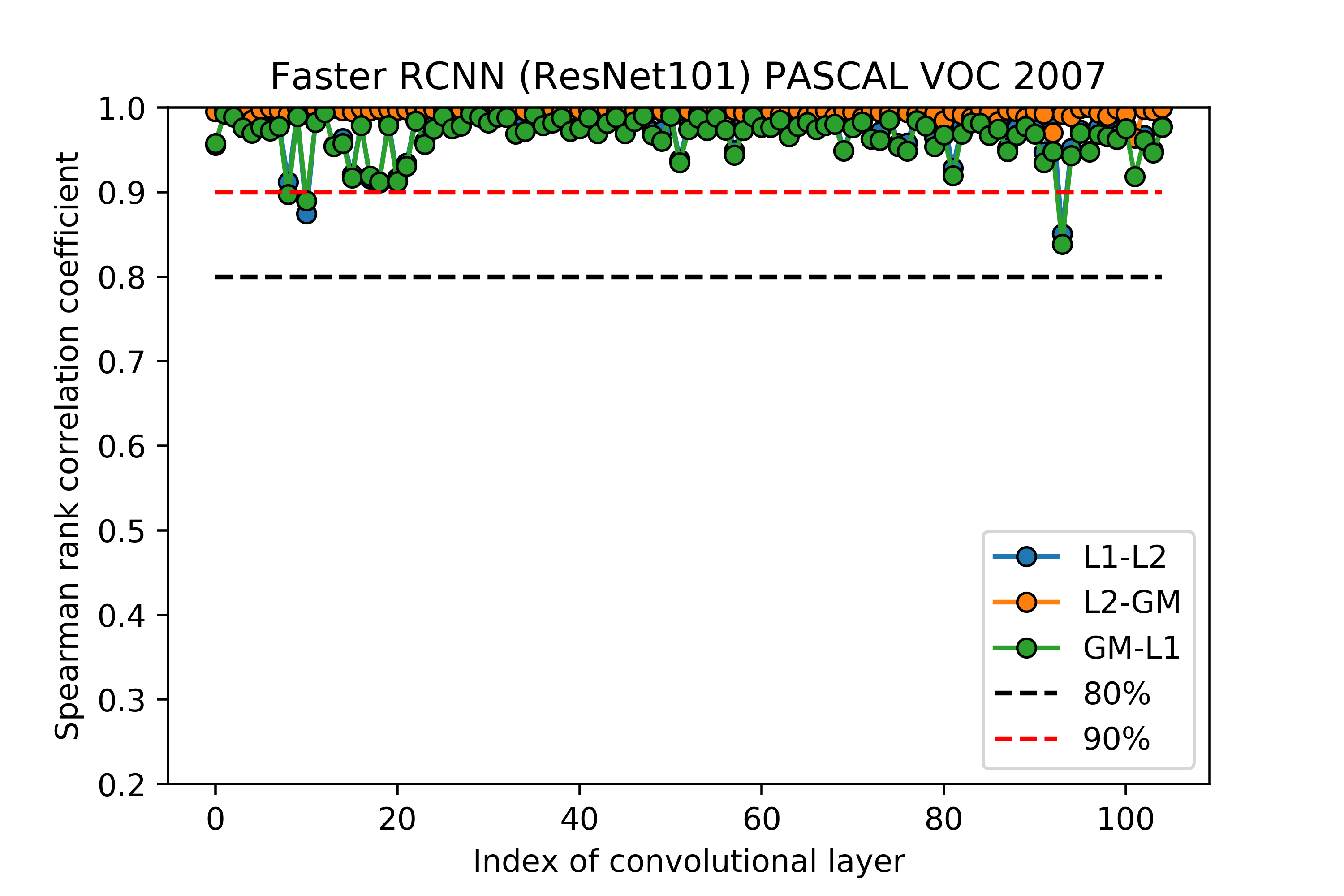} 
	\includegraphics[height=2.2in, width=2.5in]{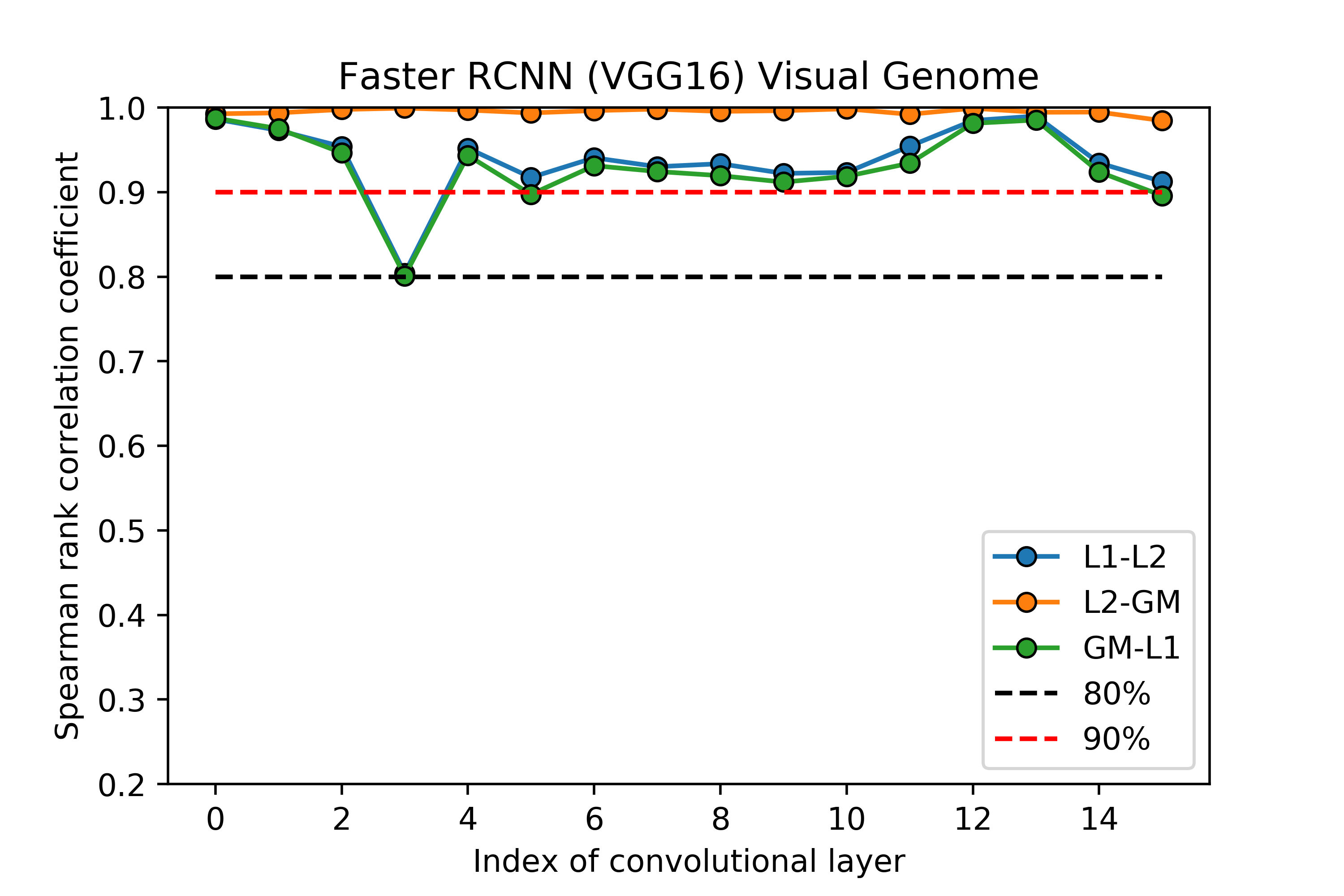} 
	\caption{Other task: Faster RCNN}
\end{figure}

\begin{figure}
	\centering 
	\includegraphics[height=2.2in, width=2.5in]{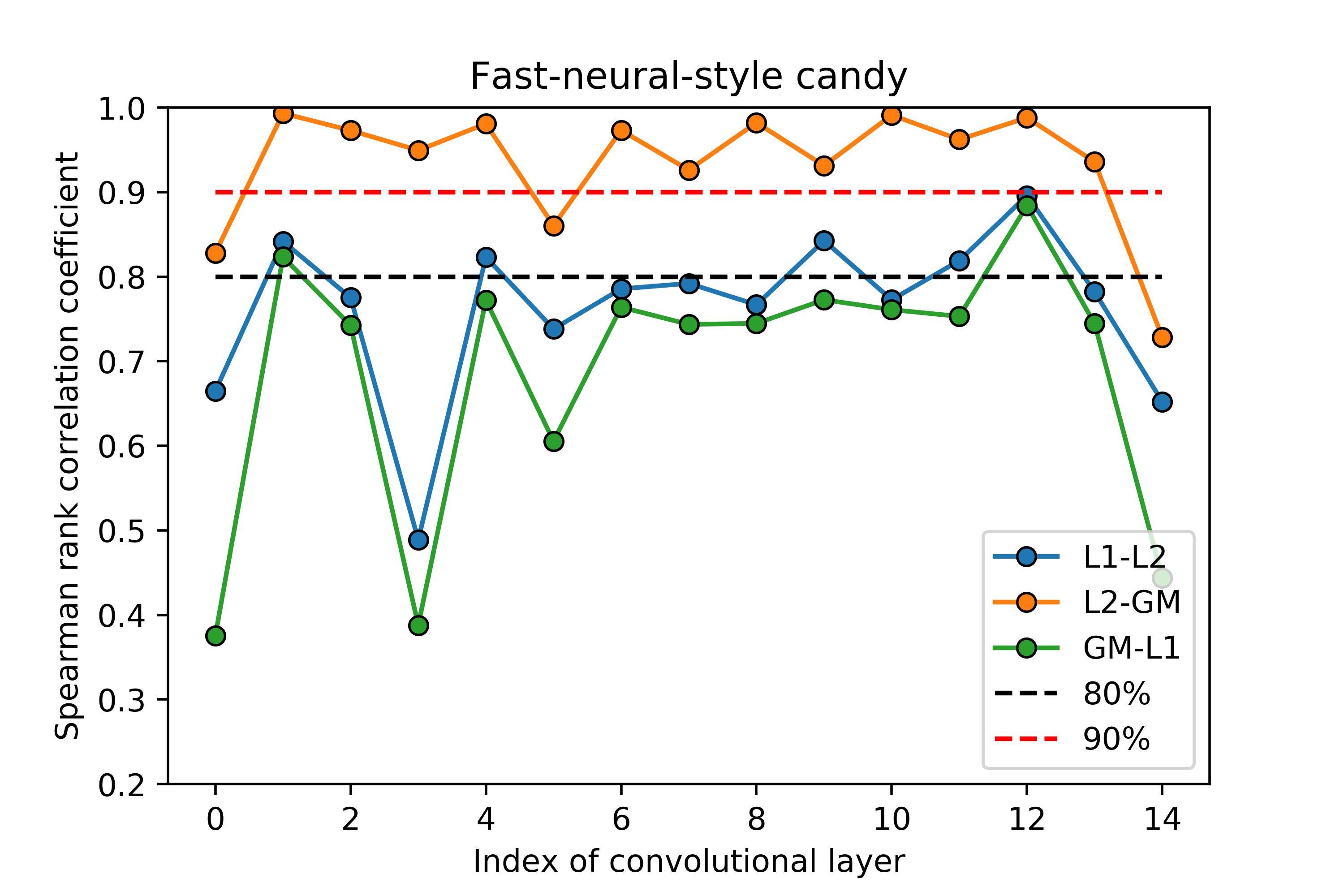} 
	\includegraphics[height=2.2in, width=2.5in]{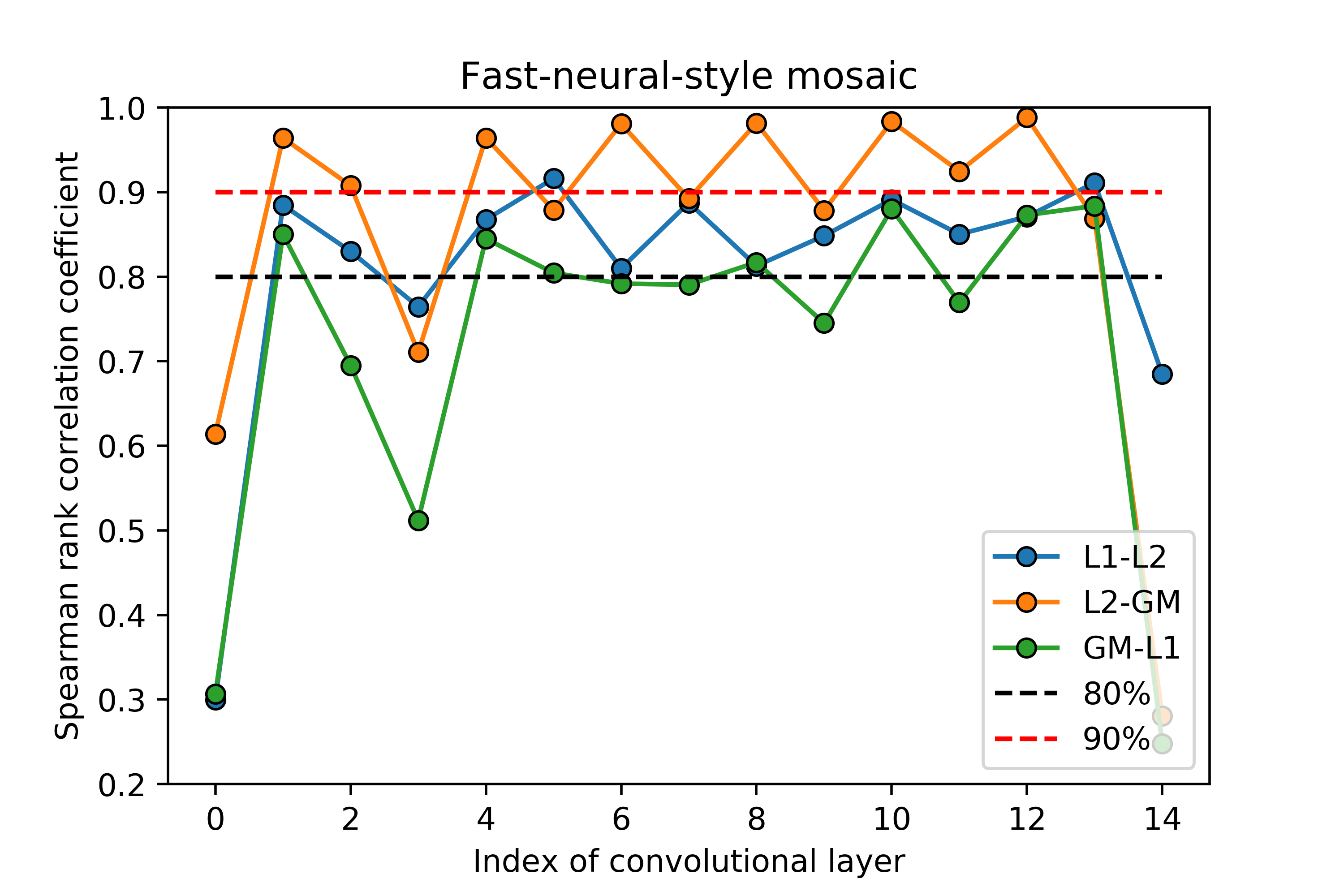} 
	\includegraphics[height=2.2in, width=2.5in]{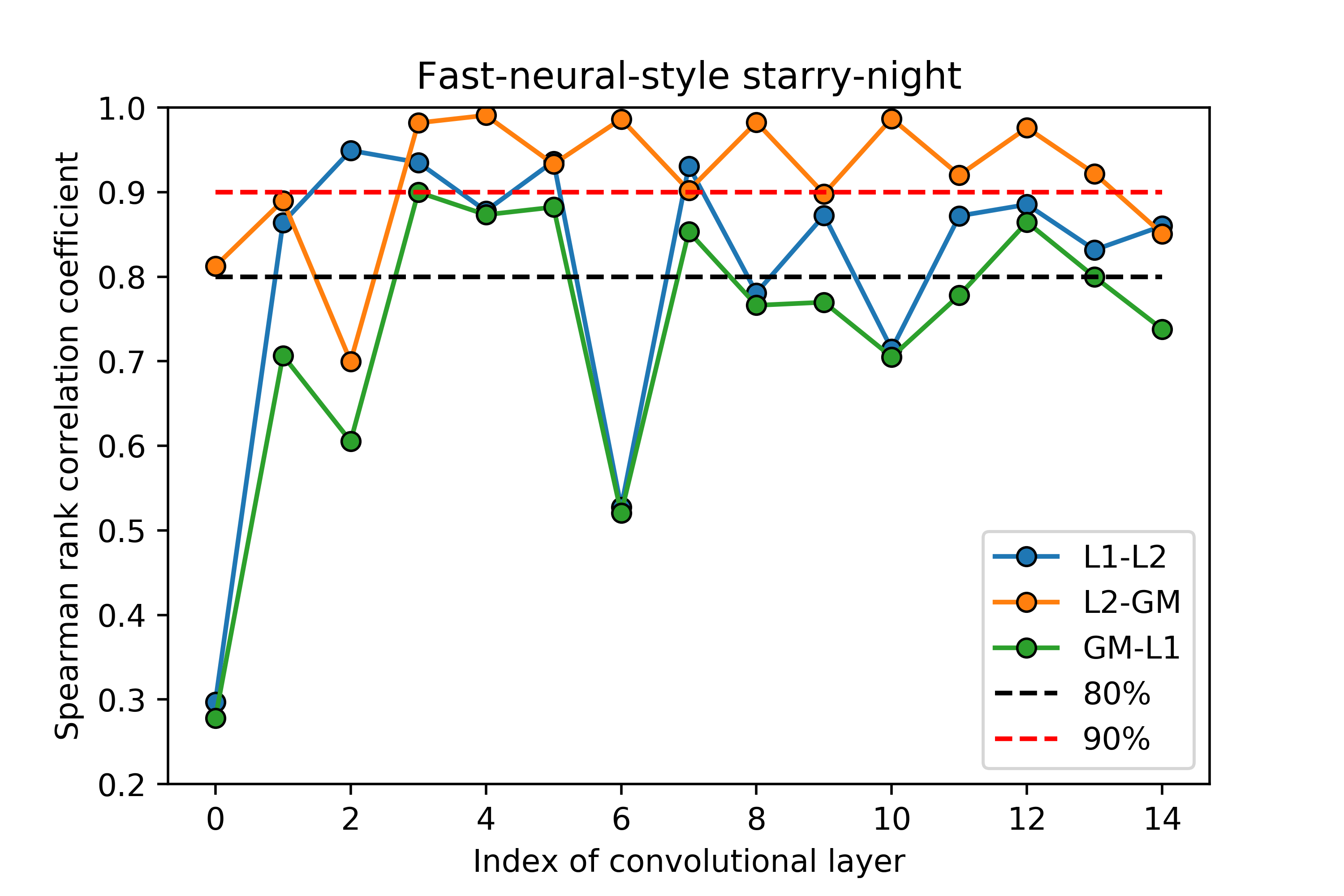} 
	\includegraphics[height=2.2in, width=2.5in]{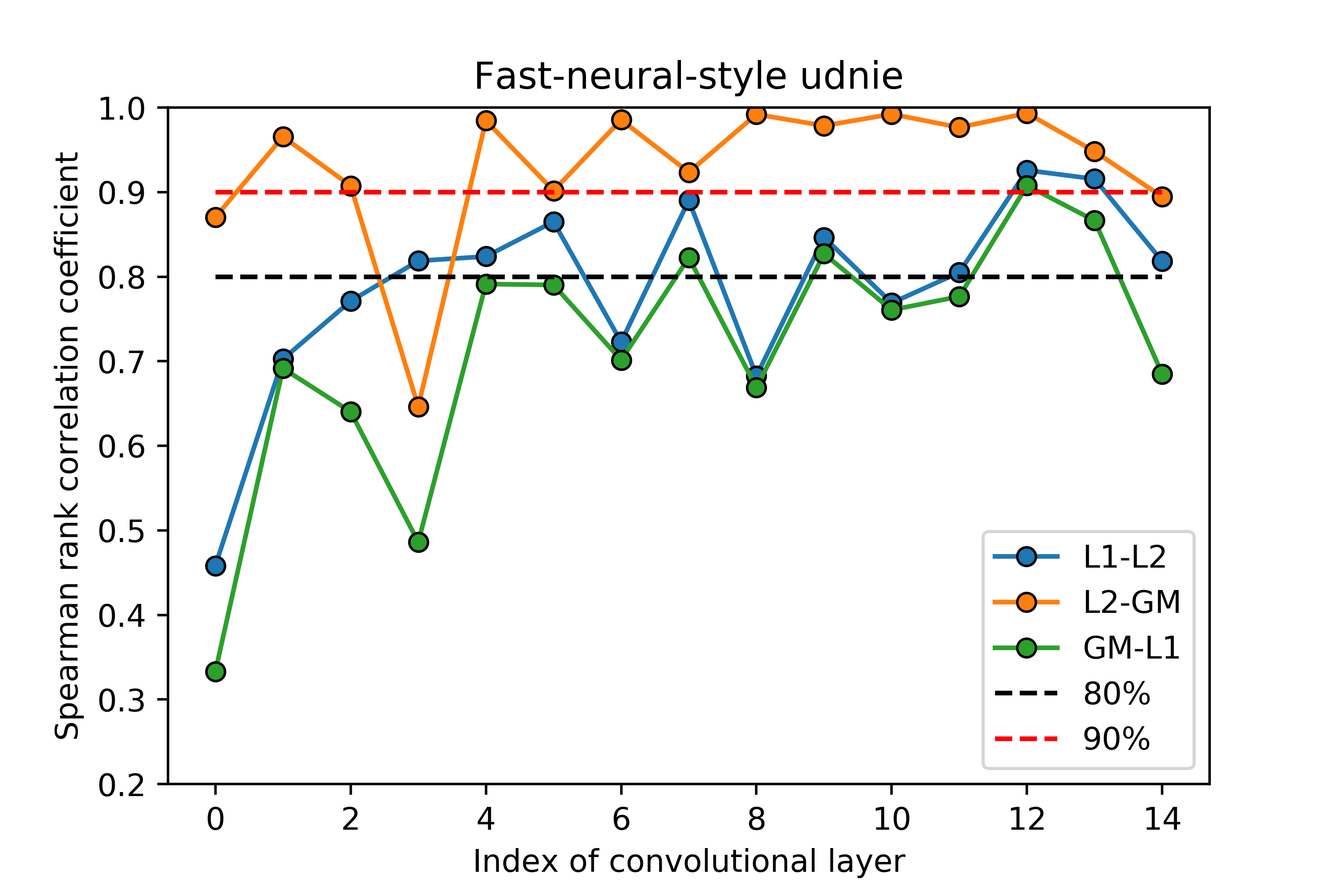}
	\caption{Other task: style transfer}
\end{figure}

\begin{figure}
	\centering 
	\includegraphics[height=2.2in, width=2.5in]{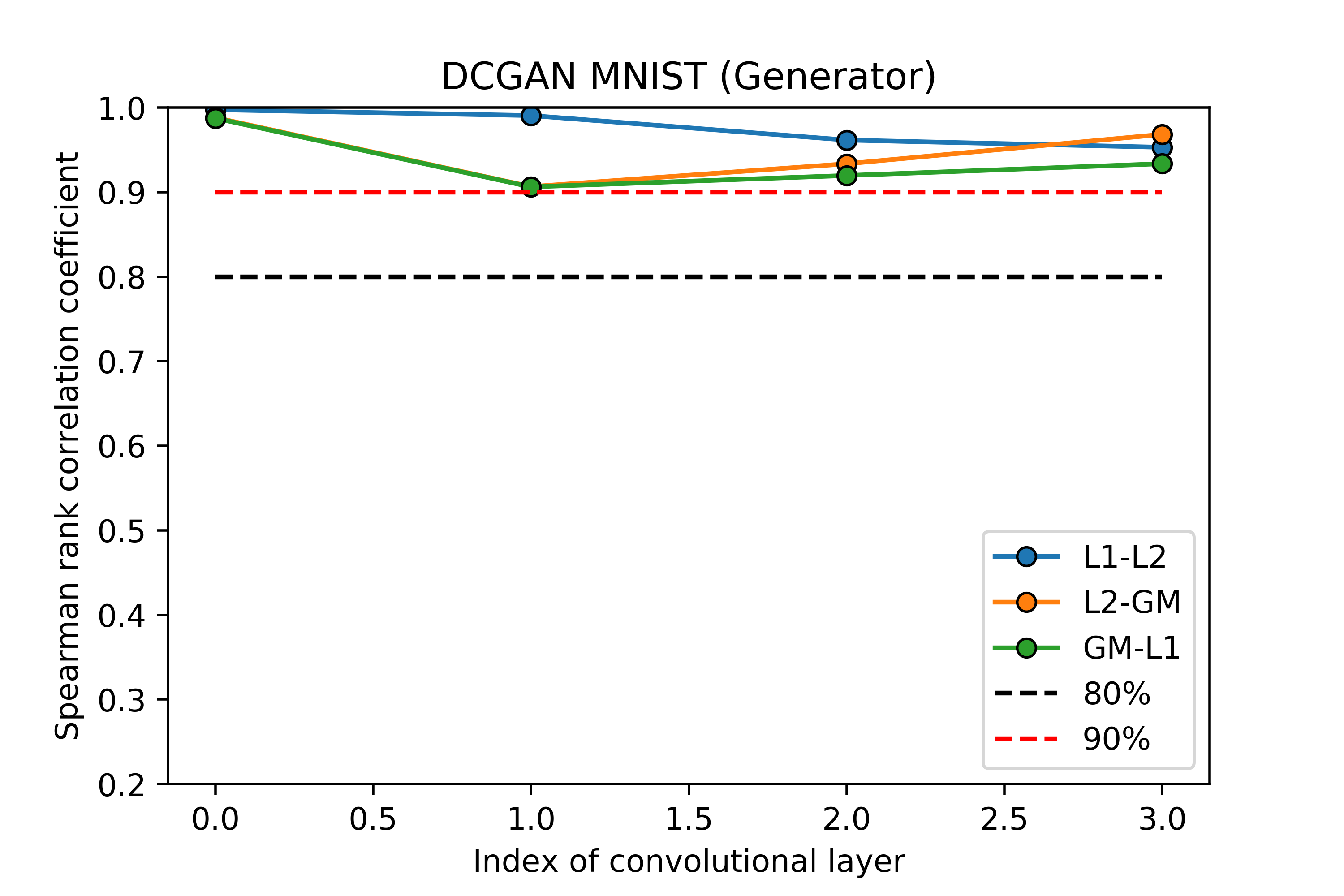} 
	\includegraphics[height=2.2in, width=2.5in]{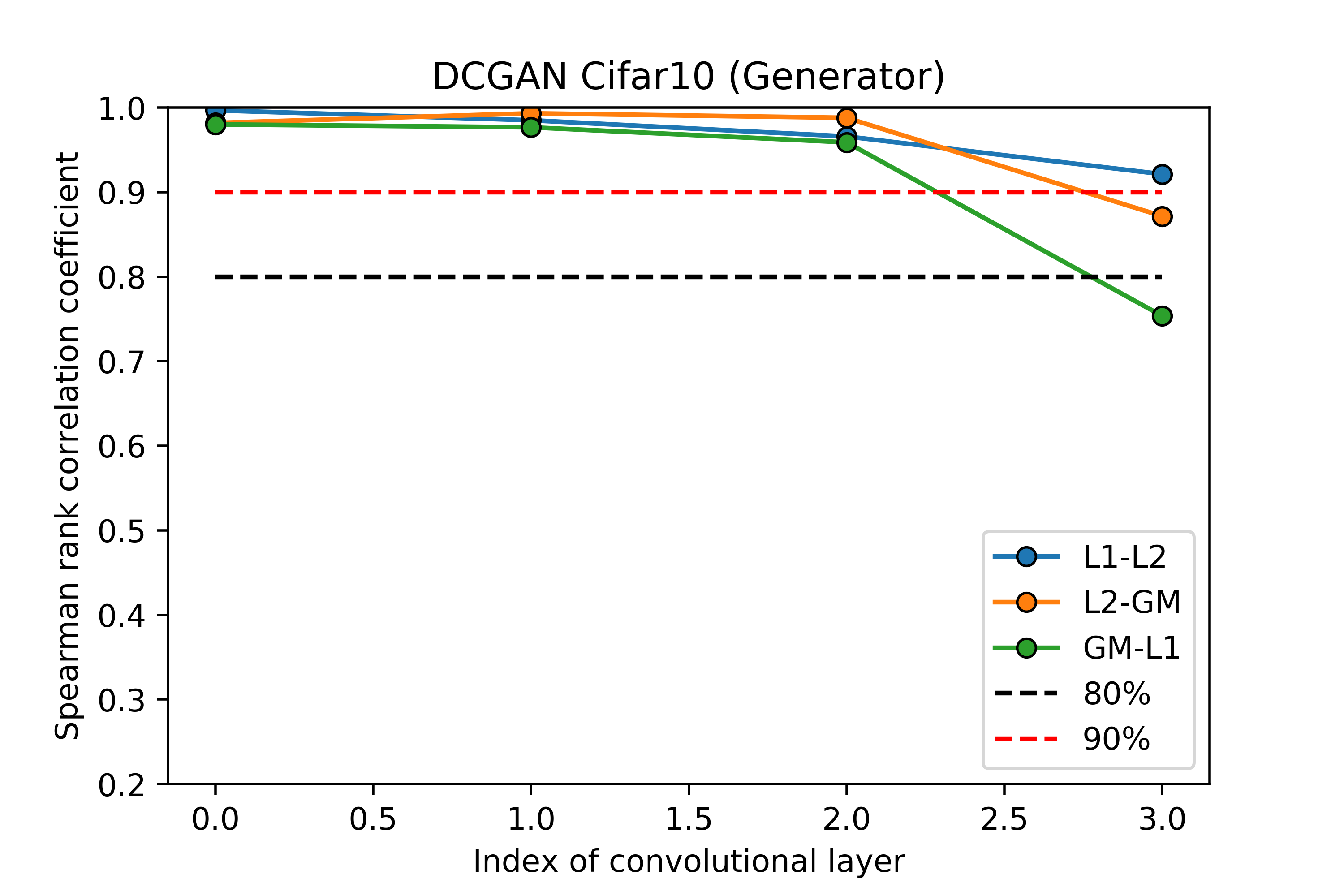} 
	\includegraphics[height=2.2in, width=2.5in]{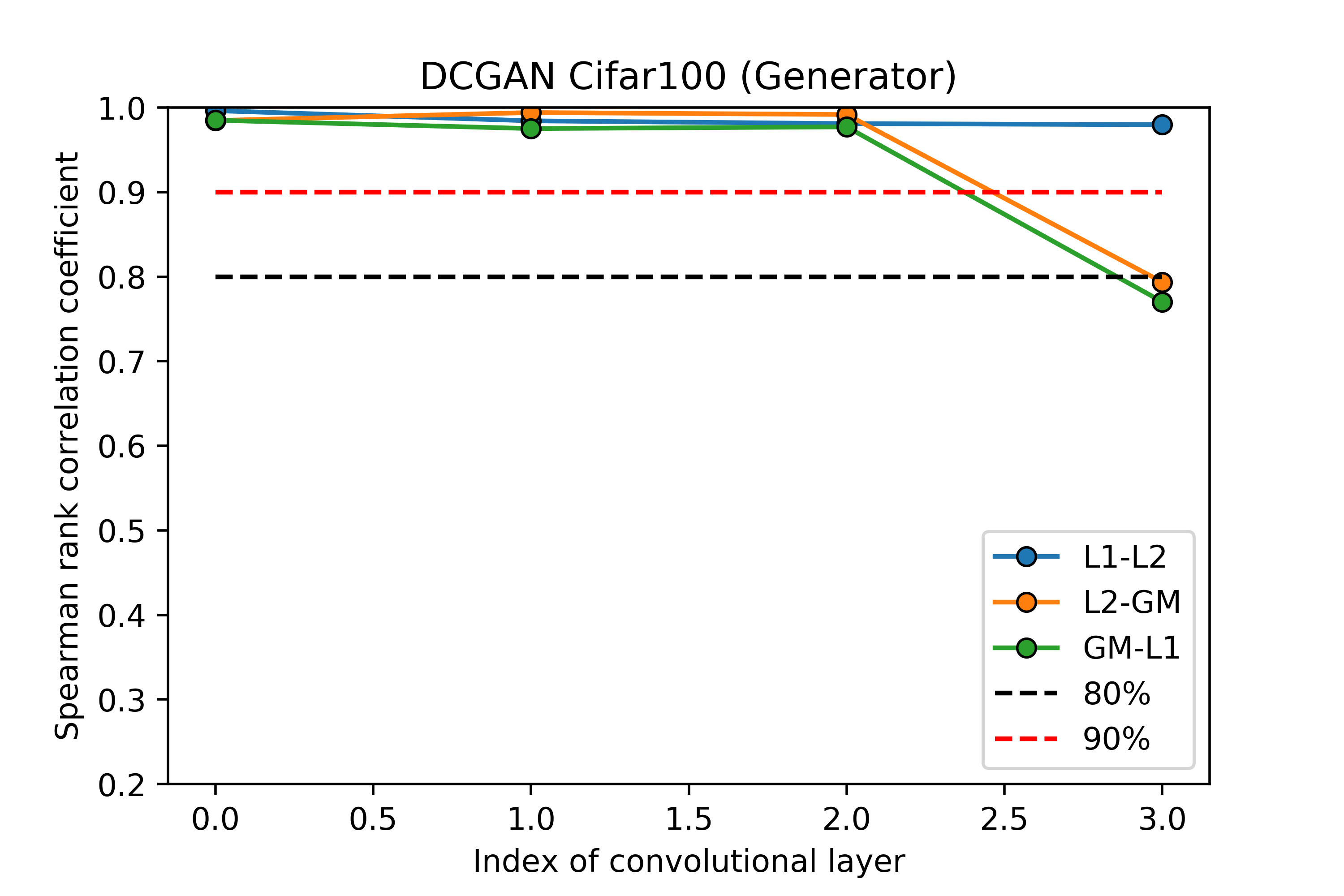} 
	\caption{Other task: GAN}
\end{figure}

\begin{figure}
	\centering 
	\includegraphics[height=2.2in, width=2.5in]{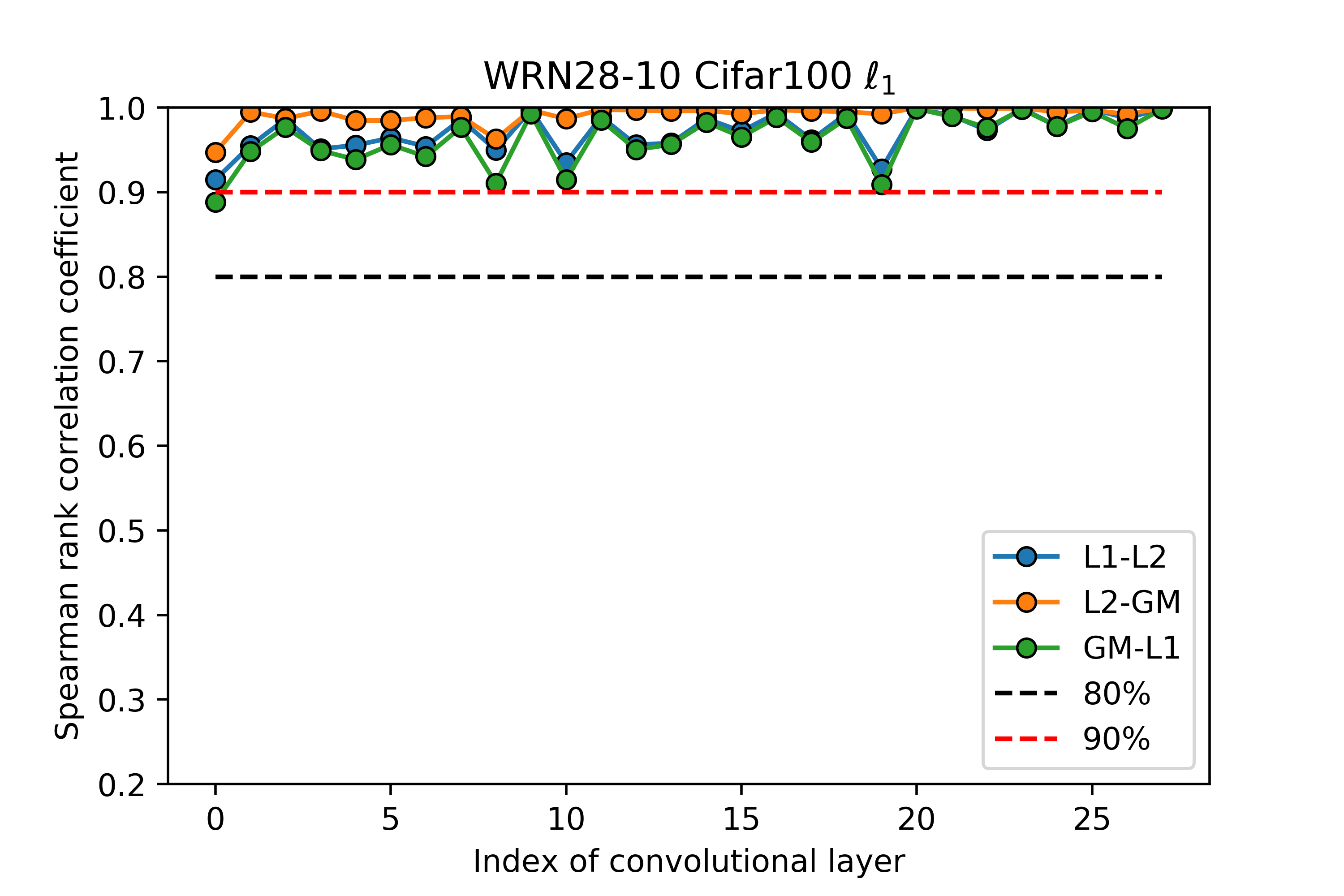} 
	\includegraphics[height=2.2in, width=2.5in]{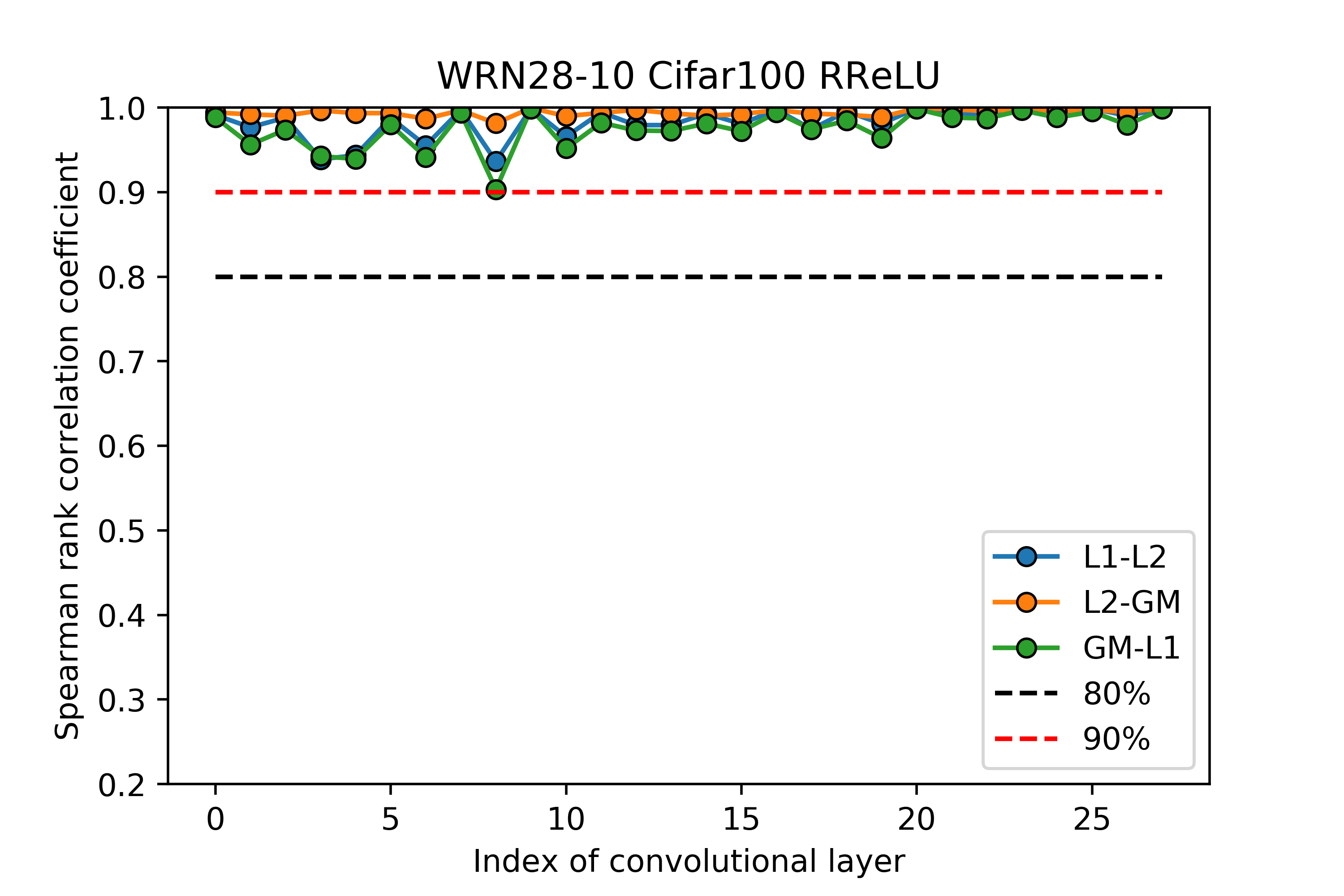} 
	\includegraphics[height=2.2in, width=2.5in]{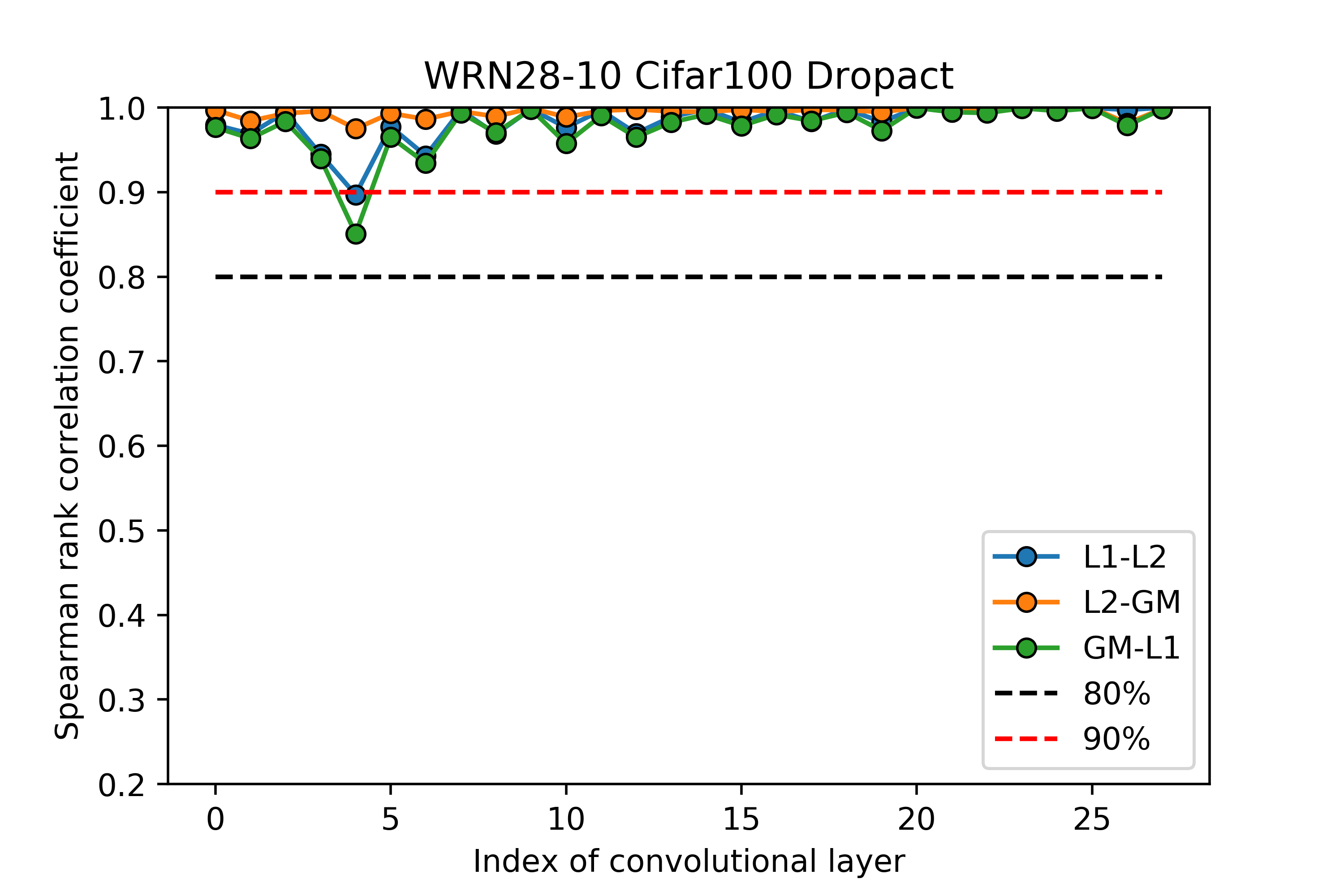} 
	\includegraphics[height=2.2in, width=2.5in]{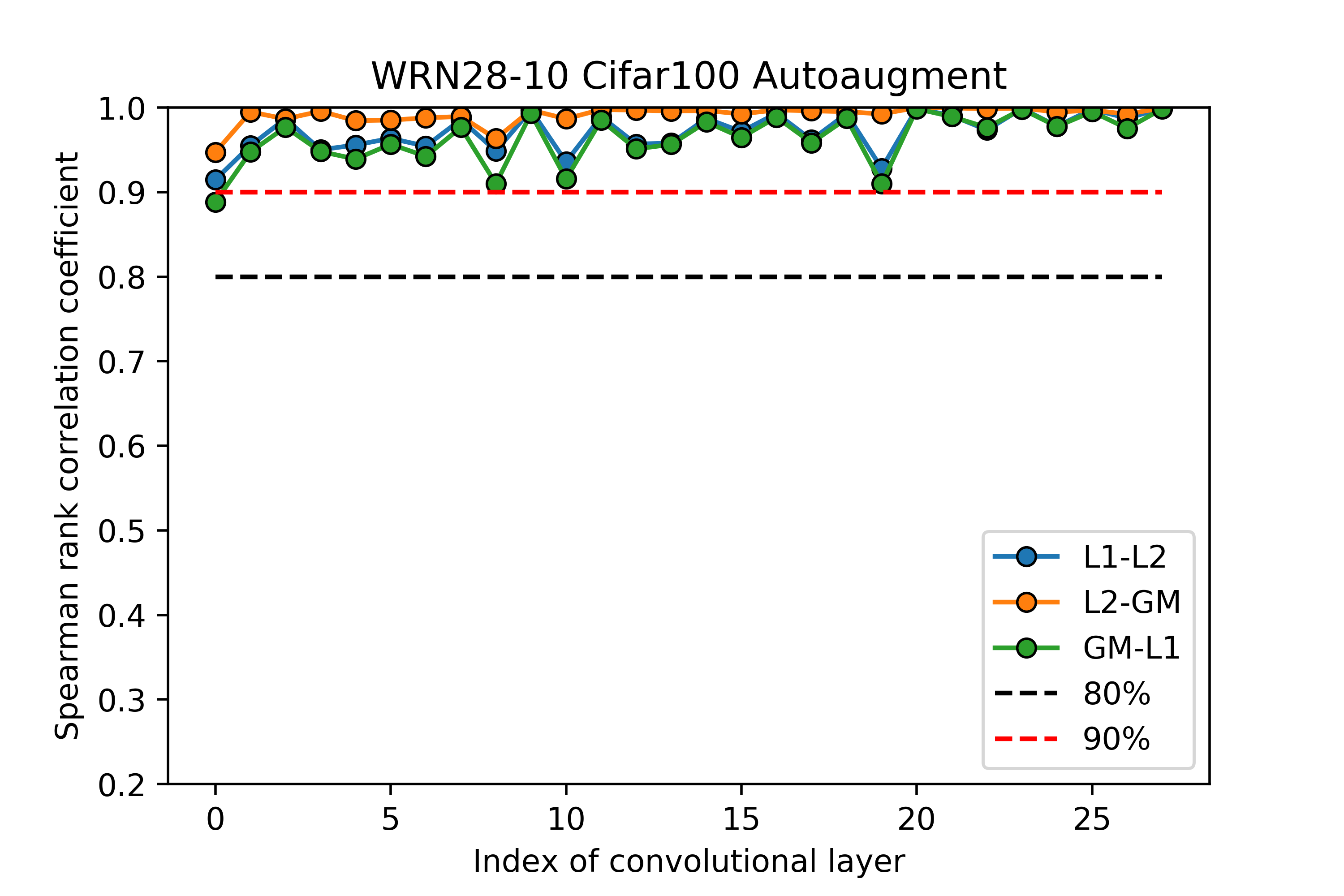} 
	\includegraphics[height=2.2in, width=2.5in]{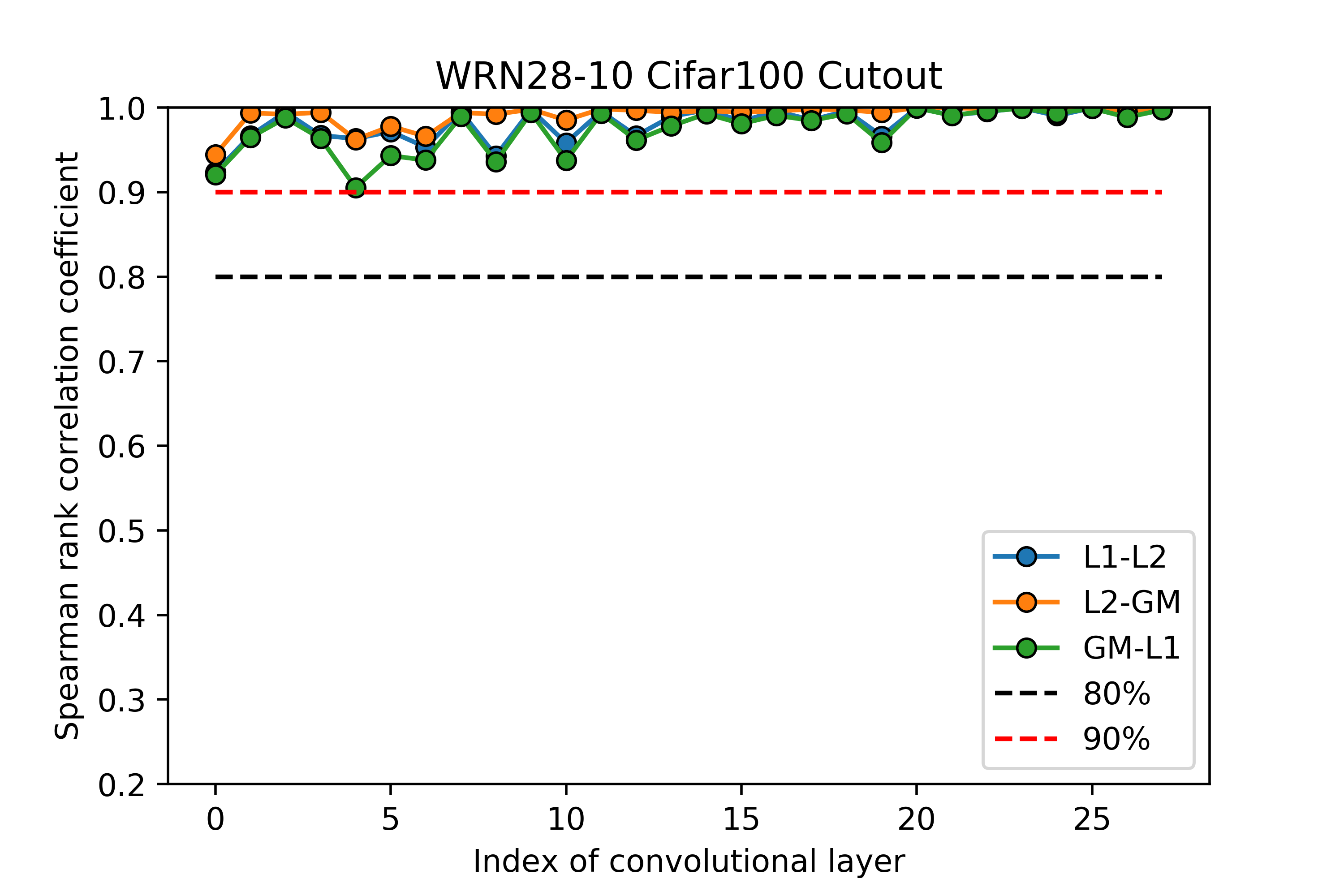} 
	\includegraphics[height=2.2in, width=2.5in]{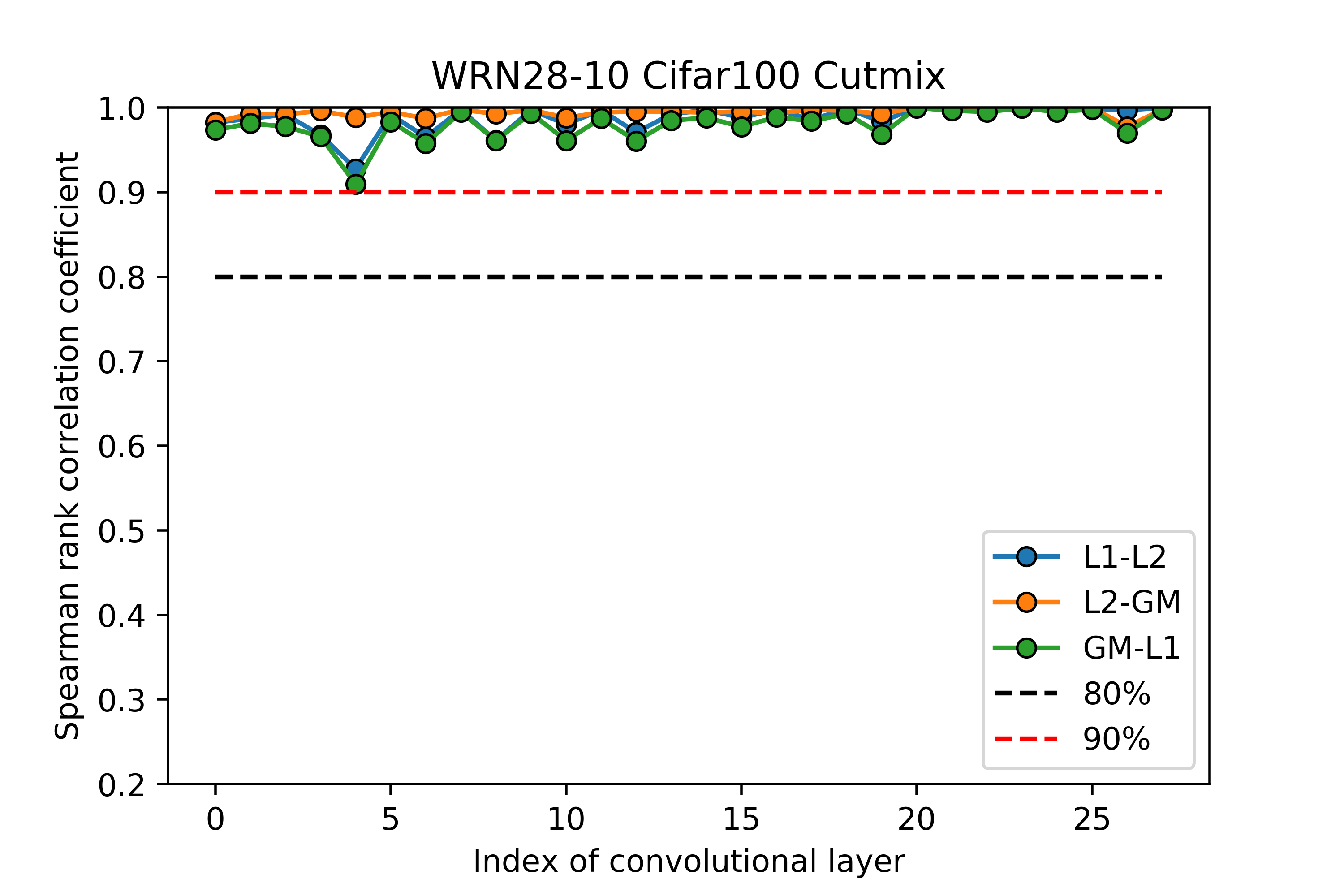} 
	
	\caption{Other task: Regularization}
\end{figure}

\begin{figure}
	\centering 
	\includegraphics[height=2.2in, width=2.5in]{wrn_cifar100.png} 
	\includegraphics[height=2.2in, width=2.5in]{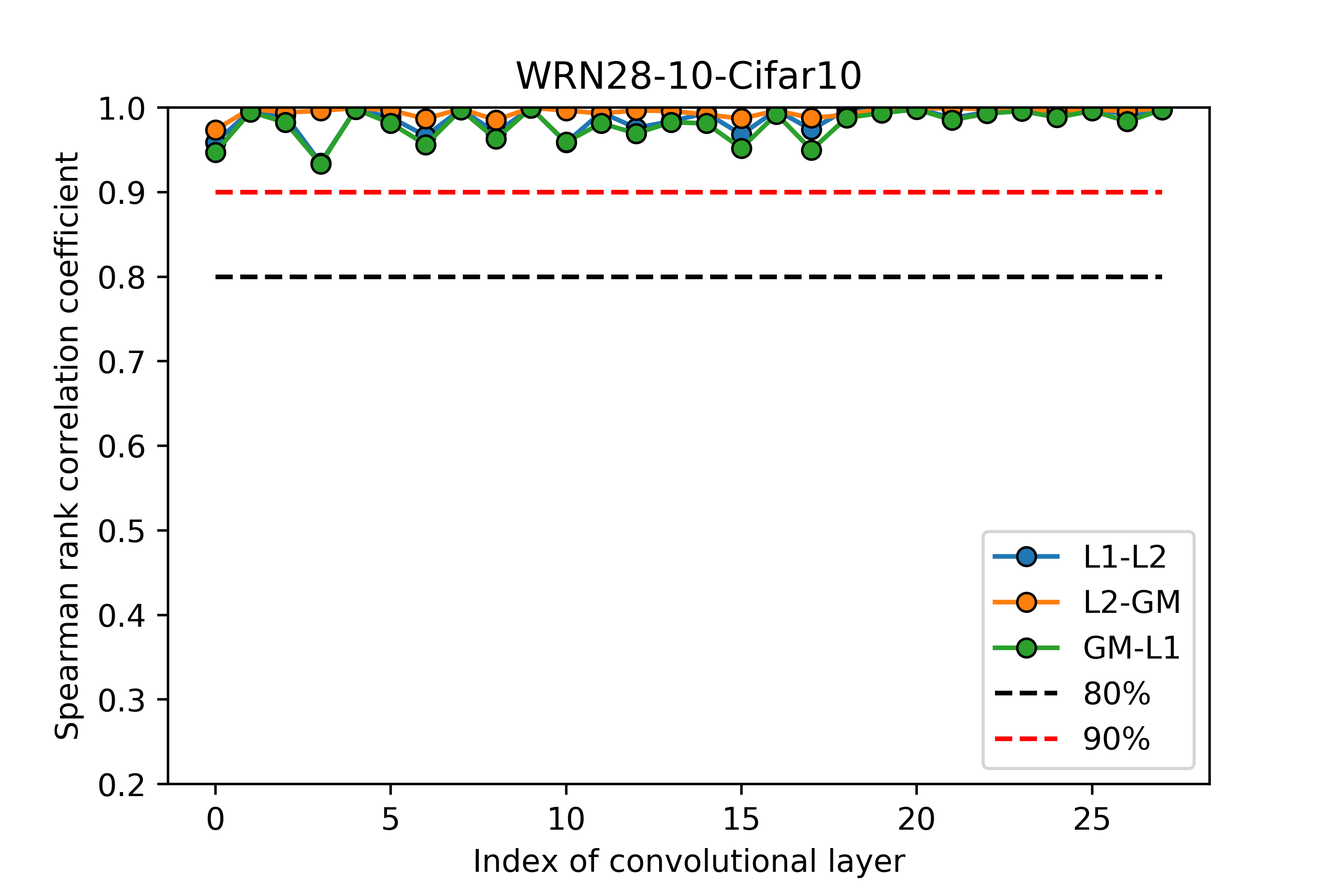} 
	\includegraphics[height=2.2in, width=2.5in]{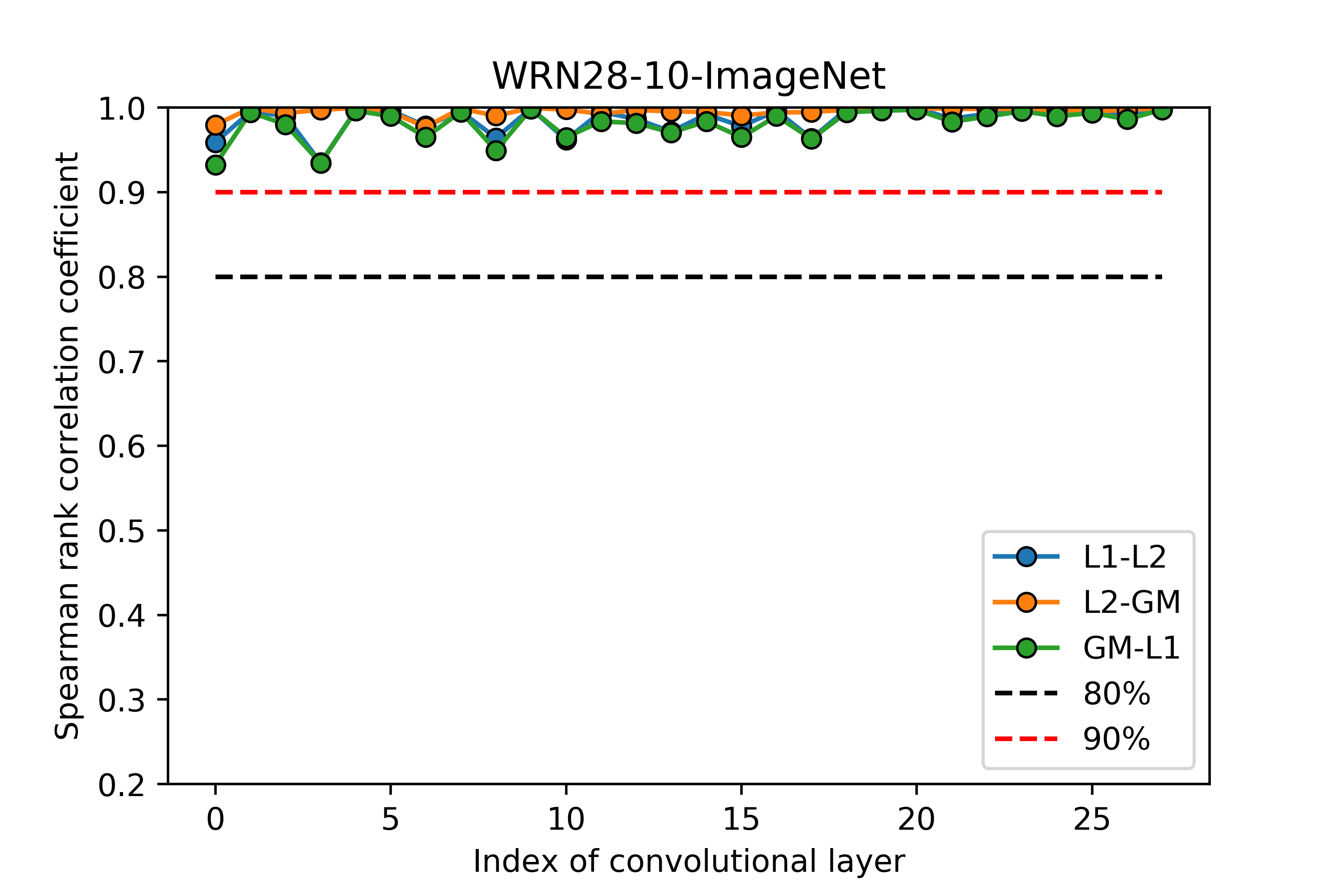} 
	\includegraphics[height=2.2in, width=2.5in]{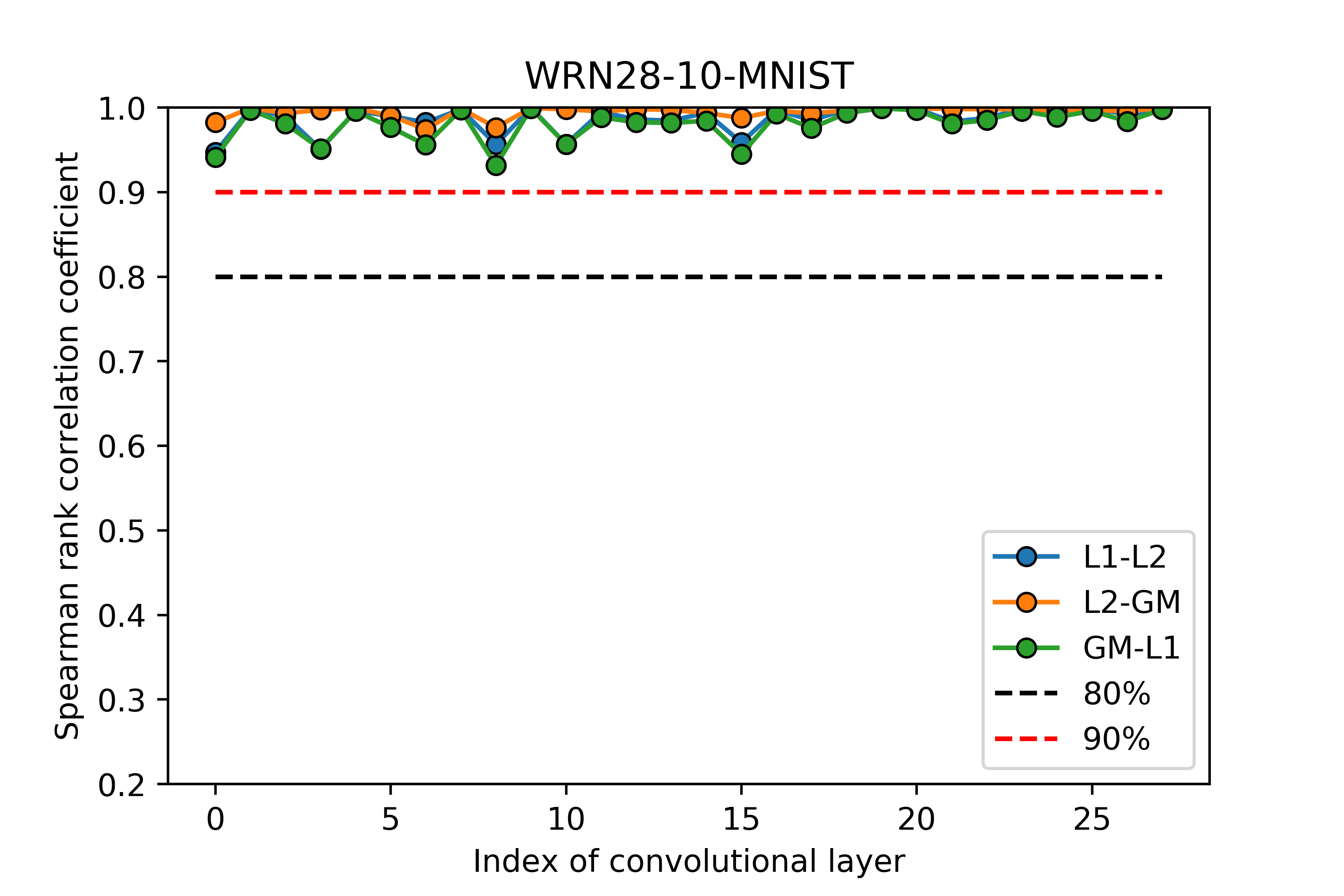} 
	\caption{Dataset}
\end{figure}    

\begin{figure}
	\centering 
	\includegraphics[height=2.2in, width=2.5in]{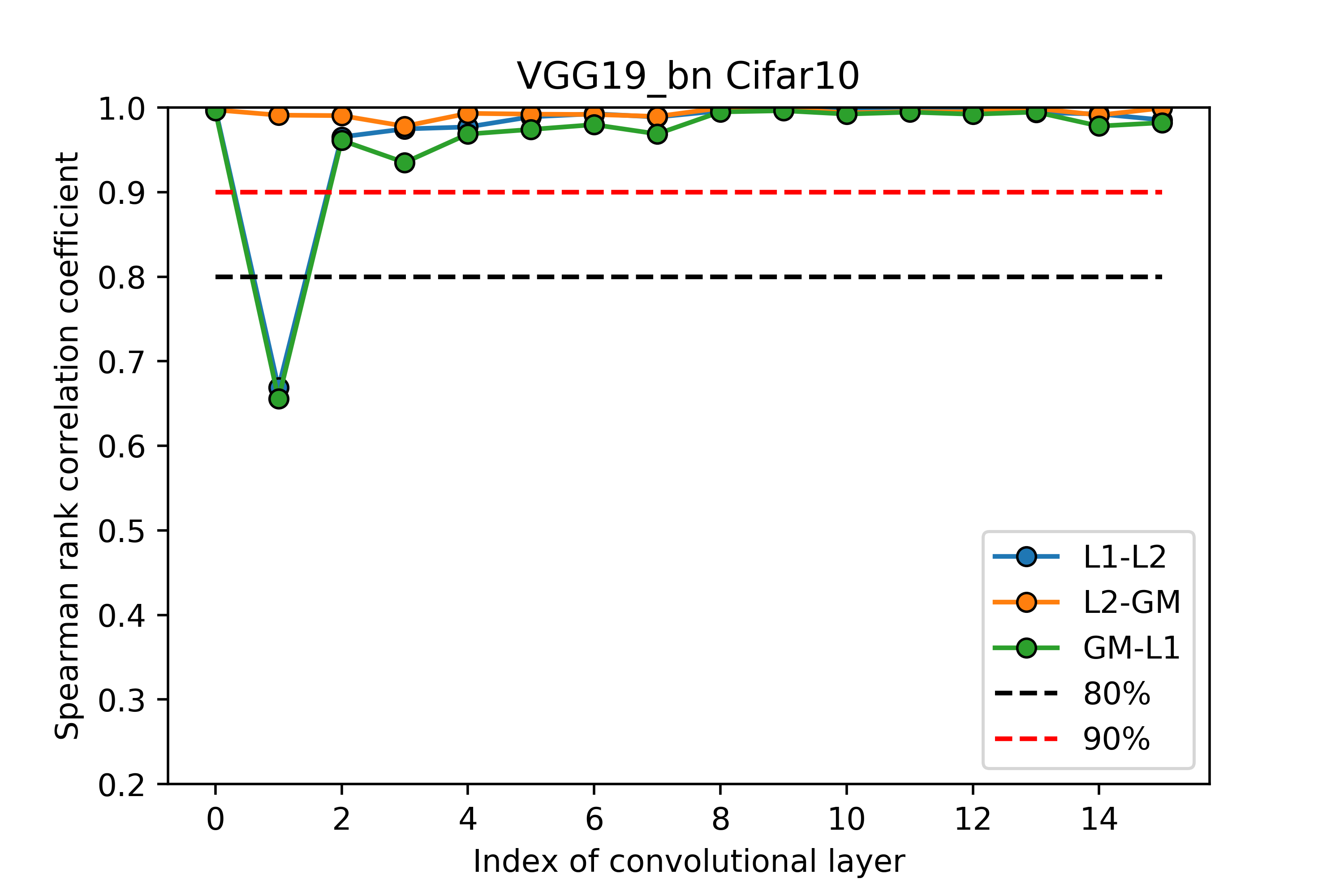} 
	\includegraphics[height=2.2in, width=2.5in]{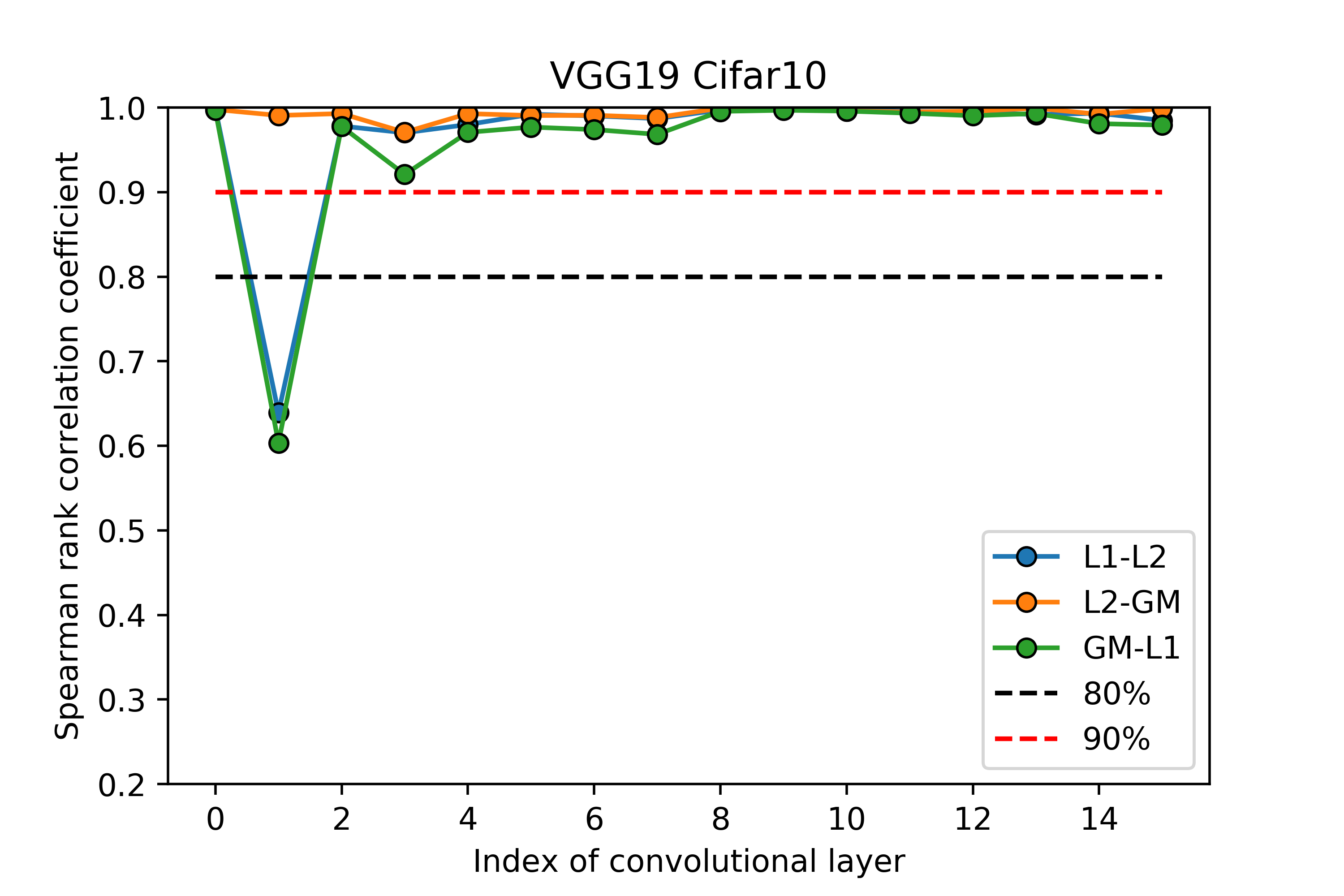} 
	\caption{Batch normalization}
\end{figure}

\begin{figure}
	\centering 
	\includegraphics[height=2.2in, width=2.5in]{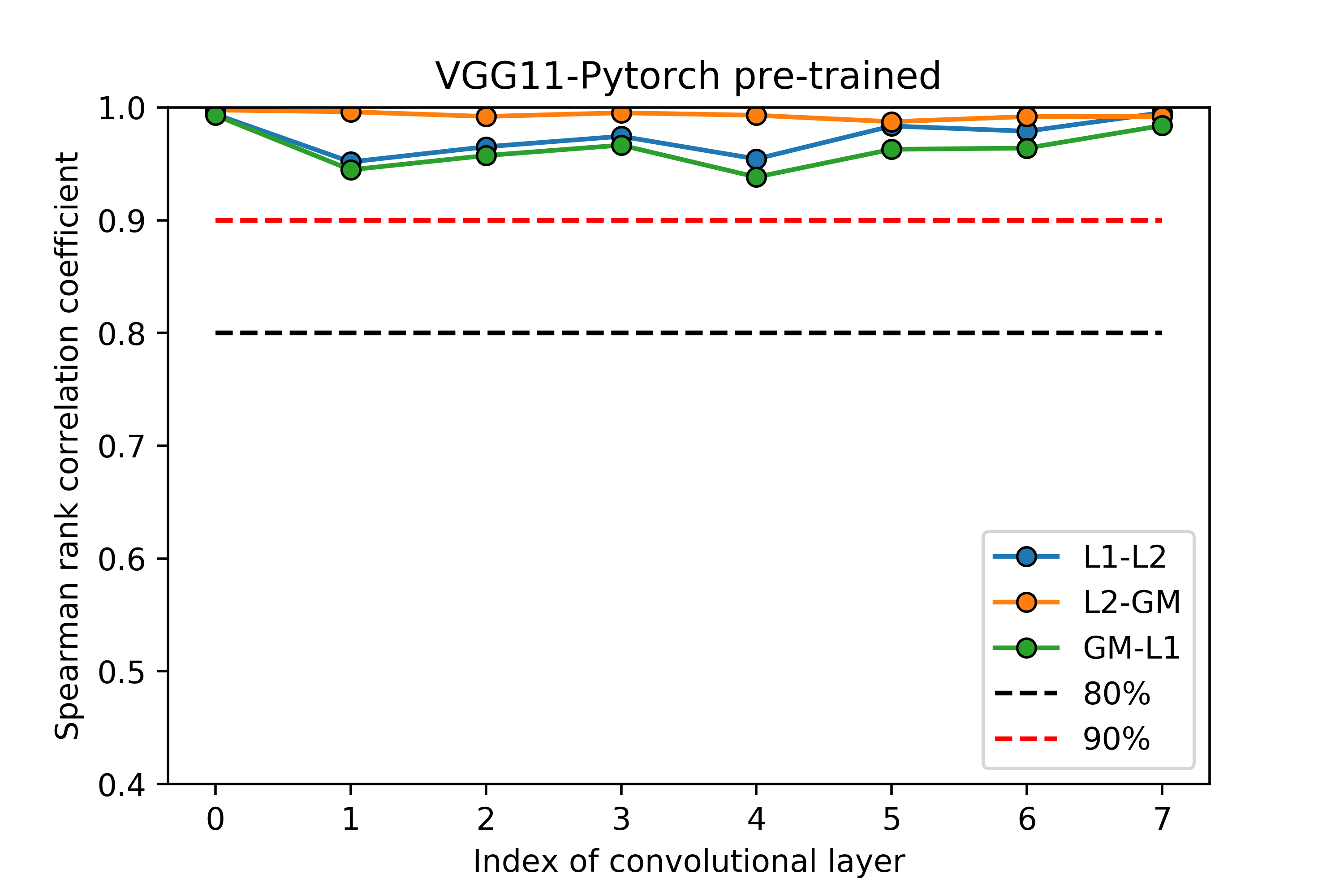} 
	\includegraphics[height=2.2in, width=2.5in]{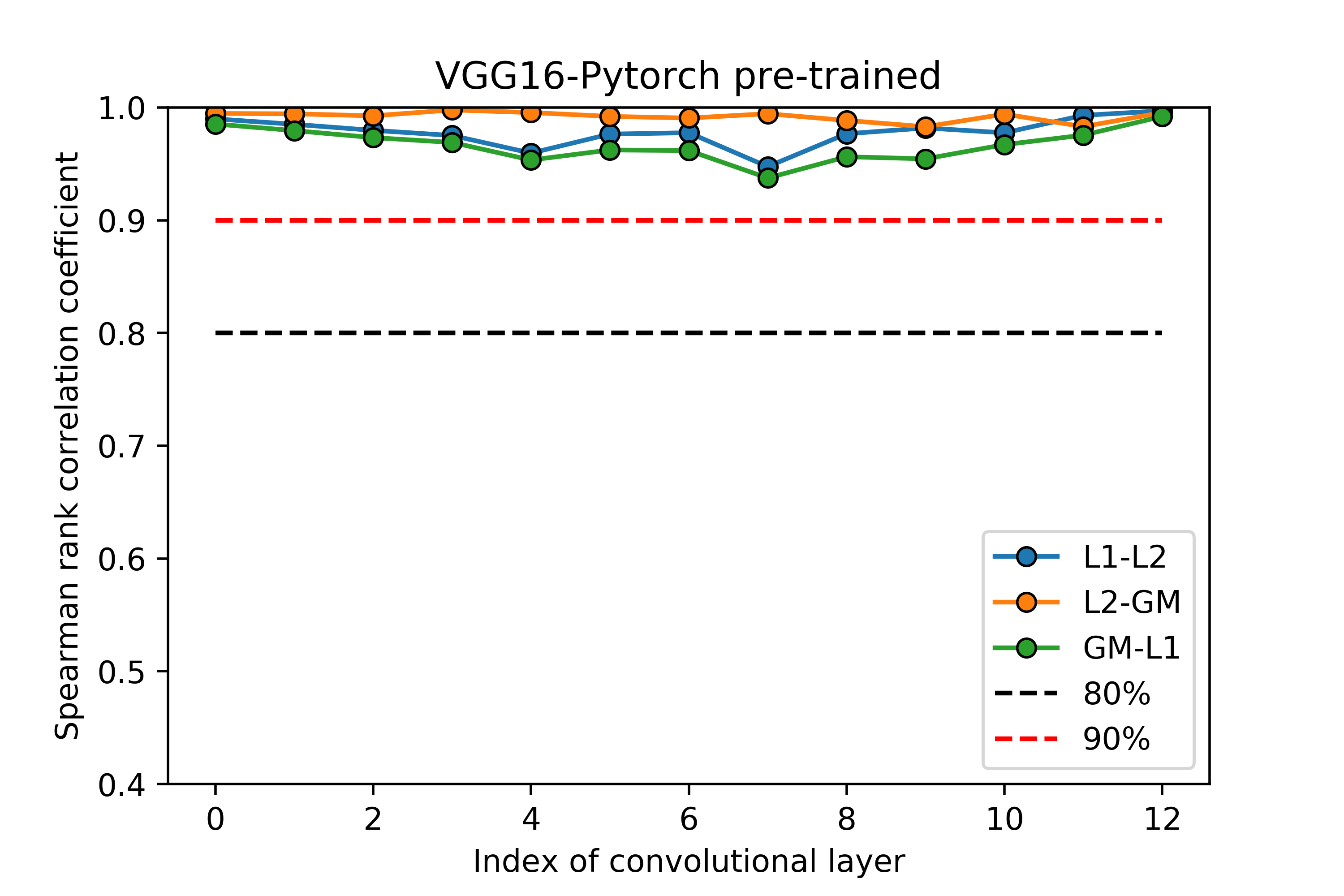} 
	\includegraphics[height=2.2in, width=2.5in]{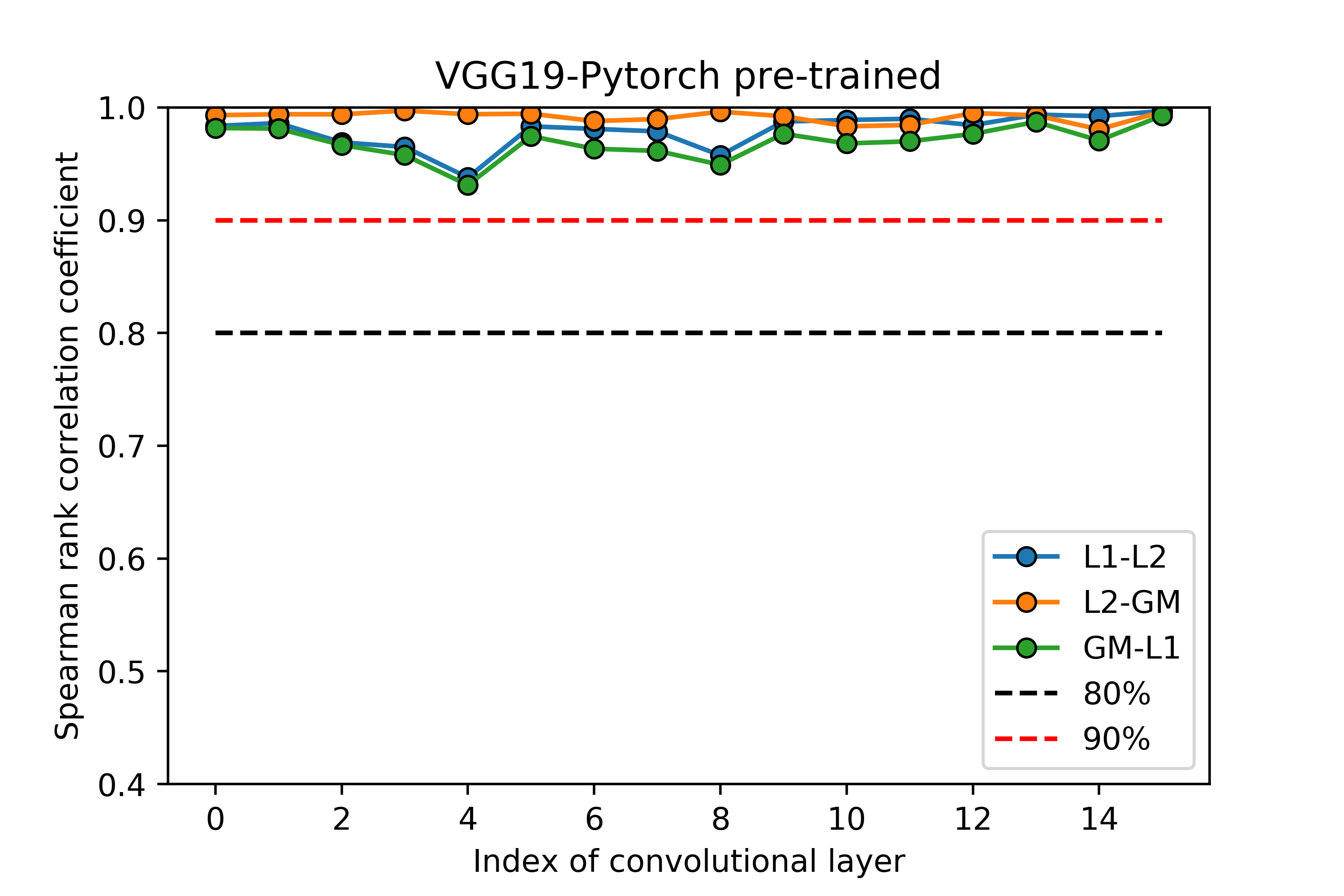} 
	\includegraphics[height=2.2in, width=2.5in]{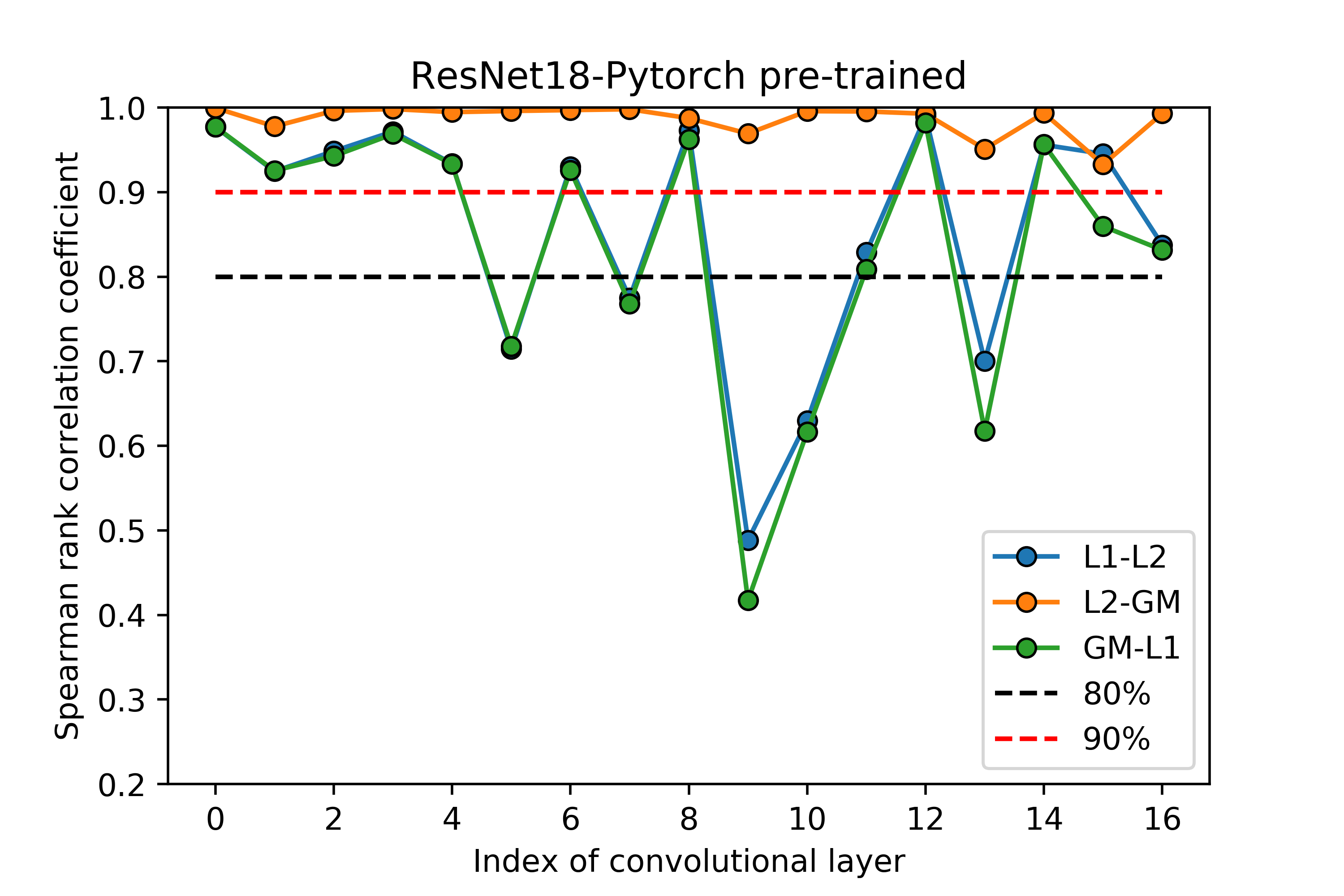} 
	\includegraphics[height=2.2in, width=2.5in]{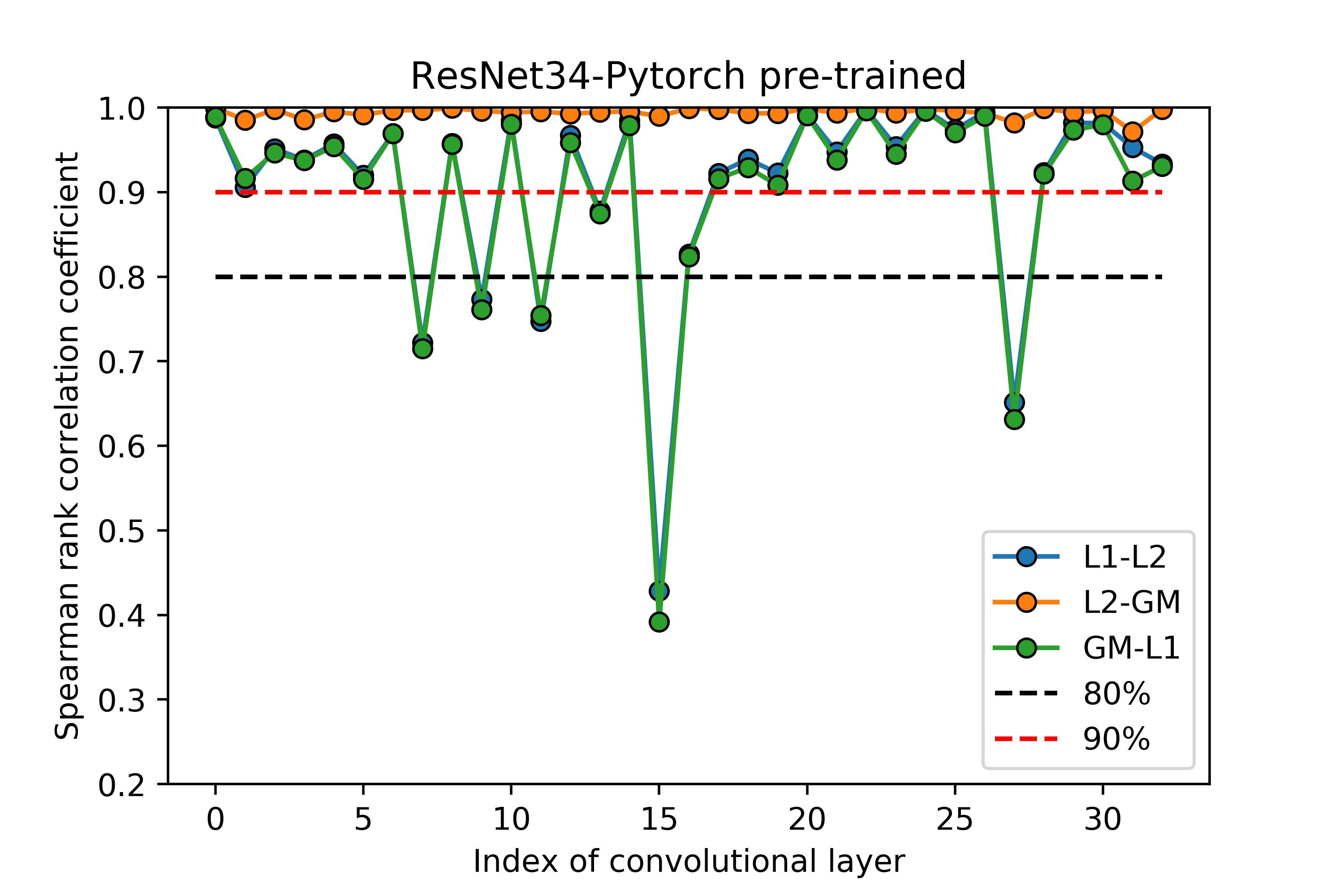} 
	\includegraphics[height=2.2in, width=2.5in]{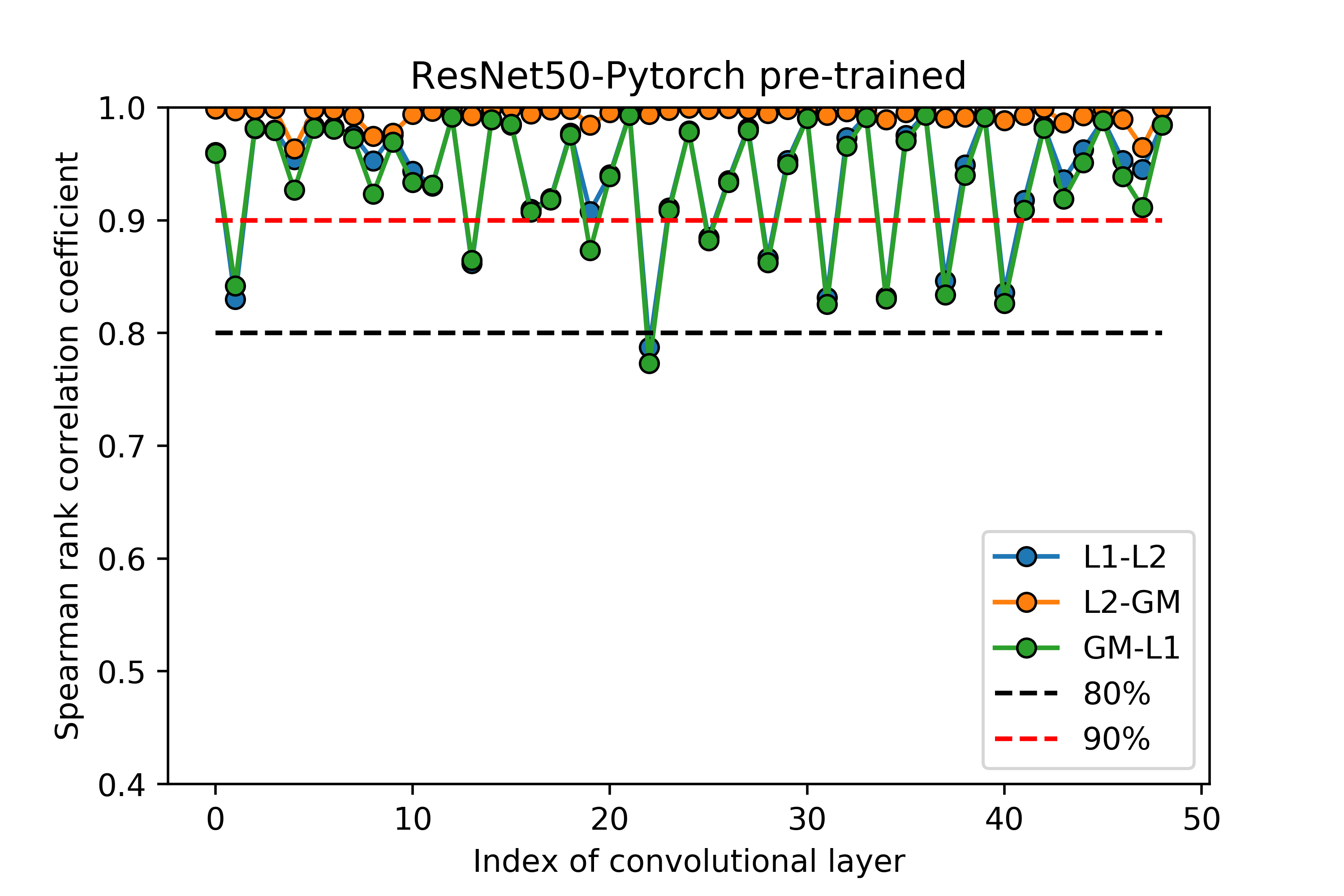}

	\caption{Pytorch pre-trained Model} 
\end{figure}    

\begin{figure}
	\centering 
	\includegraphics[height=2.2in, width=2.5in]{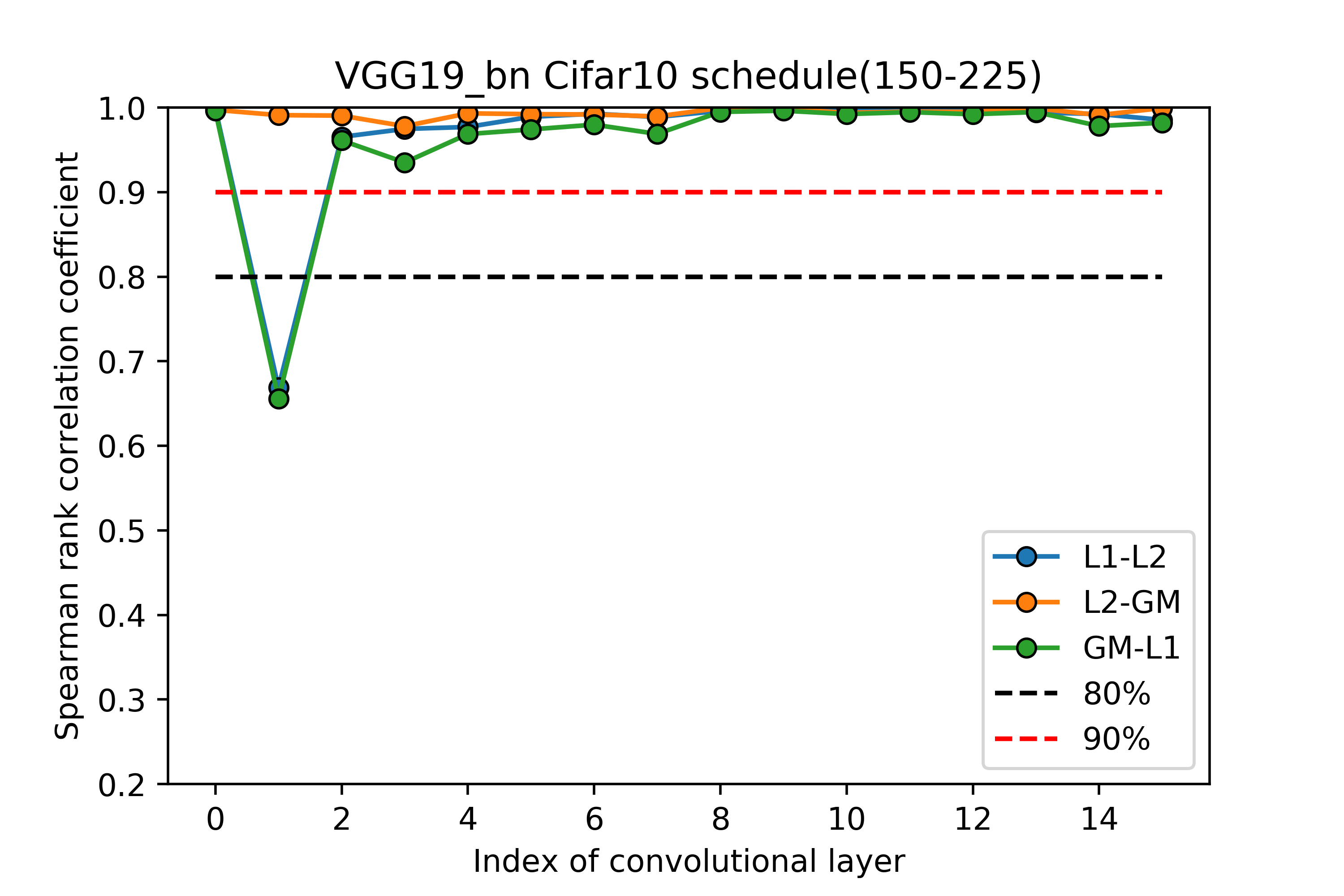} 
	\includegraphics[height=2.2in, width=2.5in]{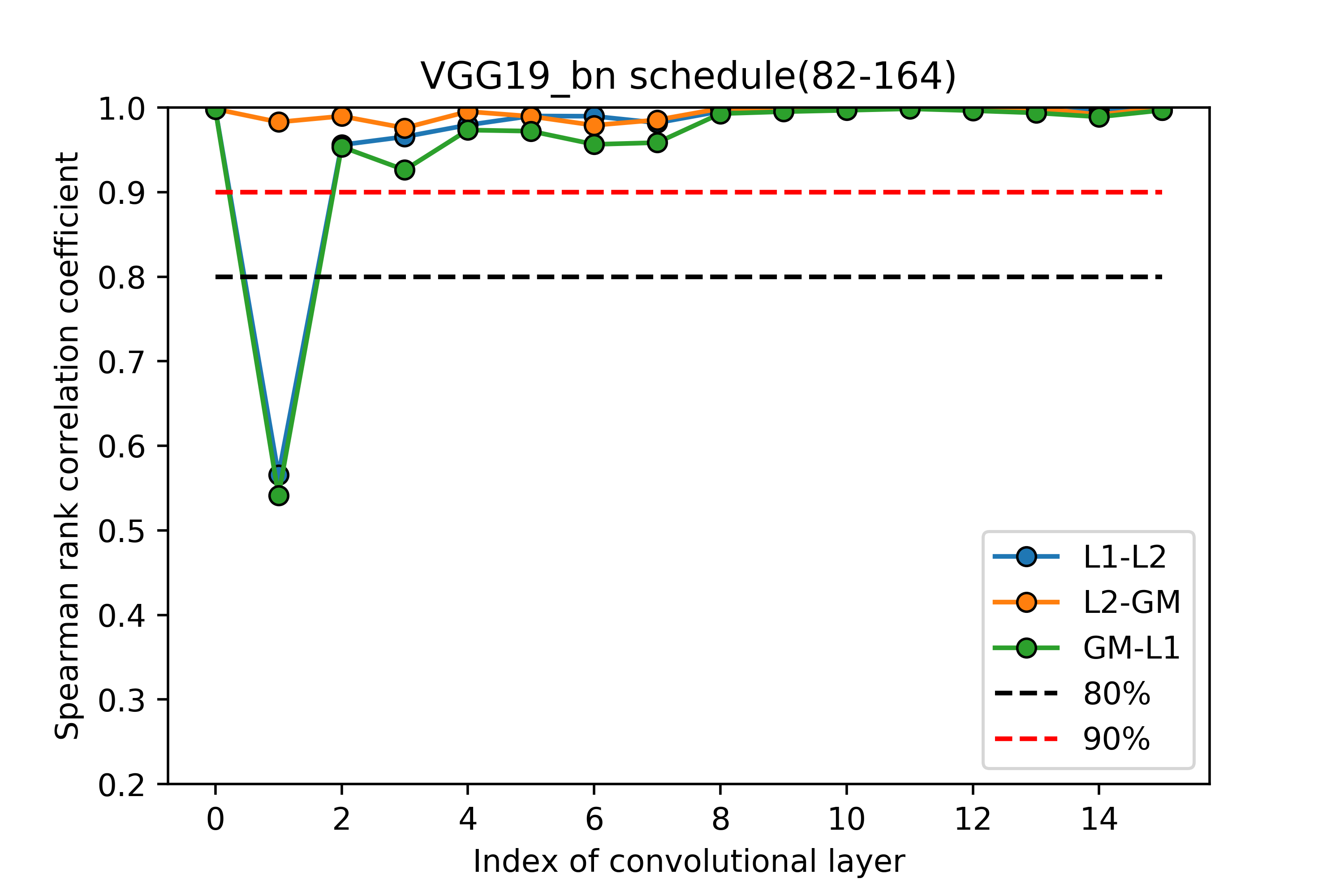} 
	\includegraphics[height=2.2in, width=2.5in]{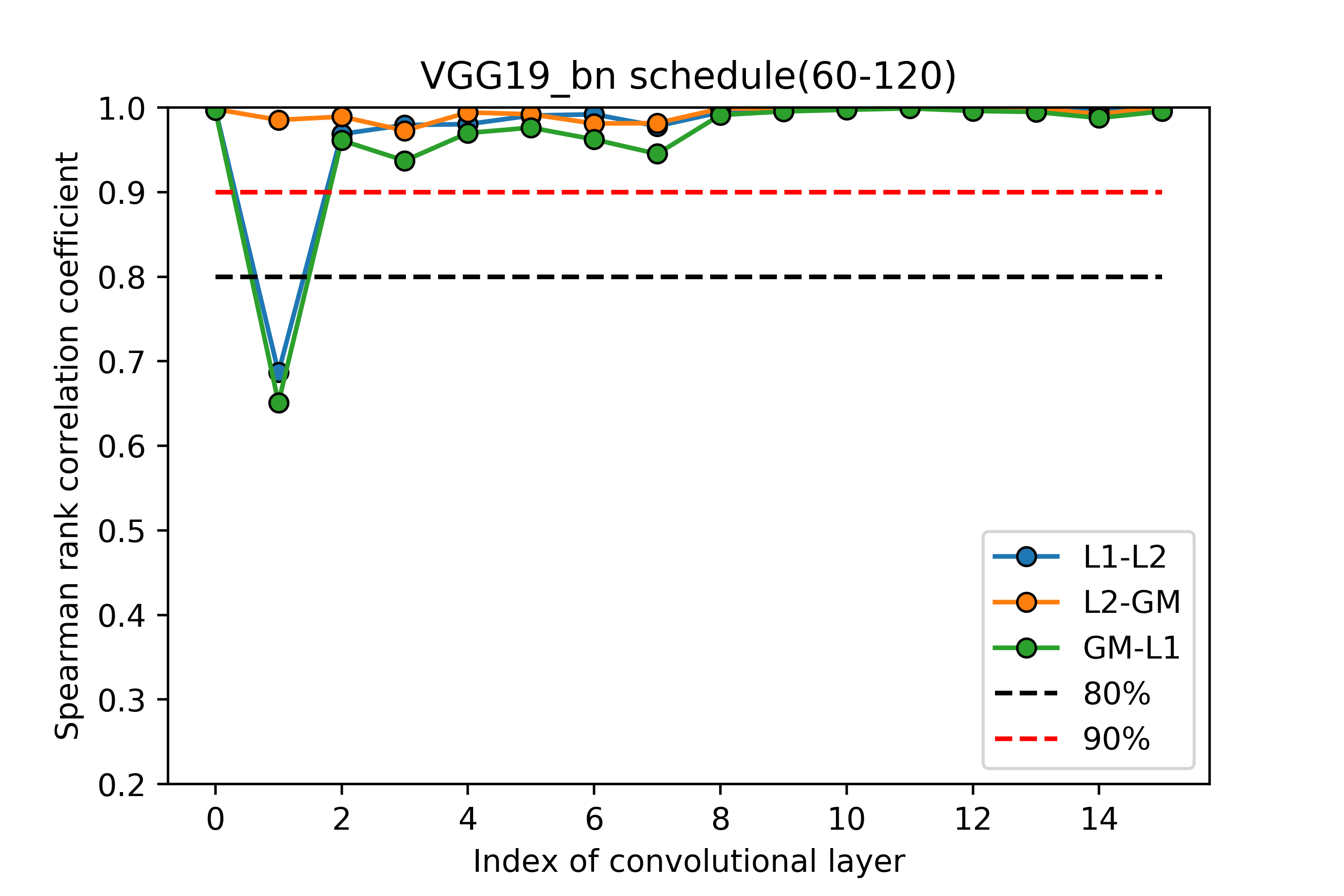} 
	\includegraphics[height=2.2in, width=2.5in]{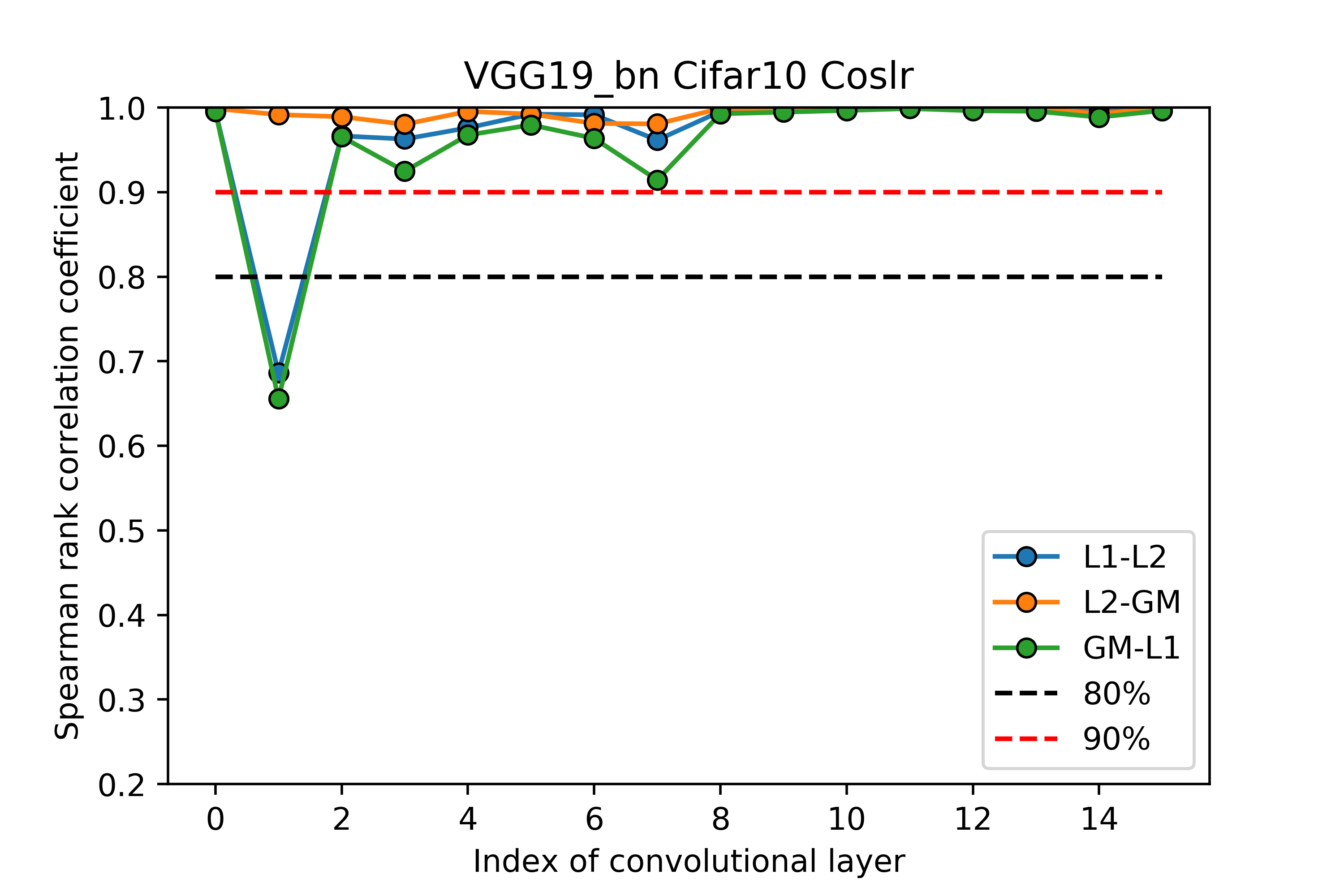}

	\caption{Learning rate} 
\end{figure}


\end{document}